\documentclass[twoside,11pt]{article}

\usepackage[abbrvbib, preprint]{jmlr2e}

\usepackage{amsmath}
\usepackage{mathrsfs}
\usepackage{thmtools}
\usepackage{mathtools}

\usepackage{commands}

\usepackage{natbib}

\usepackage{xcolor}
\usepackage{subcaption}

\usepackage{tikz}
\usetikzlibrary{positioning, intersections}
\usepackage{pgfplots}
\usepgfplotslibrary{groupplots}
\usepgfplotslibrary{fillbetween}

\pgfplotsset{compat=1.18}

\definecolor{Set1.red}{rgb}{0.894117,0.101960,0.109803}
\definecolor{Set1.blue}{rgb}{0.215686,0.494117,0.721568}
\definecolor{Set1.green}{rgb}{0.301960,0.686274,0.290196}
\definecolor{Set1.purple}{rgb}{0.596078,0.305882,0.639215}
\definecolor{Set1.orange}{rgb}{1,0.498039,0}
\definecolor{Set1.yellow}{rgb}{1,1,0.2}
\definecolor{Set1.brown}{rgb}{0.650980,0.337254,0.156862}
\definecolor{Set1.pink}{rgb}{0.968627,0.505882,0.749019}
\definecolor{Set1.gray}{rgb}{0.6,0.6,0.6}

\newcommand{\solnridge}{{\bw_{\text{\tiny{ridge}}}}}
\newcommand{\mni}{{\bw_{\text{\tiny{MNI}}}}}
\newcommand{\mnitilde}{{\tilde{\bw}_{\text{\tiny{MNI}}}}}
\newcommand{\mnicheck}{{\check{\bw}_{\text{\tiny{MNI}}}}}
\newcommand{\mm}{{\bw_{\text{\tiny{MM}}}}}

\newcommand{\mnic}{{\bw^c_{\text{\tiny{MNI}}}}}
\newcommand{\mnitildec}{{\tilde{\bw}^c_{\text{\tiny{MNI}}}}}

\newcommand{\Numalt}{{N_a}}
\newcommand{\Varalt}{{V_a}}
\newcommand{\Diamondalt}{{\Diamond_a}}

\newcommand\uptokstar {{0:k^*}}
\newcommand\kstartoinf {{k^*:\infty}}

\newcommand{\muperpridge}{\bmu_{\mathrel{\stackrel{\mathclap{\perp}}{\vphantom{.}\smash{\scriptscriptstyle\sim}}}}}

\newcommand{\muperpridgebar}{\bar{\bmu}_{\mathrel{\stackrel{\mathclap{\perp}}{\vphantom{.}\smash{\scriptscriptstyle\sim}}}}}
\newcommand{\solnridgebar}{{\bar{\bw}_{\text{\tiny{ridge}}}}}

\title{Benign Overfitting and the Geometry of the Ridge Regression Solution in Binary Classification}

\author{\name Alexander Tsigler \email alexander\_tsigler@berkeley.edu\\
 \addr Department of Statistics\\
  University of California, Berkeley\\
  367 Evans Hall, Berkeley, CA 94720-3860\\
 \AND
 \name Luiz F.O.~Chamon \email luiz.chamon@polytechnique.edu\\
 \addr Department of Applied Mathematics\\
 	\'{E}cole polytechnique\\
 	91128 Palaiseau CEDEX\\
 \AND
 \name Spencer Frei \email  sfrei@google.com\\
 \addr Google DeepMind \\
 1600 Amphitheatre Parkway, Mountain View, CA 94043
 \AND
 \name Peter L.~Bartlett  \email peter@berkeley.edu\\
 \addr Departments of Statistics and EECS\\
  University of California, Berkeley and Google DeepMind\\ 367 Evans Hall, Berkeley, CA 94720-3860
}
\date{August 2023}

\ShortHeadings{Geometry of Ridge Regression in Classification}{Tsigler et al.}

\begin{document}

\maketitle

\begin{abstract}
In this work, we investigate the behavior of ridge regression in an overparameterized binary classification task.  We assume examples are drawn from (anisotropic) class-conditional cluster distributions with opposing means and we allow for the training labels to have a constant level of label-flipping noise. We characterize the classification error achieved by ridge regression under the assumption that the covariance matrix of the cluster distribution has a high effective rank in the tail.  We show that ridge regression has qualitatively different behavior depending on the scale of the cluster mean vector and its interaction with the covariance matrix of the cluster distributions. In regimes where the scale is very large, the conditions that allow for benign overfitting turn out to be the same as those for the regression task. We additionally provide insights into how the introduction of label noise affects the behavior of the minimum norm interpolator (MNI).  The optimal classifier in this setting is a linear transformation of the cluster mean vector and in the noiseless setting the MNI approximately learns this transformation.  On the other hand, the introduction of label noise can significantly change the geometry of the solution while preserving the same qualitative behavior.
\end{abstract}




\tableofcontents

\section{Introduction}
\label{sec::introduction}


Following empirical observations that large neural networks can achieve vanishing error on noisy training data and still perform well on test data, there has been a plethora of work aimed at theoretically understanding the generalization performance of overparameterized models~(see~\cite{Bartlett_Montanari_Rakhlin_2021} and~\cite{Belkin_2021} for relevant surveys). In particular, there is significant interest in understanding when interpolating noisy data~(i.e., \emph{overfitting}) can be \emph{benign}, i.e., when an estimator can achieve near-zero training error while simultaneously generalizing well. The model class, the choice of the estimation procedure, and properties of the training data as well as the population distribution are key factors that interact to influence the consequences of overfitting.

To date, the most complete picture we have of this \emph{benign overfitting} phenomenon is for the minimum $\ell_2$-norm solution of well-specified linear regression. Here, this phenomenon arises due to a self-induced regularization coming from a population covariance matrix with large effective rank~(relative to the number of samples)~\citep{benign_overfitting}. Beyond linear regression, however, our understanding is far less developed, even in the setting of linear classification. The main goal of this paper is to characterize how benign overfitting can occur in this setting.

\subsection{Binary classification problem setting}
\label{sec::binary classification problem setting}

We consider a mixture of two classes with the same covariance $\bSigma \in \R^{p\times p}$ and symmetric (with respect to the origin) centers $\{-\bmu,  \bmu\}\subset \R^p$. Both classes have the same probabilities. More precisely, the matrix whose rows are the data points is given by
\[
\bX := \by\bmu^\top + \bZ\bSigma^{1/2} \in \R^{n \times p},
\]
where $\by\in \{-1, 1\}^n$ is the vector of class labels, whose components are i.i.d.\ Rademacher random variables, and $\bZ \in \R^{n \times p}$ is a matrix with i.i.d.\ isotropic rows. We also denote $\bQ := \bZ\bSigma^{1/2}$---the matrix of covariates with centers of clusters subtracted. We consider the overparameterized regime, that is $p > n$, and linear classifiers that assign the label $\sign(\bw^\top\bx)$ to the point $\bx \in \R^p$. Here, $\bw \in \R^p$ is the weight vector of the classifier.

Imagine%
\footnote{We say ``imagine'' because we do not actually impose any assumption on $\phi$ throughout the paper.}
that we had some control over the deviations of the rows of $\bZ$ in all directions, namely, imagine there is an increasing function $\phi: \R_{\geq 0} \to \R$ such that, for any $\bv \in \cS^{p-1}$ and $t > 0$
\[
\P(\bz^\top\bv < -t) \leq \phi(t),
\]
where $\bz^\top$ is a random draw from the same distribution as the rows of $\bZ$. Then, the probability of error for the classifier $\bx \mapsto \sign(\bw^\top \bx)$ could be bounded as
\begin{align*}
\P\left(\bw^\top(\bmu + \bSigma^{1/2}\bz) < 0\right)
=\P\left(\frac{\bw^\top\bSigma^{1/2}}{\|\bw\|_\bSigma}\bz < -\frac{\bw^\top\bmu}{\|\bw\|_\bSigma}\right)
\leq \phi\left(\frac{\bmu^\top\bw}{\|\bw\|_\bSigma}\right),
\end{align*}
where we introduced $\|\bw\|_\bSigma := \sqrt{\bw^\top\bSigma \bw}$.
That is, the quantity ${\bmu^\top\bw}/{\|\bw\|_\bSigma}$ controls the probability of predicting the wrong label on a new data point. Moreover, if rows of $\bZ$ have a Gaussian distribution, then by the same argument the probability of an error is exactly $\Phi(-{\bmu^\top\bw}/{\|\bw\|_\bSigma})$, where $\Phi$ is the normal CDF.

The main quantitative results of this paper are bounds on  ${\bmu^\top\solnridge}/{\|\solnridge\|_\bSigma}$, where $\solnridge$ is the solution of the~(noisy) ridge regression problem defined as follows. First, we assume we are given a vector of labels $\hat{\by}$ perturbed by some label-flipping noise, that is,  $\hat{\by}$ is obtained from $\by$ by flipping the sign of each of its coordinates independently with probability $\eta$. Then, for a given regularization parameter $\lambda \in \R$, we define
\begin{equation}\label{eq:solnridge}
\solnridge := \bX^\top(\bX\bX^\top + \lambda \bI_n)^{-1}\hat{\by},
\end{equation}
where $\bI_n \in \R^{n\times n}$ is the identity matrix. An interesting particular case of this solution arises when $\lambda = 0$. In that case, $\bX\bw = \hat{\by}$, that is, the solution in~\eqref{eq:solnridge} exactly interpolates the labels $\hat{\by}$. When $\lambda = 0$, we introduce a separate notation for $\solnridge$, namely, $\mni$. Here, MNI stands for the \emph{Minimum Norm Interpolating} solution.

Throughout the paper, we will also use the notation $\bA := \lambda \bI_n + \bQ \bQ^\top$ and $\bnu := \bQ\bmu$.

\subsection{Assumptions on the distribution of the covariates}
\label{sec::covariance spectrum assumptions}

When it comes to the assumptions that we impose on the data distribution, we follow the steps of \citet{BO_ridge} that generalizes \citep{benign_overfitting}.

First, denote the eigenvalues of $\bSigma$ in non-increasing order as $\{\lambda_i\}_{i=1}^p$ and fix the basis to be the eigenbasis of $\bSigma$, that is, $\bSigma = \diag(\lambda_1, \dots, \lambda_p)$. The remainder of the paper considers this basis. We assume that for some $k$---that is small compared to $n$---removing the first $k$ columns of the matrix $\bQ$ makes the ``effective rank" of its rows large compared to $n$. The exact notion of large effective rank that we use is somewhat technical and we postpone its introduction to Section \ref{sec::definition of sA_k}. When the data is Gaussian (or, more generally, if the matrix $\bZ$ has independent sub-Gaussian elements), that condition would be
\begin{equation}
\label{eq::effective rank gaussian}
\lambda + \sum_{i > k}\lambda_i > c\left(n\lambda_{k+1} + \sqrt{n\sum_{i > k}\lambda_i^2}\right),
\end{equation}
where $c$ is a large constant. Note that the regularization parameter $\lambda$ adds to the energy of the tail of the covariance spectrum $\sum_{i > k}\lambda_i$ in the left hand side of the expression. That is, the notion of effective rank depends not only on $\bSigma$, but also on the level of regularization applied. In \citep{benign_overfitting} it was shown that for $\lambda = 0$, condition \eqref{eq::effective rank gaussian} is necessary for benign overfitting to happen in linear regression.

In view of this separation of the first $k$ components, we introduce the following notation taken from \cite{BO_ridge}: for any $k \in \{0, 1, \dots, p\}$ and any matrix $\bM \in \R^{n\times p}$, we denote by~$\bM_\uptok$ the matrix comprised of the first $k$ columns of $\bM$.%
\footnote{When $k=0$, this matrix is empty and all the terms involving the indices $\uptok$ become zero.}
Analogously, we denote by $\bM_\ktoinf$ the matrix comprised of the last $p-k$ columns of $\bM$. For any $\bu\in \R^p$, we denote by $\bu_\uptok$ the vector comprised of the first $k$ components of $\bu$ and by~$\bu_\ktoinf$ the remaining components. Finally, we denote $\bSigma_\uptok = \diag(\lambda_1, \dots, \lambda_{k})$ and $\bSigma_\ktoinf = \diag(\lambda_{k+1}, \dots, \lambda_{p}).$  We sometimes refer to the first $k$ components as the ``spiked part of the covariance" and to the remaining components as the ``tail of the covariance."

Apart from the effective rank condition described above, we also need several concentration inequalities to hold. Those inequalities, however, are rather standard (such as law of large numbers for i.i.d.\ random variables or sample covariance concentration in low dimensions), so we choose to assume that those inequalities hold directly instead of deriving them from assumptions on the distribution of the data. We introduce these inequalities in Section \ref{sec::definition of sB_k}.

Our motivation to consider such a regime was initially technical: we believed that the quantities of interest could be accurately evaluated in this regime. However, our results suggest that such a structure of the data is necessary for benign overfitting to occur in classification. We elaborate on this point in Section \ref{sec::conclusions}.

\subsection{First result: recovering the geometry in the noiseless setting}
\label{sec::recovering the geometry introduction}

Even though our goal is to study benign overfitting, the first result that we obtain is actually for the ``noiseless" setting, that is, $\eta = 0$ and $\by = \hat{\by}$. In this setting, our bounds show that  $\solnridge$ performs effectively  as $\bQ^\top \bA^{-1}\by + (\bSigma + n^{-1}\Lambda\bI_p)^{-1}\bmu$, where $\Lambda = \lambda + \sum_{i > k}\lambda_i$, and $\bA = \lambda \bI_n + \bQ \bQ^\top$. Let us analyze those two terms separately. First, the vector $\bQ^\top\bA^{-1}\by$ has a symmetric distribution and no dependence on $\bmu$, so it plays the role of a noise term. The vector $\bmu$ should therefore be large enough for this noise term not to dominate. On the other hand, the term $(\bSigma + n^{-1}\Lambda\bI_p)^{-1}\bmu$ can be seen as a ridge regularized version of the optimal classifier. Indeed, the vector $\bw$ that minimizes $\bmu^\top\bw/\|\bw\|_\bSigma$ is proportional to $\bSigma^{-1}\bmu$ and $(\bSigma + n^{-1}\Lambda \bI_p)$ is akin to a regularized version of the covariance, where the tail of the covariance $\sum_{i > k}\lambda_i$ adds to the explicit ridge regularization $\lambda$. Another way to think about this is to introduce $k^*$ as in \citep{benign_overfitting}:
\[
k^* := \min\left\{k: n\lambda_{k+1} < \lambda + \sum_{i > k}\lambda_i\right\}, \quad \Lambda_* := \lambda + \sum_{i > k^*}\lambda_i.
\]
Then, we obtain, up to a constant factor,
\[
(n\Lambda^{-1}\bSigma + \bI_p)^{-1}\bmu \approx \begin{pmatrix} n^{-1}\Lambda_*\bSigma_\uptokstar^{-1}\bmu_\uptokstar\\ \bmu_\kstartoinf\end{pmatrix}.
\]
We see that the ridge regression solution performs the optimal linear transform in the first $k^*$ coordinates, but is proportional to $\bmu$ without any transformation in the remaining coordinates. Throughout this paper we refer to this effect as ``recovering the geometry" in the first $k^*$ components.

Therefore, we show that there are three regimes for the noiseless setting: when $\bmu$ is small in magnitude, the ``noise term"  $\bQ^\top \bA^{-1}\by$  dominates both $\bmu^\top\solnridge$ and $\|\solnridge\|_\bSigma$, and the quantity $\bmu^\top\solnridge$ will be negative with probability close to 50\%, resulting in no meaningful bound on classification performance. As the magnitude of $\bmu$ grows, the term $(\bSigma + n^{-1}\Lambda\bI_p)^{-1}\bmu$ starts dominating in the scalar product with $\bmu$, while the term $\bQ^\top \bA^{-1}\by$ still dominates in the norm. This regime already yields a non-vacuous classification guarantee, but still does not exhibit the full ``recovery of the geometry." Finally, when $\bmu$ becomes large enough, $\solnridge$ performs effectively as $(\bSigma + n^{-1}\Lambda\bI_p)^{-1}\bmu$.

It is worth noting that the direction $(\bSigma + n^{-1}\Lambda\bI_p)^{-1}\bmu$ approaches $\bSigma^{-1}\bmu$ as $\Lambda$ decreases, which suggests that one should always use the smallest possible (perhaps even negative) regularization to achieve the best classification performance. This conclusion is not straightforward, however, since decreasing $\Lambda$ also increases the relative magnitude of the noise term $\bQ^\top\bA^{-1}\by$. Nevertheless, we show that it is indeed the case and one cannot gain a significant (in a certain sense) increase in performance by increasing $\lambda$ beyond the point at which the data has high effective rank in the tail of the covariance when the noise level $\eta = 0$.

\subsection{Second result: benign overfitting}
\label{sec::introduction label flipping noise story}

When it comes to the case with label-flipping noise, we show that the structure of the solution vector may change significantly compared to the noiseless case, depending on the magnitude of $\bmu$. For MNI, adding label-flipping noise multiplies the noiseless solution by a random scalar and aggregates an additional ``noise component" in the orthogonal direction that has no dependence on $\bmu$. As $\bmu$ becomes large, that multiplicative scalar becomes close to zero-mean, that is, it flips the direction of the noiseless solution with probability close to 50\%. Moreover, the new orthogonal ``noise component" begins to dominate over the noiseless solution.

Nevertheless, though the solution vector for the noisy case may look very different from the noiseless case, the bounds on $\bmu^\top\solnridge$ and $\|\solnridge\|_\bSigma$ remain rather similar. In fact, the bound on $\bmu^\top\solnridge$ remains practically the same and the bound on $\|\solnridge\|_\bSigma$ only picks up an additional term corresponding to magnitude of the new orthogonal ``noise component."
As a result, our bounds suggest%
\footnote{Unfortunately, we do not provide a matching upper bound on $\bmu^\top\solnridge/\|\solnridge\|_\bSigma$ for the case with label-flipping noise, so we can only make claims about the bound rather than the ratio itself. We do believe, however, that this statement applies to $\bmu^\top\solnridge/\|\solnridge\|_\bSigma$ and provide a rationale as to why in  Section \ref{sec::conclusions}.}
that the noisy solution goes through the same regimes as the noiseless one, with an additional regime when $\bmu$ is very large in magnitude. In the latter, our bound loses the dependence on $\bmu$ and becomes a function only of the covariance. Interestingly, the conditions under which the ratio~${\bmu^\top\solnridge}/{\|\solnridge\|_\bSigma}$ becomes small, that is, the conditions under which benign overfitting occurs, are exactly the same as those from the regression literature \citep{benign_overfitting,BO_ridge}.

\subsection{Comparison to the regression problem}
\label{sec::comparison with regression}

As we mentioned earlier, our setting and results parallel those from the papers on benign overfitting in linear and ridge regression \citep{benign_overfitting, BO_ridge}. In this section, we recap their main conclusions and highlight some connections and differences with our classification setup.

\citet{benign_overfitting} and \citet{BO_ridge} consider the minimum-norm interpolating solution for an overparameterized linear regression problem.  In our notation, it can be formulated as
\[
\hat{\btheta} := \argmin_{\btheta \in \R^p} \|\btheta\| \text{ s.t. } \bQ\btheta = \bQ\btheta^* + \beps,
\]
where we introduced $\btheta^*\in \R^p$---the coefficients of the ground truth linear model---and $\beps \in \R^n$---a noise vector, which has independent centered components with variances $v_\eps$. The random noise $\beps$ is independent from $\bQ$. The solution $\hat{\btheta}$ has a closed form expression, namely,
\[
\hat{\btheta} = \bQ^\top \bA^{-1}(\bQ\btheta^* + \beps),
\]
and thus the root mean squared error on a test point can be bounded as in
\begin{equation}
\label{eq::regression bias variance decomposition}
\left\|\btheta^* - \hat{\btheta}\right\|_\bSigma\leq \left\|(\bI_p - \bQ^\top \bA^{-1}\bQ)\btheta^*\right\|_\bSigma + \left\|\bQ^\top \bA^{-1}\beps\right\|_\bSigma.
\end{equation}

The first term on the right-hand side of \eqref{eq::regression bias variance decomposition} constitutes the bias of the MNI regression solution and the second term constitutes its variance. The main result of \citet{benign_overfitting} is that under the structure that we introduced in Section \ref{sec::covariance spectrum assumptions}, the variance term becomes small. This occurs because in the first $k$ components, the noise vector $\beps$ gets projected from dimension $n$ onto a small dimension $k$, so that its energy gets damped by a factor $k/n$. The remainder of the noise vector is then smeared over the components $\ktoinf$. Since the tail of the covariance has high effective rank, a newly sampled data point will be almost orthogonal to $\hat{\btheta}$ in the components $\ktoinf$, so that the fact that they absorbed the noise is irrelevant. When it comes to the bias term, \citet{BO_ridge} show that no learning happens in the components $\ktoinf$, that is, almost all the energy of the signal $\|\btheta^*_\ktoinf\|_{\bSigma_\ktoinf}$ goes into the bias term. In the first $k$ components, on the other hand, the bias term behaves as in classical $k$-dimensional ridge regression with regularization $\sum_{i > k}\lambda_i$. In short, the tail of the covariance provides implicit regularization to the low-dimensional linear regression in the first $k$ components. The signal is not learned in the tail at all, but at the same time, it is not relevant for predictions.

In contrast, the classification setting considered in this paper is fundamentally different. Indeed, in regression, the ``signal vector" $\btheta^*$ is an element of a dual space measured through the data matrix $\bQ$. In our classification setting, the ``signal vector" $\bmu$ is baked into the design matrix $\bX = \bQ + \by\bmu^\top$ and $\bQ$ obscures $\bmu$ instead of helping to measure it. Because of that, it is not clear how to apply the high-level conclusions of the regression papers to the classification setting. For example, a naive application would suggest that if $\bmu$ is supported on the tail of the covariance (i.e., if $\|\bmu_\uptok\| = 0$), this would result in high classification error since no learning happens in the tail. This is not correct, as our bounds imply arbitrarily high classification accuracy in this setting. Since $\bmu$ is baked into $\bX$, a plausible hypothesis could be that the ``useful space" in which the learning happens is the span of the first $k$ eigendirections of the covariance together with $\bmu$ itself. Unfortunately, we did not find such a decomposition of the space useful in our derivations.

Yet, our argument shows very strong connections to the regression setting. In Section~\ref{sec::recovering the geometry introduction}, we stated that in the noiseless case, $\solnridge$ behaves as $\bQ^\top \bA^{-1}\by + (\bSigma + n^{-1}\Lambda\bI_p)^{-1}\bmu$. Observe that the first term in that expression is almost exactly the variance part of the regression solution~[last term in~\eqref{eq::regression bias variance decomposition}]. Moreover, the vector $(\bSigma + n^{-1}\Lambda\bI_p)^{-1}\bmu$ arises as an approximation of the vector $n\Lambda^{-1}(\bI_p - \bQ^\top \bA^{-1}\bQ)\bmu$, which is directly analogous to the bias part of the regression solution. Therefore, for the setting without label-flipping noise, the bound on $\|\solnridge\|_\bSigma$ has almost the same expression as the bound on $\|\btheta^* - \hat{\btheta}\|_\bSigma$. Since the quantity of interest in our paper is $\bmu^\top\solnridge/\|\solnridge\|_\bSigma$, we see (at least on a technical level) that having high classification accuracy is strongly related to having small prediction error in regression.  In the presence of label-flipping noise~(Section \ref{sec::introduction label flipping noise story}), we mention an additional ``orthogonal noise component.'' As it turns out, this component has a very similar structure to $\bQ^\top\bA^{-1}\by$ (see Section \ref{sec::mni geometry with noise} for the precise derivation). Hence, in the large $\bmu$ regime with label-flipping noise, our bound on $\bmu^\top\solnridge/\|\solnridge\|_\bSigma$ has the form $1/\|\bQ^\top\bA^{-1}\by\|_{\bSigma}$. That is, if $\bmu$ is large, the conditions under which the classification accuracy is high are exactly the same as the conditions under which the variance of regression is low. This suggests that the mechanism by which the regression solution ``hides" the noise it interpolates is similar to that used by the classification solution.

Overall, though we identify very concrete connections between our classification results and the regression results from \citet{benign_overfitting} and \citet{BO_ridge},
identifying either a clear unifying picture of these two settings or a fundamental difference between them remains an intriguing open question.

\subsection{Related work}

Despite the fact that the literature on linear classification in high dimensions is vast, only a few papers studied problems with general covariance structure and the impact of that structure on the prediction accuracy. Here, we only review those works, referring the reader to \citet{loureiro2021learning} and \citet{wang2021benign} for a broader review of the related literature.

Existing results can be split into asymptotic and non-asymptotic. The most common asymptotic regime is the ``proportional asymptotic regime," that is, both $p$ and $n$ go to infinity, while their ratio goes to a constant. The results obtained in this regime always require some assumptions on the spectral decay of covariances of the clusters, for example, that all eigenvalues of covariances are bounded from above and below by fixed constants.

\citet{montanari2023generalization} consider i.i.d.\ data $\bx_i$ from a centered Gaussian distribution with covariance $\bSigma$ in the proportional asymptotic regime. There are two classes and the probability that a point $\bx$ belongs to the first class is a function of $\bx^\top \btheta_*$, for some true parameter vector $\btheta_*$. They express the asymptotic classification error of the maximum margin classifier through a solution to a certain system of non-linear equations. \citet{loureiro2021learning} considers a problem of multi-class classification of Gaussian mixtures with generic covariances and means in the proportional asymptotic regime. They express asymptotic in-sample and out-of-sample classification errors of a generic convex-loss-minimization algorithm through the solution to a system of non-linear equations under the condition that that solution exists and is unique.

Asymptotic methods often stem from analytical methods from statistical physics that can accurately predict properties of certain large stochastic systems.  \citet{Huang_2021} consider a binary linear classification problem in which classes have arbitrary means and covariances, and use a non-rigorous computation based on replica-symmetry trick to obtain expressions for the  distribution of the solution to a generic loss minimization problem. Interestingly, for the case of symmetric clusters with the same covariance, the empirical mean of that distribution is $(\alpha \bI_p + \beta\bSigma)^{-1}\bmu$, where scalars $\alpha, \beta$ are a solution to a certain system of non-linear equations. Because of this form of the result, one can say that they observed what we call ``recovering the geometry." However, the number of components in which the geometry is recovered is hidden behind the system of non-linear equations.

\citet{shamir22a} considers a data structure similar to ours. The main result of that paper can be formulated as follows: in an asymptotic setting in which the rows of $\bX_\ktoinf$ become very close to orthonormal, the maximum margin solution to the classification problem effectively minimizes hinge loss on the first $k$ components. This paper was an important motivation for our work as it suggested that the regime considered in the regression literature may also lead to fruitful results in classification.

These papers, however, are substantially different from this work. They either consider the maximum margin solution instead of the ridge regression solution or obtain results in the form of a solution to a system of equations that is difficult to approach analytically. Because of that, we do not provide more detailed comparisons between our results and theirs. The works we discuss in the remainder of this section, however, turn out to be directly comparable to ours.

Still in the asymptotic literature, \citet{muthukumar2021classification} considers a model similar to that of \citet{montanari2023generalization}, but with a different choice of the covariance structure. Their main result is given for a certain ``bi-level" covariance whose eigenvalues take one of two values: there is a small number of eigenvalues with a large value and a large number of eigenvalues of small value. One can immediately see the similarity between that structure and the structure we introduced in Section \ref{sec::covariance spectrum assumptions}. Even though the goal of \citet{muthukumar2021classification} is to study the maximum margin solution, the approach they take is motivated by a recent observation  that under certain assumptions, the maximum margin solution coincides with the MNI solution \citep{hsu2021proliferation, NEURIPS2021_26d4b431}. This phenomenon is known as ``support proliferation." Because of that, the main result of \citet{muthukumar2021classification} is actually derived for the minimum norm interpolating solution, which makes it possible to compare it to our result.

When it comes to the non-asymptotic literature, \citet{Cao_Gu_Belkin} study the classification error of the MNI solution and \citet{wang2021binary} study the ridge regression solution, both motivated by the idea of support proliferation. Finally, \citet{chatterji2020linearnoise} obtain a bound on the misclassification probability of maximum margin solution in binary classification. To approach the maximum margin solution, they use its characterization for separable data as the limit of gradient descent on logistic loss. As it turns out, however, the assumptions they consider imply that support proliferation happens on the event when their proof works.

We give detailed comparisons with \citet{Cao_Gu_Belkin, wang2021binary, chatterji2020linearnoise, muthukumar2021classification} in Section \ref{sec::comparisons with other papers}. Interestingly, support proliferation requires the whole data distribution to have high effective rank, so most of the results from those papers correspond to our result with $k = 0$.

Finally, we note that a version of our work was previously included as part of the first author's PhD thesis \citep{tsigler2024benign}.

\subsection{Structure of the paper}

We start in Section \ref{sec::geometry discussion} by considering MNI and providing geometric descriptions of the solution vector. In Section \ref{sec::assumptions}, we introduce our assumptions on the distribution of the data to obtain quantitative bounds. Section \ref{sec::main results main body} gives those quantitative results and explains the regimes they go through depending on the magnitude of $\bmu$. We put almost all the technical steps of the proofs of those bounds in the appendix, presenting only their outline in Section~\ref{sec::proofs}. In Section~\ref{sec::regularization main body}, we study the influence of ridge regularization on the error of the classifier.  Finally, we provide detailed comparisons with the previous literature in Section \ref{sec::comparisons with other papers}.

\subsection{Additional notation}
\label{sec::additional_notation}

We use the symbol $:=$ to introduce definitions: for example, $b := a + 1$ means that we introduce a new quantity $b$ which is defined as $a + 1$.
For any scalars $a, b$ we denote $\min(a, b)$ by $a \wedge b$ and $\max(a, b)$ by $a \vee b$. We denote $a\vee 0$ as $a_+$. For any $i \in \{1, 2, \dots, p\}$ we denote the $i$-th coordinate vector in $\R^p$ as $\be_i$. We use $a\approx b$ to denote an informal statement that $a$ and $b$ are within a constant factor of each other with high probability (which we abbreviate as w.h.p.). Analogously, we use $a \gtrsim b$ ($a \lesssim b$) to denote ``w.h.p. $a$ is at least (at most) constant times $b$" and $a \gg b$ ($a \ll b$) to denote ``w.h.p. $a$ is at least a large (at most a small) constant times $b$." 
More precisely, when a part of an assumption (e.g., in a statement of the form ``if $a\gg b$ then \dots"), $a \gg b$ means that there is a large enough constant $C$ such that if $a > Cb$ then the conclusion holds. When a part of a conclusion (e.g., ``if \dots then $a \gg b$"), $a \gg b$ means that $a > Cb$, and the constant $C$ can be arbitrarily large if the constants in the assumption are taken large enough. For example, the statement ``if $a\gg b$ then $c \gg d$" means that for any constant $C_{c,d} > 0$ there exists a constant $C_{a,b} > 0$ such that if $a > C_{a,b}b$, then $c > C_{c,d}d$.

For any $\bu, \bv \in \R^d$ we write $\bu \geq \bv$ to denote that all the components of $\bu - \bv$ are non-negative. We use $\diag(\bv)$ to denote the diagonal matrix in $\R^{d\times d}$ whose diagonal elements are elements of $\bv$. For any positive integer $d$, we denote by $\bzero_d \in \R^d$ the vector of all zeros, by $\bone_d \in \R^d$ the vector of all ones, and by $\bI_d \in \R^{d\times d}$ the identity matrix. For a linear space $\cA \subseteq \R^p$, we denote its orthogonal complement by $\cA^\perp$.

There is a slight collision with notation because we use $\bmu$ for the centers of the clusters and $\mu_i$ to denote the $i$-th component of $\bmu$, but we also use $\mu_i(\bM)$ to denote the $i$-th largest eigenvalue of the symmetric matrix $\bM$. We abbreviate positive-definite as PD and positive-semi-definite as PSD. For any PSD matrix $\bM \in \R^{d\times d}$ and any $\bu \in \R^d$, we define $\|\bu\|_\bM := \sqrt{\bu^\top\bM\bu}$.  For any matrix $\bM \in \R^{n\times p}$, we denote its pseudo-inverse by $\bM^\dagger \in \R^{p \times n}$.

We say that a random element $\bV$ of some real vector space has a symmetric distribution if $\bV$ has the same distribution as $-\bV$. We use $\P(\sA)$ to denote the probability of an event $\sA$ and $\E[\xi]$ to denote expectation of a random variable $\xi$. We use the notation $\P_{\xi}$ and $\E_\xi$ for probability and expectation with respect to a draw of the random element $\xi$.

\section{Geometric picture for minimum norm interpolation}
\label{sec::geometry discussion}

In this section, we present a geometric view of binary classification. We restrict the discussion to MNI, that is, the ridge solution $\solnridge$ with zero regularization $\lambda = 0$. We do so because while it is straightforward to think about MNI geometrically, ridge regularization is a more algebraic construct.

Explicitly, the MNI solution can be defined as
\begin{equation}\label{eq:mni}
    \mni := \arg\min_{\bw\in\R^p}\|\bw\| \text{ s.t. } \bX\bw = \hat{\by},
\end{equation}
i.e., the vector with minimum Euclidean norm that interpolates the (potentially noisy) labels. There is an explicit formula for~\eqref{eq:mni}, namely $\mni = \bX^\dagger\hat{\by}$, where $\bX^\dagger$ denotes the pseudo-inverse of~$\bX$. Unfortunately, this characterization is not  convenient for us since we want to decouple the contributions of~$\bmu$ from~$\bQ$. An alternative, more amenable definition can be obtained from the dual form of the convex optimization problem~\eqref{eq:mni}, namely,
\begin{equation}\label{eq:mni_dual}
    \mnitilde := \bX^\top\bD_{\hat{\by}}\balpha, \text{ where }\balpha = \argmin_{\balpha \in \R^n}\|\bX^\top\bD_{\hat{\by}}\balpha\|  \text{ s.t. } \balpha^\top\bone_n = 1,
\end{equation}
where~$\bD_{\hat{\by}} := \diag(\hat{\by})$.
In other words, $\mnitilde$ is the projection of $\bzero_p$ onto the affine span of the columns of $\bX^\top\bD_{\hat{\by}}$. Recall that the affine span is the set of linear combinations whose coefficients sum to one, i.e., the set $\{\bX^\top\bD_{\hat{\by}}\balpha : \balpha^\top\bone_n = 1\}$. The precise result is given by the following proposition.

\begin{proposition}\label{thm:w_tilde}
The vectors $\mni$ and $\mnitilde$ have the same direction, but different norms. They are related to each other as follows:
\begin{equation}
\label{eq::mni and mnic norm relation}
\mni = \frac{\mnitilde}{\|\mnitilde\|^2}, \quad \mnitilde = \frac{\mni}{\|\mni\|^2}.
\end{equation}
\end{proposition}

\begin{proof}
We provide here a geometric argument. Alternatively, the proof can be carried out using Lagrangian duality to relate~\eqref{eq:mni} and~\eqref{eq:mni_dual}.
First, because $\by\in\{-1,1\}^n$, we can rewrite the definition of MNI as
\[
\mni = \arg\min_{\bw\in\R^p}\|\bw\| \text{ s.t. } \bD_{\hat{\by}}\bX\bw = \bone_n.
\]
We see that $\mni$ has the same scalar products with all the columns of $\bX^\top\bD_{\hat{\by}}$, which implies that it has the same scalar product with all elements of the affine span of those columns.  Therefore, it must be perpendicular to that affine span. Note that it also lies in their linear span, and there is only one direction in that linear span that is perpendicular to the affine span: the direction of the projection of zero onto the affine span. Thus, we already obtained that $\mni$ and $\mnitilde$ have the same direction. As for the norm, notice that since $\mnitilde$ belongs to the affine span, it must be that $\mni^\top\mnitilde = 1$. This yields \eqref{eq::mni and mnic norm relation}.
\end{proof}

Since both~$\mni$ and~$\mnitilde$ have the same direction, they result in the same classification rule. Nevertheless, it is sometimes more convenient to work with one or the other. Indeed, $\mnitilde$ is more suitable to study the noiseless case~(Section \ref{sec::mni geometry no noise}) whereas~$\mni$ turns out to be better suited in the presence of label-flipping noise~(Section \ref{sec::mni geometry with noise}).

\subsection{MNI without label-flipping noise}
\label{sec::mni geometry no noise}

In this section, we study the case without label-flipping noise, that is, $\eta = 0$ and $\hat{\by} = \by$. In this case, we write~\eqref{eq:mni} and~\eqref{eq:mni_dual} as
\begin{align*}
\mnic :=& \arg\min_{\bw\in\R^p}\|\bw\| \text{ s.t. } \bX\bw = {\by},\\
\mnitildec :=& \bX^\top\bD_\by\balpha, \text{ where }\balpha = \argmin \|\bX^\top\bD_\by\balpha\| \text{ s.t. } 
\balpha^\top\bone_n = 1.
\end{align*}
Here, the superscript $c$ stands for ``clean." We start by using the expression for $\bX$ to get
\[
\bX^\top\bD_{\hat{\by}} = (\bQ + \by\bmu^\top)^\top\bD_{\by} = \bQ^\top\bD_{\by} + \bmu\bone_n^\top.
\]
We see that changing $\bmu$ simply shifts all the columns of $\bX^\top\bD_{\hat{\by}}$ by the same vector, which gives an easy way to derive the formulas for the solution. For the rest of this section, we will explicitly track the dependence on $\bmu$ in the notation, that is, we will write $\mnic(\bmu)$ and $\mnitildec(\bmu)$ instead of $\mnic$ and $\mnitildec$.

Let us start with the case $\bmu = \bzero_p$ and then see the effect of adding $\bmu$. When $\bmu = \bzero_p$, the matrix $\bX$ coincides with $\bQ$, so that
\begin{align*}
\mnic(\bzero_p) =& \bQ^\dagger\by = \bQ^\top\bA^{-1}\by,\\
\mnitildec(\bzero_p) =& \frac{\mnic(\bzero_p)}{\|\mnic(\bzero_p)\|^2} =  \frac{\bQ^\top\bA^{-1}\by}{\by^\top\bA^{-1}\by}.
\end{align*}
Note that $\|\mnitildec(\bzero_p)\|^2 = (\by^\top\bA^{-1}\by)^{-1}$.
Hence, as we add $\bmu$, the columns of $\bX^\top\bD_{\hat{\by}}$ are shifted and thus also their affine span. Denote the linear span of the columns of $\bQ^\top\bD_{\hat{\by}}$ as $\cQ$ and the linear space parallel to the affine span of those columns as $\cQ_A$. Note that $\cQ_A$ is orthogonal to $\mnitildec(\bzero_p)$ and that $\cQ = \cQ_A \oplus \langle\mnitildec(\bzero_p)\rangle$---the direct sum of $\cQ_A$ and the line spanned by $\mnitildec(\bzero_p)$. Thus, we can decompose $\bmu$ into 3 orthogonal components: one perpendicular to $\cQ$, one lying in $\cQ_A$, and one in the direction of $\mnitildec(\bzero_p)$. Explicitly,
\begin{equation}
\label{eq::mu decomposition}
\bmu = \bmu_\perp + \bmu_{\parallel \cQ_A} + \frac{\bmu^\top\mnitildec(\bzero_p)}{\|\mnitildec(\bzero_p)\|^2}\mnitildec(\bzero_p).
\end{equation}

Recall that $\mnitildec(\bmu)$ is the projection of the origin onto the affine span of the columns of  $\bX^\top\bD_{\hat{\by}}$. Note that $\bmu_{\parallel \cQ_A}$ does not change that affine span (it shifts the affine span by a vector parallel to it). The  component in the direction of $\mnitildec(\bzero_p)$ doesn't change the linear span, but it shifts the affine span orthogonally and is therefore added to the projection. Finally, $\bmu_\perp$ shifts the linear span orthogonally, so it also gets added to the projection. Therefore, we get
\[
\mnitildec(\bmu) = \mnitildec(\bzero_p) + \bmu_\perp + \frac{\bmu^\top\mnitildec(\bzero_p)}{\|\mnitildec(\bzero_p)\|^2}\mnitildec(\bzero_p).
\]
Plugging in the expressions for each variable, we get the formula
\begin{equation}
\label{eq::MNI formula new scaling}
\mnitildec(\bmu) = \frac{\bQ^\top\bA^{-1}\by}{\by^\top\bA^{-1}\by} + \underbrace{(\bI_p - \bQ^\top\bA^{-1}\bQ)\bmu}_{\bmu_\perp} + \frac{\bnu^\top\bA^{-1}\by}{\by^\top\bA^{-1}\by}\bQ^\top\bA^{-1}\by,
\end{equation}
where we use the notation $\bnu := \bQ\bmu$.

Now that we have this decomposition, we can discuss its quantitative implications. Recall that we are interested in the quantity $\bmu^\top\mnitildec(\bmu)/\|\mnitildec(\bmu)\|_\bSigma.$ Thus, we compare the scalar product with $\bmu$ and the norm in $\bSigma$ for the terms above.

In Section \ref{sec::recovering the geometry introduction}, we claimed that the noiseless solution behaves as $\bQ^\top \bA^{-1}\by + (\bSigma + n^{-1}\Lambda\bI_p)^{-1}\bmu$. In the case of MNI, this happens because $\bmu_\perp$ behaves like $(n\Lambda^{-1}\bSigma + \bI_p)^{-1}\bmu$ and $\by^\top\bA^{-1}\by \approx n\Lambda^{-1}$. Additionally, our bounds show that the second term, $\bmu_\perp$, always dominates the third term in terms of both scalar product with $\bmu$ and the norm in $\bSigma$ in the regime that we consider. More concretely, Lemma \ref{lm::high probability bounds} in Appendix \ref{sec::main bound put together} gives $\bmu^\top \bmu_\perp \approx \bmu^\top (n\Lambda^{-1}\bSigma + \bI_p)^{-1}\bmu$ and $\|\bmu_\perp\|_\bSigma \lesssim \|(n\Lambda^{-1}\bSigma + \bI_p)^{-1}\bmu\|_\bSigma$. Tightness of the latter bound was shown in \citet{BO_ridge} since the same quantity (up to the change of notation) arises as the bias term in regression. We have already mentioned this connection to the regression literature in Section \ref{sec::comparison with regression}, where we also pointed out that $\bQ^\top \bA^{-1}\by$ is directly analogous to the variance part of the MNI solution for regression.

\subsection{MNI with label-flipping noise and benign overfitting}
\label{sec::mni geometry with noise}

Let us proceed with the case where labels are contaminated by label-flipping noise. As we previously mentioned, we now use the MNI characterization~$\mni$ in~\eqref{eq:mni} rather than~$\mnitilde$ from~\eqref{eq:mni_dual}.

Start by recalling that we denoted the linear span of the columns of $\bQ^\top\bD_{\by}$ as $\cQ$. For any $\bv \in \R^p$, let the projection of $\bv$ on $\cQ$ be $\bv_\parallel$ and the projection of $\bv$ on $\cQ^\perp$ be $\bv_\perp$. Note that $\bQ\bv_\perp = \bzero_n$ for any $\bv \in \R^p$. Let us now consider which labels are interpolated by $\mni_\perp$ and $\mni_\parallel$. Using the expression for~$\bX$, we obtain
\begin{gather*}
(\bQ + \by\bmu^\top)\mni_\perp = \by\bmu^\top\mni_\perp = \alpha\by
\\
(\bQ + \by\bmu^\top)\mni_\parallel = \hat{\by} - \alpha\by
\Rightarrow
\bQ\mni_\parallel = \hat{\by} - \alpha\by - \by\bmu^\top\mni_\parallel = \hat{\by} - \beta\by,
\end{gather*}
where we introduced the scalar quantities $\alpha$ and $\beta$. To proceed, note that there is a unique vector $\bw \in \cQ$ such that $\bQ\bw = \hat{\by} - \beta\by$, namely $\bw = \bQ^\dagger(\hat{\by} - \beta\by)$. Thus,
\[
\mni_\parallel = \bQ^\dagger(\hat{\by} - \beta\by) = \bQ^\top\bA^{-1}(\hat{\by} - \beta\by).
\]
On the other hand,~\eqref{eq:mni} implies that $\mni$ always lies in the span of the columns of $\bX^\top$ and since $\bX^\top = \bQ^\top-\bmu\by^\top$, projections of those columns onto $\cQ^\perp$ must be $\pm \bmu_\perp$. Thus, $\mni_\perp$  must be collinear with $\bmu_\perp$. Therefore, there exist scalars $a, b$ such that
\[
\mni = \bQ^\top\bA^{-1}\hat{\by} + a\bQ^\top\bA^{-1}\by + b\bmu_\perp,
\]
which reduces the problem to two dimensions.

The next simplifying step is to move to an orthogonal basis. Since $\bmu_\perp$ is already orthogonal to both $\bQ^\top\bA^{-1}\hat{\by}$ and $\bQ^\top\bA^{-1}\by$, it suffices to find a scalar $\xi$ such that $\bQ^\top\bA^{-1}(\hat{\by} - \xi\by)$ is orthogonal to $\bQ^\top\bA^{-1}\by$. We write
\begin{equation}\label{eq:xi}
\by^\top\bA^{-1}\bQ\bQ^\top\bA^{-1}(\hat{\by} - \xi\by) = 0
\Rightarrow
\by^\top\bA^{-1}(\hat{\by} - \xi\by) = 0
\Rightarrow
\xi = \frac{\by^\top\bA^{-1}\hat{\by} }{\by^\top\bA^{-1}\by}.
\end{equation}
Note that $\E_{\by, \hat{\by}}[\by^\top\bA^{-1}\hat{\by}] = (1 - 2\eta)\tr(\bA^{-1})$ and $\E_{\by}[\by^\top\bA^{-1}\by] = \tr(\bA^{-1})$, so informally $\xi \approx 1 - 2\eta$. Now, define $\tilde{\by} := \hat{\by} - \xi\by$ to get
\[
\mni =  \bQ^\top\bA^{-1}\tilde{\by} + \Delta\bw,
\]
where $\Delta\bw$ belongs to the span of $\bmu_\perp$ and $\bQ^\top\bA^{-1}\by$.  Since $ \Delta\bw$ is orthogonal to $\bQ^\top\bA^{-1}\tilde{\by}$ and $\mni$ is the minimum norm solution that interpolates the labels $\hat{\by}$, it must be that $\Delta\bw$ is the minimum norm vector that interpolates the labels
\[
\hat{\by} - (\bQ + \by\bmu^\top)\bQ^\top\bA^{-1}\tilde{\by} = (\xi - \bnu^\top\bA^{-1}\tilde{\by})\by.
\]
Note that these are rescaled versions of the clean labels $\by$. Hence, $\Delta\bw$ is a scaled noiseless solution, that is,
\begin{equation}\label{eq::mni through rescaling of mnic}
\mni =  \bQ^\top\bA^{-1}\tilde{\by} +(\xi - \bnu^\top\bA^{-1}\tilde{\by}) \mnic.
\end{equation}

In~\eqref{eq::mni through rescaling of mnic}, $\bQ^\top\bA^{-1}\tilde{\by}$ acts as a ``noise vector": it has no dependence on $\bmu$ and it has a symmetric distribution. The term $\xi \mnic$ is a scaling of the noiseless solution. Recall that $\xi \approx 1 - 2\eta$, which is close to $1$ when the noise level $\eta$ is small. The last term $- \bnu^\top\bA^{-1}\tilde{\by}\mnic$ is also a scaling of the noiseless solution, but the scaling factor $\bnu^\top\bA^{-1}\tilde{\by}$ has a symmetric distribution since it is a linear function of $\tilde{\by}$. That is, this term points in the opposite direction of the noiseless solution with probability $0.5$ and also acts as a ``noise vector."

Now let us consider the magnitude of these terms. To do so, it is informative to consider the scale of $\bmu$ as a parameter and see how it affects the Euclidean norm of each term.
The first term, $\bQ^\top\bA^{-1}\tilde{\by}$, does not depend on $\bmu$, so its Euclidean norm remains the same.
As for $\mnic$, it is equal to $\bQ^\top\bA^{-1}\by$ when $\|\bmu\| = 0$. It is therefore similar, except it contains the clean labels $\by$ instead of the ``label noise" $\tilde{\by}$. If we consider the noise to be a small constant, both vectors should have Euclidean norms of the same order. As $\bmu$ grows, however, the vector $\mnic$ changes and its norm~(asymptotically) decreases inversely proportional to the scale of $\bmu$. Indeed, as we saw in Section~\ref{sec::mni geometry no noise}, $\mnitildec$ has affine dependence on $\bmu$, and, due to Proposition~\ref{thm:w_tilde}, the norm of $\mnic$ is reciprocal to the norm of $\mnitildec$.
The third term, $- \bnu^\top\bA^{-1}\tilde{\by}\mnic$ starts at zero since $\bnu$ scales with $\bmu$. Due to that scaling, as $\bmu$ grows, it converges to a vector of finite length. Recall that this vector is equally likely to point in the same direction as $\mnic$ as in the opposite direction.

Overall, we see that adding a small constant amount of label-flipping noise makes the solution look significantly different compared to the noiseless one. First, it picks up an additional scaling factor, which may potentially flip the sign of the projection on the direction of the noiseless solution when $\bmu$ is large enough.  Second, it picks up an additional noise component in the orthogonal direction, whose magnitude can be comparable to or even much larger than the magnitude of the noiseless solution.

Despite those differences, however, the bound for the noisy case is surprisingly similar to the bound in the noiseless case. Recall that we need to estimate two scalar quantities: $\|\mni\|_\bSigma$ and $\bmu^\top \mni$. When it comes to the first of them, our bounds suggest%
\footnote{We can only make an informal statement here since our formal arguments work with slightly different expressions: instead of $\tilde{\by}$ we use the formulas involving $\Delta \by := \hat{\by} - \by$ since it has i.i.d.\ components. We also have not proven that our bounds are tight for the case with label-flipping noise.}
that the sum of $\|\bQ^\top\bA^{-1}\tilde{\by}\|_\bSigma$ and $\|\mnic\|_\bSigma$ dominates $|\bnu^\top\bA^{-1}\tilde{\by}|\|\mnic\|_\bSigma$, so that $\|\mni\|_\bSigma$ only picks up one term compared to $\|\mnic\|_\bSigma$, namely, $\|\bQ^\top\bA^{-1}\tilde{\by}\|_\bSigma$
(compare eq.~\eqref{eq:mu.dot.wmnic.and.norm.wmnic.noiseless} and~\eqref{eq:norm.wmnic.noisy} below).
We provide a more detailed discussion of this domination in Section~\ref{sec::main lower bound} after we present our main quantitative bound --- Theorem~\ref{th::main}.

When it comes to the scalar product with $\bmu$, even more cancellations occur. First, recall that for the clean solution we obtained that
\[
\mnic = \frac{\mnitildec}{\|\mnitildec\|^2}, \quad \text{with } \mnitildec = \frac{\bQ^\top\bA^{-1}\by}{\by^\top\bA^{-1}\by} + \bmu_\perp + \frac{\bnu^\top\bA^{-1}\by}{\by^\top\bA^{-1}\by}\bQ^\top\bA^{-1}\by.
\]
Letting $S := \by^\top\bA^{-1}\by\|\mnitildec\|^2$ and using the definition of~$\xi$ from~\eqref{eq:xi}, we obtain that
\begin{equation}\label{eq:inner_product_noisy}
    S \bmu^\top\mni = \by^\top\bA^{-1}\hat{\by} \|\bmu_\perp\|^2 + (1 +  \bnu^\top\bA^{-1}\by) \bnu^\top \bA^{-1}\hat{\by}.
\end{equation}
(See Appendix~\ref{sec::formulas for solution} for the derivation.)

Hence, the formula for the scalar product in the noisy case turns out to be almost the same as in the noiseless case and the bound does not change. Indeed, note that the noiseless case is a particular case of the noisy case with $\hat{\by} = \by$, so the formula for the noiseless case can be obtained from Equation~\eqref{eq:inner_product_noisy} by simply dropping the hat of $\hat{\by}$, which introduces very little change:
\begin{align*}
    S \bmu^\top \mnic(\bmu) = \by^\top \bA^{-1} \by \|\bmu_\perp\|^2 + (1 + \bnu^\top \bA^{-1} \by) \bnu^\top \bA^{-1} \by,
\end{align*}
(cf. Equation~\eqref{eq::MNI formula new scaling}).

Part of the reason for this cancellation can be seen from our derivation. We start by saying that we need the component~$\bQ^\top\bA^{-1} \tilde{\by}$ to interpolate the ``label noise" $\tilde{\by}$. When multiplied by $\bQ + \by\bmu^\top$, this vector picks up additional labels proportional to $\by$ because of the scalar product with $\bmu$. We then proceed by forcing the extra term $\Delta\bw$ to kill those additional labels. However, the labels that $\Delta\bw$ interpolates largely come from its scalar product with $\bmu$. In the end, this leads to the cancellation when we compute the scalar product of $\bmu$ and the sum of $\bQ^\top\bA^{-1} \tilde{\by}$ and $\Delta \bw$. This, however, only provides an algebraic explanation of why some of the terms disappear. 
Finding an intuitive explanation for why the formulas for the scalar product are so similar in the noisy and noiseless cases remains an intriguing question.

\section{Assumptions on the data}
\label{sec::assumptions}

So far we have seen how MNI behaves geometrically and which terms arise in the expressions of interest. To make a quantitative statement about classification, however, one needs to bound those terms. Those bounds, in their turn, require assumptions. In this section, we explain the assumptions we impose on the distribution of the rows of $\bZ$ and on the sequence $\{\lambda_i\}_{i=1}^p$ in order to obtain our results.

\subsection{Gram matrix of the tails}
\label{sec::definition of sA_k}

The central object in our analysis is the (regularized) Gram matrix of the ``tails" of the data, which we denote as
\begin{equation}
\label{eq::definition of A_k}
\bA_k := \bQ_\ktoinf\bQ_\ktoinf^\top + \lambda \bI_n.
\end{equation}
Just as in \citep{BO_ridge} the main assumption under which our arguments work is that the condition number of the matrix $\bA_k$ is bounded by some constant. Therefore, we introduce the following event.

\begin{definition}
\label{def::sA_k(L)}
For any $k \in \{0, 1, \dots, p-1\}$ and $L \geq 1$, we denote by $\sA_k(L)$ the event
\begin{equation}
\sA_k(L) := \left\{\dfrac{1}{L}\Bigl(\lambda + \sum_{i > k}\lambda_i\Bigr) \leq \mu_n(\bA_k) \leq \mu_1(\bA_k) \leq L\Bigl(\lambda + \sum_{i > k}\lambda_i\Bigr)\right\}.
\end{equation}
\end{definition}

Note that $\E\bA_k = \bI_n \cdot \left(\lambda + \sum_{i > k}\lambda_i\right)$. Thus, on $\sA_k(L)$, the eigenvalues of $\bA_k$ are within a constant factor of the eigenvalues of its expectation.

Throughout the paper we will also always impose assumptions of the form
\begin{equation}
\label{eq::assumption on high effective rank}
\lambda + \sum_{i > k}\lambda_i > c\left(n\lambda_{k+1} + \sqrt{n\sum_{i > k}\lambda_i^2}\right),
\end{equation}
where $c$ is some large constant. It is closely related to saying that the event $\sA_k$ holds with high probability as well as to the notions of effective ranks used by \cite{benign_overfitting} and \cite{BO_ridge}. Indeed, consider the following lemma.

 \begin{lemma}[Lemma 16 from \citep{BO_ridge}]
\label{lm:: eigenvalues of A_k indep coord}
Suppose that elements of $\bZ$ are $\sigma_x$-sub-Gaussian and independent. There exists a constant $c$ that only depends on $\sigma_x$ such that, with probability at least $1 - ce^{-n/c}$,
\begin{align*}
\mu_1(\bA_k) = \lambda + \mu_1(\bQ_\ktoinf\bQ_\ktoinf^\top) \leq& \lambda + \sum_{i > k} \lambda_i + c\left(n\lambda_{k+1} + \sqrt{n\sum_{i > k} \lambda_i^2}\right),\\
\mu_n(\bA_k) = \lambda +\mu_n(\bQ_\ktoinf\bQ_\ktoinf^\top) \geq& \lambda + \sum_{i > k} \lambda_i - c\left(n\lambda_{k+1} + \sqrt{n\sum_{i > k} \lambda_i^2}\right).
\end{align*}
\end{lemma}

Lemma \ref{lm:: eigenvalues of A_k indep coord} shows that if components of the data are independent and sub-Gaussian, then \eqref{eq::assumption on high effective rank} implies that $\sA_k(L)$ holds with high probability for some constant $L$. The reason why we introduce event $\sA_k(L)$ instead of assuming independence of the components is that we do not believe independence to be \emph{necessary} for $\sA_k(L)$ to hold. The same logic was followed in \cite{BO_ridge} (see their Section 4). Moreover, \cite[Section 5]{BO_ridge} explains that much weaker conditions, such as sub-Gaussianity and some small-ball condition, are sufficient for the event $\sA_k(L)$ to hold with high probability. In \cite[Section 5.4]{BO_ridge}, they even show that $\sA_k(L)$ can hold with high probability for some heavy-tailed distributions.

\subsection{Algebraic assumptions}
\label{sec::definition of sB_k}

Similarly to \cite{BO_ridge}, our arguments for the main lower bound cleanly decompose into an algebraic and a probabilistic part. Hence, we do not need to formulate that bound with some probability over the draw of $\bQ$, but we can just specify the exact event on which our results hold. We have already introduced the event $\sA_k(L)$. Another event that we need is as follows.

\begin{definition}
For any $k \in \{0, 1, \dots, p-1\}$ and $c_B > 0$, we let $\sB_k(c_B)$ be the event on which all of the following hold:
\begin{enumerate}
\item $\mu_1(\bZ_\uptok^\top\bZ_\uptok) \leq c_B n$ and $\mu_n(\bZ_\uptok^\top\bZ_\uptok) \geq n/c_B$.
\item $\|\bQ_\ktoinf\bmu_\ktoinf\|^2 \leq c_Bn\|\bmu_\ktoinf\|^2_{\bSigma_\ktoinf} $.
\item $\tr(\bQ_\ktoinf \bSigma_\ktoinf \bQ_\ktoinf^\top) \leq c_Bn \sum_{i > k} \lambda_i^2.$
\item $\tr(\bZ_\uptok^\top\bZ_\uptok) \leq c_Bnk$.
\item $\|\bQ_\ktoinf\bSigma_\ktoinf\bQ_\ktoinf^\top \| \leq c_B\left(\sum_{i > k}\lambda_i^2 + n\lambda_{k+1}^2\right).$
\end{enumerate}
\end{definition}

It is easy to see that the event $\sB_k(c_B)$ holds with high probability if the constant $c_B$ is large enough. For the case of sub-Gaussian data, this can be stated more precisely.

\begin{lemma}
\label{lm::sB under sub-Gaussianity}
Suppose the distribution of the rows of $\bZ$ is $\sigma_x$-sub-Gaussian. One can take the constant $c_B$ large enough depending only on  $\sigma_x$ such that for any $k < n/c_B$ the  probability of the event $\sB_k(c_B)$ is at least $1 - c_Be^{-n/c_B}$.
\end{lemma}

\begin{proof}
We need to show that all five bounds from the definition of $\sB_k(c_B)$ hold with probability at least $1 - c_B e^{-n/c_B}$, where $c_B$ only depends on $\sigma_x$. Bounds 1--4 were derived in the proof of \citep[Theorem~5]{BO_ridge}~(see the display in the middle of page~59). To transform their notation into ours, one needs to replace $X$ by $\bQ$, $\Sigma$ by $\bSigma$ and $Z$ by $\bZ$.
The last statement follows directly from \cite[Lemma 24]{BO_ridge}, one just needs to plug in $\bSigma^2$ instead of $\Sigma$.
\end{proof}

As with the definition of the event $\sA_k$, we introduce the event $\sB_k(c_B)$ instead of assuming sub-Gaussianity because we believe that sub-Gaussianity is not necessary for $\sB_k(c_B)$ to hold with high probability. This was also discussed by \cite{BO_ridge} (see their Section 6). Indeed, the first condition in the definition of $\sB_k(c_B)$ is just concentration of sample covariance in dimension $k$ with $n$ data points, which is known to hold for heavy-tailed distributions (see \citet{tikhomirov2018sample} and references therein). The inequalities 2--4 are just the law of large numbers (concentration of the sum of $n$ i.i.d. random variables). Only inequality 5~(bound on the norm of the Gram matrix) is somewhat less standard. Note, however, that the Gram matrix has the same spectral norm as the sample covariance matrix multiplied by $n$. Therefore, that inequality could be obtained as a direct corollary of a dimension-free bound on the spectral norm of a sample covariance matrix. An example of a heavy-tailed result of this type can be found in \cite[Theorem 2]{abdalla2023covariance}.

\section{Main results}
\label{sec::main results main body}

In order to formulate the results more succinctly throughout the paper we introduce additional notation. First of all, we denote the bound on the sub-Gaussian constant of the label-flipping noise as
\begin{equation}
\label{eq::sigma_eta}
\sigma_\eta := 1/\sqrt{\ln\frac{3+\eta^{-1}}{2}}.
\end{equation}
Next, for a given $k$, we define
\begin{align*}
\Lambda :=& \lambda + \sum_{i > k}\lambda_i,\\
V :=& n^{-1}\tr\left(\left(\Lambda n^{-1}\bSigma_\uptok^{-1} + \bI_k\right)^{-2}\right)+ \Lambda^{-2}n\sum_{i > k}\lambda_i^2,\\
\Delta V :=& \frac{1}{n}\wedge\frac{n\lambda_1^2}{\Lambda^2} + \frac{n\lambda_{k+1}^2 + \sum_{i > k}\lambda_i^2}{\Lambda^2},\\
B :=& n^{-2}\Lambda^2\left\|\left(\Lambda n^{-1}\bSigma_\uptok^{-1} + \bI_k\right)^{-1}\bSigma_\uptok^{-1/2}\bmu_\uptok\right\|^2 + \|\bmu_\ktoinf\|_{\bSigma_\ktoinf}^2,\\
\Diamond^2 :=&n\Lambda^{-2}B,\\
M :=& \frac{\Lambda}{n}\left\|\left(\Lambda n^{-1}\bSigma_\uptok^{-1} + \bI_k\right)^{-1/2}\bSigma_\uptok^{-1/2}\bmu_\uptok\right\|^2 + \|\bmu_\ktoinf\|^2,\\
N:=& n\Lambda^{-1}M.
\end{align*}
Note that we do not track the dependence on $k$ in the notation since we always introduce $k$ before using it.

To explain how these quantities arise, let us consider the ridgeless case ($\lambda = 0$) and return to the geometric picture from Section~\ref{sec::geometry discussion}. First of all, note that $\Lambda$ is the energy of the tail of the covariance and, as we already mentioned in Section \ref{sec::recovering the geometry introduction}, $\Lambda$ represents the implicit regularization that this tail imposes on the learning problem. Next, the term $V$ corresponds to $\E_{\by}\|\bQ^\top\bA^{-1}\by\|_\bSigma^2$ which, as discussed in Section \ref{sec::comparison with regression}, is nothing but the variance term from the regression literature \citep{benign_overfitting, BO_ridge}. As in those papers, $V$ is bounded by a constant, but can be arbitrarily small. The quantity $\Delta V$ controls deviations of $\|\bQ^\top\bA^{-1}\by\|_\bSigma^2$ with respect to the randomness in $\by$.

The terms $B$ and $M$ arise as $B \approx \|\bmu_\perp\|_\bSigma^2$ and $M \approx \bmu^\top\bmu_\perp$. Once again, notice that $\|\bmu_\perp\|_\bSigma^2$ is exactly the bias term from \citep{benign_overfitting, BO_ridge}. Interestingly, $\Diamond$, which is a rescaling of $\sqrt{B}$, also controls the magnitude of $\bnu^\top\bA^{-1}\by$.

A reader familiar with the regression literature may notice that our expressions for $V$ and $B$ are somewhat different from the main bounds in \citep{BO_ridge}. This is because we choose a different presentation strategy: while \citet{BO_ridge} give bounds in a simpler form, they are only tight for the right choice of $k$. The way we formulate our bounds makes them tight for any choice of $k$ under which the assumptions are satisfied at the cost of more involved expressions. We elaborate more on the differences in techniques from \citet{BO_ridge} in Section \ref{sec::lower bound proof sketch}, but a formal connection is established in the following proposition. The fact that the bounds are the same up to a constant multiplier for the right choice of $k$ can be seen from Lemma \ref{lm::bounds via k star}.
\begin{proposition}
\label{prop::V and B upper bound as in BO_ridge}
\begin{equation*}
V \leq \frac{k}{n} + \Lambda^{-2}n\sum_{i > k}\lambda_i^2 \quad \text{and} \quad
B \leq n^{-2}\Lambda^2\left\|\bmu_\uptok\right\|_{\bSigma_\uptok^{-1}}^2 + \|\bmu_\ktoinf\|_{\bSigma_\ktoinf}^2.
\end{equation*}
\end{proposition}
\begin{proof}
\begin{align*}
V &= n^{-1}\tr\left(\left(\Lambda n^{-1}\bSigma_\uptok^{-1} + \bI_k\right)^{-2}\right)+ \Lambda^{-2}n\sum_{i > k}\lambda_i^2
\\
{}&\leq n^{-1}\tr\left(\bI_k\right) + \Lambda^{-2}n\sum_{i > k}\lambda_i^2 =  \frac{k}{n} + \Lambda^{-2}n\sum_{i > k}\lambda_i^2,
\\
B &= n^{-2}\Lambda^2\left\|\left(\Lambda n^{-1}\bSigma_\uptok^{-1} + \bI_k\right)^{-1}\bSigma_\uptok^{-1/2}\bmu_\uptok\right\|^2 + \|\bmu_\ktoinf\|_{\bSigma_\ktoinf}^2
\\
{}&\leq n^{-2}\Lambda^2\left\|\bSigma_\uptok^{-1/2}\bmu_\uptok\right\|^2 + \|\bmu_\ktoinf\|_{\bSigma_\ktoinf}^2 = n^{-2}\Lambda^2\left\|\bmu_\uptok\right\|_{\bSigma_\uptok^{-1}}^2 + \|\bmu_\ktoinf\|_{\bSigma_\ktoinf}^2.
\end{align*}
\end{proof}

When it comes to noisy labels, informally, $V$ controls $\|\bQ^\top\bA^{-1}\tilde{\by}\|_\bSigma^2$ and $\Diamond$ controls $\bnu^\top\bA^{-1}\tilde{\by}$. We say ``informally'' because our proofs do not deal with the vector $\tilde{\by}$ directly: it is more convenient to work with the vector $\Delta \by := \hat{\by} - \by$ as it has i.i.d. coordinates.

Finally, by virtue of algebra, our results extend to ridge regression. Note, however, that only $\Lambda$ directly depends on the regularization parameter $\lambda$. Thus, explicit regularization only adds to the implicit regularization from the data without qualitatively changing the results.

When it comes to the interpretation from Section \ref{sec::introduction}, it is not hard to see that $N$ is within a constant factor of $\bmu^\top(\bSigma + \Lambda n^{-1}\bI_p)^{-1}\bmu$, while $n\Diamond^2$ is within a constant factor of $\left\|(\bSigma + \Lambda n^{-1}\bI_p)^{-1}\bmu\right\|_\bSigma^2$ (see Lemma \ref{lm::alternative form of bounds} for a precise statement).

Some useful relations between those quantities are shown by the following lemma, whose proof can be found in Appendix \ref{sec::important relations}.

\begin{restatable}[Relations between the main quantities]{lemma}{importantrelations}
\label{lm::relations}
Suppose that
\begin{equation}
\label{eq::assumption on k for important relations}
k \leq n \quad \text{ and} \quad \Lambda > n\lambda_{k+1} \vee \sqrt{n\sum_{i > k}\lambda_i^2}.
\end{equation}
Then, \quad
$\displaystyle
n\Diamond^2 \leq N, \quad n\Diamond^2 \leq N\sqrt{n\Delta V}, \quad V \leq 2, \quad \Delta V \leq \frac{3}{n},
\quad\text{and}\quad \Delta V \leq 4V.
$
\end{restatable}

\subsection{Lower bound}
\label{sec::main lower bound}

Our main lower bound on the quantity $\bmu^\top\solnridge/\|\solnridge\|_\bSigma$ is given by the following theorem.
\begin{restatable}[Main lower bound]{theorem}{mainresult}
\label{th::main}
For any $c_B > 0$ and~$L> 1$, there exists a constant $c$ that only depends on $c_B$ and $L$, such that the following holds. Assume that $\eta < c^{-1}$, $k < n/c$, and
\[
\Lambda > cn\lambda_{k+1} \vee \sqrt{n\sum_{i > k}\lambda_i^2}.
\]
For any $t \in (0, \sqrt{n}/c)$, conditionally on the event $\sA_k(L)\cap \sB_k(c_B)$, with probability at least $1 - ce^{-t^2/2}$ over the draw of $(\by, \hat{\by})$, the following inequalities hold for a certain scalar $S > 0$:
\begin{align}
S\bmu^\top\solnridge \geq& c^{-1}N - ct\Diamond ,\label{eq::main bound numerator}\\
S\|\solnridge\|_\bSigma \leq& c\left(\left[1 + N\sigma_\eta \right]\sqrt{ V + t^2\Delta V} + \Diamond\sqrt{n}\right).\label{eq::main bound denominator}
\end{align}
That is, if $N > 2c^2t\Diamond$, then on the same event,
\[
\frac{\bmu^\top\solnridge}{\|\solnridge\|_\bSigma} \geq \frac{1}{2c^2}\frac{N}{\left[1 + N\sigma_\eta \right]\sqrt{ V + t^2\Delta V} + \Diamond\sqrt{n}}.
\]
\end{restatable}

It is informative to explain how this result relates to the expressions we derived in Section~\ref{sec::geometry discussion}. Recall that we restrict ourselves to the case $\lambda = 0$ in that section, that is, $\solnridge = \mni$ and $S = \by^\top\bA^{-1}\by\|\mnitildec\|^2$. As in Section \ref{sec::geometry discussion}, let us start with the noiseless case, that is, $\eta = \sigma_\eta = 0$ and $\by = \hat{\by}$. Then,
\begin{equation*}
S\mnic = \by^\top\bA^{-1}\by\mnitildec = \by^\top\bA^{-1}\by\bmu_\perp  + (1 + \bnu^\top\bA^{-1}\by)\bQ^\top\bA^{-1}\by.
\end{equation*}
We need to bound $S\bmu^\top\mnic$ from below and $S\|\mnic\|_\bSigma$ from above, so we write
\begin{align*}
S\bmu^\top\mnic &=  \by^\top\bA^{-1}\by\|\bmu_\perp\|^2  + (1 + \bnu^\top\bA^{-1}\by)\bnu^\top\bA^{-1}\by,\\
S\|\mnic\|_\bSigma &\leq \by^\top\bA^{-1}\by\|\bmu_\perp\|_\bSigma  + (1 + |\bnu^\top\bA^{-1}\by|)\|\bQ^\top\bA^{-1}\by\|_\bSigma.
\end{align*}
The bound from Theorem \ref{th::main} can now be obtained by plugging in
\begin{gather*}
\by^\top\bA^{-1}\by \approx n\Lambda^{-1},\quad
\|\bmu_\perp\|^2 \approx M,\quad
|\bnu^\top\bA^{-1}\by| \lesssim t\Diamond,\\
\|\bmu_\perp\|_\bSigma \lesssim \sqrt{B},\quad \text{and} \quad
\|\bQ^\top\bA^{-1}\by\|_\bSigma \lesssim \sqrt{V  + t\Delta V}.
\end{gather*}
As stated in Section \ref{sec::mni geometry no noise}, the contribution of the term $\bmu_\perp$ dominates the contribution of $\bnu^\top\bA^{-1}\by$ and $\bQ^\top\bA^{-1}\by$ in both bounds, that is,
\begin{equation*}
n\Lambda^{-1}M \gtrsim  t^2\Diamond^2 \quad \text{and} \quad
\quad n\Lambda^{-1}\sqrt{B} = \sqrt{n}\Diamond \gtrsim t\Diamond\sqrt{V  + t\Delta V}.
\end{equation*}
Overall, we get
\begin{equation} \label{eq:mu.dot.wmnic.and.norm.wmnic.noiseless}
S\bmu^\top\mnic \gtrsim N - ct\Diamond \quad \text{and} \quad S\|\mnic\|_\bSigma \lesssim \sqrt{V  + t\Delta V} + \sqrt{n}\Diamond.
\end{equation}

Now, let us consider the case with label-flipping noise. As we already mentioned in Section~\ref{sec::mni geometry with noise}, due to certain algebraic cancellations, the formula for the scalar product in the noisy case is very similar to the formula of the noiseless case, namely,
\[
S \bmu^\top\mni = \by^\top\bA^{-1}\hat{\by} \|\bmu_\perp\|^2 + (1 +  \bnu^\top\bA^{-1}\by) \bnu^\top \bA^{-1}\hat{\by}.
\]
Since the vector $\hat{\by}$ is just a noisy version of the vector $\by$, the quantities $\by^\top\bA^{-1}\hat{\by}$ and $\bnu^\top \bA^{-1}\hat{\by}$ are close to $\by^\top\bA^{-1}{\by}$ and $\bnu^\top \bA^{-1}{\by}$ correspondingly, which yields the same bound on the scalar product as in the noiseless case.

When it comes to the denominator, only one extra significant term appears compared to the noiseless case. All others are dominated. To see why, let us write
\[
S\|\mni\|_\bSigma \leq S\|\bQ^\top\bA^{-1}\tilde{\by}\|_\bSigma + (\xi + |\bnu^\top\bA^{-1}\tilde{\by}|) S\|\mnic\|_\bSigma.
\]
We already have the upper bound on $S\|\mnic\|_\bSigma$. The other bounds come from the following inequalities:
\begin{gather*}
S = (1 + \bnu^\top\bA^{-1}\by)^2 + \by^\top\bA^{-1}\by\|\bmu_\perp\|^2
\lesssim (1 + t\Diamond)^2 + N, \\
\|\bQ^\top\bA^{-1}\tilde{\by}\|_\bSigma \lesssim \sigma_\eta\sqrt{V + t\Delta V},\quad
\xi \approx 1,\quad \text{and} \quad
|\bnu^\top\bA^{-1}\tilde{\by}| \lesssim \sigma_\eta t\Diamond.
\end{gather*}
Combining these estimates yields
\[
S\|\mni\|_\bSigma \lesssim \left((1 + t\Diamond)^2 + N\right)\sigma_\eta\sqrt{V + t\Delta V} + (1 + \sigma_\eta t\Diamond)\left(\sqrt{V  + t\Delta V} + \sqrt{n}\Diamond\right).
\]
By Lemma \ref{lm::relations} and since $t < \sqrt n$, we deduce that $N$ dominates $t^2\Diamond^2$, and $\sigma_\eta N\sqrt{V  + t\Delta V}$ dominates $\sigma_\eta t\Diamond\cdot \sqrt{n}\Diamond$. Dropping the dominated terms leaves us with
\[
S\|\mni\|_\bSigma \lesssim \left(3t\Diamond + N\right)\sigma_\eta\sqrt{V + t\Delta V} + \sqrt{V  + t\Delta V} + \sqrt{n}\Diamond.
\]
Finally, the term $3t\Diamond\sigma_\eta\sqrt{V + t\Delta V}$ is dominated by $\sqrt{n}\Diamond$, which gives us the final bound:
\begin{equation}\label{eq:norm.wmnic.noisy}
S\|\mni\|_\bSigma \lesssim \left(1 + N\sigma_\eta\right)\sqrt{V  + t\Delta V} + \sqrt{n}\Diamond.
\end{equation}
Recall, however, that this derivation is only informal, as our proof does not give rigorous bounds on quantities involving $\tilde{\by}$ and deal with $\Delta \by = \hat{\by} - \by$ instead.

\subsection{Upper bound}

Obtaining an upper bound on $\bmu^\top\solnridge/\|\solnridge\|_\bSigma$ turns out to be more technically challenging. We elaborate on that in Section \ref{sec::upper bound proof sketch}, where we explain and justify which additional assumptions are needed to derive an upper bound. Due to these technical difficulties, we only provide the upper bound for the regime without label-flipping noise in the following theorem.

\begin{restatable}[Main upper bound]{theorem}{mainupper}
\label{th::main upper bound}
Suppose that $\eta = 0$---there is no label-flipping noise---and the rows of $\bZ$ are $\sigma_x$-sub-Gaussian. For any $L > 1$, there are large constants $a, c$ that only depend on $\sigma_x$ and $L$ and an absolute constant $c_y$ such that the following holds. Suppose that $k < n/c$ and
\begin{equation*}
\Lambda > c\left(n\lambda_{k+1} + \sqrt{n\sum_{i > k}\lambda_i^2}\right).
\end{equation*}
Assume that  $\bQ_\ktoinf$ is independent from $\bQ_\uptok$ and the distribution of $\bQ_\ktoinf$ is symmetric.
\begin{enumerate}
\item If $N < a^{-1}\Diamond$, then $\bmu^\top\solnridge < 0$ with probability at least $c_y^{-1}(\P(\sA_k(L)) - ce^{-n/c})_+$, where $u_+$ denotes $u \vee 0$ for any $u \in \R$.

\item If $N \geq a^{-1}\Diamond$, then for any $t \in (0, \sqrt{n}/c_y)$, the probability of the event
\[
\left\{\frac{\bmu^\top\solnridge}{\|\solnridge\|_\bSigma} \leq c(1 + t)\frac{N}{\sqrt{V + n\Diamond^2}}\right\}
\]
is a least
\[
(c_y^{-1} - c_ye^{-t^2/c_y} - c_ye^{-n/c})_+(\P(\sA_k(L)) - ce^{-n/c})_+.
\]
\end{enumerate}
\end{restatable}

\subsection{Tight bound for a quantile}

Theorems \ref{th::main} and \ref{th::main upper bound} give lower and upper bounds on the quantity $\bmu^\top\solnridge/\|\solnridge\|_\bSigma$ correspondingly. That quantity, however, is random, and the bounds depend on a parameter $t$ that controls the probability with which the bound holds. We do not expect those bounds to be sharp for all possible values of $t$, but we show that they are sharp when $t$ is a constant.

\begin{definition}
\label{def::eps quantile}
For any $\eps \in (0, 1)$, we denote by $\alpha_\eps$ the $\eps$-quantile of the distribution of~$\bmu^\top\solnridge/\|\solnridge\|_\bSigma$, i.e.,
\begin{equation}\label{eq::def of eps quantile}
\alpha_\eps := \inf\left\{\alpha \in \R: \P\left(\frac{\bmu^\top\solnridge}{\|\solnridge\|_\bSigma} < \alpha\right) > \eps\right\}.
\end{equation}
\end{definition}

The following theorem shows that our upper and lower bounds on $\alpha_\eps$  are within a constant factor of each other when $\eps$ is set to a certain absolute constant.
\begin{restatable}[Tightness of the bounds]{theorem}{quantiletightness}
\label{th::constant quantile tight bounds}
Suppose that the distribution of the rows of $\bZ$ is $\sigma_x$-sub-Gaussian. Suppose that $\eta = 0$---there is no label-flipping noise. For any $L > 1$, there exist constants $a, c$ that only depend on $L, \sigma_x$ and absolute constants $\eps, \delta$ such that the following holds. Suppose that $n > c$,  $k < n/c$,
\[
\Lambda > c\left(n\lambda_{k+1}\vee \sqrt{n\sum_{i > k}\lambda_i^2}\right),
\]
and the probability  of the event $\sA_k(L)$ is at least $1 - \delta$. Assume that $\bQ_\ktoinf$ is independent from $\bQ_\uptok$ and the distribution of $\bQ_\ktoinf$ is symmetric. Then,
\[
\alpha_\eps \leq c\frac{N}{\sqrt{V} + \sqrt{n}\Diamond}.
\]
If additionally $N \geq a\Diamond$, then
\[
\alpha_\eps \geq c^{-1}\frac{N}{\sqrt{V} + \sqrt{n}\Diamond}.
\]
\end{restatable}

\subsection{Regimes of the lower bound}
\label{sec::regimes of the lower bound}

In this section we discuss the forms that the bound from Theorem~\ref{th::main} can take depending on which terms dominate in the expressions. To do so, we let $t$ be a large constant and $\eta$ be a small constant. That is, the bound holds with a high, constant probability and the probability of label-flipping is a small constant. We also fix the covariance~$\bSigma$, the number of data points $n$, and the direction of $\bmu$, treating its magnitude as a parameter. Under this setting, we introduce the following notation: for two quantities $a, b$, we denote by~$a\uparrow b$ the smallest magnitude of $\bmu$ for which $a \gtrsim b$~(i.e., w.h.p. $a$ is at least a constant times $b$; see Section~\ref{sec::additional_notation}).

Let us start with the bound on $S\|\solnridge\|_\bSigma$ in \eqref{eq::main bound denominator}. It has three terms: $\sqrt{V}$, which does not scale with~$\bmu$; $\Diamond\sqrt{n}$, which is linear in~$\bmu$; and $N\sqrt{V}$, which is quadratic. Thus, the term $\sqrt{V}$ will dominate the bound when $\bmu$ is small and $N\sqrt{V}$ will dominate when it is large.  Regarding the term $\Diamond\sqrt{n}$, it may or may not dominate for some range of $\|\bmu\|$. Both cases are illustrated in Figure~\ref{fig:cases constant linear quadratic}. Note that in the noiseless case~($\eta = 0$), the term $N\sqrt{V}$ vanishes, in which case it is $\Diamond\sqrt{n}$ that dominates the bound for large $\bmu$.

\begin{figure}[htbp]
    \centering
    \begin{subfigure}[b]{0.49\textwidth}
        \centering

\begin{tikzpicture}
\scriptsize

\def\const{1.9}
\def\lin{2.5}
\def\quad{1}

\begin{axis}[
    width=\columnwidth,
    height=0.75\columnwidth,
    axis lines=middle,
    xlabel={$\|\bmu\|$},
    ylabel={},
    domain=0:2.95,
    enlarge x limits=0.05,
    clip=false,
    xtick=\empty,
    ytick=\empty,
    grid=none,
    axis line style={thick},
    legend style={at={(0.07,1)}, anchor=north west, legend columns=1,
    	fill opacity=0.5, text opacity=1},
	every axis x label/.style={at=(ticklabel cs:1.08), yshift=2pt}
]

\addplot[samples=2, thick, black!25, dashed, forget plot]
    coordinates {({\const/\lin},0) ({\const/\lin},9)};
\node[anchor=south, xshift=-1ex] at (axis cs:{\const/\lin},-1.25)
	{$\sqrt{n}\Diamond \uparrow \sqrt{V}$};

\addplot[samples=2, thick, black!25, dashed, forget plot]
	coordinates {({sqrt(\const/\quad)},0) ({sqrt(\const/\quad)},9)};
\node[anchor=south, xshift=0.5ex] at (axis cs:{sqrt(\const/\quad)},-1.25)
	{$N \uparrow 1$};

\addplot[samples=2, thick, black!25, dashed, forget plot]
	coordinates {({\lin/\quad},0) ({\lin/\quad},9)};
\node[anchor=south] at (axis cs:{\lin/\quad},-1.25)
	{$N\sqrt{V} \uparrow \sqrt{n}\Diamond$};

\addplot[
    samples=2,
    thick,
    Set1.blue
] {\const};
\addlegendentry{$\sqrt{V}$}

\addplot[
    samples=2,
    thick,
    Set1.green
] {\lin*x};
\addlegendentry{$\sqrt{n}\Diamond$}

\addplot[
    samples=100,
    thick,
    Set1.red
] {\quad*x^2};
\addlegendentry{$N\sqrt{V}$}

\end{axis}
\end{tikzpicture}
        \caption{Case 1: $N\uparrow 1 > \sqrt{n}\Diamond \uparrow \sqrt{V}$}
        \label{fig:cases constant linear quadratic case 1}
    \end{subfigure}
    \hfill
    \begin{subfigure}[b]{0.49\textwidth}
        \centering

\begin{tikzpicture}
\scriptsize

\def\const{3.1}
\def\lin{1.2}
\def\quad{1}

\begin{axis}[
    width=\columnwidth,
    height=0.75\columnwidth,
    axis lines=middle,
    xlabel={$\|\bmu\|$},
    ylabel={},
    domain=0:2.95,
    enlarge x limits=0.05,
    clip=false,
    xtick=\empty,
    ytick=\empty,
    grid=none,
    axis line style={thick},
    legend style={at={(0.07,1)}, anchor=north west, legend columns=1,
    	fill opacity=0.5, text opacity=1},
	every axis x label/.style={at=(ticklabel cs:1.08), yshift=2pt}
]

\addplot[samples=2, thick, black!25, dashed, forget plot]
    coordinates {({\const/\lin},0) ({\const/\lin},9)};
\node[anchor=south] at (axis cs:{\const/\lin},-1.25)
	{$\sqrt{n}\Diamond \uparrow \sqrt{V}$};

\addplot[samples=2, thick, black!25, dashed, forget plot]
	coordinates {({sqrt(\const/\quad)},0) ({sqrt(\const/\quad)},9)};
\node[anchor=south, xshift=0.5ex] at (axis cs:{sqrt(\const/\quad)},-1.25)
	{$N \uparrow 1$};

\addplot[samples=2, thick, black!25, dashed, forget plot]
	coordinates {({\lin/\quad},0) ({\lin/\quad},9)};
\node[anchor=south, xshift=-3ex] at (axis cs:{\lin/\quad},-1.25)
	{$N\sqrt{V} \uparrow \Diamond\sqrt{n}$};

\addplot[
    samples=2,
    thick,
    Set1.blue
] {\const};
\addlegendentry{$\sqrt{V}$}

\addplot[
    samples=2,
    thick,
    Set1.green
] {\lin*x};
\addlegendentry{$\Diamond\sqrt{n}$}

\addplot[
    samples=100,
    thick,
    Set1.red
] {\quad*x^2};
\addlegendentry{$N\sqrt{V}$}

\end{axis}
\end{tikzpicture}
        \caption{Case 2: $N\uparrow 1 < \sqrt{n}\Diamond \uparrow \sqrt{V}$}
        \label{fig:cases constant linear quadratic case 2}
    \end{subfigure}

    \caption{Two possible relations between $\sqrt{V}$, $\sqrt{n}\Diamond$, and $N\sqrt{V}$ depending on $\|\bmu\|$.}
    \label{fig:cases constant linear quadratic}
\end{figure}
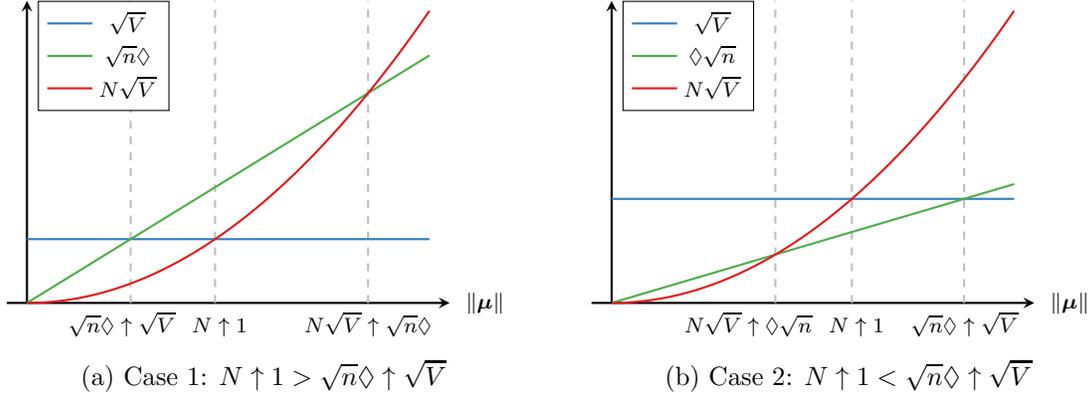

Whether both cases in Figure~\ref{fig:cases constant linear quadratic} are possible for a constant $\eta$ depends on the relation between $\bSigma$ and the direction of $\bmu$. Indeed, consider two extreme cases. First, suppose that $\bmu = m\be_1$---a vector supported on the first coordinate. We then have
\[
\Diamond\sqrt{n} = \frac{\sqrt{\lambda_1} m}{\lambda_1 + \Lambda/n} \quad \text{and} \quad N = \frac{m^2}{\lambda_1 + \Lambda/n}.
\]
If $\lambda_1 \gg \Lambda/n$ and $V \ll 1$, then $m = \sqrt{\lambda_1}$ implies $\Diamond\sqrt{n} \gg (1 + N)\sqrt{V}$. Thus, $\Diamond\sqrt{n}$ dominates and we are in the setting of Figure~\ref{fig:cases constant linear quadratic case 1}.
On the other hand, consider $\bmu = m\be_p$---a vector supported on the last coordinate. Then,
\[
\Diamond \sqrt{n} = n\Lambda^{-1}\sqrt{\lambda_p}m \quad \text{and} \quad N = n\Lambda^{-1}m^2.
\]
If additionally $V \gg n\lambda_p/\Lambda$, which is possible since $\lambda_p$ can be arbitrarily small, we can write
\[
(1 + N)\sqrt{V} = (1 + n\Lambda^{-1}m^2)\sqrt{V} \gg (1 + n\Lambda^{-1}m^2)\sqrt{n\Lambda^{-1}\lambda_p}
\geq 2\sqrt{n\Lambda^{-1}m^2}\sqrt{n\Lambda^{-1}\lambda_p}
= 2\Diamond \sqrt{n}.
\]
That is, for such choice of covariance and $\bmu$, the term $\Diamond \sqrt{n}$ does not dominate the bound for $\|\solnridge\|_\bSigma$ in \eqref{eq::main bound denominator} for any choice of $m$, and thus, the case from Figure~\ref{fig:cases constant linear quadratic case 2} must happen.

Now, consider the bound on $S\bmu^\top\solnridge$ in \eqref{eq::main bound numerator}. It has two terms (up to constant multipliers): $N$, which is quadratic in~$\bmu$, and $-\Diamond$, which is linear. Thus, for small $\bmu$~(i.e., $\|\bmu\|\leq N\uparrow \Diamond$) the second term will dominate and the bound will be negative. As $\|\bmu\|$ becomes larger than $N\uparrow \Diamond$, however, $N$ will dominate~(Figure~\ref{fig:formulas scheme}, second line).

\begin{figure}[tbp]
    \centering

\begin{tikzpicture}[node distance=6mm]
\pgfdeclarelayer{bg}
\pgfsetlayers{bg,main}

\def\boxwidth{9em}
\def\arrowwidth{\columnwidth-\boxwidth-3.2em}
\def\one{4.5em}
\def\two{9.8em}
\def\three{14.4em}
\def\four{19.4em}
\def\five{24.1em}
\def\final{\arrowwidth}
\def\ticksize{2mm}
\def\linesep{3.2mm}

\newcommand{\tick}[2]{
	\draw ($(#1.east) + (#2, -0.5*\ticksize)$) -- ($(#1.east) + (#2, 0.5*\ticksize)$);
}
\newcommand{\guideone}[2]{
	\draw[dashed, opacity=0.3]
		($(norm noiseless one.east) + (#1, 1em)$) -- ($(mu one) + (#1, -1ex)$)
		node[anchor=south, yshift=-4ex, text opacity=1, font=\small] {#2};
}
\newcommand{\guidetwo}[2]{
	\draw[dashed, opacity=0.3]
		($(norm noiseless two.east) + (#1, 1em)$) -- ($(mu two) + (#1, -1ex)$)
		node[anchor=south, yshift=-4ex, text opacity=1, font=\small] {#2};
}
\newcommand{\guidefull}[2]{
	\draw[dashed, opacity=0.3]
		($(S.east) + (#1, 1em)$) -- ($(mu two) + (#1, -1ex)$)
		node[anchor=south, yshift=-4ex, text opacity=1, font=\small] {#2};
}
\newcommand{\regime}[4]{
	\draw[latex-latex] ($(#1.east) + (#2,0)$) --
		node[midway, inner sep=5pt, fill=white] {#4}
		++(#3-#2,0);
}

\node[align=right, text width=\boxwidth] (S) {$S \approx$};
\regime{S}{0}{\three}{$1$}
\regime{S}{\three}{\final}{$N$}

\node[align=right, text width=\boxwidth, below=\linesep of S] (sign)
	{$\sign\big( \mni^\top \mnic \big) =$};
\regime{sign}{0}{\five}{$+$}
\regime{sign}{\five}{\final}{$\pm$}

\node[align=right, text width=\boxwidth, below=\linesep of sign] (inner)
	{$S\bmu^\top \bw^{(c)}_{\text{\tiny{MNI}}} \approx$};
\regime{inner}{0}{\one}{$\pm\Diamond$}
\regime{inner}{\one}{\final}{$N$}

\node[align=right, below=3*\linesep of inner, xshift=7mm] (case one) {(a) Case 1: $N\uparrow 1 > \sqrt{n}\Diamond \uparrow \sqrt{V}$:};

	\node[align=right, text width=\boxwidth, below=\linesep of case one, xshift=-7mm] (norm noiseless one)
		{$S\|\mnic\|_\bSigma  \approx$};
	\regime{norm noiseless one}{0}{\two}{\color{Set1.blue}$\sqrt{V}$}
	\regime{norm noiseless one}{\two}{\final}{$\color{Set1.green}\sqrt{n}\Diamond$}

	\node[align=right, text width=\boxwidth, below=\linesep of norm noiseless one] (norm noisy one)
		{$S\|\mni\|_\bSigma \approx$};
	\regime{norm noisy one}{0}{\two}{$\color{Set1.blue}\sqrt{V}$}
	\regime{norm noisy one}{\two}{\four}{$\color{Set1.green}\sqrt{n}\Diamond$}
	\regime{norm noisy one}{\four}{\final}{\color{Set1.red}$N\sqrt{V}$}

	\node[align=right, text width=\boxwidth, below=\linesep of norm noisy one] (bound noiseless one)
		{$\dfrac{\bmu^\top \mnic}{\|\mnic\|_\bSigma} \gtrsim$};
	\regime{bound noiseless one}{0}{\one}{$-\infty$}
	\regime{bound noiseless one}{\one}{\two}{$\dfrac{N}{\color{Set1.blue}\sqrt{V}}$}
	\regime{bound noiseless one}{\two}{\final}{$\dfrac{N}{\color{Set1.green}\sqrt{n}\Diamond}$}

	\node[align=right, text width=\boxwidth, below=\linesep of bound noiseless one] (bound noisy one)
		{$\dfrac{\bmu^\top \mni}{\|\mni\|_\bSigma} \gtrsim$};
	\regime{bound noisy one}{0}{\one}{$-\infty$}
	\regime{bound noisy one}{\one}{\two}{$\dfrac{N}{\color{Set1.blue}\sqrt{V}}$}
	\regime{bound noisy one}{\two}{\four}{$\dfrac{N}{\color{Set1.green}\sqrt{n}\Diamond}$}
	\regime{bound noisy one}{\four}{\final}{$\dfrac{1}{\color{Set1.red}\sqrt{V}}$}

\draw[-latex, shorten <= -7em] ($(bound noisy one.east) - (0,4*\linesep)$) node (mu one) {}
	-- ++(\arrowwidth,0)
	node[font=\small, below, xshift=-2pt, yshift=-2pt] {$\| \bmu \|$};

\node[align=right, below=6.5*\linesep of bound noisy one, xshift=7mm] (case two) {(b) Case 2: $N\uparrow 1 < \sqrt{n}\Diamond \uparrow \sqrt{V}$:};

	\node[align=right, text width=\boxwidth, below=\linesep of case two, xshift=-7mm] (norm noiseless two)
		{$S\|\mnic\|_\bSigma  \approx$};
	\regime{norm noiseless two}{0}{\four}{\color{Set1.blue}$\sqrt{V}$}
	\regime{norm noiseless two}{\four}{\final}{\color{Set1.green}$\sqrt{n}\Diamond$}

	\node[align=right, text width=\boxwidth, below=\linesep of norm noiseless two] (norm noisy two)
		{$S\|\mni\|_\bSigma \approx$};
	\regime{norm noisy two}{0}{\three}{\color{Set1.blue}$\sqrt{V}$}
	\regime{norm noisy two}{\three}{\final}{\color{Set1.red}$N\sqrt{V}$}

	\node[align=right, text width=\boxwidth, below=\linesep of norm noisy two] (bound noiseless two)
		{$\dfrac{\bmu^\top \mnic}{\|\mnic\|_\bSigma} \gtrsim$};
	\regime{bound noiseless two}{0}{\one}{$-\infty$}
	\regime{bound noiseless two}{\one}{\four}{$\dfrac{N}{\color{Set1.blue}\sqrt{V}}$}
	\regime{bound noiseless two}{\four}{\final}{$\dfrac{N}{\color{Set1.green}\sqrt{n}\Diamond}$}

	\node[align=right, text width=\boxwidth, below=\linesep of bound noiseless two] (bound noisy two)
		{$\dfrac{\bmu^\top \mni}{\|\mni\|_\bSigma} \gtrsim$};
	\regime{bound noisy two}{0}{\one}{$-\infty$}
	\regime{bound noisy two}{\one}{\three}{$\dfrac{N}{\color{Set1.blue}\sqrt{V}}$}
	\regime{bound noisy two}{\three}{\final}{\color{Set1.red}$\dfrac{1}{\sqrt{V}}$}

\draw[-latex, shorten <= -7em] ($(bound noisy two.east) - (0,4*\linesep)$) node (mu two) {}
	-- ++(\arrowwidth,0)
	node[font=\small, below, xshift=-2pt, yshift=-2pt] {$\| \bmu \|$};

\begin{pgfonlayer}{bg}
	\guidefull{\one}{$N \uparrow \Diamond$}{mu two}

	\guidefull{\three}{$N \uparrow 1$}{mu two}

	\guidefull{\five}{$\Diamond \uparrow 1$}{mu two}

	\guideone{\two}{${\color{Set1.green} \sqrt{n}\Diamond} \uparrow {\color{Set1.blue} \sqrt{V}}$}
	\guideone{\four}{${\color{Set1.red} N\sqrt{V}} \uparrow {\color{Set1.green} \sqrt{n}\Diamond}$}

	\guidetwo{\four}{${\color{Set1.green} \sqrt{n}\Diamond} \uparrow {\color{Set1.blue} \sqrt{V}}$}
\end{pgfonlayer}

\end{tikzpicture}
    \caption{Main bounds on the quantities of interest~(up to a constant factor) as a function of~$\|\bmu\|$. We write~$\bw_{\text{MNI}}^{(c)}$ to denote both~$\mni$ and~$\mnic$ and $\bW = (\bSigma + n^{-1}\Lambda\bI_p)^{-1}$.}
		\label{fig:formulas scheme}
\end{figure}
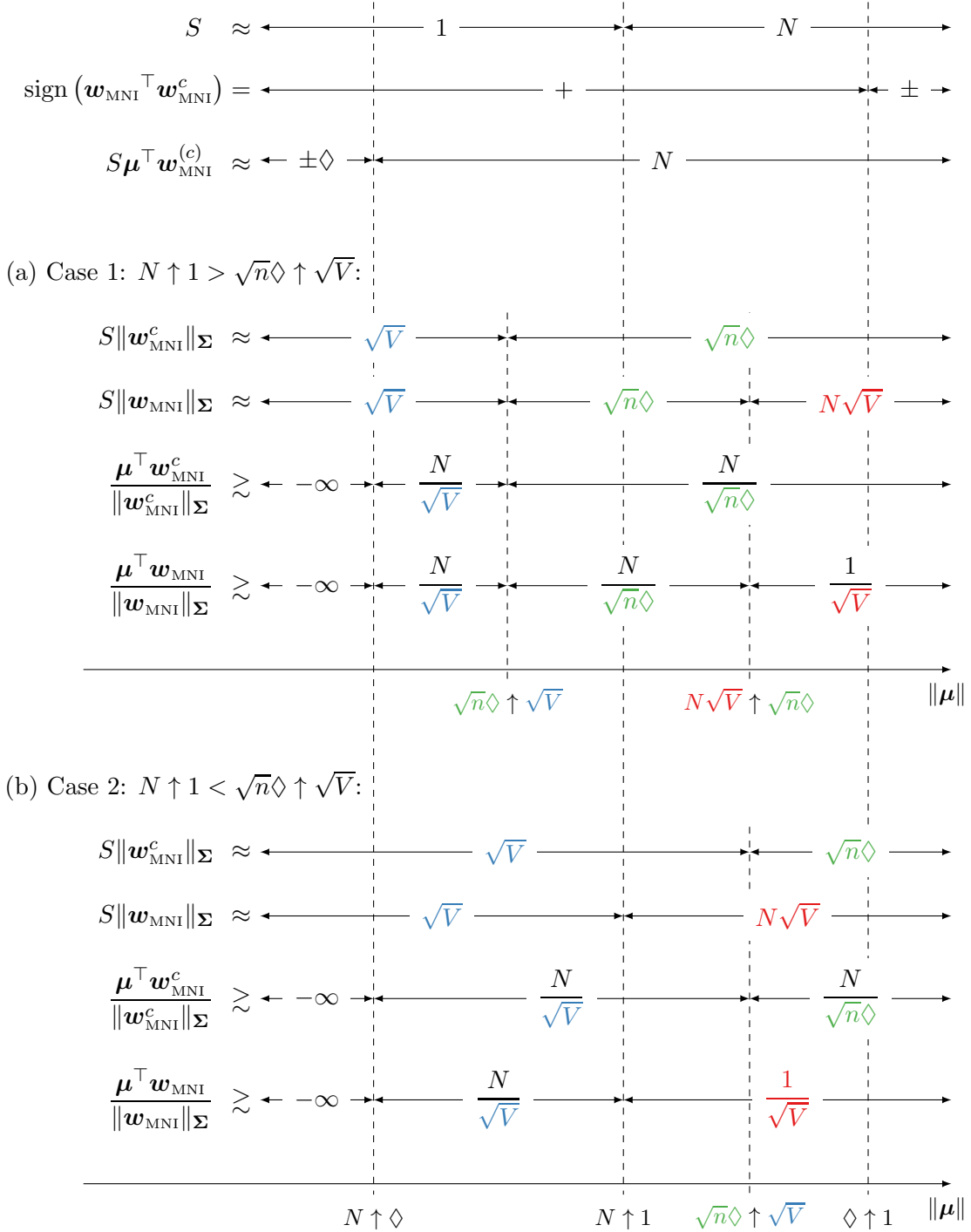

Let us now investigate how the transition $N\uparrow \Diamond$ relates to the transitions from Figure~\ref{fig:cases constant linear quadratic}. According to Lemma~\ref{lm::relations},  $\sqrt{n}\Diamond^2 \leq N\sqrt{\Delta V}\leq 2N\sqrt{ V}$, which implies that
\[
\frac{\sqrt{V}}{\Diamond \sqrt{n}} \geq \frac{\Diamond}{2N}.
\]
That is, if $\Diamond$ is at least a constant times $N$, then $\sqrt{V}$ is at least a constant times $\Diamond \sqrt{n}$, so $N\uparrow \Diamond \leq \sqrt{n}\Diamond \uparrow V$. Moreover, since $V\leq 2$, if $N\sqrt{V}\geq \sqrt{n}\Diamond$, then $N \geq \Diamond$, so $N\uparrow \Diamond \leq N\sqrt{V}\uparrow \sqrt{n}\Diamond.$ So, the transition in the bound on  $S\bmu^\top\solnridge$ happens before all the transitions in the bound on $S\|\solnridge\|_\bSigma$.

So far we have discussed the transitions in $\bmu^\top\mni$ and $\|\mni\|_\bSigma$. Now, let us discuss how $\mni$ looks in Euclidean geometry. Recall that $\bQ^\top\bA^{-1}\tilde{\by} \perp \mnic$, that is, $\mni$ has two orthogonal components: one in the direction of $\bQ^\top\bA^{-1}\tilde{\by}$, and one in the direction of $\mnic$. The latter is proportional to $\mnic$ with coefficient $\xi - \bnu^\top\bA^{-1}\tilde{\by}$.  As discussed in Section~\ref{sec::mni geometry with noise}, $\bnu^\top\bA^{-1}\tilde{\by}$ is a zero-mean random variable whose magnitude grows linearly with $\bmu$. Hence, as soon as it dominates $\xi$, the contribution of the clean solution to $\mni$ becomes ``washed out," i.e., the $\mnic$ component becomes close to zero mean instead of being positive.
Our arguments don't provide rigorous bounds for the quantities involving $\tilde{\by}$, but we can speculate about the relation between this transition and the transitions of the main bound by plugging in $|\bnu^\top\bA^{-1}\tilde{\by}| \approx |\bnu^\top\bA^{-1}\by|$ for constant $\eta$. Since the magnitude of the latter is controlled by $\Diamond$, the transition happens when $\Diamond$ becomes larger than $\xi \approx 1$. Using Lemma \ref{lm::relations}, we obtain
\[
\Diamond \gtrsim 1 \Rightarrow \sqrt{n}\Diamond \gg 1 \gtrsim\sqrt{V}
	\quad \text{and} \quad N\sqrt{V} \geq N\sqrt{\Delta V}/2 \geq \sqrt{n}\Diamond^2 \gtrsim \sqrt{n}\Diamond.
\]
Thus, the transition $\Diamond \uparrow 1$ happens after all the transitions from Figure~\ref{fig:cases constant linear quadratic}. We put together all the relevant quantities and the dominating (up to a constant factor) terms in their bound in Figure~\ref{fig:formulas scheme}.

Let us conclude by tying these regimes to the geometrical discussions from Section \ref{sec::geometry discussion}. For the case without label-flipping noise, we obtain from Proposition~\ref{thm:w_tilde} and~\eqref{eq::MNI formula new scaling} that
\[
\mnic \propto \mnitildec = \frac{\bQ^\top\bA^{-1}\by}{\by^\top\bA^{-1}\by} + \bmu_\perp
	+ \frac{\bnu^\top\bA^{-1}\by}{\by^\top\bA^{-1}\by}\bQ^\top\bA^{-1}\by.
\]
The term $\bmu_\perp$ dominates over the last term and effectively acts as $(n\Lambda^{-1}\bSigma + \bI_p)^{-1}\bmu$. Hence, as stated in Section \ref{sec::recovering the geometry introduction}, the MNI effectively performs as
\begin{equation*}
     \bQ^\top \bA^{-1}\by + (\bSigma + n^{-1}\Lambda\bI_p)^{-1}\bmu \approx n\Lambda^{-1}\mnitildec.
\end{equation*}
As a result, when $\bmu$ is small enough so that $\Diamond$ dominates $N$, the ``noise term" $\bQ^\top \bA^{-1}\by$ takes over both the scalar product with $\bmu$ and the norm in $\bSigma$ resulting in a negative bound~(Figure~\ref{fig:formulas scheme}). As the magnitude of $\bmu$ grows, the term $(\bSigma + n^{-1}\Lambda\bI_p)^{-1}\bmu$ starts dominating the scalar product with $\bmu$~($N$ becomes larger than $\Diamond$). Yet, the term $\bQ^\top \bA^{-1}\by$ continues to dominate the norm in $\bSigma$~($V$ dominates over $\sqrt{n}\Diamond$), resulting in the bound $N/\sqrt{V}$ (up to a constant factor). Finally, as $\bmu$ becomes even larger, $\mnic$ performs effectively as $(\bSigma + n^{-1}\Lambda\bI_p)^{-1}\bmu$ and the bound becomes $N/(\sqrt{n}\Diamond)$ (up to a constant factor).

The noisy case is harder to explain without an intuitive explanation for why the formula for $\bmu^\top\mni$ is so similar to that for $\bmu^\top\mnic$ [see~\eqref{eq:inner_product_noisy}]. Nevertheless, this similarity implies that the conditions under which the bound becomes positive are the same: $\bmu$ must be large enough so that $N$ dominates $\Diamond$. Now, let us geometrically explain the regime transitions of $\|\mni\|_\bSigma$. First of all, if $\bmu = \bzero_p$, then $\mni = \bQ^\top\bA^{-1}\hat{\by}$ and $\mnic = \bQ^\top\bA^{-1}{\by}$. Since both $\by$ and $\hat{\by}$ are independent from $\bQ$ and have the same distribution, $\|\mni\|_\bSigma \approx \|\mnic\|_\bSigma \approx \sqrt{V}$.  That is, for small $\bmu$, $\mni$ and $\mnic$ display the same behavior in terms of both their scalar product with $\bmu$ and their norm in $\bSigma$. Because of that, both noisy and noiseless bound go through the same initial regime $N/\sqrt{V}$.

Next, recall from \eqref{eq::mni through rescaling of mnic} that the MNI solution under label-flipping noise can be written as
\[
\mni =  \bQ^\top\bA^{-1}\tilde{\by} +(\xi - \bnu^\top\bA^{-1}\tilde{\by}) \mnic,
\]
where $\tilde{\by} = \hat{\by} - \xi\by$ can be interpreted as the ``label noise'' and~$\xi \approx 1 - 2\eta$. Using this decomposition, as we stated in Section~\ref{sec::mni geometry with noise}, our bounds suggest that
\[
\|\mni\|_\bSigma \approx \|\bQ^\top\bA^{-1}\tilde{\by}\|_\bSigma + \|\mnic\|_\bSigma.
\]
Unfortunately, as we mentioned before, our arguments do not provide rigorous bounds on quantities involving $\tilde{\by}$. However, we can speculate about their magnitude by assuming that for constant noise level they behave as if $\tilde{\by}$ was $\by$. Under this assumption, $\|\bQ^\top\bA^{-1}\tilde{\by}\|_\bSigma \approx \sqrt{V}$---the same magnitude that~$\|\mnic\|_\bSigma$ exhibits for small $\bmu$.

As $\bmu$ grows, however, $\|\mnic\|_\bSigma$ changes. Recall from Section~\ref{sec::mni geometry no noise} that $\mnitildec$ changes linearly with $\bmu$ and that since~$\mnic = \mnitildec/\|\mnitildec\|^2$, we have~$\|\mnic\| = 1/\|\mnitildec\|$. Because of that, in Euclidean geometry, the magnitude of~$\mnic$ is almost always decreasing as $\bmu$ grows. Nevertheless, when analyzing $\|\mnic\|_\bSigma$, this decrease in magnitude is not the only factor to consider due to the contribution of the direction of $\mnic$. As $\bmu$ grows, the direction of $\mnic$ approximately aligns with that of $\bmu_\perp$. If, in its turn, $\bmu_\perp$ is approximately aligned with the first few eigendirections of $\bSigma$, then the change in direction may dominate the change in magnitude and $\|\mnic\|_\bSigma$ may grow to dominate $\|\bQ^\top\bA^{-1}\tilde{\by}\|_\bSigma$. Because of that, $\|\mni\|_\bSigma$ may stay approximately the same as $\|\mnic\|_\bSigma$ and the noisy bound may go over the second regime of the noiseless bound, namely, $N/(\sqrt{n}\Diamond)$, depending on the relationship between~$\bSigma$ and~$\bmu$.

Regardless of whether we observe this regime, $\mnic$ converges to zero as $\bmu$ goes to infinity, so that $\|\bQ^\top\bA^{-1}\tilde{\by}\|_\bSigma$ eventually dominates $\|\mnic\|_\bSigma$. Since $\bmu^\top\mni \approx \bmu^\top\mnic$, we obtain
\[
\frac{\bmu^\top\mni}{\|\mni\|_\bSigma} \approx \frac{\bmu^\top\mnic}{\|\bQ^\top\bA^{-1}\tilde{\by}\|_\bSigma}.
\]
Observe that $\bmu^\top\mnic$ is the label that the MNI solution assigns to the center of the positive cluster. Since $\bmu$ is large, $\mnic$ has high classification accuracy and $\bmu^\top\mnic \approx 1$. Additionally, we have $\|\bQ^\top\bA^{-1}\tilde{\by}\|_\bSigma \approx \sqrt{V}$ for constant $\eta$. Thus, for large enough $\bmu$, the bound becomes $1/\sqrt{V}$. Note that this quantity no longer depends on $\bmu$, so that the bound can only increase as $V$ decreases. In Proposition \ref{prop::V and B upper bound as in BO_ridge}, as well as Section \ref{sec::comparison with regression}, we have shown that $V$ is (up to a multiplicative constant) equivalent to the bound on the variance term for regression obtained in \cite{benign_overfitting}. Hence, the sufficient condition for benign overfitting that we get in the large $\bmu$ regime is equivalent to the benign overfitting condition in linear regression.

\subsection{Benign overfitting}

Expanding on the informal observations from the end of last section, one interesting phenomenon that Theorem \ref{th::main} implies is that the misclassification error can be arbitrarily close to zero even if we have a small constant level of label-flipping noise. One can see that it happens for $t = o(\sqrt{n})$, $V = o(1)$,  and $N = \omega(\sqrt{V} +  t/\sqrt{n} + \Diamond \sqrt{n})$ (which has the form of an SNR condition). In order to simplify the presentation, we formulate a rigorous corollary for the Gaussian distribution.

\begin{corollary}
\label{cor::BO Gaussian data}
Suppose the rows of matrix $\bQ$ come from i.i.d.\ samples from a Gaussian distribution. Take $\lambda = 0$ (that is, consider the MNI solution). There exists a large absolute constant $c$ such that, for any $C > 1$, if
\begin{align}
&\text{- the noise is bounded by a constant:}&& \eta < c^{-1},
\\
&\text{- the ``spiked part" of the covariance has low dimension:}&& k < n/(cC^2),
\label{eq::BO cor small k assumption}
\\
&\text{- the ``tail" of the covariance has high effective rank:}&& \sum_{i > k}\lambda_i > cn\lambda_{k + 1} \vee cC\sqrt{n\sum_{i > k}\lambda_i^2},
\label{eq::BO cor large effective rank}
\\
&\text{- the scale of $\bmu$ is large enough:}&&N \geq 1 + cC\left(\sqrt{V} + \Diamond \sqrt{n}\right),
\label{eq::BO cor large bmu}
\end{align}
then, with probability at least $1 - ce^{-n/(cC)^2}$,
\begin{equation}\label{eq::BO cor result}
\frac{\bmu^\top\mni}{\|\mni\|_{\bSigma}} \geq C.
\end{equation}
\end{corollary}

\begin{proof}
First of all, due to \eqref{eq::BO cor large effective rank}, if $c$ is large enough, by Lemma \ref{lm:: eigenvalues of A_k indep coord}, for absolute constants $c_A, L$ the probability of the event $\sA_k(L)$ is at least $1 - c_Ae^{-n/c_A}$. Moreover, due to \eqref{eq::BO cor small k assumption} and Lemma \ref{lm::sB under sub-Gaussianity}, for an absolute constant $c_B$ the event $\sB_k(c_B)$ holds with probability at least $1 - c_Be^{-n/c_B}$.

Next, by Theorem \ref{th::main}, for an absolute constant $c_1$ on the event $\sA_k(L)\cap \sB_k(c_B)$ we have for a certain scalar $S > 0$ for any $t \in (0, \sqrt{n}/c_1)$, with probability at least $1 - c_1e^{-t^2/2}$,
\begin{align}
S\bmu^\top\mni \geq& c_1^{-1}N - c_1t\Diamond,\label{eq::BO cor proof numerator}\\
S\|\mni\|_\bSigma \leq& c_1\left(\left[1 + N\sigma_\eta \right]\sqrt{ V + t^2\Delta V} + \Diamond\sqrt{n}\right).
\end{align}
Since $t < \sqrt{n}/c_1$, if $c$ is large enough, due to \eqref{eq::BO cor large bmu} and \eqref{eq::BO cor proof numerator}, there exists an absolute constant $c_2$ such that $S\bmu^\top\mni \geq c_2^{-1}N$.

So far, for an absolute constant $c_3$, we have with probability at least $1 - c_1e^{-t^2/2}$,
\begin{align*}
\frac{\bmu^\top\mni}{\|\mni\|_\bSigma} \geq& c_3^{-1} \frac{N}{\left[1 + N\sigma_\eta \right]\sqrt{ V + t^2\Delta V} + \Diamond\sqrt{n}} \\
\geq& \frac{1}{2c_3}\left(\frac{N}{\Diamond\sqrt{n} + \sqrt{ V + t^2\Delta V}} \wedge \frac{1}{\sigma_\eta\sqrt{ V + t^2\Delta V}}\right).
\end{align*}
For simplicity, we use Lemma \ref{lm::relations} to bound $\Delta V < 3/n$. We see that to achieve the bound in \eqref{eq::BO cor result}, we need to show the following two conditions:
\begin{equation*}
N \geq 2Cc_3\left(\Diamond\sqrt{n} + \sqrt{ V} + \sqrt{ 3t^2/n}\right) \quad \text{and} \quad
\sigma_\eta(\sqrt{ V} + \sqrt{3t^2/n}) \leq (2c_3C)^{-1}.
\end{equation*}
Recall that $\eta < c^{-1}$, so that $\sigma_\eta < 1$. The two conditions above can therefore be achieved by first taking $t^2 = n/(c_4C^2)$ for a large enough absolute constant $c_4$. Then, \ref{eq::BO cor large bmu} implies the first condition and \eqref{eq::BO cor small k assumption}--\eqref{eq::BO cor large effective rank} imply the second due to Proposition \ref{prop::V and B upper bound as in BO_ridge}.
\end{proof}

We believe these sufficient conditions for benign overfitting in classification to be novel. For example, in a recent paper, \cite{wang2021binary} obtain similar conditions only for the case of isotropic data, i.e., $\bSigma = \bI_p$ (see their Theorem 7).

\section{Proof outline and different forms of the main quantities}
\label{sec::proofs}
\subsection{Lower bound proof sketch}
\label{sec::lower bound proof sketch}

On a high level we follow the same logic as \citet{BO_ridge}: start by using algebraic formulas for the ridge/MNI solutions to disentangle the contribution of the~$\uptok$ from the~$\ktoinf$ components, then estimate each quantity using concentration inequalities. This is, however, more challenging for our mixture model since the clusters are not zero-mean~(they are centered at~$\pm \bmu$). Algebraically, the rank-one modification in~$\bX = \bQ + \by\bmu^\top$ hinders the straightforward application of the machinery from \citep{BO_ridge}.

Hence, the first step in the proof is to transform the expression for $\solnridge = \bX^\top(\bX\bX^\top + \lambda \bI_n)^{-1}\hat{\by}$ into a form that involves the inverse of $\bA = \bQ \bQ^\top + \lambda \bI_n$  instead of $\bX\bX^\top + \lambda \bI_n$. The corresponding formula is obtained in Lemma~\ref{lm::solution formulas}~(Appendix~\ref{sec::formulas for solution}). We then proceed to deriving sharp bounds for each of its terms. Note that there are two independent sources of randomness in our setting: the matrix $\bQ$ and the labels $(\by, \hat{\by})$. We start by addressing the second source, making high probability statements over the draw of $(\by, \hat{\by})$ conditionally on $\bQ$ in Lemma~\ref{lm::randomness in y}~(Appendix~\ref{sec::randomness in y}). As a result, the problem now reduces to bounding expressions that only depend on $\bQ$.

At this point we can apply the ideas developed in \citep{BO_ridge} with some modifications. As in that paper, we start by deriving algebraic bounds that hold almost surely on the event that the matrix $\bA_k$ is PD~(Lemma \ref{lm::algebraic decompositions}, Appendix \ref{sec::algebraic decompositions}). Some of the terms that we bound there have already appeared in \citep{BO_ridge}, namely $\|\muperpridge\|_\bSigma^2$ and $\tr(\bA^{-1}\bQ\bSigma\bQ^\top\bA^{-1})$ are respectively the bias and variance from \citep{BO_ridge}. Though we could, in principle, reuse the results in that paper, we use a slightly different derivation to obtain them in a new form.

Explicitly, we directly use the Sherman-Morrison-Woodbury identity to obtain our algebraic decompositions~(Lemma \ref{lm::application of SMW}, Appendix~\ref{sec::algebraic decompositions}), the most important embodiment of which is $\bA^{-1}\bQ_\uptok = \bA_k^{-1}\bQ_\uptok\left(\bI_k + \bQ_\uptok^\top \bA_k^{-1}\bQ_\uptok\right)^{-1}$. We then use the results in Lemma~\ref{lm::preserve sigma uptok inverse} to replace $\bSigma_\uptok^{1/2}\left(\bI_k + \bQ_\uptok^\top \bA_k^{-1}\bQ_\uptok\right)^{-1}\bSigma_\uptok^{1/2}$ by~$\alpha^{-1}(\beta^{-1}\bSigma_\uptok^{-1} + \bI_k)^{-1}$, where $\alpha, \beta$ are scalars that concentrate within a constant factor of their typical value. This is in contrast to the strategy used in \citep{BO_ridge}, which uses the fact that the matrix~$\left(\bI_k + \bQ_\uptok^\top \bA_k^{-1}\bQ_\uptok\right)^{-1}$ is dominated by~$\bI_k$ for~$k = k^*$.
Our technique yields sharp results for any choice of~$k$ and upper bounds that have the same form as lower bounds. In~\citet{BO_ridge}, upper and lower bounds have different forms and a separate conversion is needed to show that they coincide for $k = k^*$. We pay for that, however, with a bulkier form of the bounds.

The bounds from Lemma \ref{lm::algebraic decompositions} are formulated in terms of quantities that we assume to be concentrated around their typical values on the events $\sA_k(L)$ and $\sB_k(c_B)$. Plugging those values in the bounds is done in Lemma \ref{lm::randomness in Q but no eigvals of A} (Appendix \ref{sec::randomness in Q}). We conclude the proof of Theorem \ref{th::main} in Appendix \ref{sec::main bound put together} by first combining the bounds from Lemmas \ref{lm::randomness in y} and \ref{lm::randomness in Q but no eigvals of A} in Lemma~\ref{lm::high probability bounds} and then plugging the result into the algebraic formulas from Lemma \ref{lm::solution formulas}.

\subsection{Upper bound proof sketch}
\label{sec::upper bound proof sketch}

When we deal with the bounds within a constant factor,  it is usually more difficult to obtain a bound from below than from above. This is because to bound a sum from above one can use a triangle inequality $|a+ b| \leq |a| + |b|$ to reduce the problem to bounding separate terms from above. If, however, we want to bound a sum $|a + b|$ from below, the triangle inequality yields $|a+ b| \geq (|a| - |b|)\vee (|b| - |a|)$, which is only sharp when one term dominates another in magnitude.

In our case, we want to bound the fraction $\bmu^\top\solnridge/\|\solnridge\|_\bSigma$. To bound it from below we bound $\bmu^\top\solnridge$ from below and $\|\solnridge\|_\bSigma$ from above. Moreover, it turns out that there is only one term in the expression for $\bmu^\top\solnridge$ that we actually need to bound from below; other terms play the role of noise and can be bounded from above in absolute value. Because of that, bounding $\bmu^\top\solnridge/\|\solnridge\|_\bSigma$ from below is a much more straightforward task than bounding it from above.

Two problems arise when we bound $\bmu^\top\solnridge/\|\solnridge\|_\bSigma$ from above: first, we need to bound $\|\solnridge\|_\bSigma$ from below, and there is no one dominating term in its expression. The second problem is to show that the numerator $\bmu^\top\solnridge$ will be negative with constant probability if $N$ is not large enough compared to $\Diamond$, for which we also need to bound the magnitude of the noise terms in the numerator from below.

To alleviate the problem with the triangle inequality we resort to the following trick: we assume that $\bQ_\ktoinf$ has a symmetric distribution and is independent from $\bQ_\uptok$. This means that the joint distribution of $(\bQ_\uptok, \bQ_\ktoinf, \by)$ is the same as the distribution of $(\bQ_\uptok, \eps_q\bQ_\ktoinf, \eps_y\by)$, where we introduced two new Rademacher random variables $(\eps_q, \eps_y)$, which are independent of all previously defined random variables and from each other. The basic idea behind the introduction of these random variables is as follows: suppose $\eps$ is a Rademacher random variable, which is independent of random variables $a, b$. Then, conditionally on $a, b$,  with probability 0.5 over the draw of $\eps$, $|a + \eps b| = |a| + |b|$. If we now have high-probability lower bounds on $|a|$ and $|b|$, then we get a constant probability lower bound on $|a + \eps b|$.

To explain how this idea applies to the quantities that we need to bound, we will need to look at their exact expressions. Denote
\[
\bar{\bQ} := [\bQ_\uptok, \eps_q\bQ_\ktoinf],\quad
\bar{\by} := \eps_y\by,\quad
\solnridgebar := (\bar{\bQ} + \bar{\by}\bmu^\top)^\top (\underbrace{\bar{\bQ}\bar{\bQ}^\top + \lambda \bI_n}_{= \bA})^{-1}\bar{\by}.
\]
(Recall that for the upper bounds we only consider the case with no label-flipping noise, i.e., $\by = \hat{\by}$.) The expression for the numerator is now
\[
\bar{S}\bmu^\top\solnridgebar =  \bar{\by}^\top\bA^{-1}\bar{\by} \bmu^\top\muperpridgebar + (1 +  \bar{\bnu}^\top\bA^{-1}\bar{\by}) \bar{\bnu}^\top \bA^{-1}\bar{\by},
\]
where $\bar{\bnu} = \bar{\bQ}\bmu$, $\muperpridgebar = (\bI_p - \bar{\bQ}^\top\bA^{-1}\bar{\bQ})\bmu$, and $\bar{S}$ is a scalar, which is non-negative with high probability. The proof of Theorem \ref{th::main} already provides a sharp high-probability bound on the term $\bar{\by}^\top\bA^{-1}\bar{\by} \bmu^\top\muperpridgebar$, as well as the upper bound on $|\bar{\bnu}^\top \bA^{-1}\bar{\by}|$. Thus, the difficulty is only in proving a lower bound on $|\bar{\bnu}^\top \bA^{-1}\bar{\by}|$ to say that the term $\bar{\bnu}^\top \bA^{-1}\bar{\by}$ will make the numerator negative with constant probability unless $N$ is large compared to $\Diamond$. The random variable $\eps_q$ helps as follows: with probability $0.5$ over the draw of $\eps_q$,
\[
\|\bA^{-1}\bar{\bnu}\|^2 =  \|\bA^{-1}(\bQ_\uptok\bmu_\uptok + \eps_q \bQ_\ktoinf\bmu_\ktoinf)\|^2 \geq \|\bA^{-1}\bQ_\uptok\bmu_\uptok\|^2 + \|\bQ_\ktoinf\bmu_\ktoinf\|^2.
\]
We bound the terms $\|\bA^{-1}\bQ_\uptok\bmu_\uptok\|^2 $ and $\|\bQ_\ktoinf\bmu_\ktoinf\|^2$ from below in Lemma \ref{lm::high probability lower bounds numerator}. We then use them in Lemma \ref{lm::numerator upper} to get the full upper bound on the numerator $\bar{S}\bmu^\top\solnridgebar$.

When it comes to the denominator $\|\bar{S}\solnridgebar\|_\bSigma$, the expression that we use is
\[
\bar{S}\solnridgebar = (1 + \bar{\bnu}^\top \bA^{-1}\bar{\by})\bar{\bQ}^\top\bA^{-1}\bar{\by} + \bar{\by}^\top \bA^{-1}\bar{\by}\muperpridgebar.
\]
The term $|\bar{\bnu}^\top \bA^{-1}\bar{\by}|\|\bar{\bQ}^\top\bA^{-1}\bar{\by}\|_\bSigma$ is dominated by others, so we can reuse an upper bound on it from the proof of Theorem \ref{th::main}. When it comes to the remaining terms, note that $\bar{\bQ}^\top\bA^{-1}\bar{\by} = \eps_y\bar{\bQ}^\top\bA^{-1}\by,$ while $ \bar{\by}^\top \bA^{-1}\bar{\by}\muperpridgebar = {\by}^\top \bA^{-1}{\by}\muperpridgebar$, which does not depend on $\eps_y$. Thus, with probability $0.5$ over $\eps_y$, the cross-term that arises from those two terms is non-negative and can be ignored for the purposes of obtaining a lower bound. Next, note that $\|\bar{\bQ}^\top\bA^{-1}\by\|_\bSigma = \|{\bQ}^\top\bA^{-1}\by\|_\bSigma $ does not depend on $\eps_q$, thus with probability $0.5$ over the draw $\eps_q$ we can ignore the cross terms that arise from interaction between components $\uptok$ and $\ktoinf$ in  $\muperpridgebar$. Overall, with probability at least $0.25$ over the draw of $(\eps_q, \eps_y)$ we can ignore a few terms in the expression for $\|\bar{S}\solnridgebar\|_\bSigma$ to obtain a lower bound on it. The precise statement is given in Lemma \ref{lm:randomness in eps denominator lower}.

The strategy for the remainder of the proof of the upper bound is the same as in the proof of the lower bound: in Lemma \ref{lm:randomness in y denominator lower} we make high-probability statements with respect to the draw of $\by$, and in Lemma \ref{lm::high probability lower bounds denominator} we make high probability statements with respect to the draw of $\bQ$. We put together the lower bound on the denominator in Lemma \ref{lm::denominator lower} and combine it with the upper bound on the numerator in Theorem \ref{th::main upper bound}, whose proof is given in Appendix \ref{sec::upper bound on the ratio appendix}.

Note that due to the nature of the proof, we only obtain the upper bounds with constant probability.

\section{Effect of regularization}
\label{sec::regularization main body}

In this section we discuss the effects of the regularization on the accuracy of the learned classifier in the noiseless setting (i.e., $\eta = 0$). We will touch on the noisy setting in Section~\ref{sec::regularization noisy setting}, but we can only talk about the dependence of the lower bound on regularization there since we don't provide a matching upper bound. 

The main result that we have for the noiseless case is Theorem \ref{th::constant quantile tight bounds}, which proves tightness of the bound for a quantile $\alpha_\eps$. Throughout this section we will study how changing regularization affects that quantile.

Before we start, however, let's introduce two alternative forms of the bound on that quantile, that are somewhat more useful in terms of tracking the effect of $\lambda$. As we already pointed out, our bounds are closely related to bounds from the earlier work on regression \citep{BO_ridge} but have somewhat different form. In the next two lemmas we show how our quantities of interest could be alternatively defined to have similar form to the quantities from that paper.

The following lemma gives a form of the bounds using the notion of $k^*$ introduced by \citet{benign_overfitting}. This form corresponds to the form of the main results obtained by \citet{BO_ridge}.
\begin{restatable}[Bounds via $k^*$]{lemma}{boundsviakstar}
\label{lm::bounds via k star}
Suppose that 
\[
k \leq n/2 \quad \text{ and} \quad \Lambda > n\lambda_{k+1}.
\]
Define 
\begin{align*}
k^* :=& \min\left\{\kappa \in \{0, 1, \dots, k\}: \lambda + \sum_{i > \kappa}\lambda_i \geq n\lambda_{\kappa + 1}\right\},\\
\Lambda_* :=& \lambda + \sum_{i > k^*}\lambda_i,\\
V_* :=& \frac{k^*}{n}+ \Lambda_*^{-2}n\sum_{i > k^*}\lambda_i^2,\\
\Diamond_*^2 :=& n^{-1}\left\|\bmu_\uptokstar\right\|_{\bSigma_\uptokstar^{-1}}^2 + n\Lambda_*^{-2}\|\bmu_\kstartoinf\|_{\bSigma_\ktoinf}^2,\\
N_* :=& \left\|\bmu_\uptokstar\right\|_{\bSigma_\uptokstar^{-1}}^2 + n\Lambda_*^{-1}\|\bmu_\kstartoinf\|^2.
\end{align*}

Then
\[
2N_* \geq N \geq N_*/2, \quad 2\Diamond_* \geq \Diamond \geq \Diamond_*/2, \quad 4V_* \geq V \geq V_*/4,\quad \Lambda_* \geq \Lambda \geq \Lambda_*/2.
\]
\end{restatable}

The next lemma gives an alternative form of the bounds, which makes their dependence on $k$ less pronounced. They are analogous to the bounds from  Section 7 of \citep{BO_ridge}.
\begin{restatable}[Alternative form of the bounds]{lemma}{altformbounds}
\label{lm::alternative form of bounds}
Suppose that $k < n$ and $\Lambda > n\lambda_{k+1}.$ Denote
\begin{align*}
\Numalt := \sum_i \frac{\mu_i^2}{\lambda_i + \Lambda/n},\quad
\Varalt :=  \sum_i \frac{\lambda_i^2/n }{(\lambda_i + \Lambda/n)^2},\quad
\Diamondalt^2 := \sum_i \frac{\lambda_i  \mu_i^2/n}{(\lambda_i + \Lambda/n)^2}.
\end{align*}
Then 
\begin{align*}
N\geq  \Numalt \geq N/2,\quad
V\geq  \Varalt \geq V/4,\quad
\Diamond^2 \geq \Diamondalt^2\geq \Diamond^2/4.
\end{align*}
\end{restatable}
Note that this form of the results already appeared in our discussion in Section \ref{sec::recovering the geometry introduction}.

Now we are in position to return to studying the effect of regularization. To track changing values of regularization parameter $\lambda$, for the rest of this section we add it explicitly to the notation, i.e., in this section we will write $\alpha_\eps(\lambda), \Lambda(\lambda), N(\lambda), V(\lambda), \Diamond(\lambda), \sA_k(L, \lambda)$ etc. Note that if $L > 1$ and $\lambda' > \lambda$, then $\sA_k(L, \lambda) \subseteq \sA_k(L, \lambda')$, that is, we always only need to assume that $\sA_k(L, \lambda)$ holds for the smallest value of the regularization parameter that we consider.

\subsection{Increasing regularization never helps in the noiseless case}
\label{sec::regularization doesn't help}
Due to Theorem \ref{th::constant quantile tight bounds}, the main quantity of interest in the setting without label-flipping noise is $\frac{N(\lambda)}{\sqrt{V(\lambda)} + \sqrt{n}\Diamond(\lambda)}.$ According to Lemma \ref{lm::alternative form of bounds}, this quantity is within a constant factor of $\frac{\Numalt(\lambda)}{\sqrt{\Varalt(\lambda)} + \sqrt{n}\Diamondalt(\lambda)}$. The first interesting observation is that $\Numalt(\lambda)/\Diamondalt(\lambda)$ is a non-increasing function of $\lambda$. To see this denote
\begin{align*}
t := \Lambda/n,\quad
\bv := \bSigma^{-1/2}\bmu,\quad
\bw := (\bSigma + t\bI_p)^{-1}\bSigma^{1/2}\bmu
= (\bI_p + t\bSigma^{-1})^{-1}\bv.
\end{align*}
With this notation, it becomes 
\[
\frac{\Numalt(\Lambda)}{\sqrt{n}\Diamondalt(\Lambda)} = \frac{\bv^\top\bw}{\|\bw\|},
\]
which is non-increasing by the following lemma, whose proof can be found in Appendix \ref{sec::ridge analysis appendix}.
\begin{restatable}{lemma}{rationovardecreasing}
\label{lm::ratio with no variance is non-increasing}
Consider a non-zero vector $\bv \in \R^p$ and a PD symmetric matrix $\bM \in \R^{p\times p}$. Introduce the function $f:\R^p \to \R$ as 
$f(\bw) = \bv^\top\bw/\|\bw\|.$
Then $f\left((\bI_p + t\bM)^{-1}\bv\right)$ is a non-increasing function of $t$ on $[0, +\infty)$. 
\end{restatable}

This observation already suggests that increasing regularization should not lead to an increase in the bound. The only way that it could happen is when the term $\sqrt{V(\lambda)}$ dominates $\sqrt{n}\Diamond$ in the denominator. As it turns out, however, in this case the vector $\bmu$ cannot be large enough for the bound to be larger than a constant. A formal statement is given by the following lemma, which is proven in Appendix \ref{sec::ridge analysis appendix}.
\begin{restatable}[Increasing the regularization cannot make the bound large]{lemma}{largelambdasmallbound}\label{lm::increasing regularization cannot make bound large}
Suppose that $k < n$ and $\Lambda(\lambda) > n\lambda_k$. Then for some absolute constant $c > 0$ and any $\lambda' > \lambda$ 
\[
\frac{N(\lambda')}{\sqrt{V(\lambda')} + \sqrt{n}\Diamond(\lambda')} \leq c\left(1 + \frac{N(\lambda)}{\sqrt{V(\lambda)} + \sqrt{n}\Diamond(\lambda)}\right).
\]
\end{restatable}

Combining this with Theorem \ref{th::constant quantile tight bounds} gives the following.

\begin{corollary}
Suppose that the distribution of the rows of $\bZ$ is $\sigma_x$-sub-Gaussian. For any $L > 1$ there exist constants $a, c$ that only depend on $L$ and $\sigma_x$ and absolute constants $\delta, \eps$ such that the following holds. Assume that $n > c$, $k < n/c$, $\P(\sA_k(L, \lambda)) > 1-\delta$,  $N(\lambda) \geq a\Diamond(\lambda)$, and 
\[
\Lambda(\lambda) > c\left(n\lambda_{k+1}\vee \sqrt{n\sum_{i > k}\lambda_i^2}\right).
\]
Suppose that $\bQ_\ktoinf$ has a symmetric distribution and is independent from $\bQ_\uptok$.

For every $\lambda' \geq \lambda$ 
\[
\alpha_\eps(\lambda') \leq c\left(1 + \alpha_\eps(\lambda)\right) .
\]

\end{corollary}
\begin{proof}
Take $a, \delta, \eps$ the same as in Theorem \ref{th::constant quantile tight bounds}. Denote the constant $c$ from that theorem as $c_1$. Note that $\P(\sA_k(L, \lambda')) \geq \P(\sA_k(L, \lambda)) > 1-\delta$, which means that Theorem \ref{th::constant quantile tight bounds} applies for all values of the regularization parameter $\lambda' > \lambda$. Thus,
\[
\alpha_\eps(\lambda') \leq c_1\frac{N(\lambda')}{\sqrt{V(\lambda')} + \sqrt{n}\Diamond(\lambda')}, \quad \alpha_\eps(\lambda) \geq c_1^{-1}\frac{N(\lambda)}{\sqrt{V(\lambda)} + \sqrt{n}\Diamond(\lambda)}.
\]

Combining it with Lemma \ref{lm::increasing regularization cannot make bound large} and taking $c$ large enough depending on $c_1$ yields the result.
\end{proof}

Note, however, that our argument only works if the probability of the event $\sA_k(L, \lambda)$ is high for some constant $L$, and that $\Lambda(\lambda)$ is large compared to $n\lambda_{k+1}$. Increasing $\lambda$ increases both $\Lambda(\lambda)$ and the probability of $\sA_k(L, \lambda)$. Therefore, the results above don't say that smaller values of regularization are always better. A more precise interpretation would be ``if $\lambda$ is large enough so that $\Lambda(\lambda) \gg n\lambda_{k+1}$ and $\sA_k(L, \lambda)$ holds with high probability, then there is no benefit from increasing it further".

\subsection{Increasing regularization does nothing in some regimes}
Even though we showed that there is not much use (in a certain sense) in increasing regularization, we haven't yet shown that it is harmful. For example, the following question arises: can decreasing regularization increase $\bmu^\top\solnridge(\lambda)/\|\solnridge(\lambda)\|_{\bSigma}$ by more than a constant factor? As it turns out, it depends on how $\bmu$ is spread across the eigendirections of $\bSigma$. For example, increasing regularization will always preserve the bound within a constant factor if $\bmu$ is supported in the tail or $\bmu$ is an eigenvector of the covariance, and $\bmu$ is large enough so that $V(\lambda)$ is dominated by $n\Diamond(\lambda)^2$. The formal statement is given by the following corollary, which is proven in Appendix \ref{sec::ridge analysis appendix}.
\begin{restatable}[Regularization doesn't matter for certain $\bmu$]{corollary}{lambdadoesntmatter}\label{cor::lambda doesnt matter for some mu}
Suppose that the distribution of the rows of $\bZ$ is $\sigma_x$-sub-Gaussian. For any  $L > 1$ there exist constants $a, c$ that only depend on $L, \sigma_x$ and absolute constants $\eps, \delta$ such that the following holds.
Suppose that  $n > c$, $k < n/c$, $\P(\sA_k(L, \lambda)) > 1-\delta$,  $N(\lambda) \geq a\Diamond(\lambda)$, and
\[
\Lambda(\lambda) > c\left(n\lambda_{k+1}\vee \sqrt{n\sum_{i > k}\lambda_i^2}\right).
\]
Suppose that $\bQ_\ktoinf$ has a symmetric distribution and is independent from $\bQ_\uptok$.

If either for some $i \leq k$ 
\[
\bmu = \mu_i\be_i, \quad\text{and}\quad\frac{n\lambda_i\mu_i^2}{(1 + n\lambda_i/\Lambda(\lambda))^2} \geq \sum_i \lambda_i^2,
\]
(here $\be_i $ is the $i$-th eigenvector of $\bSigma$), or 
\[
\|\bmu_\uptok\| = 0\quad\text{and}\quad
\sum_i \lambda_i^2 \leq n\|\bmu_\ktoinf\|_{\bSigma_\ktoinf}^2,
\]
then  for any $\lambda' \geq \lambda$,
\[
\alpha_\eps(\lambda')/c \leq \alpha_\eps(\lambda) \leq c\alpha_\eps(\lambda').
\]
\end{restatable}

The results that we obtained so far seem to contradict the conclusion made by \cite{wang2021binary}, who considered a particular case of our model with Gaussian data and $k = 0$ and concluded that increasing regularization always decreases the classification error (see their Section 6.1), and checked that empirically in simulations. According to our results, increasing $\lambda$ does not change $\bmu^\top\solnridge(\lambda)/\|\solnridge(\lambda)\|_\bSigma$ by more than a constant factor in this regime. There is no actual contradiction, because \cite{wang2021binary} only proved that their bound is decreasing. They neither proved that the bound is sharp, nor that it can decrease by more than a constant factor. We provide a detailed comparison with their results in Section \ref{sec::comparison with wang2021binary}.

\subsection[Increasing regularization may cause harm by breaking the balance]{Increasing regularization may cause harm by breaking the balance between the tail and the spiked part}
Now let's investigate for which $\bmu$ having regularization as small as possible actually provides more than a constant factor gain. Lemma \ref{lm::bounds via k star} gives, perhaps, the most convenient formulas to look at. For simplicity let's restrict ourselves to the case where $\bmu$ is large enough, so the term $\sqrt{V(\lambda)}$ is dominated in the denominator.  Let's write out the quantity of interest:
\[
\frac{N_* }{\sqrt{n}\Diamond_*} = \frac{\left\|\bmu_\uptokstar\right\|_{\bSigma_\uptokstar^{-1}}^2 + n\Lambda_*^{-1}\|\bmu_\kstartoinf\|^2}{\sqrt{\left\|\bmu_\uptokstar\right\|_{\bSigma_\uptokstar^{-1}}^2 + n^2\Lambda_*^{-2}\|\bmu_\kstartoinf\|_{\bSigma_\kstartoinf}^2}}.
\]
 Increasing regularization does two things: it changes the value of $\Lambda_*$, which serves as a scaling factor in front of the contribution of the tail, and it decreases $k^*$, therefore recovering the geometry in fewer components. We are going to look at those effects separately. 

 First, consider the case when $k^*$ doesn't change from changing $\lambda$. Note that if the term $\left\|\bmu_\uptokstar\right\|_{\bSigma_\uptokstar^{-1}}^2$ dominates in both the numerator and the denominator, then the ratio becomes just $\left\|\bmu_\uptokstar\right\|_{\bSigma_\uptokstar^{-1}}$ up to a constant factor, that is, it is not sensitive to the changes in $\Lambda_*$. The same happens if the term $\left\|\bmu_\uptokstar\right\|_{\bSigma_\uptokstar^{-1}}^2$ is dominated in both the numerator and the denominator: the ratio becomes $\|\bmu_\kstartoinf\|^2/\|\bmu_\kstartoinf\|_{\bSigma_\kstartoinf}$, and again it is not sensitive to the changes in $\Lambda_*$. Moreover, since we always assume $\Lambda_* > n\lambda_{k^*+1}$ we have
\[
n\Lambda_*^{-1}\|\bmu_\kstartoinf\|^2 \geq n\Lambda_*^{-1}\lambda_{k^*+1}^{-1}\|\bmu_\kstartoinf\|_{\bSigma_\kstartoinf}^2 \geq n^2\Lambda_*^{-2}\|\bmu_\kstartoinf\|_{\bSigma_\kstartoinf}^2.
\]

Thus, the case in which changing $\lambda$ can change the bound by more than a constant in this regime is 
\[
n\Lambda_*^{-1}\|\bmu_\kstartoinf\|^2 \geq \left\|\bmu_\uptokstar\right\|_{\bSigma_\uptokstar^{-1}}^2 \geq n^2\Lambda_*^{-2}\|\bmu_\kstartoinf\|_{\bSigma_\kstartoinf}^2.
\]

In this case the bound becomes equal to $n\Lambda_*^{-1}\|\bmu_\kstartoinf\|^2/\left\|\bmu_\uptokstar\right\|_{\bSigma_\uptokstar^{-1}}$ up to a constant factor, so the dependence on $\Lambda_*$ is inversely proportional. Recall, however, that in order to not change $k^*$ we need to always have $\Lambda_* \leq n\lambda_{k^*}$. Putting it together with  $\Lambda_* \geq n\lambda_{k^*+1} $ we see that changing regularization in this regime can change the bound by at most $\lambda_{k^*}/\lambda_{k^*+1}$. Thus, there should be a big relative gap between $\lambda_{k^*}$ and $\lambda_{k^*+1}$ for that quantity to be large. 

The discussion above reveals a recipe for constructing regimes in which increasing regularization can significantly impair the classification accuracy. The formal statement is as follows, its proof can be found in Appendix \ref{sec::ridge analysis appendix}.
\begin{restatable}{lemma}{smallregfixedkrecipe}
\label{lm::small regularization is better no change in k star}
For any  $\sigma_x \geq 1, L > 1$ there exist constants $a, c$ that only depend on $\sigma_x$ and absolute constants $\eps, \delta$ such that the following holds. Suppose that  $n > c$, $0 < k < n/c$. Take any $C > 1$ and construct  the classification problem as follows:
\begin{enumerate}
\item Take $\bZ_\ktoinf$ with $\sigma_x$-sub-Gaussian rows and the sequence $\{\lambda_i\}_{i > k}$ and regularization parameter $\lambda$ such that $\P(\sA_k(L, \lambda)) \geq 1 - \delta$ and 
\[
\Lambda(\lambda) \geq c\left(n\lambda_{k+1}\vee \sqrt{n\sum_{i > k}\lambda_i^2}\right).
\]
\item Take $\bZ_\uptok$ with $\sigma_x$-sub-Gaussian rows independent from $\bZ_\ktoinf$, and $\{\lambda_i\}_{i=1}^k$ such that $n\lambda_k \geq C \Lambda(\lambda)$.
\item Take $\bmu_\ktoinf$ whose most energy is spread among the eigendirections of $\bSigma$ with small eigenvalues, that is, 
\[
\|\bmu_\ktoinf\|_{\bSigma_\ktoinf}^2 \leq C^{-1}n^{-1}\Lambda(\lambda)\|\bmu_\ktoinf\|^2.
\]
\item Take\footnote{Note that such $\bmu_\uptok$ exists because of how we chose $\bmu_\ktoinf$.} $\bmu_\uptok$ which balances $\bmu_\ktoinf$ in the following sense:
\begin{equation}
\label{eq::mu uptok balances mu ktoinf}
nC^{-1}\Lambda(\lambda)^{-1}\|\bmu_\ktoinf\|^{2} \geq \|\bmu_\uptok\|_{\bSigma_\uptok^{-1}}^2 \geq n^2\Lambda(\lambda)^{-2}\|\bmu_\ktoinf\|_{\bSigma_\ktoinf}^2.
\end{equation}
\item Scale $\bmu$ up\footnote{Note that the previous conditions were homogeneous in $\bmu$, so multiplying it by a scalar does not break them.} if needed, so it holds that
\[
 n\Diamond^2(\lambda) \geq V(\lambda)\quad \text{ and }N(\lambda) \geq a\Diamond(\lambda).
\]
\end{enumerate}

Then for any $\lambda'$ such that $\Lambda(\lambda') \geq C\Lambda(\lambda)$ 
\[
\alpha_\eps(\lambda) \geq \frac{C}{c}\alpha_\eps(\lambda').
\]
\end{restatable}

The following corollary, whose proof can be found in Appendix \ref{sec::ridge analysis appendix}, shows a particular example when the optimal regularization is negative:
\begin{restatable}{corollary}{negativeregfixedkexample}
There exists absolute constants $a, b$ such that the following holds. Take $p = \infty$, $n > a$ and $1\leq k < n/a$. 
Consider the following classification problem with Gaussian data (in infinite dimension) and no label-flipping noise ($\eta = 0$):
\[
\lambda_i = \begin{cases}
2b, & i \leq k,\\
e^{-(i - k)/(bn)}, & i > k.
\end{cases},\quad
\mu_i = 
 \begin{cases}
4\sqrt{b/k}, & i \leq k,\\
4\sqrt{b}\cdot 2^{-(i-k)/2}, & i > k.
\end{cases}
\]
Then the value of $\lambda$ that maximizes $\alpha_\eps(\lambda)$ is negative.
\end{restatable}

\subsection{Increasing regularization can harm by destroying ``recovery of the geometry"}
\label{sec::regularization destroys recovering geometry}
Now let's consider a scenario where $k^*$ changes all the way to zero because of increase in regularization. That is, we stop ``recovering the geometry" of the first $k^*$ components because of it. For simplicity, consider the case with no tail, that is, $\|\bmu_\kstartoinf\| = 0$. Informally, increasing regularization will change the classifier from $(\solnridge(\lambda))_\uptokstar \propto \bSigma_\uptokstar^{-1}\bmu_\uptokstar$ to $(\solnridge(\lambda'))_\uptokstar \propto \bmu_\uptokstar$ and the value of $ \bmu^\top \solnridge/\|\solnridge\|_\bSigma$ will go from $\|\bmu_\uptokstar\|_{\bSigma_\uptokstar^{-1}}$ to $\|\bmu_\uptokstar\|^2/\|\bmu_\uptokstar\|_{\bSigma_\uptokstar}$, which may be much smaller depending on $\bmu_\uptok$. This results in the following lemma, whose proof can be found in Appendix \ref{sec::ridge analysis appendix}:
\begin{restatable}{lemma}{smallregchangingkrecipe}
\label{lm::smallregchangingkrecipe}
For any  $\sigma_x > 1, L > 1$ there exist constants $a, c$ that only depend on $L, \sigma_x$ and absolute constants $\eps, \delta$ such that the following holds. Suppose that  $n > c$, $0 < k < n/c$. Take any $C > 1$ and construct  the classification problem as follows:
\begin{enumerate}
\item Take $\bZ_\ktoinf$ with $\sigma_x$-sub-Gaussian rows and the sequence $\{\lambda_i\}_{i > k}$ and regularization parameter $\lambda$ such that $\P(\sA_k(L, \lambda)) \geq 1 - \delta$ and 
\[
\Lambda(\lambda) > c\left(n\lambda_{k+1}\vee \sqrt{n\sum_{i > k}\lambda_i^2}\right).
\]
\item Take $\bZ_\uptok$ with $\sigma_x$-sub-Gaussian rows independent from $\bZ_\ktoinf$, and $\{\lambda_i\}_{i=1}^k$ such that $n\lambda_k \geq \Lambda(\lambda)$.
\item Take $\bmu$ that is only supported on the first $k$ coordinates (i.e.,  $\|\bmu_\ktoinf\|= 0$) such that
\begin{equation}
\label{eq::C-adversarial mu uptok}
\|\bmu_\uptok\|_{\bSigma_\uptok}\|\bmu_\uptok\|_{\bSigma_\uptok^{-1}} \geq C\|\bmu_\uptok\|^{2}.
\end{equation}

\item Scale $\bmu$ up if needed, so that
\[
 n\Diamond^2(\lambda) \geq V(\lambda)\quad \text{ and }N(\lambda) \geq a\Diamond(\lambda).
\]
\end{enumerate}

Then for any $\lambda'$ such that $\Lambda(\lambda') \geq n\lambda_1$ 
\[
\alpha_\eps(\lambda) \geq \frac{C}{c}\alpha_\eps(\lambda').
\]
\end{restatable}

A natural question is when one can choose such $\bmu_\uptok$ that Equation \eqref{eq::C-adversarial mu uptok} is satisfied. The answer is given by the following.
\begin{lemma}
For any $\bmu_\uptok \neq \bzero_k$
\[
1 \leq \frac{\|\bmu_\uptok\|_{\bSigma_\uptok}\|\bmu_\uptok\|_{\bSigma_\uptok^{-1}}}{\|\bmu_\uptok\|^{2}} \leq \frac{\lambda_1 + \lambda_k}{2\sqrt{\lambda_1\lambda_k}}.
\]
The upper bound is achieved for $\bmu_\uptok = \be_1 + \be_k$.
\end{lemma}
\begin{proof}
 Without loss of generality we can put $\left\|\bmu_\uptok\right\|^2 = 1$. Now the numbers $\{\mu_i^2\}$ act as weights: $\left\|\bmu_\uptok\right\|_{\bSigma_\uptok^{-1}}^2$ is the weighted average of $\{\lambda_i^{-1}\}_{i=1}^k$ with weights $\{\mu_i^2\}$, while $\left\|\bmu_\uptok\right\|_{\bSigma_\uptok}^{-2}$ is inverse of the weighted average of  $\{\lambda_i\}_{i=1}^k$. Thus, for the convex function $f(x) = 1/x$ we can write
 \[
\|\bmu_\uptok\|_{\bSigma_\uptok}^2\|\bmu_\uptok\|_{\bSigma_\uptok^{-1}}^2 = \frac{\sum_{i = 1}^k\mu_i^2 f(\lambda_i)}{f\left(\sum_{i = 1}^k\mu_i^2\lambda_i\right)}.
 \]
Thus, the lower bound follows from Jensen's inequality. Moreover, if $f$ is a non-negative convex function and $X$ is a random variable with a support $[a, b]$, then the ratio $\E[f(X)]/f(\E[X])$ is maximized by a distribution of $X$ is supported on $\{a, b\}$.  That is, we should have $\mu_k^2 = 1 - \mu_1^2$ and $\mu_i = 0$ for $i \not\in \{1, k\}$. Now we only need to find the scalar $\mu_1^2$ that maximizes the following:
\[
\frac{\|\bmu_\uptok\|_{\bSigma_\uptok}^2\|\bmu_\uptok\|_{\bSigma_\uptok^{-1}}^2}{\|\bmu_\uptok\|^{4}} = \left(\lambda_k + (\lambda_1 - \lambda_k)\mu_1^2\right)\left(\lambda_k^{-1} + (\lambda_1^{-1} - \lambda_k^{-1})\mu_1^2\right).
\]
Putting the derivative equal to zero yields:
\[
0 =(\lambda_1 - \lambda_k)\left(\lambda_k^{-1} + (\lambda_1^{-1} - \lambda_k^{-1})\mu_1^2\right) + \left(\lambda_k + (\lambda_1 - \lambda_k)\mu_1^2\right)(\lambda_1^{-1} - \lambda_k^{-1}),
\]
\[
\mu_1^2 = 0.5.
\]
The maximum value is equal to $\frac{(\lambda_1 + \lambda_k)^2}{4\lambda_1\lambda_k}$.
\end{proof}

The following corollary, whose proof can be found in Appendix \ref{sec::ridge analysis appendix}, shows another example when the optimal regularization is negative:
\begin{restatable}{corollary}{negativeregchangingkexample}
There exist absolute constants $b > c$ such that the following holds. Take $p > bn$, and $b\leq k < n/b$. 
Consider the following classification problem with Gaussian data (in dimension $p$) and no label-flipping noise ($\eta = 0$):
\[
\lambda_i = \begin{cases}
k^{-4i/k}, & i \leq k,\\
\frac{cn}{pk^4}, & i > k.
\end{cases},\quad
\mu_i = 
 \begin{cases}
\frac{b\ln(k)}{k^5}\left(\frac{k}{n} + \frac{n}{p}\right), & i \leq k,\\
0, & i > k.
\end{cases}
\]
Then the value of $\lambda$ that maximizes $\alpha_\eps(\lambda)$ is negative.
\end{restatable}

\subsection{Regularization with label-flipping noise}
\label{sec::regularization noisy setting}
Since we don't have a matching upper bound for the case with label-flipping noise, we can only consider the effect of the regularization on the lower bound given in Theorem \ref{th::main}. That bound, up to a constant factor,  is given by the following formula:
\begin{align*}
\frac{N - ct\Diamond }{\left(\left[1 + N\sigma_\eta \right]\sqrt{ V + t^2\Delta V} + \Diamond\sqrt{n} \right)}.
\end{align*}
Let's look at it in the simple regime when $t$ is a constant and $\bmu$ is large enough so that $ct\Diamond$ is dominated by $N$.  Thus, we are going to consider the formula
\[
\frac{N}{\left(\left[1 + N\sigma_\eta \right]\sqrt{V} + \Diamond\sqrt{n} \right)}.
\]

We can rewrite it up to a constant as a minimum of two terms:
\[
\frac{1}{\sigma_\eta\sqrt{V}} \wedge \frac{N}{\sqrt{V} + \Diamond\sqrt{n}},
\]
and the second term is just the bound for the case $\eta = 0$. We've already seen this in Section~\ref{sec::regimes of the lower bound}, where we stated that the bound for the case with label-flipping noise goes over the same regimes, and only picks up a new regime for large $\bmu$. We already know that increasing $\lambda$ ``doesn't help" in the noiseless regime. It does, however, increase the first term ($V(\lambda)$ is obviously a decreasing function of $\lambda$). Thus, regularization can only provide a significant benefit in that new ``large $\bmu$" regime. Since we don't have a proof of tightness for this bound, however, we leave a more careful study of this effect to future work.

\section{Detailed comparison with previous results}
\label{sec::comparisons with other papers}
In this section, we compare our results with the results of four recent papers: \citep{Cao_Gu_Belkin}, \citep{wang2021binary}, \citep{chatterji2020linearnoise}, and \citep{muthukumar2021classification}.  Sections \ref{sec::comparison with Cao_Gu_Belkin} and \ref{sec::comparison with wang2021binary}   show that our results generalize the results of the first two of those papers. In Section \ref{sec::comparison with chatterji2020linearnoise} we discuss the relation between our results and those of \citet{chatterji2020linearnoise}, and show that their bound is weaker then ours for the case of Gaussian data.  Finally, the paper of \citet{muthukumar2021classification} considers a slightly different model.  In Section \ref{sec::comparison with muthukumar2021classification} we explain how this model is related to ours, and show that some of their conclusions can be recovered from our analysis too, even though our results do not strictly generalize theirs.

Before we begin the comparisons, it is worth talking about some similarities that all those papers possess. All four of them had studying the maximum margin solution as their aim. In our notation, the maximum margin solution (MM) is defined as 
\begin{equation}
\label{eq::mm definition}
\mm = \argmin_{\bw \in \R^p} \|\bw\| \text{ s.t. } \bD_{\hat{\by}}\bX\bw \geq \bone_n,
\end{equation}
where $\bD_{\hat{\by}}= \diag(\hat{\by})$. A standard argument with Lagrange multipliers shows that the solution to the optimization problem \eqref{eq::mm definition} is a conic combination of the columns of the matrix $\bX^\top\bD_{\hat{\by}}$, and strictly positive coefficients in that conic combination correspond to the inequalities on the right hand side of Equation \eqref{eq::mm definition} that are saturated, i.e., they are satisfied with equality. The data points (columns of $\bX^\top$) which correspond to those strictly positive coefficients are called support points. One of the core ideas of \citet{muthukumar2021classification, Cao_Gu_Belkin, wang2021binary} is that in some cases ``support proliferation" happens with high probability, which means that all points are support points. In this case all inequalities in the constraints become equalities, i.e., $\bD_{\hat{\by}}\bX\bw = \bone_n$, and MM coincides with MNI. Motivated by this observation, those papers actually study MNI under support proliferation or in a certain vicinity of that regime. Because of that, our results can be directly compared to the results of those papers. 

When it comes to \citet{chatterji2020linearnoise}, they don't explicitly rely on support proliferation to study MM, but, as we explain in Section \ref{sec::comparison with chatterji2020linearnoise}, their proof implies that support proliferation must happen, and thus we can compare our results to theirs too.

Interestingly, one of the conditions under which support proliferation happens is that the whole data distribution has high effective rank. Because of that, most of the results from the above mentioned papers correspond to our results with $k = 0$.

\subsection{Comparison with \texorpdfstring{\citep{Cao_Gu_Belkin}}{(Cao et al., 2021)}}
\label{sec::comparison with Cao_Gu_Belkin}
\citet{Cao_Gu_Belkin} study the same data generating model as ours. Their main result, reformulated in our notation, is given by the following theorem.
\begin{theorem}[Theorem 3.1 and Proposition 4.1 from \citep{Cao_Gu_Belkin}]
Suppose the elements of matrix $\bZ$ are independent and $\sigma_x$-sub-Gaussian, and $\eta = 0$. There are constants $C, C'$ that only depend on $\sigma_x$ such that the following holds. Assume that
\begin{equation}
\label{eq::mni is mm assumption Cao}
\tr(\bSigma) \geq C\max\left\{n^{3/2}\|\bSigma\|, n\|\bSigma\|_F, n\sqrt{\log(n)}\|\bmu\|_\bSigma\right\},
\end{equation}
and $\|\bmu\|^2 \geq C\|\bmu\|_\bSigma$. Then, with probability at least $1 - n^{-1}$, $\mm = \mni$ and 
\begin{equation}
\label{eq::bound Gao Gu Belkin}
\frac{(\bmu^\top\mni)^2}{\|\mni\|_\bSigma^2} \geq C'\frac{n\|\bmu\|^4}{n\|\bmu\|_\bSigma^2 + \|\bSigma\|_F^2 + n\|\bSigma\|^2}.
\end{equation}
\end{theorem}

Note that if we take $k = 0$ and $\lambda = 0$, the assumption imposed in Equation \eqref{eq::mni is mm assumption Cao} implies that  $\Lambda \geq n^{3/2}\lambda_1 \gg n$. Moreover, since the data is assumed to be sub-Gaussian, the events $\sA_k(L)$ and $\sB_k(c_B)$ hold with high probability for  constants $L, c_B$ that only depend on $\sigma_x$, due to Lemmas \ref{lm:: eigenvalues of A_k indep coord} and \ref{lm::sB under sub-Gaussianity}. Therefore, our Theorem \ref{th::main} is applicable and gives the following bound with probability $1 - ce^{-t^2/2}$ (up to a constant factor):
\begin{align*}
\frac{N - ct\Diamond}{\sqrt{V + t^2\Delta V} + \sqrt{n}\Diamond}.
\end{align*}

Thus, the following proposition, whose proof can be found in Appendix \ref{sec::comparissons proofs appendix}, shows that our bound is at least as good as the bound from \citet{Cao_Gu_Belkin}.
\begin{restatable}{proposition}{comparisonGaoGuBelkin}
\label{prop::comparison with GaoGuBelkin}
Take $k = 0$ and some $c > 1$. Suppose that $n\lambda_1 < \Lambda$ and  $\|\bmu\|^2 \geq 2c\|\bmu\|_\bSigma$. Then for $t < \sqrt{n}$,
\begin{equation}
\label{eq::our bound k star = 0}
\frac{N - ct\Diamond}{\sqrt{V + t^2\Delta V} + \sqrt{n}\Diamond} \geq \frac14\frac{n\|\bmu\|^2}{n\|\bmu\|_\bSigma + \sqrt{n}\|\bSigma\|_F + n\|\bSigma\|}.
\end{equation}
\end{restatable}

Note that the resulting bound does not depend on $\lambda$. We have already observed that in Corollary \ref{cor::lambda doesnt matter for some mu}: indeed, since $k = 0$, $\bmu$ is supported on the tail of the covariance, and regularization does not change the bound by more than a constant factor. Moreover, since they effectively considered $k = 0$, \citet{Cao_Gu_Belkin} did not observe the effect of ``recovering the geometry." 

\subsection{Comparison with \texorpdfstring{\citep{wang2021binary}}{(Wang and Thrampoulidis, 2021)}}
\label{sec::comparison with wang2021binary} 
The next paper we compare ours with is \citep{wang2021binary}. They consider Gaussian $\bQ$. When it comes to $\bSigma$, they consider two ensembles, which they call ``balanced" (see their Definition 2.1) and ``bi-level" (see their Definition 2.2). Translating to our terminology, for a  balanced ensemble, $k^* = 0$, and for bi-level, $k^* = 1$.

Their result for the balanced ensemble is as follows. 
\begin{theorem}[Theorem 3 from \citep{wang2021binary}]
\label{th::wang main balanced}
There are large absolute constants $a, b, c$ such that the following holds. Assume that rows of $\bQ$ come from a  Gaussian distribution, $k = 0$ and 
\begin{equation}
\label{eq::wang2021binary balanced high effective rank}
n\lambda_1 < b\sum_{i }\lambda_i.
\end{equation}
 Take $\lambda \geq 0$. Assume that $\|\bmu\|^2 \geq a\left(n\Lambda^{-1}\|\bmu\|^2_{\bSigma} + \|\bmu\|_{\bSigma}\right)$. Then with probability at least $1 - e^{-n^2/c}$
\begin{equation}
\label{eq::wang balanced}
\frac{\bmu^\top\solnridge}{\|\solnridge\|_\bSigma} \geq c^{-1}\frac{\|\bmu\|^2 - a\left(n\Lambda^{-1}\|\bmu\|^2_{\bSigma} + \|\bmu\|_{\bSigma}\right)}{(1 \vee n\Lambda^{-1}\|\bmu\|_{\bSigma})\sqrt{\sum_i \lambda_i^2} + \|\bmu\|_{\bSigma}}.
\end{equation}
\end{theorem}

We see that their bound is at most within a constant factor of 
\[
\frac{\|\bmu\|^2}{\|\bSigma\|_F +  \|\bmu\|_{\bSigma}+ \|\bmu\|_{\bSigma}\frac{n\|\bSigma\|_F}{\lambda + \tr(\bSigma)}}.
\]
Comparing to our bound from Equation \ref{eq::our bound k star = 0}, we see that the bound from \citep{wang2021binary} has $\|\bSigma\|_F$ in the denominator, which is larger (up to a constant) than $\|\bSigma\|_F/\sqrt{n} +  \|\bSigma\|$ that stands in the denominator of Equation \eqref{eq::our bound k star = 0} (after dividing both numerator and denominator by $n$). Moreover, it picks up an additional term $\|\bmu\|_{\bSigma}\frac{n\|\bSigma\|_F}{\lambda + \tr(\bSigma)}$ in the denominator.  Thus, the bound from Theorem \ref{th::wang main balanced} is worse than the one from Equation~\eqref{eq::our bound k star = 0}. Note that, just as in the previous section, Equation \eqref{eq::wang2021binary balanced high effective rank} implies that our Theorem \ref{th::main} is applicable with high probability. That is, Proposition \ref{prop::comparison with GaoGuBelkin} shows that our result generalizes the result for balanced ensembles from \citet{wang2021binary}.

When it comes to the bi-level ensemble, the result of \citet{wang2021binary} translated to our notation is given by the following theorem.
\begin{theorem}[Theorem 5 from \citep{wang2021binary}]
\label{th::wang bi-level}
There are large absolute constants $a, b, c$ such that the following holds. Assume that rows of $\bQ$ come from a  Gaussian distribution. Take $\bmu$ that is supported on one coordinate, i.e., $\bmu = \mu_j \be_j$ for some $j$, and $j > 1$.  Assume that $\|\bmu\|^2 \geq a\left(n\Lambda^{-1}\|\bmu\|^2_{\bSigma} + \|\bmu\|_{\bSigma}\right)$. Assume that $k = 1$, $\lambda \geq 0$ and
\begin{equation}
\label{eq::wang bi-level effective rank assumption}
bn\lambda_1 > \sum_{i> 1}\lambda_i \quad \text{ and } \quad bn\lambda_2 < \sum_{i> 2}\lambda_i.
\end{equation}
Denote
\begin{align*}
A = \lambda_1\frac{\Lambda + n\|\bmu\|_\bSigma}{n\lambda_1 + \Lambda},\quad
B = (1 + n\Lambda^{-1}\|\bmu\|_\bSigma)\sqrt{\sum_{i \neq 1, j}\lambda_i^2}.
\end{align*}

Then with probability at least $1 - e^{-n/c}$,
\begin{equation}
\label{eq::wang bi-level}
\frac{\bmu^\top\solnridge}{\|\solnridge\|_\bSigma} \geq c^{-1}\frac{\|\bmu\|^2(1 - cn\Lambda^{-1}\lambda_j) - c\|\bmu\|_{\bSigma}}{A + B + \lambda_j + \|\bmu\|_\bSigma}.
\end{equation}
\end{theorem}

Note that for $k = 1$ the second part of Equation \eqref{eq::wang bi-level effective rank assumption} yields $\Lambda > b\lambda_{k+1}$. Thus, if $b$ is large enough, under that assumption, the events $\sA_k(L)$ and $\sB_k(c_B)$ hold with high probability for absolute constants $L, c_B$, due to Lemmas \ref{lm:: eigenvalues of A_k indep coord} and \ref{lm::sB under sub-Gaussianity}. Therefore,  our Theorem \ref{th::main} is applicable just as it was in Section \ref{sec::comparison with Cao_Gu_Belkin}. So, the following proposition, whose proof can be found in Appendix \ref{sec::comparissons proofs appendix}, shows that our bound generalizes the bound for bi-level ensembles from \citep{wang2021binary}.

\begin{restatable}{proposition}{comparisonwangbinary}
Take $k = 1$ and some $c > 1$. Assume that $\lambda > 0$, $n\lambda_{k+1} \leq \sum_{i > k}\lambda_i$, $\|\bmu_\uptok\| = 0$, and $\|\bmu\|^2 \geq 2c\|\bmu\|_{\bSigma}$. Take any $j > 1$ and define $A, B$ as in Theorem \ref{th::wang bi-level}. Then for $t \leq \sqrt{n}$
\[
\frac{N - ct\Diamond}{\sqrt{V + t^2\Delta V} +  \sqrt{n}\Diamond} \geq \frac16  \frac{\|\bmu\|^2}{A + B + \lambda_j + \|\bmu\|_\bSigma}.
\]
\end{restatable}

Overall, we see that the bounds from \citet{wang2021binary} are not sharp. Moreover, since they  considered either $k = 0$ or $k=1$ and $\bmu$ supported on a single coordinate, they did not observe the effect of ``recovering the geometry."
\subsection{Comparison with \texorpdfstring{\citep{chatterji2020linearnoise}}{(Chatterji and Long, 2021)}}
\label{sec::comparison with chatterji2020linearnoise}

\citet{chatterji2020linearnoise} consider almost the same data generating model as ours: two clusters with symmetric means and the same covariances. Only their definition of the noise is different: they consider arbitrary corruptions of the distribution of $(\bx, y)$ that preserve the marginal distribution of $\bx$ and have bounded total variation distance with the initial distribution. Label-flipping noise can be seen as a particular case of such corruption. 

Nevertheless, comparing our results with those of \citet{chatterji2020linearnoise} is not straightforward for two reasons. First, they consider the MM solution, while we consider the ridge and MNI solutions. Second, \citet{chatterji2020linearnoise} impose assumptions that are incomparable with ours, for example they assume that elements of $\bQ$ have bounded sub-Gaussian norms, while we only have proofs that the events $\sA_k(L)$ and $\sB_k(c_B)$ hold with high probability when elements of $\bZ$ are sub-Gaussian (see our Lemmas \ref{lm:: eigenvalues of A_k indep coord} and  \ref{lm::sB under sub-Gaussianity}).

Regarding the first potential issue, we note that they in fact consider a regime where the max-margin solution coincides with MNI.  To see this, note that by Lemma A.2 of~\citet{frei2023benign}, for max-margin to coincide with MNI, it suffices for the training data to be `nearly-orthogonal’ in the sense that for every training sample $(\bx_k, y_k)$ it holds that $\| \bx_k\|^2 \gg n \max_{i, j} \frac{ \| \bx_i\|^2}{\|\bx_j\|^2} \max_{i\neq j} |\bx_i^\top \bx_j|$.  One can verify this property holds in their setting by using their Lemma 10 together with their assumption (A.3).

To alleviate the problem with the differences in the assumptions, we compare the results for the case of Gaussian distributions, where both our and their results are directly applicable. We also only consider the label-flipping noise here, since it is a particular case of the noise considered in  \citet{chatterji2020linearnoise}. When translated into our notation, that result is given by the following.
\begin{theorem}[Theorem 1 from \citep{chatterji2020linearnoise}, Gaussian case]
\label{th::chatterji2020linearnoise main result}
Fix\\ some constant $\kappa \in(0, 1)$. Suppose rows of $\bQ$ are i.i.d.~samples from a Gaussian distribution. Suppose that $\lambda_i \leq 1$ for every $i \in \{1, \dots, p\}$ and $\sum_i \lambda_i \geq \kappa p$. There is a constant $c$ that only depends on $\kappa$ and an absolute constant $b$ such that the following holds. 

Take $\delta \in (e^{-n/c}, c^{-1}).$ Assume that $p\geq cn^2\log(n/\delta)$, $p/(cn) \geq \|\bmu\|^2 \geq c\log(n/\delta)$, and $\eta \leq 1/c$.  Then with probability $1-\delta$ over the draw of $\bX, \hat{\by}$
\begin{equation}
\label{eq::chatterji long main bound}
\P_{(\bx, \hat{y})}(\hat{y}\bx^\top\mm < 0) \leq \eta + \exp\left(-b\frac{\|\bmu\|^4}{p}\right),
\end{equation}
where $(\bx, \hat{y})$ is a new data point from the data distribution with label-flipping noise.
\end{theorem}

Note that assumptions of Theorem \ref{th::chatterji2020linearnoise main result} yield for $n > e/c$:
\[
\sum_i \lambda_i \geq \kappa p \geq cn^2\log(nc) \geq cn\lambda_1.
\]
Thus, for  $k = 0$ we have $\Lambda > cn\lambda_{k+1}$. According to Lemmas \ref{lm:: eigenvalues of A_k indep coord} and  \ref{lm::sB under sub-Gaussianity}, if $c$ is a large enough absolute constant, both events $\sA_k(L)$ and $\sB_k(c_B)$  hold with probability at least $1 - ce^{-n/c}$ for some absolute constants $L$ and $c_B$. Thus, Theorem \ref{th::main} is applicable for $k = 0$, and yields the result with probability $1 - ce^{-n/c} - ce^{-t^2/2}$. To match the probability of $1 - \delta$ from Theorem \ref{th::chatterji2020linearnoise main result} we should take $t = \sqrt{2\log (1/\delta)}$. Finally, in Section \ref{sec::binary classification problem setting} we saw that in the Gaussian case the error probability on a new noiseless point $(\bx, y)$ is $\Phi(-{\bmu^\top\bw}/{\|\bw\|_\bSigma})$. Since $\Phi(-z) \leq e^{-z^2/2}$ for every $z > 0$, to recover the result of Theorem \ref{th::chatterji2020linearnoise main result} we just need to show that ${\bmu^\top\bw}/{\|\bw\|_\bSigma} \gtrsim \|\bmu\|^2/\sqrt{p}$. Thus, the following proposition, whose proof can be found in Appendix \ref{sec::comparissons proofs appendix}, shows that our result is stronger than Theorem \ref{th::chatterji2020linearnoise main result}.
\begin{restatable}{proposition}{comparisonchatterjilinearnoise}
Assume that $\lambda_i \leq 1$ for any $i$ and $\sum_{i=1}^p \lambda_i \geq \kappa p$ for some constant $\kappa \in (0, 1]$. Take $k = 0$, $\lambda = 0$ and some $c > 1$. 
Suppose additionally that $\kappa p/n \geq \|\bmu\|^2 \geq (2ct)^2/(\kappa^2n)$, and $t^2 <  n\kappa$.

Then 
\[
\frac{N - ct\Diamond }{\left[1 + N\sigma_\eta \right]\sqrt{ V + t^2\Delta V} + \Diamond\sqrt{n}} \geq \frac{1}{10} \frac{\|\bmu\|^2\sqrt{n\kappa}}{\sqrt{p}}.
\]
\end{restatable}
That is, our lower bound picks up an additional factor of $\sqrt{n}$ compared to the bound from Theorem \ref{th::chatterji2020linearnoise main result}.

\subsection{Comparison with \texorpdfstring{\citep{muthukumar2021classification}}{(Muthukumar et al., 2021)}}
\label{sec::comparison with muthukumar2021classification}
Apart from the data generating model considered in this paper, there is another model that was recently considered in the literature on linear classification. Consider a centered Gaussian distribution with covariance $\bSigma$. When a point $\bxi$ is generated from this distribution, it gets assigned the label $\sign(\bxi^\top\balpha)$ for some vector $\balpha \in \R^p$. Thus, the domain is split into two clusters. It is easy to see that the centers of the clusters are 
\[
\E[\sign(\bxi^\top\balpha)\bxi] = \E[\sign(\bz^\top\bSigma^{1/2}\balpha)\bSigma^{1/2}\bz] = \bSigma^{1/2} \cdot \sqrt{\frac{2}{\pi}}\frac{\bSigma^{1/2}\balpha}{\|\bSigma^{1/2}\balpha\|} = \sqrt{\frac{2}{\pi\balpha^\top\bSigma\balpha}}\bSigma\balpha =: \boldm,
\]
where we used $\bz$ to denote a vector from the isotropic Gaussian distribution, and denoted the centers of the clusters as $\pm \boldm$, which plays the role of $\bmu$. The covariance within a cluster is not $\bSigma$, but a rank-one correction to it, namely
\begin{multline*}
\bSigma' = \E\left[(\bxi\sign(\bxi^\top\balpha) - \boldm)(\sign(\bxi^\top\balpha)\bxi)^\top\right] = \E[\bxi\bxi^\top] - \boldm \E[\sign(\bxi^\top\balpha)\bxi]^\top =\\ = \bSigma - \boldm\boldm^\top
= \bSigma -  \frac{2\bSigma\balpha\balpha^\top\bSigma}{\pi\balpha^\top\bSigma\balpha} = \bSigma^{1/2}\left(\bI_p - \frac{2}{\pi}\bP_{\bSigma^{1/2}\balpha}\right)\bSigma^{1/2},
\end{multline*}
where we denoted the projection on the direction of $\bSigma^{1/2}\balpha$ as $\bP_{\bSigma^{1/2}\balpha}$. Because of the factor $2/\pi < 1 $ in front of it, the matrix $\bI_p - \frac{2}{\pi}\bP_{\bSigma^{1/2}\balpha}$ is still within a constant factor of identity, so the covariance of a cluster is within a constant factor of $\bSigma$.

Moreover, for a classifier $\bxi \to \sign(\bxi^\top \bw)$, the probability to assign a wrong label is 
\[
\P\left(\sign(\bz^\top \bSigma^{1/2}\bw) \neq \sign(\bz^\top \bSigma^{1/2}\balpha)\right) = \frac{\angle(\bSigma^{1/2}\bw, \bSigma^{1/2}\balpha)}{\pi} = \frac{1}{\pi}\arccos\left(\frac{\balpha^\top\bSigma\bw}{\|\balpha\|_\bSigma \|\bw\|_\bSigma}\right),
\]
where we used $\angle(\cdot, \cdot)$ to denote the angle between two vectors. Note that the argument of $\arccos$ is almost the same as  the quantity $\bmu^\top\bw/\|\bw\|_\bSigma$ studied in this paper: indeed,  plugging in the formulas for the mean and the covariance of the cluster we obtain
\[
\frac{\boldm^\top\bw}{\sqrt{\bw\bSigma'\bw}} = \sqrt{\frac{2}{\pi}} \frac{\balpha^\top\bSigma\bw}{\sqrt{\balpha^\top\bSigma\balpha}\sqrt{\bw\bSigma'\bw}},
\]
and we also saw that  $(1 - 2/\pi)\bw\bSigma\bw \leq \bw\bSigma'\bw \leq \bw\bSigma\bw$.

Thus, in principle, the results in our paper can apply directly to this model. The caveat, however, is that our bounds are only defined up to a constant multiplier, while the quantity $\frac{\balpha^\top\bSigma\bw}{\|\balpha\|_\bSigma \|\bw\|_\bSigma}$ is always between minus one and one. For example, our bounds cannot distinguish between perfect classification and some constant probability of error that is less than $0.5$. 

Now let's compare our results with the result of \citet{muthukumar2021classification}, who consider such a model. The main result of \citet{muthukumar2021classification} is their Theorem~13, which considers the following construction: there are three non-negative real-valued parameters  $q, r, s$  such that $r < 1 < s$, $q < s-r$. The covariance is diagonal, that is, $\bSigma = \diag(\lambda_1, \dots, \lambda_p)$, and $p = n^s$.  The spectrum of $\bSigma$  has a bi-level structure, that is 
\begin{equation}
\label{eq::muthukumar bi-level spectrum}
\lambda_i = 
\begin{cases}
n^{s - q  -r}, & \text{for } i \leq n^r,\\
(1-n^{-q})/(1 - n^{r-s})&\text{for } i > n^r.
\end{cases}
\end{equation}
Finally, \citet{muthukumar2021classification} consider $\balpha = \be_1$ (note that, similarly to \citep{wang2021binary}, taking such $\balpha$ hides the effect of ``recovering the geometry," since $\balpha$ has the same direction as $\bSigma^{-1}\balpha$). For this choice of $\balpha$, the mean of the positive cluster becomes $\boldm = \sqrt{\frac{2}{\pi\balpha^\top\bSigma\balpha}}\bSigma\balpha = \sqrt{\frac{2\lambda_1}{\pi}}\be_1$.

\citet{muthukumar2021classification} consider the asymptotic setting with $n$ approaching infinity, and compute the classification error of the MNI. Namely, their Theorem 13 shows that if $q + r < (s+1)/2$ then the misclassification probability approaches zero, while for $q + r > (s+1)/2$ it approaches $0.5$. That is, the quantity  $\frac{\balpha^\top\bSigma\bw}{\|\balpha\|_\bSigma \|\bw\|_\bSigma}$ approaches 1 when $q + r < (s+1)/2$, and 0 if $q + r > (s+1)/2$.

Proposition \ref{prop::recover threshold from muthukumar2021classification} below shows that one can see the same phase transition in our results. However, as our bounds are only defined up to a constant multiplier, we do not recover the result of \citet{muthukumar2021classification} precisely.

\begin{restatable}{proposition}{comparisonmuthukumarclassification}
\label{prop::recover threshold from muthukumar2021classification}
Take real $q, r, s$  such that $0 \leq r < 1 < s$, $0 \leq q < s-r$. Consider $p = n^s$, $\bSigma = \diag(\lambda_1, \dots, \lambda_p)$, and $\bmu = \sqrt{2\lambda_1/\pi}\be_1$, where $\{\lambda_i\}_{i=1}^p$ are given by Equation \eqref{eq::muthukumar bi-level spectrum}. Take $\lambda = 0$, $k = n^r$, and $c$ to be any constant that doesn't depend on $n$.

Then, as $n$ goes to infinity, for $t < n^{0.499r}$  the following holds:
\begin{align*}
\frac{N - ct\Diamond}{\sqrt{V + t^2\Delta V} +  \sqrt{n}\Diamond} 
 = (1 + o_n(1))\frac{N}{\sqrt{V } +  \sqrt{n}\Diamond}  =&\begin{cases}
o_n(1), & 2q + 2r - 1 - s > 0,\\
\frac{1 + o_n(1)}{\sqrt{2\pi}}& 2q + 2r - 1 - s = 0,\\
\sqrt{\frac{2}{\pi}} + o_n(1) & 2q + 2r - 1 - s < 0.
\end{cases}
\end{align*}
Here we use $o_n(1)$ to  denote  quantities that converge to zero as $n$ goes to infinity.
\end{restatable}

\section{Conclusions and further directions}
\label{sec::conclusions}

In this paper we studied classification accuracy of the ridge regression solution in a binary classification problem. We derived tight bounds for the case without label-flipping noise, and a lower bound for the case with label-flipping noise. Our bounds are additionally supported by geometric derivations for the minimum norm interpolating solution, which explain the structure of the solution vector. Even though we don't provide a matching upper bound for the case with label-flipping noise, the geometric derivations show that the vector $\bQ\bA^{-1}\tilde{\by}$ plays an important role, thus suggesting that the term $\sigma_\eta\sqrt{V}$ should indeed appear in the bound for $\|\solnridge\|_\bSigma$, and that our bound is indeed tight.

Our bounds yield several novel qualitative conclusions. We discover the effect of ``recovering the geometry" in the first $k^*$ components, which was seemingly missed in the previous literature. For the setting without label-flipping noise, we show that there is no benefit (in a certain sense) of increasing regularization beyond the point where the (regularized) covariance obtains a tail of high effective rank, and that the optimal regularization can even be negative. When it comes to the case with label flipping noise and benign overfitting, we discover that the bound for this case exhibits the same behavior as the bound for the noiseless case, unless $\bmu$ is large in magnitude. In the latter case, our bound loses dependence on $\bmu$ completely, and the conditions for benign overfitting in this regime coincide with the conditions from the regression literature.

Despite all the above mentioned progress, there are still gaps in our understanding of benign overfitting in this model, which we leave for future work. The most 
obvious task is to obtain a matching upper bound for the case with label-flipping noise. As explained above, we believe that our bound should be tight, at least when $\eta$ is a constant. The dependence of the bound on $\eta$, however, is probably not sharp when $\eta$ becomes small. This is because our argument relies on sub-Gaussianity of a Bernoulli random variable with parameter $\eta$, but when that parameter is small, the Bernoulli random variable behaves as a heavy-tailed one. Thus, the argument using sub-Gaussianity may not be sharp.

Next, our argument only works if there exists $k$ for which the tail of the covariance has high effective rank. However, the bound that we obtained suggests that this structure may be necessary for benign overfitting to occur. Indeed, as we explained above, the sufficient conditions for benign condition that we obtain are very similar to those from the regression literature \citep{benign_overfitting}, and it was shown there that high effective rank of the tail is necessary for benign overfitting to occur. Proving the necessity of this regime is another direction of future work.

Finally, even though we use a very similar regime to that considered in the regression literature, and there are a lot of technical similarities between the results, we do not have a high-level explanation of benign overfitting that would unify the regression and classification settings.
Resolving this, and understanding when and how the noise that is interpolated in training does not impact classification accuracy, are important directions for future work.

\section*{Acknowledgements}
We gratefully acknowledge the support of the NSF for FODSI through grant DMS-2023505, of the NSF and the Simons Foundation for the Collaboration on the Theoretical Foundations of Deep Learning through awards DMS-2031883 and 814639, and of the ONR through MURI award N000142112431.

\bibliography{bib}

\appendix

\section{Formulas for the solution}
\label{sec::formulas for solution}
In this section we derive the explicit formula for $\solnridge$, which operates with the inverse of matrix $\bA$ instead of $\bX\bX^\top$. The version of this formula for the case of MNI solution $\mni$ with clean labels $\hat{\by} = \by$ already appeared in \cite{Cao_Gu_Belkin}, who, in turn, took it from \cite{wang2021binary}.

\begin{restatable}[Explicit formulas for MNI]{lemma}{solutionformulasmni}\label{lm::mni solution formulas}
Denote $\Delta\by := \hat{\by} - \by$.  Denote the projection of $\bmu$ on the orthogonal complement to the span of the columns of $\bQ^\top$ as $\bmu_\perp$, and take $\lambda = 0$. Denote
\[
S = (1 + \bnu^\top \bA^{-1}\by)^2 + \|\bmu_\perp\|^2\by^\top\bA^{-1}\by.
\]
Then
\begin{align*}
S\mni =& \left[(1 + \bnu^\top \bA^{-1}\by)^2 + \|\bmu_\perp\|^2\by^\top\bA^{-1}\by\right] \bQ^\top\bA^{-1}\Delta\by\\
+& \left[(1 + \bnu^\top \bA^{-1}\by)(1 - \bnu^\top\bA^{-1}\Delta\by) - \|\bmu_\perp\|^2\Delta\by^\top \bA^{-1}\by\right]\bQ^\top\bA^{-1}\by\\
+& \left[\by^\top \bA^{-1}\by + (1 + \bnu^\top \bA^{-1}\by)\Delta\by^\top \bA^{-1}\by -\by^\top\bA^{-1}\by\bnu^\top\bA^{-1}\Delta\by\right]\bmu_\perp,\\
S\bmu^\top\mni =&  \by^\top\bA^{-1}\hat{\by} \|\bmu_\perp\|^2 + (1 +  \bnu^\top\bA^{-1}\by) \bnu^\top \bA^{-1}\hat{\by}\\
=&  \by^\top\bA^{-1}\by \|\bmu_\perp\|^2 + (1 +  \bnu^\top\bA^{-1}\by) \bnu^\top \bA^{-1}\by\\
+&  \by^\top\bA^{-1}\Delta\by\|\bmu_\perp\|^2 
+ (1 + \bnu^\top \bA^{-1}\by)\bnu^\top\bA^{-1}\Delta\by.\\
\end{align*}

In particular, when $\hat{\by}=\by$
\begin{align*}
S\mni =& (1 + \bnu^\top \bA^{-1}\by)\bQ^\top\bA^{-1}\by + \by^\top \bA^{-1}\by \bmu_\perp,\\
S\bmu^\top\mni =&  \by^\top\bA^{-1}{\by} \|\bmu_\perp\|^2 + (1 +  \bnu^\top\bA^{-1}\by) \bnu^\top \bA^{-1}{\by}.
\end{align*}
\end{restatable}

\begin{proof}
We defined $\mni$ as $\bX^\top(\bX\bX^\top)^{-1}\hat{\by}$. Our goal is to derive a different formula, which would have inverse of $\bQ\bQ^\top$ instead of $\bX\bX^\top$.  This derivation could be made algebraically using the fact that $\bX\bX^\top$ is a low-rank correction to $\bQ\bQ^\top$ and applying Sherman-Morrison-Woodbury identity. Such a derivation would be very bulky, so we take another path here and derive the required formula from scratch using geometric considerations. 
We are going to use the fact that $\mni$ can be equivalently defined as the unique vector $\hat{\bw}$ that lies in the span of columns of $\bX^\top$ such that $\bX\hat{\bw} = \hat{\by}$. 

Denote the span of the columns of $\bQ^\top$ as $\cQ$, and the projector onto $\cQ$ as $\bP_\bQ :=  \bQ^\top \bA^{-1}\bQ$. For any $\bv \in \R^p$ denote the projection of $\bv$ on $\cQ$ as $\bv_\parallel$ and the projection of $\bv$ on $\cQ^\perp$ as $\bv_\perp$:
\[
\bv_\perp:= (\bI_p - \bP_\bQ)\bv, \quad \bv_\parallel:=  \bP_\bQ\bv = \bv - \bv_\perp.
\]

Consider any vector $\bw$ in the span of the columns of $\bX^\top$. The projection of this vector on $\cQ^\perp$ must be a scalar multiple of $\bmu_\perp$ because the projection of the $i^{\text{th}}$ column of $\bX^\top$ is $y_i \bmu_\perp$. That is, $\bw_\perp = \alpha \bmu_\perp$ for some scalar $\alpha$. Now let's answer the following question: which labels does $\bw$ give to data points? The part in $\cQ$ doesn't interact with $\bmu_\perp$ (they are orthogonal) and vice versa, so
\begin{align*}
\bX\bw =& (\bQ + \by\bmu_\parallel^\top)\bw_\parallel + \by\bmu_\perp^\top\bw_\perp\\
=& \bQ\bw_\parallel + (\bmu_\parallel^\top\bw_\parallel + \alpha\|\bmu_\perp\|^2)\by 
\end{align*}

Recall that we want to find the minimum norm interpolator for labels $\hat{\by}$, that is such $\hat{\bw}$ that
\[
 \bQ\hat{\bw}_\parallel + (\bmu_\parallel^\top\hat{\bw}_\parallel + \alpha\|\bmu_\perp\|^2)\by = \hat{\by}.
\]

Denote $\beta := \bmu_\parallel^\top\hat{\bw}_\parallel + \alpha\|\bmu_\perp\|^2$. We see that $\hat{\bw}_\parallel$ is such vector in $\cQ$ that $\bQ\hat{\bw}_\parallel = \hat{\by} - \beta\by$. Therefore, it is the minimum norm interpolator of labels $\hat{\by} - \beta\by$ with the data matrix $\bQ$ and we can use the formula for MNI to obtain
\[
\hat{\bw}_\parallel = \bQ^\top(\underbrace{\bQ\bQ^\top}_{\bA})^{-1}(\hat{\by} - \beta\by).
\]

Thus far, we learned that for some scalars $\alpha, \beta$ it holds that 
\begin{align*}
\hat{\bw}_\perp =& \alpha \bmu_\perp,\\
\hat{\bw}_\parallel =& \bQ^\top\bA^{-1}\hat{\by} - \beta\bQ^\top\bA^{-1}\by,\\
\beta =& \bmu_\parallel^\top\hat{\bw}_\parallel + \alpha\|\bmu_\perp\|^2\\
=&\bnu^\top\bA^{-1}\hat{\by} - \beta\bnu^\top\bA^{-1}\by + \alpha\|\bmu_\perp\|^2.
\end{align*}

There is, however, one more condition that we are missing: there is only one pair $\alpha, \beta$ that satisfies the relation above for which the vector $\bQ^\top\bA^{-1}\hat{\by} - \beta\bQ^\top\bA^{-1}\by + \alpha \bmu_\perp$ lies in the span of the columns of $\bX^\top$ --- the one with the minimal norm. Thus, we arrive at the following optimization problem in $\alpha, \beta$:
\begin{align*}
&\alpha^2\|\bmu_\perp\|^2 + \beta^2\by^\top\bA^{-1}\by - 2 \beta\hat{\by}^\top \bA^{-1}\by \to \min_{\alpha, \beta},\\
\text{s.t. }& \beta(1 + \bnu^\top \bA^{-1}\by) - \bnu^\top\bA^{-1}\hat{\by} = \alpha\|\bmu_\perp\|^2,
\end{align*}
where in the first line we wrote the expression for $\|\hat{\bw}\|^2 = \|\hat{\bw}_\parallel\|^2  + \|\hat{\bw}_\perp\|^2$ and dropped the term $\| \bQ^\top\bA^{-1}\hat{\by}\|^2$ which doesn't depend on $\alpha, \beta$.

To solve this problem we parameterize 
\[
\beta =t\|\bmu_\perp\|^2, \quad  \alpha = t(1 + \bnu^\top \bA^{-1}\by) - \|\bmu_\perp\|^{-2}\bnu^\top\bA^{-1}\hat{\by}.
\]
The optimization problem becomes to minimize the following in $t$:
\begin{multline*}
t^2(1 + \bnu^\top \bA^{-1}\by)^2\|\bmu_\perp\|^2 - 2t(1 + \bnu^\top \bA^{-1}\by)\bnu^\top\bA^{-1}\hat{\by} + t^2\|\bmu_\perp\|^4\by^\top\bA^{-1}\by - 2 t\|\bmu_\perp\|^2\hat{\by}^\top \bA^{-1}\by,
\end{multline*}
which is a simple minimization of a quadratic function in one variable. We get
\begin{align*}
t =& \frac{(1 + \bnu^\top \bA^{-1}\by)\bnu^\top\bA^{-1}\hat{\by} + \|\bmu_\perp\|^2\hat{\by}^\top \bA^{-1}\by}{(1 + \bnu^\top \bA^{-1}\by)^2\|\bmu_\perp\|^2 + \|\bmu_\perp\|^4\by^\top\bA^{-1}\by},\\
\beta = & \frac{(1 + \bnu^\top \bA^{-1}\by)\bnu^\top\bA^{-1}\hat{\by} + \|\bmu_\perp\|^2\hat{\by}^\top \bA^{-1}\by}{(1 + \bnu^\top \bA^{-1}\by)^2 + \|\bmu_\perp\|^2\by^\top\bA^{-1}\by},\\
\alpha =&\frac{(1 + \bnu^\top \bA^{-1}\by)\bnu^\top\bA^{-1}\hat{\by} + \|\bmu_\perp\|^2\hat{\by}^\top \bA^{-1}\by}{(1 + \bnu^\top \bA^{-1}\by)^2\|\bmu_\perp\|^2 + \|\bmu_\perp\|^4\by^\top\bA^{-1}\by}(1 + \bnu^\top \bA^{-1}\by) - \|\bmu_\perp\|^{-2}\bnu^\top\bA^{-1}\hat{\by}\\
=&\frac{(1 + \bnu^\top \bA^{-1}\by)^2\bnu^\top\bA^{-1}\hat{\by} + (1 + \bnu^\top \bA^{-1}\by)\|\bmu_\perp\|^2\hat{\by}^\top \bA^{-1}\by  }{(1 + \bnu^\top \bA^{-1}\by)^2\|\bmu_\perp\|^2 + \|\bmu_\perp\|^4\by^\top\bA^{-1}\by} \\
-& \frac{\bnu^\top\bA^{-1}\hat{\by}\bigl((1 + \bnu^\top \bA^{-1}\by)^2 + \|\bmu_\perp\|^2\by^\top\bA^{-1}\by\bigr)}{(1 + \bnu^\top \bA^{-1}\by)^2\|\bmu_\perp\|^2 + \|\bmu_\perp\|^4\by^\top\bA^{-1}\by}\\
=&\frac{ (1 + \bnu^\top \bA^{-1}\by)\hat{\by}^\top \bA^{-1}\by -\by^\top\bA^{-1}\by\bnu^\top\bA^{-1}\hat{\by} }{(1 + \bnu^\top \bA^{-1}\by)^2 + \|\bmu_\perp\|^2\by^\top\bA^{-1}\by}
\end{align*}

Recall that $\Delta\by := \hat{\by} - \by$. Using this notation
\[
\hat{\bw}_\parallel = \bQ^\top\bA^{-1}\Delta\by + (1 - \beta)\bQ^\top\bA^{-1}\by,\quad \hat{\bw}_\perp = \alpha \bmu_\perp.
\]
Denote
\[
S := (1 + \bnu^\top \bA^{-1}\by)^2 + \|\bmu_\perp\|^2\by^\top\bA^{-1}\by.
\]
We have
\begin{align*}
S(1 - \beta) =& (1 + \bnu^\top \bA^{-1}\by)(1 -  \bnu^\top\bA^{-1}\Delta\by) - \|\bmu_\perp\|^2\by^\top \bA^{-1}\Delta\by,\\
S\alpha =& (1 + \bnu^\top \bA^{-1}\by)\hat{\by}^\top \bA^{-1}\by -\by^\top\bA^{-1}\by\bnu^\top\bA^{-1}\hat{\by}.
\end{align*}
which gives the desired formula for $\hat{\bw}$. When it comes to $\bmu^\top\hat{\bw}$, we  directly compute the scalar product using the formula for $\hat{\bw}$:
\begin{align*}
S\bmu^\top\hat{\bw} =& \left[(1 + \bnu^\top \bA^{-1}\by)^2 + \|\bmu_\perp\|^2\by^\top\bA^{-1}\by\right] \bnu^\top\bA^{-1}\Delta\by\\
+& \left[(1 + \bnu^\top \bA^{-1}\by)(1 - \bnu^\top\bA^{-1}\Delta\by) - \|\bmu_\perp\|^2\Delta\by^\top \bA^{-1}\by\right]\bnu^\top\bA^{-1}\by\\
+& \left[\by^\top \bA^{-1}\by + (1 + \bnu^\top \bA^{-1}\by)\Delta\by^\top \bA^{-1}\by -\by^\top\bA^{-1}\by\bnu^\top\bA^{-1}\Delta\by\right]\|\bmu_\perp\|^2\\
=&\bnu^\top\bA^{-1}\Delta\by\cdot \bigg[
        (1 + \bnu^\top \bA^{-1}\by)^2 + \|\bmu_\perp\|^2\by^\top\bA^{-1}\by \notag \\
    &\quad -(1 + \bnu^\top \bA^{-1}\by)\bnu^\top\bA^{-1}\by 
        - \|\bmu_\perp\|^2\by^\top\bA^{-1}\by 
    \bigg]\\
+& \by^\top\bA^{-1}\Delta\by\cdot\left[(1 + \bnu^\top \bA^{-1}\by)\|\bmu_\perp\|^2 -  \bnu^\top \bA^{-1}\by\|\bmu_\perp\|^2\right]\\
+&  \by^\top\bA^{-1}\by \|\bmu_\perp\|^2 + (1 +  \bnu^\top\bA^{-1}\by) \bnu^\top \bA^{-1}\by\\
=& \by^\top\bA^{-1}\by \|\bmu_\perp\|^2 + (1 +  \bnu^\top\bA^{-1}\by) \bnu^\top \bA^{-1}\by\\
+& \bnu^\top\bA^{-1}\Delta\by\cdot(1 + \bnu^\top \bA^{-1}\by) +  \by^\top\bA^{-1}\Delta\by\cdot \|\bmu_\perp\|^2.
\end{align*}
\end{proof}

\begin{restatable}[Explicit formulas for the ridge solution]{lemma}{solutionformulas}\label{lm::solution formulas}
Denote 
\begin{align*}
\Delta\by :=& \hat{\by} - \by,\\
\muperpridge :=& (\bI_p - \bQ^\top\bA^{-1}\bQ)\bmu,\\
S :=& (1 + \bnu^\top \bA^{-1}\by)^2 + \bmu^\top\muperpridge\by^\top\bA^{-1}\by.
\end{align*}
Then for any $\lambda$ such that the matrix $\bA$ is PD the following holds:
\begin{align*}
\label{eq::ridge solution formulas}
\begin{split}
S\solnridge =& \left[(1 + \bnu^\top \bA^{-1}\by)^2 + \by^\top\bA^{-1}\by\bmu^\top\muperpridge\right] \bQ^\top\bA^{-1}\Delta\by\\
+& \left[(1 + \bnu^\top \bA^{-1}\by)(1 - \bnu^\top\bA^{-1}\Delta\by) - \Delta\by^\top \bA^{-1}\by\bmu^\top\muperpridge\right]\bQ^\top\bA^{-1}\by\\
+& \left[\by^\top \bA^{-1}\by + (1 + \bnu^\top \bA^{-1}\by)\Delta\by^\top \bA^{-1}\by -\by^\top\bA^{-1}\by\bnu^\top\bA^{-1}\Delta\by\right]\muperpridge,\\
S\bmu^\top\solnridge =&  \by^\top\bA^{-1}\hat{\by}\bmu^\top\muperpridge + (1 +  \bnu^\top\bA^{-1}\by) \bnu^\top \bA^{-1}\hat{\by}\\
=&  \by^\top\bA^{-1}\by \bmu^\top\muperpridge + (1 +  \bnu^\top\bA^{-1}\by) \bnu^\top \bA^{-1}\by\\
+& \bnu^\top\bA^{-1}\Delta\by\cdot(1 + \bnu^\top \bA^{-1}\by) +  \by^\top\bA^{-1}\Delta\by\cdot \bmu^\top\muperpridge.
\end{split}
\end{align*}

In particular, when $\hat{\by}=\by$
\begin{align}
S\solnridge =& (1 + \bnu^\top \bA^{-1}\by)\bQ^\top\bA^{-1}\by + \by^\top \bA^{-1}\by\muperpridge,\\
S\bmu^\top\solnridge =&  (1 +  \bnu^\top\bA^{-1}\by) \bnu^\top \bA^{-1}{\by} + \by^\top\bA^{-1}{\by} \bmu^\top\muperpridge.
\end{align}
\end{restatable}
\begin{proof}
First of all, we obtain formulas for the particular case when $\lambda = 0$ in Lemma \ref{lm::mni solution formulas} below. These formulas can be extended to the case of positive $\lambda$ by a standard trick.  Recall the definitions:
\begin{align*}
\solnridge &= \bX^\top(\bX\bX^\top + \lambda \bI_n)^{-1}\hat{\by},\\
\mni &= \bX^\dagger\hat{\by} = \bX^\top(\bX\bX^\top )^{-1}\hat{\by}.
\end{align*}
Ridge solution can be obtained from the MNI solution with augmented data, namely denote
\[
\check{\bQ} := [\bQ, \sqrt{\lambda}\bI_n], \quad \check{\bmu} := \begin{pmatrix} \bmu\\ \bzero_n \end{pmatrix}.
\]
and
\[
\check{\bX} := \check{\bQ} + \by\check{\bmu}^\top = [\bX, \sqrt{\lambda}\bI_n].
\]
Now  MNI solution for the augmented data becomes
\[
\mnicheck = \check{\bX}^\dagger\hat{\by} = \begin{pmatrix} \bX^\top\\ \sqrt{\lambda}\bI_n \end{pmatrix}(\bX\bX^\top + \lambda \bI_n)^{-1}\hat{\by},
\]
that is, $\solnridge$ is equal to the first $p$ coordinates of $\mnicheck$. Moreover, note that $\bmu^\top\solnridge = \check{\bmu}^\top \mnicheck$. To apply Lemma \ref{lm::mni solution formulas} and obtain the formula for $\mnicheck$ and $\check{\bmu}^\top \mnicheck$ we need to plug in the following objects instead of $\bQ, \bA, \bnu$, and $\bmu_\perp$ correspondingly:
\begin{align*}
\check{\bQ} :=& [\bQ, \sqrt{\lambda}\bI_n],\\
\check{\bA} :=& \check{\bQ} \check{\bQ}^\top = \bQ{\bQ}^\top + \lambda \bI_n = \bA,\\
\check{\bnu} :=& \check{\bQ}\check{\bmu} = \bQ\bmu = \bnu,\\
\check{\bmu}_\perp :=& (\bI_{p+n} - \check{\bQ}^\top \check{\bA}^{-1} \check{\bQ})\check{\bmu} = \begin{pmatrix} (\bI_p - \bQ^\top\bA^{-1}\bQ)\bmu\\ -\sqrt{\lambda}{\bA}^{-1}\bnu\end{pmatrix} =  \begin{pmatrix} \muperpridge\\ -\sqrt{\lambda}{\bA}^{-1}\bnu\end{pmatrix}
\end{align*}
The only thing that is not straightforward to plug in is $\|\check{\bmu}_\perp\|^2$, which we derive next:
\begin{align*}
\|\check{\bmu}_\perp\|^2 =& \| (\bI_p - \bQ^\top\bA^{-1}\bQ)\bmu\|^2 + \lambda \|{\bA}^{-1}\bnu\|^2\\
=& \|\bmu\|^2 - 2\bnu^\top\bA^{-1}\bnu + \bnu^\top \bA^{-1}\bQ\bQ^\top\bA^{-1}\bnu + \lambda \bnu^\top\bA^{-1}p\bA^{-1}\bnu\\
=&\|\bmu\|^2 - 2\bnu^\top\bA^{-1}\bnu + \bnu^\top \bA^{-1}(\underbrace{\bQ\bQ^\top + \lambda \bI_n}_{\bA})\bA^{-1}\bnu\\
=& \|\bmu\|^2 - \bnu^\top\bA^{-1}\bnu\\
=& \bmu^\top\muperpridge.
\end{align*}

Now we can obtain the result for $\lambda \geq 0$: Plugging all those objects in Lemma \ref{lm::mni solution formulas} gives the formulas for $\mnicheck$ and  $\check{\bmu}^\top \mnicheck = \bmu^\top\solnridge$. The formula for $\solnridge$ is then obtained from $\mnicheck$ by trimming the last $n$ coordinates.

Finally, to extend the result to the case of negative $\lambda$ note that the expressions on the both sides of equations in \eqref{eq::ridge solution formulas} are analytic functions of $\lambda$ on the domain $\{\lambda \in \mathbb{C}: \Re(\lambda) > -\mu_n(\bQ\bQ^\top)\}$. Since those equations hold on $\{\lambda \in \R: \lambda > 0\}$ they coincide on that whole domain, in particular for $\{\lambda \in \R: \lambda > -\mu_n(-\mu_n(\bQ\bQ^\top))\}$.
\end{proof}

\section{General probabilistic results}
First of all, we define sub-Gaussian norms of random variables and vectors. 
\begin{definition}
\label{def::sub-Gaussianity}
For any centered random variable $v$ we define its sub-Gaussian norm as 
\[
\|v\|_{\psi_2}:=\inf\left\{t>0:\E\exp(v^2/t^2)\le 2\right\}.
\]
For any random vector $\bv$ in $\R^p$ we define its sub-Gaussian norm as
\[
\|\bv\|_{\psi_2}:= \sup_{\bu \in \R^p: \|\bu\|=1} \|\bu^\top \bv\|_{\psi_2}.
\]

If $\|\bv\|_{\psi_2} \leq \sigma$, we say that the distribution of $\bv$ is $\sigma$-sub-Gaussian.
\end{definition}

\begin{lemma}
\label{lm::subgaussian constant of noise}
Consider a random variable $\xi$ such that 
\[
\eta/2 = \P(\xi=1) = \P(\xi = -1) = (1 - \P(\xi = 0))/2.
\]
Then 
\begin{align*}
\|\xi\|_{\psi_2} =& 1/\sqrt{\ln(1 + 1/\eta)} \leq 1/\sqrt{\ln\frac{3+\eta^{-1}}{2}},\\
\|\xi^2 - \eta\|_{\psi_2} \leq& 1/\sqrt{\ln\frac{3+\eta^{-1}}{2}}.
\end{align*}
\end{lemma}
\begin{proof}
By Definition \ref{def::sub-Gaussianity}, since $\xi$ is a centered random variable
\[
\|\xi\|_{\psi_2}:=\inf\left\{t>0:\E\exp(\xi^2/t^2)\le 2\right\}.
\]
We write  out
\begin{align*}
\E\exp(\xi^2/t^2) =& \eta e^{t^{-2}} + (1 - \eta) \leq 2,
\end{align*}
and see that it is equivalent to $e^{t^{-2}} \leq 1 + 1/\eta.$ Thus, $\|\xi\|_{\psi_2} = 1/\sqrt{\ln(1 + 1/\eta)}.$ 

Now let's do the same for $\xi^2 - \eta$: we need to find some $t$ such that
\begin{align*}
\E\exp\Bigl((\xi^2 - \eta)^2/t^2\Bigr) =& \eta e^{(1-\eta)^2/t^2} + (1 - \eta)e^{\eta^2/t^2} \leq 2.
\end{align*}
Let's find such $t$ that a stronger condition is satisfied, namely
\begin{align*}
e^{\eta^2/t^2} \leq& \frac32,\\
\eta e^{(1-\eta)^2/t^2} \leq \eta e^{1/t^2}\leq& \frac12 + \frac32\eta.
\end{align*}
We  take 
\[
t^{-2} = \min\left(\eta^{-2}\ln\frac32, \ln\frac{3+\eta^{-1}}{2}\right).
\]
Since $\eta^{-1} \geq 1$, we have
\[
\eta^{-2}\ln\frac32 = \ln(3e^{\eta^{-2}}/2) \geq \ln(3(1 + \eta^{-2})/2) \geq \ln(3(1 + \eta^{-1})/2), 
\]
so $\|\xi^2 - \eta\|_{\psi_2} \leq t = 1/\sqrt{\ln\frac{3+\eta^{-1}}{2}}.$

Finally, we compare two bounds that we obtained:
\[
1 + \eta^{-1} - \frac{3 + \eta^{-1}}{2} = \frac{\eta^{-1}-1}{2} \geq 0.
\]
We see that $1/\sqrt{\ln(1 + 1/\eta)} \leq 1/\sqrt{\ln\frac{3+\eta^{-1}}{2}}$.
\end{proof}

\begin{lemma}
\label{lm::subgauss norm no variance}
Suppose that $\{\eta_i\}_{i=1}^n$ are i.i.d. centered random variables with sub-Gaussian norms $\sigma$. Then for some absolute constant $c > 0$ and any $t> 0$ with probability at least $1 - 2e^{-t^2/c}$
\[
\sqrt{\sum_i \eta_i^2} \leq \sigma(\sqrt{n} + t)
\]
\end{lemma}
\begin{proof}
We basically repeat the proof of Theorem 3.1.1 from   \cite{vershynin_hdp}, but we don't use the assumption that $\{\eta_i\}_{i=1}^n$ have unit variances.

Without loss of generality we can assume that $\sigma = 1$. Indeed, if $\sigma \neq 1$ we can just work with random variables $\{\eta_i/\sigma\}_{i=1}^n$ instead of $\{\eta_i\}_{i=1}^n$.

Denote $v = \sqrt{\E[\eta_1^2]}$ --- standard deviation for $\{\eta_i\}_{i=1}^n$. Recall (or note) that $v \leq \sigma \leq 1$. 

As in the proof of Theorem 3.1.1 from   \cite{vershynin_hdp}, we get that random variables $\{\eta_i^2\}_{i=1}^n$ are sub-Exponential, with sub-Exponential norms bounded by an absolute constant. Applying Bernstein's inequality (Corollary 2.8.3) from  \cite{vershynin_hdp}, we get that for some absolute constant $c > 0$ and any $u \geq 0$ with probability at least $1 - 2\exp(-cn(u\wedge u^2))$
\[
n^{-1}\sum_i \eta_i^2 \leq v^2 + u \leq 1 + u \leq (1 + (\sqrt{u} \wedge u))^2.
\]
Finally, we make a change of variables: $t = \sqrt{n(u\wedge u^2)} = \sqrt{n}(\sqrt{u}\wedge u)$, and get that with probability at least $1 - 2e^{-ct^2}$
\[
\sqrt{n^{-1}\sum_i \eta_i^2} \leq 1 + t/\sqrt{n}.
\]

\end{proof}

\begin{lemma}[Weakened Hanson-Wright inequality]
\label{lm::quadratic form concentration}
Suppose $\bM \in \R^{n\times n}$ is  a (random) matrix and $\beps \in \R^n$ is a centered vector whose components $\{\eps_i\}_{i=1}^n$ are independent, have  variances $v^2$ and sub-Gaussian norms at most $\sigma$. If $\beps$ is independent from $\bM$, then for some absolute constant $c$ and any $s > 0$ 
\[
\P\left\{|\beps^\top \bM \beps - v^2\tr(\bM) | > \sigma^2\max(\sqrt{s}\|\bM\|_F,s\|\bM\|)\right\} \leq 2\exp\left\{-s/c\right\}.
\]

\end{lemma}
\begin{proof}
This is basically a rewriting of Theorem 6.2.1 (Hanson-Wright inequality) in \cite{vershynin_hdp}. According to that theorem,  for some absolute constant $c$  for any $t > 0$,
\[
\P\left\{|\beps^\top \bM \beps  - \E_{\beps}\beps^\top \bM \beps | \geq t\right\} \leq 2\exp\left(-c^{-1}\min\left\{\frac{t^2}{\|\bM\|_F^2\sigma^4 }, \frac{t}{\|\bM\|\sigma^2 }\right\}\right),
\]
where $\E_\eps$ denotes expectation over $\beps$.

Since for any $i$, $\E\eps_i = 0$, and $\Var(\eps_i) = v^2$, we have
\[
\E\beps^\top \bM \beps = v^2\tr(\bM).
\]
Plug in $t = \sigma^2\max(\sqrt{s}\|\bM\|_F,s\|\bM\|) $, and note that $\frac{t^2}{\|\bM\|_F^2\sigma^4 }\geq s$ and $\frac{t}{\|\bM\|\sigma^2 }\geq s$:
\[
\P\left\{|\beps^\top \bM \beps - v^2\tr(\bM) | \geq \sigma^2\max(\sqrt{s}\|\bM\|_F,s\|\bM\|)\right\} \leq 2\exp\left\{-c^{-1}s\right\}.
\]

\end{proof}

\begin{corollary}[Weakened Hanson-Wright for PSD matrices]
\label{cor::psd quadratic form concentration}
In the setting of Lemma \ref{lm::quadratic form concentration} assume that $\bM$ is almost surely PSD. Then for some absolute constant $c > 0$ and any $s > 0$
\[
\P\left\{\beps^\top \bM \beps  > c\sigma^2(\tr(\bM) + s\|\bM\|)\right\} \leq 2\exp\left\{-s/c\right\}.
\]
\end{corollary}
\begin{proof}
 We just need to transform the result of Lemma \ref{lm::quadratic form concentration} using the fact that $\bM$ is PSD. Note that this fact implies that $\|\bM\|_F^2 \leq \tr(\bM)\|\bM\|$ so we obtain that with probability at least $1 - 2\exp\left\{-c_1^{-1}s\right\}$
\[
|\beps^\top \bM \beps  - v^2\tr(\bM) | \leq  \sigma^2\max(\sqrt{s\|\bM\|\tr(\bM)},s\|\bM\|),
\]
where $c_1$ is the constant from Lemma \ref{lm::quadratic form concentration}. Now on the same even we can write
\begin{align*}
\beps^\top \bM \beps \leq& v\tr(\bM) + \sigma^2\sqrt{s\|\bM\|\tr(\bM)} + s\|\bM\|\\
\leq& \sigma^2(\tr(\bM) + \sqrt{s\|\bM\|\tr(\bM)} + s\|\bM\|)\\
\leq& \frac{3}{2}\sigma^2(\tr(\bM) +  s\|\bM\|),
\end{align*}
where we used the fact that $\sigma \geq v$ (sub-Gaussian norm is greater or equal to variance for any centered distribution) in the second line, and AM-GM inequality $2\sqrt{s\|\bM\|\tr(\bM)} \leq \tr(\bM) + s\|\bM\|$ in the last line.

Taking $c$ large enough depending on $c_1$ yields the result.
\end{proof}

\begin{lemma}[Lemma 24 from \cite{BO_ridge}]
\label{lm::subgauss Gram spectral norm}
Suppose that $\tilde{\bZ} \in \R^{n \times p}$ is a matrix with i.i.d. isotropic sub-Gaussian rows with sub-Gaussian constant $\sigma$. Suppose that $\bM \in \R^{p \times p}$ is a symmetric PSD matrix that is independent of $\tilde{\bZ}$. Then there exists an absolute constant $c$ such that for any $t > 0$ with probability at least $1-6e^{-t/c}$
\[
\|\tilde{\bZ}\bM\tilde{\bZ}^\top\|\leq c\sigma^2\bigl(\|\bM\|(t+n) + \tr(\bM)\bigr).
\]
\end{lemma}
\begin{corollary}
\label{cor::Q Sigma Q spectral norm}
There exists a constant $c$ that only depends on $\sigma_x$ such that with probability at least $ce^{-n/c}$
\[
\|\bQ_\ktoinf\bSigma_\ktoinf\bQ_\ktoinf^\top \| \leq c\left(\sum_{i > k}\lambda_i^2 + n\lambda_{k+1}^2\right).
\]
\end{corollary}
\begin{proof}
Note that $\bQ_\ktoinf\bSigma_\ktoinf\bQ_\ktoinf^\top = \bZ_\ktoinf\bSigma_\ktoinf^2\bZ_\ktoinf^\top$, apply Lemma \ref{lm::subgauss Gram spectral norm} for $\tilde{\bZ} = \bZ_\ktoinf$ and $\bM = \bSigma_\ktoinf^2$.
\end{proof}

\begin{lemma}
\label{lm::scalar prod const prob lower}
Consider $\by\in \{-1, 1\}^n$ --- random vector with i.i.d. Rademacher coordinates. Suppose that $\bv \in \R^n$ is independent from $\by$. Then for some absolute constant $c$ with probability at least $c^{-1}$
\[
|\bv^\top \by| \geq c^{-1}\|\bv\|.
\]
\end{lemma}
\begin{proof}
Since $\by$ is a vector with centered independent coordinates that have constant sub-Gaussian norms, the random variable $\xi := \bv^\top \by/\|\bv\|$ has sub-Gaussian norm at most $c_1$, where $c_1$ is an absolute constant.

Thus, for some absolute constant $c_2$ we and any $t > 0$
\[
\P(\xi > t) \leq 2e^{-t^2/c_2}. 
\]

The idea is to consider variance $\E\xi^2 = 1$. Since the tails of the random variable $\xi$ decay very fast, only a small fraction of that variance can come from the tail, which means that most of it must come from a segment of constant length, from which it is easy to deduce the bound by Markov's inequality. 

Formally, we can write for any  $c_3$ and $c_4 > c_3$
\begin{align*}
1 = &\E[\xi^2]\\
=& \int_{0}^\infty \P(|\xi|^2 > t)\, dt\\
\leq& \int_0^{c_3}1\,dt + \int_{c_3}^{c_4}\P(|\xi|^2 > t)\,dt + 2\int_{c_4}^\infty e^{-t/c_2}\,dt\\
\leq & c_3 + (c_4 - c_3)\P(\xi^2 > c_3) + 2c_2e^{-c_4/c_2}.
\end{align*}
We see that 
\[
\P(\xi^2 > c_3) \geq \frac{1 - c_3 - 2c_2e^{-c_4/c_2}}{c_4}
\]
Taking $c_4$ to be a large enough absolute constant, and $c_3$ --- small enough, yields the result.
\end{proof}

\begin{lemma}
\label{lm::quad form const prob lower}
Consider $\by\in \{-1, 1\}^n$ --- random vector with i.i.d. Rademacher coordinates. Suppose that $\bM \in \R^{n\times n}$ is a matrix that is independent from $\by$ and almost surely PSD. Then for some absolute constant $c$ with probability at least $c^{-1}$
\[
|\by^\top \bM \by| \geq c^{-1}\tr(\bM).
\]
\end{lemma}
\begin{proof}
First of all, since $\by$ has i.i.d. centered coordinates with sub-Gaussian norms bounded by an absolute constant, by Corollary \ref{cor::psd quadratic form concentration}  for some absolute constant $c_1$ and any $s > 0$
\[
\P(\by^\top\bM\by > c_1(\tr(\bM) + s\|\bM\|)) \leq 2e^{-s/c_1}.
\]

 Denote $\xi = \by^\top\bM\by /\tr(\bM)$. Recall that our goal is to to show that $\P(\xi > c^{-1}) > c^{-1}$.

Note that $\|\bM\|\leq \tr(\bM)$ since $\bM$ is PSD. Thus, it follows from the above that for any $s > 0$
\[
\P(\xi > c_1(1 + s)) \leq 2e^{-s/c_1}.
\]
For further convenience we rewrite that as follows: for any $t > c_1$
\[
\P(\xi >t) \leq 2e^{-(t/c_1 - 1)/c_1}.
\]

Now we follow the same strategy as in the proof of Lemma \ref{lm::scalar prod const prob lower}. We write for some small $c_2$,  and large $c_3 > c_1$
\begin{align*}
1 = &\E[\xi]\\
=& \int_{0}^\infty \P(|\xi| > t)\, dt\\
\leq& \int_0^{c_2}1\,dt + \int_{c_2}^{c_3}\P(|\xi| > t)\,dt + 2\int_{c_3}^\infty e^{-(t/c_1 - 1)/c_1}\,dt\\
\leq & c_2 + (c_3 - c_2)\P(\xi > c_2) + 2c_1e^{-(c_3/c_1 - 1)/c_1}.
\end{align*}
\[
\P(\xi > c_2) \geq \frac{1 - c_2 - 2c_1e^{-(c_3/c_1 - 1)/c_1}}{c_3 - c_2}.
\]
Taking $c_2$ to be a small enough constant, and $c_3$ --- large enough, yields the result.
\end{proof}

\section{Some important relations}
\label{sec::important relations}

\importantrelations*
\begin{proof}

For the first inequality, we write
\begin{align*}
\Diamond^2 =& n\Lambda^{-2}\|\bmu_\ktoinf\|_{\bSigma_\ktoinf}^2 + n^{-1}\left\|\left(\Lambda n^{-1}\bSigma_\uptok^{-1} + \bI_k\right)^{-1}\bSigma_\uptok^{-1/2}\bmu_\uptok\right\|^2\\
\leq& n\Lambda^{-2}\lambda_{k+1}\|\bmu_\ktoinf\|^2 + n^{-1}\left\|\left(\Lambda n^{-1}\bSigma_\uptok^{-1} + \bI_k\right)^{-1/2}\bSigma_\uptok^{-1/2}\bmu_\uptok\right\|^2\\
\leq&\Lambda^{-1}\|\bmu_\ktoinf\|^2 + n^{-1}\left\|\left(\Lambda n^{-1}\bSigma_\uptok^{-1} + \bI_k\right)^{-1/2}\bSigma_\uptok^{-1/2}\bmu_\uptok\right\|^2\\
=& \Lambda^{-1}M,
\end{align*}
where we used that $\left\|\left(\Lambda n^{-1}\bSigma_\uptok^{-1} + \bI_k\right)^{-1}\right\| \leq 1$ in the second line, and $\Lambda > n\lambda_{k+1}$ in the third line. Alternatively, we could use $\left\|\left(\Lambda n^{-1}\bSigma_\uptok^{-1} + \bI_k\right)^{-1}\right\| \leq n\Lambda^{-1}\lambda_1$ in the second line to obtain
\begin{align*}
\Diamond^2 \leq& n\Lambda^{-2}\lambda_{k+1}\|\bmu_\ktoinf\|^2 + n\Lambda^{-1}\lambda_1 \cdot n^{-1}\left\|\left(\Lambda n^{-1}\bSigma_\uptok^{-1} + \bI_k\right)^{-1/2}\bSigma_\uptok^{-1/2}\bmu_\uptok\right\|^2\\
\leq&n\Lambda^{-2}\lambda_1\left(\|\bmu_\ktoinf\|^2 +  n^{-1}\Lambda\left\|\left(\Lambda n^{-1}\bSigma_\uptok^{-1} + \bI_k\right)^{-1/2}\bSigma_\uptok^{-1/2}\bmu_\uptok\right\|^2\right)\\
=& n\Lambda^{-2}\lambda_1M,
\end{align*}
which means that 
\[
\Diamond^2 \leq \Lambda^{-1}M \wedge n\Lambda^{-2}\lambda_1M = \left(\frac{1}{\sqrt{n}} \wedge \frac{\sqrt{n}\lambda_1}{\Lambda}\right)\sqrt{n}\Lambda^{-1}M \leq \Lambda^{-1}M\sqrt{n\Delta V}.
\]

Now let's upper bound $V$:
\begin{align*}
V =&  n^{-1}\tr\left(\left(\Lambda n^{-1}\bSigma_\uptok^{-1} + \bI_k\right)^{-2}\right)+ \Lambda^{-2}n\sum_{i > k}\lambda_i^2\\
\leq& k/n + \Lambda^{-2}n\sum_{i > k}\lambda_i^2\\
\leq& 2,
\end{align*}
where we used Equation \eqref{eq::assumption on k for important relations} to make the second transition, and we used the fact that $\left(\Lambda n^{-1}\bSigma_\uptok^{-1} + \bI_k\right)^{-2}$ is a $k \times k$ symmetric matrix, all eigenvalues of which are in $(0, 1)$.

When it comes to $\Delta V$, we write
\begin{align*}
n\Delta V \leq& 1 + \frac{n^2\lambda_{k+1}^2 + n\sum_{i > k}\lambda_i^2}{\Lambda^2}\leq 3,
\end{align*}
where the last transition follows directly from Equation \eqref{eq::assumption on k for important relations}.

Finally, let's compare $V$ and $\Delta V$. In case $k = 0$ we get
\[
4V = 4\Lambda^{-2}n\sum_i\lambda_i^2 \geq \Lambda^{-2}\left(2n\lambda_1^2 + \sum_i\lambda_i^2\right) = \Delta V.
\]
If $k > 0$, we have
\begin{align*}
4V =& 4n^{-1}\tr\left(\left(\Lambda n^{-1}\bSigma_\uptok^{-1} + \bI_k\right)^{-2}\right)+ 4\Lambda^{-2}n\sum_{i > k}\lambda_i^2\\
\geq& 4n^{-1}\frac{1}{(1 + \Lambda n^{-1}\lambda_1^{-1})^2} + \Lambda^{-2}n \lambda_{k+1}^2 + \Lambda^{-2}\sum_{i > k}\lambda_i^2\\
\geq& n^{-1}\frac{1}{(1 \vee \Lambda n^{-1}\lambda_1^{-1})^2} + \Lambda^{-2}n \lambda_{k+1}^2 + \Lambda^{-2}\sum_{i > k}\lambda_i^2\\
=& \Delta V.
\end{align*}
\end{proof}

\boundsviakstar*
\begin{proof}
First of all, let's compare $\Lambda$ and $\Lambda_*$. Since $k^* \leq k$, we obviously have $\Lambda_* \geq \Lambda$. On the other hand,
\begin{align*}
\Lambda_* =& \lambda + \sum_{i = k^* + 1}^k \lambda_i + \sum_{i > k}\lambda_i\\
\leq& \lambda + (k - k^*) \lambda_{k^* + 1} + \sum_{i > k}\lambda_i\\
\leq& \frac{k - k^*}{n}\Lambda_* + \Lambda\\
\leq& \frac12 \Lambda_* + \Lambda.
\end{align*}
Therefore, $\Lambda_* \leq 2\Lambda$.

Suppose that $k^* \neq 0$ (we will deal with the case $k^* = 0$ separately in the end. Let's show that $k^*$ is the ``the place where the transition happens", more precisely $\lambda_i \leq n^{-1}\Lambda_*$ for $i > k^*$ and $\lambda_i \geq n^{-1}\Lambda_*$ for $i \leq k^*$. Indeed, the first of those inequalities follows from the definition of $k^*$, and for the second we can write
\begin{gather*}
n\lambda_i \geq n\lambda_{k^*} > \lambda + \sum_{i \geq k^*}\lambda_i \geq \Lambda_*,
\end{gather*}
where the second inequality also follows from the definition of $k^*$. 
Combining with the fact that $\Lambda \leq \Lambda_*$, we also obtain that $\lambda_i \geq n^{-1}\Lambda$ for $i \leq k^*$.

Now, let's prove the remaining relations one-by-one.
\begin{enumerate}
\item $n\Lambda^{-1}M$ vs $n\Lambda_*^{-1}M_*$. 
\begin{align*}
n\Lambda^{-1}M &= \left\|\left(\Lambda n^{-1}\bSigma_\uptok^{-1} + \bI_k\right)^{-1/2}\bSigma_\uptok^{-1/2}\bmu_\uptok\right\|^2 + n\Lambda^{-1}\|\bmu_\ktoinf\|^2\\
&= \sum_{i=1}^k \frac{\mu_i^2}{\lambda_i(1 + \lambda_i^{-1}n^{-1}\Lambda)} + n\Lambda^{-1}\sum_{i > k} \mu_i^2\\
&= \sum_{i=1}^{k^*} \frac{\mu_i^2}{\lambda_i + n^{-1}\Lambda} + \sum_{i=k^* + 1}^{k} \frac{\mu_i^2}{\lambda_i + n^{-1}\Lambda} + \sum_{i> k} \frac{\mu_i^2}{n^{-1}\Lambda}\\
&\begin{cases}
\geq  \sum_{i=1}^{k^*} \frac{\mu_i^2}{2\lambda_i} + \sum_{i=k^* + 1}^{k} \frac{\mu_i^2}{2n^{-1}\Lambda_*} + \sum_{i> k} \frac{\mu_i^2}{n^{-1}\Lambda_*},\\
\leq \sum_{i=1}^{k^*} \frac{\mu_i^2}{\lambda_i} + \sum_{i=k^* + 1}^{k} \frac{\mu_i^2}{n^{-1}\Lambda_*/2} + \sum_{i> k} \frac{\mu_i^2}{n^{-1}\Lambda_*/2},
\end{cases}
\end{align*}
where we plugged in $n^{-1}\Lambda \leq \lambda_i$ for $i \leq k^*$, $\lambda_i \leq n^{-1}\Lambda_*$ for $i > k^*$, and $\Lambda_* \geq \Lambda \geq \Lambda_*/2$ in the last transition.

Since
\[
n\Lambda_*^{-1}M_* = \sum_{i=1}^{k^*} \frac{\mu_i^2}{\lambda_i} + \sum_{i> k^*} \frac{\mu_i^2}{n^{-1}\Lambda_*},
\]
 the above implies that $2n\Lambda_*^{-1}M_* \geq n\Lambda^{-1}M \geq n\Lambda_*^{-1}M_*/2$.

\item $\Diamond$ vs $\Diamond_*$.
\begin{align*}
n\Diamond^2 =& \left\|\left(\Lambda n^{-1}\bSigma_\uptok^{-1} + \bI_k\right)^{-1}\bSigma_\uptok^{-1/2}\bmu_\uptok\right\|^2 + n^2\Lambda^{-2}\|\bmu_\ktoinf\|_{\bSigma_\ktoinf}^2\\
=& \sum_{i=1}^k \frac{\mu_i^2}{\lambda_i(1 + \lambda_i^{-1}n^{-1}\Lambda)^2} + n^2\Lambda^{-2}\sum_{i > k}\lambda_i \mu_i^2\\
=& \sum_{i=1}^{k^*} \frac{\lambda_i\mu_i^2}{(\lambda_i + n^{-1}\Lambda)^2} + \sum_{i=k^* + 1}^k \frac{\lambda_i\mu_i^2}{(\lambda_i + n^{-1}\Lambda)^2} + \sum_{i > k}\frac{\lambda_i \mu_i^2}{(n^{-1}\Lambda)^2}\\
&\begin{cases}
\geq \sum_{i=1}^{k^*} \frac{\lambda_i\mu_i^2}{(2\lambda_i)^2} + \sum_{i=k^* + 1}^k \frac{\lambda_i\mu_i^2}{(2n^{-1}\Lambda_*)^2} + \sum_{i > k}\frac{\lambda_i \mu_i^2}{(n^{-1}\Lambda_*)^2},\\
\leq\sum_{i=1}^{k^*} \frac{\lambda_i\mu_i^2}{\lambda_i^2} + \sum_{i=k^* + 1}^k \frac{\lambda_i\mu_i^2}{(n^{-1}\Lambda_*/2)^2} + \sum_{i > k}\frac{\lambda_i \mu_i^2}{(n^{-1}\Lambda_*/2)^2},
\end{cases}
\end{align*}
where we plugged in $n^{-1}\Lambda \leq \lambda_i$ for $i \leq k^*$, $\lambda_i \leq n^{-1}\Lambda_*$ for $i > k^*$, and $\Lambda_* \geq \Lambda \geq \Lambda_*/2$ in the last transition.

Since
\[
n\Diamond_*^2 = \sum_{i=1}^{k^*} \frac{\lambda_i\mu_i^2}{\lambda_i^2} +  \sum_{i > k^*}\frac{\lambda_i \mu_i^2}{(n^{-1}\Lambda_*)^2},
\]
the above implies that $4n\Diamond_*^{2} \geq n\Diamond^{2}\geq n\Diamond_*^{2}/4$.
\item $V$ vs $V_*$.
\begin{align*}
V =& n^{-1}\tr\left(\left(\Lambda n^{-1}\bSigma_\uptok^{-1} + \bI_k\right)^{-2}\right)+ \Lambda^{-2}n\sum_{i > k}\lambda_i^2\\
=&n^{-1}\sum_{i = 1}^k \frac{1}{(1 + \lambda_i^{-1}n^{-1}\Lambda)^2} + \Lambda^{-2}n\sum_{i > k}\lambda_i^2\\
=&\sum_{i = 1}^{k^*} \frac{\lambda_i^2/n}{(\lambda_i + n^{-1}\Lambda)^2} + \sum_{i = k^* + 1}^k \frac{\lambda_i^2/n}{(\lambda_i + n^{-1}\Lambda)^2} + \sum_{i > k}\frac{\lambda_i^2/n}{(n^{-1}\Lambda)^2}\\
&\begin{cases}
\geq  \sum_{i = 1}^{k^*} \frac{\lambda_i^2/n}{(2\lambda_i)^2} + \sum_{i = k^* + 1}^k \frac{\lambda_i^2/n}{(2n^{-1}\Lambda_*)^2} + \sum_{i > k}\frac{\lambda_i^2/n}{(n^{-1}\Lambda_*)^2},\\
\leq \sum_{i = 1}^{k^*} \frac{\lambda_i^2/n}{\lambda_i ^2} + \sum_{i = k^* + 1}^k \frac{\lambda_i^2/n}{(n^{-1}\Lambda_*/2)^2} + \sum_{i > k}\frac{\lambda_i^2/n}{(n^{-1}\Lambda_*/2)^2},
\end{cases}
\end{align*}
where we plugged in $n^{-1}\Lambda \leq \lambda_i$ for $i \leq k^*$, $\lambda_i \leq n^{-1}\Lambda_*$ for $i > k^*$, and $\Lambda_* \geq \Lambda \geq \Lambda_*/2$ in the last transition.

Since
\[
V^* = \sum_{i = 1}^{k^*} \frac{\lambda_i^2/n}{\lambda_i^2} +  \sum_{i > k^*}\frac{\lambda_i^2/n}{(n^{-1}\Lambda_*)^2}
\]
the above implies that $4V_* \geq V \geq V_*/4$.
\end{enumerate}
\end{proof}

\altformbounds*
\begin{proof}
We prove the relations one-by-one. In the last transition in each display below we use the fact that for $i > k$ we have $(\Lambda/n)^{-1} \leq 2(\lambda_i + \Lambda/n)^{-1}$ to obtain the upper bound.
\begin{align*}
n\Lambda^{-1}M &= \left\|\left(\Lambda n^{-1}\bSigma_\uptok^{-1} + \bI_k\right)^{-1/2}\bSigma_\uptok^{-1/2}\bmu_\uptok\right\|^2 + n\Lambda^{-1}\|\bmu_\ktoinf\|^2\\
&= \sum_{i=1}^k \frac{\mu_i^2}{\lambda_i(1 + \lambda_i^{-1}n^{-1}\Lambda)} + n\Lambda^{-1}\sum_{i > k} \mu_i^2\\
&= \sum_{i=1}^k \frac{\mu_i^2}{\lambda_i + \Lambda/n} + \sum_{i> k} \frac{\mu_i^2}{\Lambda/n}\\
&\begin{cases}
\geq  \sum_i \frac{\mu_i^2}{\lambda_i + \Lambda/n}\\
\leq 2 \sum_i \frac{\mu_i^2}{\lambda_i + \Lambda/n}.
\end{cases}
\end{align*}

\begin{align*}
n\Diamond^2 =& \left\|\left(\Lambda n^{-1}\bSigma_\uptok^{-1} + \bI_k\right)^{-1}\bSigma_\uptok^{-1/2}\bmu_\uptok\right\|^2 + n^2\Lambda^{-2}\|\bmu_\ktoinf\|_{\bSigma_\ktoinf}^2\\
=& \sum_{i=1}^k \frac{\mu_i^2}{\lambda_i(1 + \lambda_i^{-1}n^{-1}\Lambda)^2} + n^2\Lambda^{-2}\sum_{i > k}\lambda_i \mu_i^2\\
=& \sum_{i=1}^k \frac{\lambda_i\mu_i^2}{(\lambda_i + \Lambda/n)^2} + \sum_{i > k}\frac{\lambda_i \mu_i^2}{(\Lambda/n)^2}\\
&\begin{cases}
\geq  \sum_i \frac{\lambda_i\mu_i^2}{(\lambda_i + \Lambda/n)^2}\\
\leq 4\sum_i \frac{\lambda_i\mu_i^2}{(\lambda_i + \Lambda/n)^2},
\end{cases}
\end{align*}

\begin{align*}
V =& n^{-1}\tr\left(\left(\Lambda n^{-1}\bSigma_\uptok^{-1} + \bI_k\right)^{-2}\right)+ \Lambda^{-2}n\sum_{i > k}\lambda_i^2\\
=&n^{-1}\sum_{i = 1}^k \frac{1}{(1 + \lambda_i^{-1}n^{-1}\Lambda)^2} + \Lambda^{-2}n\sum_{i > k}\lambda_i^2\\
=&\sum_{i = 1}^k \frac{\lambda_i^2/n}{(\lambda_i + \Lambda/n)^2} + \sum_{i > k}\frac{\lambda_i^2/n}{(\Lambda/n)^2}\\
&\begin{cases}
\geq  \sum_i \frac{\lambda_i^2/n}{(\lambda_i + \Lambda/n)^2}\\
\leq 4\sum_i\frac{\lambda_i^2/n}{(\lambda_i + \Lambda/n)^2},
\end{cases}
\end{align*}
\end{proof}

\section{Randomness in labels}
\label{sec::randomness in y}

\begin{restatable}[Factoring out randomness in labels]{lemma}{randomnessiny}\label{lm::randomness in y}
There exists an absolute constant $c$ s.t. conditionally on the draw of $\bQ$ for any $t > 0$ with probability at least $1-ce^{-t^2/c}$ over the draw of $(\by, \hat{\by})$ all the following hold:
\begin{enumerate}
\item 
\[
\max(|\bnu^\top \bA^{-1}\by|, |\bnu^\top \bA^{-1}\hat{\by}|) \leq ct\|\bA^{-1}\bnu\|.
\]
\item 
\[
|\bnu^\top \bA^{-1}\Delta\by|\leq ct\sigma_\eta\|\bA^{-1}\bnu\|.
\]
\item 
\[|\Delta\by^\top \bA^{-1}\by| \leq c\|\bA^{-1}\|\left(n\eta + t\sigma_\eta(\sqrt{n} + t)\right).\]
\item 
\begin{align*}
\by^\top \bA^{-1}\hat{\by} \geq& (n - n\eta - ct\sigma_\eta\sqrt{n} - k)\mu_1(\bA_k)^{-1}\\
-&  (n\eta + ct\sigma_\eta\sqrt{n} + ct\sqrt{n} + ct^2)\|\bA^{-1}\|.
\end{align*}
\item 
\[
n\|\bA^{-1}\| \geq \by^\top \bA^{-1}\by\geq (n - k)\mu_1(\bA_k)^{-1} -  c(t\sqrt{n} + t^2)\|\bA^{-1}\|.
\]
\item 
\begin{align*}
\|\bQ^\top\bA^{-1}\Delta\by\|_\bSigma^2 \leq& c\sigma_\eta^2\Bigl( \tr(\bA^{-1}\bQ^\top\bSigma \bQ^\top\bA^{-1}) + t^2\|\bA^{-1}\bQ^\top\bSigma \bQ^\top\bA^{-1}\|\Bigr).
\end{align*}
\item 
\begin{align*}
\|\bQ^\top\bA^{-1}\by\|_\bSigma^2 \leq&  c\Bigl(\tr(\bA^{-1}\bQ^\top\bSigma \bQ^\top\bA^{-1})+ t^2\|\bA^{-1}\bQ^\top\bSigma \bQ^\top\bA^{-1}\|\Bigr).
\end{align*}
\end{enumerate}
\end{restatable}
\begin{proof}
Throughout the whole proof we will use Lemma \ref{lm::subgaussian constant of noise}, which states that sub-Gaussian norms of the components of $\Delta\by/2$ are at most $\sigma_\eta$. Recall also that sub-Gaussian norms of the components of $\by$ and $\hat{\by}$ are equal to an absolute constant (to be precise, $1/\sqrt{\ln(2)}$). Each time we use $c$ in this proof it denotes a new absolute constant. In the end we take $c$ large enough, so all the statements hold.
\begin{enumerate}
\item $|\bnu^\top \bA^{-1}\by|, |\bnu^\top \bA^{-1}\hat{\by}|$: the bound follows directly from the fact that $\by$ and $\hat{\by}$ are sub-Gaussian vectors with sub-Gaussian norms bounded by an absolute constant (see Lemma 3.4.2 in \cite{vershynin_hdp}), and both $\by, \hat{\by}$ are independent from $\bnu^\top \bA^{-1}$.
\item $|\bnu^\top \bA^{-1}\Delta\by|$: the bound follows in the same way as above from the fact that $\bDelta \by$ is a sub-Gaussian vector with sub-Gaussian norm at most $c\sigma_\eta$, and $\Delta\by$ is independent from $\bnu^\top \bA^{-1}$.
\item 
 $|\Delta\by^\top \bA^{-1}\by|$. Denote $\by_c = \by + \Delta\by/2$ --- the vector, whose coordinates corresponding to the clean points are equal to their clean labels, and other coordinates zeroed out. Conditionally on $\Delta\by$, $\by_c$ is a vector with i.i.d. Rademacher coordinates supported on the complement of the support of $\Delta\by$. Since Rademacher R.V's. are sub-Gaussian, we have that for some absolute constant $c$ for any $t > 0$ the following holds with probability at least $1 - 2e^{-t^2/c}$:
\begin{align*}
|\Delta\by^\top \bA^{-1}\by| =& \Bigl|-\Delta\by^\top \bA^{-1}\Delta\by/2 + \Delta\by^\top \bA^{-1}\by_c\Bigr|\\
\leq& \Delta\by^\top \bA^{-1}\Delta\by/2 + ct\|\bA^{-1}\Delta\by\|\\
\leq&\|\bA^{-1}\|(\|\Delta\by\|^2/2 + ct\|\Delta\by\|).
\end{align*}
By Lemma \ref{lm::subgaussian constant of noise} squares of coordinates of $\Delta\by/2$ are $\sigma_\eta$-sub-Gaussian with mean $\eta$, so by General Hoeffding’s inequality (Theorem 2.6.2 in \cite{vershynin_hdp}) for some absolute constant $c$ and any $t > 0$ 
with probability at least $1 - 2e^{-t^2/c}$
\[
|\|\Delta\by\|^2/4 - n\eta| \leq ct\sigma_\eta\sqrt{n}.
\]
We could use this result to bound $|\|\Delta\by\|$ as well, but then $\sqrt{\sigma_\eta}$ will appear in the bounds. Instead, we use Lemma \ref{lm::subgauss norm no variance} to give an alternative bound that also holds with probability $1 - 2e^{-t^2/c}$:
\[
\|\Delta\by\|/2 \leq \sigma_\eta(\sqrt{n} + t).
\]
Combining these bounds yields the result.

\item 
$\by^\top \bA^{-1}\hat{\by}$: denote $\bS_N \in \R^{n\times n}$ to be a diagonal matrix, such that $\bS_N[i, i] = -1$ if the label of the $i$-th data point is noisy, and $\bS_N[i, i] = 1$ otherwise. 

The matrix $\bS_N$ is independent from both $\by$ and $\bA$. Now we can write
\[
\by^\top \bA^{-1}\hat{\by} = \by^\top(\bA^{-1}\bS_N)\by.
\]

By Lemma \ref{lm::quadratic form concentration} (Hanson-Wright inequality),  for some absolute constant $c$  for any $t> 0$ with probability at least $1 - 2e^{-t^2/c}$
\[
\by^\top \bA^{-1}\hat{\by} \geq \tr(\bA^{-1}\bS_N) - ct\|\bA^{-1}\bS_N\|_F - ct^2\|\bA^{-1}\bS_N\|
\]
Note that 
\begin{align*}
\|\bA^{-1}\bS_N\| =& \|\bA^{-1}\|,\\
\|\bA^{-1}\bS_N\|_F =& \|\bA^{-1}\|_F \leq \sqrt{n}\|\bA^{-1}\|.
\end{align*}
We need to bound the number of noisy data points in order to bound $ \tr(\bA^{-1}\bS_N)$ from below. The number of noisy data points is equal to 
\[
\|\Delta\by\|_0 = \|\Delta\by\|^2/4 \leq n\eta + ct\sigma_\eta\sqrt{n},
\]
where the last inequality was taken from before, and holds with probability at least $1 - 2e^{-t^2/c}$.

Recall that the $n-k$ largest eigenvalues of $\bA^{-1}$ are greater or equal to $\mu_1(\bA_k)^{-1}$. Thus, with probability at least $1 - 2e^{-t^2/c}$. 
\begin{align*}
\tr(\bA^{-1}\bS_N) \geq& (n - \|\Delta\by\|_0 - k)\mu_1(\bA_k)^{-1} - \|\Delta\by\|_0 \|\bA^{-1}\|,\\
\geq& (n - n\eta - ct\sigma_\eta\sqrt{n} - k)\mu_1(\bA_k)^{-1} - (n\eta + ct\sigma_\eta\sqrt{n}) \|\bA^{-1}\|
\end{align*}

Combining it with Hanson-Wright, we get that with probability at least $1 - 4e^{-t^2/c}$
\begin{align*}
\by^\top \bA^{-1}\hat{\by}\geq& (n - n\eta - ct\sigma_\eta\sqrt{n} - k)\mu_1(\bA_k)^{-1} -  (n\eta + ct\sigma_\eta\sqrt{n} + ct\sqrt{n} + ct^2)\|\bA^{-1}\|\\
\end{align*}

\item $ \by^\top \bA^{-1}\by$: the inequality $\by^\top \bA^{-1}\by\leq n\|\bA^{-1}\| $ holds with probability one since $\|\by\|^2 = n$ almost surely. When it comes to the lower bound, it is simply a particular case of the result for $\by^\top\bA^{-1}\hat{\by}$ proven above for $\eta = 0$. 
\item $\|\bQ^\top\bA^{-1}\Delta\by\|_\bSigma^2$: the bound is a direct consequence of Corollary  \ref{cor::psd quadratic form concentration}, applied to $\bM = \bA^{-1}\bQ^\top\bSigma \bQ^\top\bA^{-1}$ and $\beps = \Delta\by$.
\item $\|\bQ^\top\bA^{-1}\by\|_\bSigma^2$: the bound is a direct consequence of Corollary  \ref{cor::psd quadratic form concentration}, applied to $\bM = \bA^{-1}\bQ^\top\bSigma \bQ^\top\bA^{-1}$ and $\beps = \by$.

\end{enumerate}
\end{proof}

\section{Algebraic decompositions}
\label{sec::algebraic decompositions}
The purpose of this section is to provide algebraic decompositions of various terms or bounds on them as given by the following Lemma.

\begin{restatable}[Algebraic decompositions]{lemma}{algebraicdecompositions}\label{lm::algebraic decompositions}
For any $k < n$ all the following hold almost surely on the event that the matrix $\bA_k$ is PD:
\begin{enumerate}
\item 
\begin{align*}
\|\bA^{-1}\bnu\| \leq& \frac{\mu_1(\bA_k)}{\mu_n(\bA_k)}\frac{\sqrt{\mu_1(\bZ_\uptok^\top\bZ_\uptok})}{\mu_k(\bZ_\uptok^\top\bZ_\uptok)}\left\|\left(\mu_1(\bZ_\uptok^\top\bZ_\uptok)^{-1}\mu_n(\bA_k)\bSigma_\uptok^{-1} + \bI_k\right)^{-1}\bSigma_\uptok^{-1/2}\bmu_\uptok\right\|\\
+& \mu_n(\bA_k)^{-1}\|\bQ_\ktoinf \bmu_\ktoinf\|.
\end{align*}
\item 
\begin{align*}
\bmu^\top\muperpridge \geq& \frac12 \mu_n(\bA_k)\mu_1(\bZ_\uptok^\top \bZ_\uptok)^{-1}\left\|\left(\mu_n(\bA_k)\mu_1(\bZ_\uptok^\top \bZ_\uptok)^{-1}\bSigma^{-1}_\uptok + \bI_k\right)^{-1/2}\bSigma^{-1/2}_\uptok\bmu_\uptok\right\|^2 \\
+&\|\bmu_\ktoinf\|^2 - 9\mu_n(\bA_k)^{-1}\|\bQ_\uptok^\top\bmu_\ktoinf\|^2.
\end{align*}
\item 
\begin{align*}
\bmu^\top\muperpridge \leq& 3\mu_1(\bA_k)\mu_k(\bZ_\uptok^\top\bZ_\uptok)^{-1}\left\|\left(\mu_k(\bZ_\uptok^\top \bZ_\uptok)^{-1}\mu_1(\bA_k)\bSigma^{-1}_\uptok + \bI_k\right)^{-1/2}\bSigma^{-1/2}_\uptok\bmu_\uptok\right\|^2\\
+& \|\bmu_\ktoinf\|^2 + 2\|\bA_k^{-1/2}\bQ_\ktoinf\bmu_\ktoinf\|^2.
\end{align*}
\item 
\begin{align}
\label{eq::regression bias algebraic}
\begin{split}
&{\|\muperpridge\|_\bSigma^2} \\
 \leq&  2\mu_k(\bZ_\uptok^\top \bZ_\uptok)^{-2}\mu_1(\bA_k)^{2}\left\| \left(\mu_1(\bZ_\uptok^\top \bZ_\uptok)^{-1}\mu_n(\bA_k)\bSigma_\uptok^{-1} + \bI_k\right)^{-1}\bSigma_\uptok^{-1/2}\bmu_\uptok\right\|^2\\ 
+& 2\frac{\mu_1(\bA_k)^2\mu_1(\bZ_\uptok\bZ_\uptok^\top)}{\mu_n(\bZ_\uptok^\top \bZ_\uptok)^{2}\mu_n(\bA_k)^{2}}\left\|\bQ_\ktoinf\bmu_\ktoinf\right\|^2\\
+&3\|\bmu_\ktoinf\|_{\bSigma_\ktoinf}^2 + 3\|\bQ_\ktoinf\bSigma_\ktoinf\bQ_\ktoinf^\top\|\mu_n(\bA_k)^{-2}\|\bQ_\ktoinf\bmu_\ktoinf\|^2\\
+&3\|\bQ_\ktoinf\bSigma_\ktoinf\bQ_\ktoinf^\top\|\frac{\mu_1(\bA_k)^{2}\mu_1(\bZ_\uptok^\top \bZ_\uptok)}{\mu_k(\bZ_\uptok^\top \bZ_\uptok)^2\mu_n(\bA_k)^{2}}\\
\times& \left\|\left(\mu_1(\bZ_\uptok^\top \bZ_\uptok)^{-1}\mu_n(\bA_k)\bSigma_\uptok^{-1} + \bI_k\right)^{-1}\bSigma_\uptok^{-1/2}\bmu_\uptok\right\|^2.
\end{split}
\end{align}
\item 
\begin{align}
\label{eq::regression variance algebraic} 
\begin{split}
\tr(\bA^{-1}\bQ\bSigma\bQ^\top\bA^{-1}) \leq& \frac{\mu_1(\bA_k)^{2}\mu_1(\bZ_\uptok^\top\bZ_\uptok)}{\mu_k(\bZ_\uptok^\top \bZ_\uptok)^2\mu_n(\bA_k)^{2}}\\
\times&\tr\left(\left(\mu_1(\bZ_\uptok^\top \bZ_\uptok)^{-1}\mu_n(\bA_k)\bSigma_\uptok^{-1} + \bI_k\right)^{-2}\right)\\
+&  \mu_n(\bA_k)^{-2}\tr(\bQ_\ktoinf\bSigma_\ktoinf \bQ_\ktoinf^\top).
\end{split}
\end{align}
\item 
\[
\|\bA^{-1}\bQ\bSigma\bQ^\top\bA^{-1}\| \leq  \frac{\mu_1(\bA_k)^2}{\mu_n(\bA_k)^2}\frac{\mu_1(\bZ_\uptok^\top\bZ_\uptok)}{\mu_k(\bZ_\uptok^\top\bZ_\uptok)^2} \wedge \frac{\lambda_1^{2}\mu_1(\bZ_\uptok^\top\bZ_\uptok)}{\mu_n(\bA_k)^2} + \frac{\|\bQ_\ktoinf\bSigma_\ktoinf\bQ_\ktoinf^\top \|}{\mu_n(\bA_k)^2} .
\]
\end{enumerate}
\end{restatable}
The remainder of Section \ref{sec::algebraic decompositions} gives the proof of Lemma \ref{lm::algebraic decompositions}.
\subsection{Techniques and proof strategy}
The main tool that we are going to use in this section is the following application of Sherman-Morrison-Woodbury (SMW) identity for the matrix $\bA^{-1}$:
\begin{lemma}
\label{lm::application of SMW}
If $\bA_k$ is invertible, then all the following hold:
\begin{align}
\bA^{-1} =& \bA_k^{-1} - \bA_k^{-1}\bQ_\uptok\left(\bI_k + \bQ_\uptok^\top\bA_k^{-1}\bQ_\uptok\right)^{-1}\bQ_\uptok^\top\bA_k^{-1},\label{eq::SMW for A}\\
\bA^{-1}\bQ_\uptok =&  \bA_k^{-1}\bQ_\uptok\left(\bI_k + \bQ_\uptok^\top \bA_k^{-1}\bQ_\uptok\right)^{-1}\label{eq::favorite identity},\\
\bI_k - \bQ_\uptok^\top\bA^{-1}\bQ_\uptok =& \left(\bI_k + \bQ_\uptok^\top \bA_k^{-1}\bQ_\uptok\right)^{-1}\label{eq::cancellation in projector}.
\end{align}
\end{lemma}
\begin{proof}
Equation \eqref{eq::SMW for A} is a direct application of SMW as $\bA^{-1} = (\bA_k + \bQ_\uptok\bQ_\uptok^\top)^{-1}$. To derive \eqref{eq::favorite identity} we write
\begin{align*}
\bA^{-1}\bQ_\uptok =& \bA_k^{-1}\bQ_\uptok - \bA_k^{-1}\bQ_\uptok\left(\bI_k + \bQ_\uptok^\top\bA_k^{-1}\bQ_\uptok\right)^{-1}\bQ_\uptok^\top\bA_k^{-1}\bQ_\uptok\\
=& \bA_k^{-1}\bQ_\uptok\left(\bI_k - \left(\bI_k + \bQ_\uptok^\top\bA_k^{-1}\bQ_\uptok\right)^{-1}\left(\bI_k + \bQ_\uptok^\top\bA_k^{-1}\bQ_\uptok - \bI_k\right)\right)\\
=& \bA_k^{-1}\bQ_\uptok\left(\bI_k + \bQ_\uptok^\top \bA_k^{-1}\bQ_\uptok\right)^{-1}
\end{align*}
Finally, we derive \eqref{eq::cancellation in projector} from \eqref{eq::favorite identity}:
\begin{align*}
&\bI_k - \bQ_\uptok^\top\bA^{-1}\bQ_\uptok\\
=&\bI_k - \bQ_\uptok^\top\bA_k^{-1}\bQ_\uptok\left(\bI_k + \bQ_\uptok^\top \bA_k^{-1}\bQ_\uptok\right)^{-1}\\
=&\bI_k + \left(\bI_k + \bQ_\uptok^\top \bA_k^{-1}\bQ_\uptok\right)^{-1} - \left(\bI_k + \bQ_\uptok^\top\bA_k^{-1}\bQ_\uptok\right)\left(\bI_k + \bQ_\uptok^\top \bA_k^{-1}\bQ_\uptok\right)^{-1}\\
=&\left(\bI_k + \bQ_\uptok^\top \bA_k^{-1}\bQ_\uptok\right)^{-1}.
\end{align*}
\end{proof}

Another algebraic result that we will utilize is as follows:
\begin{lemma}
\label{lm::preserve sigma uptok inverse}
Suppose $\bM \in \R^{k\times k}$ is a PD matrix such that $\alpha \bI_k \preceq \bM \preceq \beta\bI_k$ for some positive scalars $\alpha < \beta$. Then for any vector $\bu \in \R^k$
\begin{align}
\left\|(\bSigma_\uptok^{-1} + \bM)^{-1}\bu\right\| \geq& \beta^{-1}\left\|(\alpha^{-1}\bSigma_\uptok^{-1} + \bI_k)^{-1}\bu\right\|\label{eq::norm v alpha},\\
\left\|(\bSigma_\uptok^{-1} + \bM)^{-1}\bu\right\| \leq& \alpha^{-1}\left\|(\beta^{-1}\bSigma_\uptok^{-1} + \bI_k)^{-1}\bu\right\|\label{eq::norm v beta}.
\end{align}
Moreover, 
\begin{equation}
\label{eq::trace v alpha beta}
\beta^{-2}\tr\left((\alpha^{-1}\bSigma_\uptok^{-1} + \bI_k)^{-2}\right) \leq \tr\left((\bSigma_\uptok^{-1} + \bM)^{-2}\right) \leq \alpha^{-2}\tr\left((\beta^{-1}\bSigma_\uptok^{-1} + \bI_k)^{-2}\right).
\end{equation}
\end{lemma}
\begin{proof}
Denote $\bv := (\bSigma_\uptok^{-1} + \bM)^{-1}\bu$, $\bv_\alpha :=(\bSigma_\uptok^{-1} + \alpha\bI_k)^{-1}\bu$, and  $\bv_\beta:= (\bSigma_\uptok^{-1} + \beta\bI_k)^{-1}\bu$. Then 
\begin{align*}
\bv_\alpha :=&(\bSigma_\uptok^{-1} + \alpha\bI_k)^{-1}\bu\\
=&(\bSigma_\uptok^{-1} + \alpha\bI_k)^{-1}(\bSigma_\uptok^{-1} + \bM)\bv\\
=& \bv + (\bSigma_\uptok^{-1} + \alpha\bI_k)^{-1}( \bM - \alpha\bI_k)\bv
\end{align*}
Thus, 
\[
\|\bv_\alpha \| \leq \|\bv\|\left(1 + \|(\bM - \alpha\bI_k)\|\|(\bSigma_\uptok^{-1} + \alpha\bI_k)^{-1}\|\right) \leq \|\bv\|\left(1 + \frac{\beta - \alpha}{\alpha}\right) = \frac{\beta}{\alpha}\|\bv\|,
\]
which yields Equation \eqref{eq::norm v alpha}. Analogously,  
\begin{align*}
\bv_\beta:=&(\bSigma_\uptok^{-1} + \beta\bI_k)^{-1}\bu\\
=&(\bSigma_\uptok^{-1} + \beta\bI_k)^{-1}(\bSigma_\uptok^{-1} + \bM)\bv\\
=& \bv + (\bSigma_\uptok^{-1} + \beta\bI_k)^{-1}( \bM - \beta\bI_k)\bv
\end{align*}

Thus, 
\[
\|\bv_\beta\| \geq \|\bv\|\left(1 - \|(\bM - \beta\bI_k)\|\|(\bSigma_\uptok^{-1} + \beta\bI_k)^{-1}\|\right) \geq \|\bv\|\left(1 - \frac{\beta - \alpha}{\beta}\right) = \frac{\alpha}{\beta}\|\bv\|,
\]
which yields Equation \eqref{eq::norm v beta}.

Finally, Equation \eqref{eq::trace v alpha beta} is a direct consequence of Equations  \eqref{eq::norm v alpha} and \eqref{eq::norm v beta}: take $\bg$ to be an isotropic Gaussian vector in $\R^k$. Equations \eqref{eq::norm v alpha} and \eqref{eq::norm v beta} give
\[
\beta^{-2}\left\|(\alpha^{-1}\bSigma_\uptok^{-1} + \bI_k)^{-1}\bg\right\|^2 \leq \left\|(\bSigma_\uptok^{-1} + \bM)^{-1}\bg\right\|^2 \leq \alpha^{-2}\left\|(\beta^{-1}\bSigma_\uptok^{-1} + \bI_k)^{-1}\bg\right\|^2.
\]
Taking expectation over $\bg$ yields Equation \eqref{eq::trace v alpha beta}.
\end{proof}

%
%
%
%
\subsection{\texorpdfstring{$\|\bA^{-1}\bnu \|$}{A inverse times nu}}
\label{sec:: A inv nu decomposition}
In this section we derive an upper bound on $\|\bA^{-1}\bnu \|$.  Recall that 
 \begin{gather*}
 \bnu = \bQ\bmu = \bQ_\uptok\bmu_\uptok + \bQ_\ktoinf\bmu_\ktoinf,\\
 \|\bA^{-1}\bnu\| \leq    \|\bA^{-1}\bQ_\uptok\bmu_\uptok\| +  \|\bA^{-1}\bQ_\ktoinf\bmu_\ktoinf\|.
 \end{gather*}

We bound those two terms separately. For the first term we use Equation \eqref{eq::favorite identity}:
\begin{align*}
\bA^{-1}\bQ_\uptok \bmu_\uptok =&\bA_k^{-1}\bQ_\uptok\left(\bI_k + \bQ_\uptok^\top \bA_k^{-1}\bQ_\uptok\right)^{-1}\bmu_\uptok\\
=&\bA_k^{-1}\bZ_\uptok\bSigma_\uptok^{1/2}\left(\bI_k + \bSigma_\uptok^{1/2}\bZ_\uptok^\top \bA_k^{-1}\bZ_\uptok\bSigma_\uptok^{1/2}\right)^{-1}\bmu_\uptok\\
=&\bA_k^{-1}\bZ_\uptok\bSigma_\uptok^{1/2}\left(\bI_k + \bSigma_\uptok^{1/2}\bZ_\uptok^\top \bA_k^{-1}\bZ_\uptok\bSigma_\uptok^{1/2}\right)^{-1}\bmu_\uptok\\
=&\bA_k^{-1}\bZ_\uptok\left(\bSigma_\uptok^{-1} + \bZ_\uptok^\top \bA_k^{-1}\bZ_\uptok\right)^{-1}\bSigma_\uptok^{-1/2}\bmu_\uptok.
\end{align*}

So,
\[
\|\bA^{-1}\bQ_\uptok \bmu_\uptok\|^2
\leq \|\bZ_\uptok^\top \bA_k^{-2}\bZ_\uptok\|\left\|\left(\bSigma_\uptok^{-1} + \bZ_\uptok^\top \bA_k^{-1}\bZ_\uptok\right)^{-1}\bSigma_\uptok^{-1/2}\bmu_\uptok\right\|^2.
\]

Now we use Lemma \ref{lm::preserve sigma uptok inverse} together with the observation that
\begin{gather*}
\|\bZ_\uptok^\top \bA_k^{-2}\bZ_\uptok\| \leq \mu_1(\bZ_\uptok^\top\bZ_\uptok) \mu_n(\bA_k)^{-2},\\
\mu_k(\bZ_\uptok^\top\bZ_\uptok)\mu_1(\bA_k)^{-1}\bI_k \preceq \bZ_\uptok^\top \bA_k^{-1}\bZ_\uptok \preceq \mu_1(\bZ_\uptok^\top\bZ_\uptok)\mu_n(\bA_k)^{-1}\bI_k
\end{gather*}
to write
\begin{align*}
&\|\bA^{-1}\bQ_\uptok \bmu_\uptok\|^2\\
\leq& \mu_1(\bZ_\uptok^\top\bZ_\uptok) \mu_n(\bA_k)^{-2} \left(\mu_k(\bZ_\uptok^\top\bZ_\uptok)\mu_1(\bA_k)^{-1}\right)^{-2}\\
\times&\left\|\left(\mu_1(\bZ_\uptok^\top\bZ_\uptok)^{-1}\mu_n(\bA_k)\bSigma_\uptok^{-1} + \bI_k\right)^{-1}\bSigma_\uptok^{-1/2}\bmu_\uptok\right\|^2\\
=& \frac{\mu_1(\bZ_\uptok^\top\bZ_\uptok)}{\mu_k(\bZ_\uptok^\top\bZ_\uptok)^2} \frac{\mu_1(\bA_k)^{2}}{\mu_n(\bA_k)^{2}}\left\|\left(\mu_1(\bZ_\uptok^\top\bZ_\uptok)^{-1}\mu_n(\bA_k)\bSigma_\uptok^{-1} + \bI_k\right)^{-1}\bSigma_\uptok^{-1/2}\bmu_\uptok\right\|^2.
\end{align*}

For the $\ktoinf$ part we can just write
\[
\| \bA^{-1}\bQ_\ktoinf \bmu_\ktoinf\|\leq \|\bA^{-1}\|\|\bQ_\ktoinf \bmu_\ktoinf\| \leq \mu_n(\bA_k)^{-1}\|\bQ_\ktoinf \bmu_\ktoinf\|.
\]

Overall,
\begin{align*}
\|\bA^{-1}\bnu\| \leq& \frac{\mu_1(\bA_k)}{\mu_n(\bA_k)}\frac{\sqrt{\mu_1(\bZ_\uptok^\top\bZ_\uptok})}{\mu_k(\bZ_\uptok^\top\bZ_\uptok)}\left\|\left(\mu_1(\bZ_\uptok^\top\bZ_\uptok)^{-1}\mu_n(\bA_k)\bSigma_\uptok^{-1} + \bI_k\right)^{-1}\bSigma_\uptok^{-1/2}\bmu_\uptok\right\|\\
+& \mu_n(\bA_k)^{-1}\|\bQ_\ktoinf \bmu_\ktoinf\|.
\end{align*}

%
%
%
%
\subsection{\texorpdfstring{$\bmu^\top\muperpridge$}{Energy of mu in the orthogonal complement}}
\subsubsection{Bound from below}
In this section we derive a lower bound on $\bmu^\top\muperpridge$. First of all, we write
\begin{align*}
\bmu^\top\muperpridge =& \bmu^\top(\bI_p - \bQ^\top\bA^{-1}\bQ)\bmu\\
=& \bmu_\uptok^\top (\bI_k - \bQ_\uptok^\top\bA^{-1}\bQ_\uptok)\bmu_\uptok\\
+& \bmu_\ktoinf^\top (\bI_{p-k} - \bQ_\ktoinf^\top\bA^{-1}\bQ_\ktoinf)\bmu_\ktoinf\\
-& 2 \bmu_\uptok^\top\bQ_\uptok^\top\bA^{-1}\bQ_\ktoinf\bmu_\ktoinf.
\end{align*}

We see that this decomposition has 3 terms: energy in the spiked part, energy in the tail, and the cross term. We expect the positive contribution to come from $\bmu_\uptok^\top (\bI_k - \bQ_\uptok^\top\bA^{-1}\bQ_\uptok)\bmu_\uptok + \bmu_\ktoinf^\top \bI_{p-k}\bmu_\ktoinf$, the other terms will be upper bounded in absolute value and subtracted  from the lower bound. The last term (the cross term) is a bit tricky, because bounding it separately leads to a potentially vacuous bound. The approach we take here is to bound it in terms of the quantities from the first two terms, and then bounding those quantities.

Due to Equation \eqref{eq::cancellation in projector} the first term becomes 
\[
\bmu_\uptok^\top\left(\bI_k + \bQ_\uptok^\top \bA_k^{-1}\bQ_\uptok\right)^{-1}\bmu_\uptok.
\]

Now let's apply a similar transformation to the cross-term: we use Equation \eqref{eq::favorite identity} to write
\begin{align*}
&\left|\bmu_\uptok^\top\bQ_\uptok^\top\bA^{-1}\bQ_\ktoinf\bmu_\ktoinf\right|\\
=&\left|\bmu_\uptok^\top\left(\bI_k + \bQ_\uptok^\top \bA_k^{-1}\bQ_\uptok\right)^{-1}\bQ_\uptok^\top\bA_k^{-1}\bQ_\ktoinf\bmu_\ktoinf\right|\\
=&\left| \bmu_\uptok^\top\left(\bI_k + \bQ_\uptok^\top \bA_k^{-1}\bQ_\uptok\right)^{-1/2}\left(\bI_k + \bQ_\uptok^\top \bA_k^{-1}\bQ_\uptok\right)^{-1/2}\bQ_\uptok^\top\bA_k^{-1}\bQ_\ktoinf\bmu_\ktoinf\right|\\
\leq& w \bmu_\uptok^\top\left(\bI_k + \bQ_\uptok^\top \bA_k^{-1}\bQ_\uptok\right)^{-1}\bmu_\uptok \\
+& w^{-1}\left\|\left(\bI_k + \bQ_\uptok^\top \bA_k^{-1}\bQ_\uptok\right)^{-1/2}\bQ_\uptok^\top\bA_k^{-1}\bQ_\ktoinf\bmu_\ktoinf\right\|^2\\
\leq& w \bmu_\uptok^\top\left(\bI_k + \bQ_\uptok^\top \bA_k^{-1}\bQ_\uptok\right)^{-1}\bmu_\uptok + w^{-1}\|\bA_k^{-1/2}\bQ_\ktoinf\bmu_\ktoinf\|^2,
\end{align*}
where we introduced an arbitrary scalar $w > 0$ when we used AM-GM inequality. In the last line we also used the following fact: 
\[
\left\|\left(\bI_k + \bQ_\uptok^\top \bA_k^{-1}\bQ_\uptok\right)^{-1/2}\bQ_\uptok^\top\bA_k^{-1/2}\right\|\leq 1.
\]
Indeed, the matrix $\bI_k + \bQ_\uptok^\top \bA_k^{-1}\bQ_\uptok$ is larger  than the matrix $\bQ_\uptok^\top\bA_k^{-1/2}(\bQ_\uptok^\top\bA_k^{-1/2})^\top$ in the PSD sense.

We take $w = 0.25$. So far we obtained that
\[
\bmu^\top\muperpridge \geq \frac12\bmu_\uptok^\top\left(\bI_k + \bQ_\uptok^\top \bA_k^{-1}\bQ_\uptok\right)^{-1}\bmu_\uptok + \|\bmu_\ktoinf\|^2 - 9\|\bA_k^{-1/2}\bQ_\uptok^\top\bmu_\ktoinf\|^2,
\]
where we also used 
\begin{align*}
\bmu_\ktoinf^\top (\bI_{p-k} - \bQ_\ktoinf^\top\bA^{-1}\bQ_\ktoinf)\bmu_\ktoinf =& \|\bmu_\ktoinf\|^2 - \bmu_\ktoinf^\top\bQ_\ktoinf^\top\underbrace{\bA^{-1}}_{\preceq \bA_k^{-1}}\bQ_\ktoinf\bmu_\ktoinf \\
\geq& \|\bmu_\ktoinf\|^2 - \|\bA_k^{-1/2}\bQ_\uptok^\top\bmu_\ktoinf\|^2.
\end{align*}
Now we just need to bound $\bmu_\uptok^\top\left(\bI_k + \bQ_\uptok^\top \bA_k^{-1}\bQ_\uptok\right)^{-1}\bmu_\uptok$ from below and $\|\bA_k^{-1/2}\bQ_\uptok^\top\bmu_\ktoinf\|^2$ from above. We write
\begin{align*}
&\bmu_\uptok^\top\left(\bI_k + \bQ_\uptok^\top \bA_k^{-1}\bQ_\uptok\right)^{-1}\bmu_\uptok\\
=&\bmu_\uptok^\top\bSigma^{-1/2}_\uptok\left(\bSigma^{-1}_\uptok + \bZ_\uptok^\top \bA_k^{-1}\bZ_\uptok\right)^{-1}\bSigma^{-1/2}_\uptok\bmu_\uptok\\
\geq&\bmu_\uptok^\top\bSigma^{-1/2}_\uptok\left(\bSigma^{-1}_\uptok + \mu_n(\bA_k)^{-1}\mu_1(\bZ_\uptok^\top \bZ_\uptok)\bI_k\right)^{-1}\bSigma^{-1/2}_\uptok\bmu_\uptok\\
=& \mu_n(\bA_k)\mu_1(\bZ_\uptok^\top \bZ_\uptok)^{-1}\left\|\left(\mu_n(\bA_k)\mu_1(\bZ_\uptok^\top \bZ_\uptok)^{-1}\bSigma^{-1}_\uptok + \bI_k\right)^{-1/2}\bSigma^{-1/2}_\uptok\bmu_\uptok\right\|^2.
\end{align*}

For the term $\|\bA_k^{-1/2}\bQ_\uptok^\top\bmu_\ktoinf\|^2$ we simply do a norm-times-norm bound:
\[
\|\bA_k^{-1/2}\bQ_\uptok^\top\bmu_\ktoinf\|^2 \leq \mu_n(\bA_k)^{-1}\|\bQ_\uptok^\top\bmu_\ktoinf\|^2
\]

Combining everything together gives the bound. 
\begin{align*}
\bmu^\top\muperpridge \geq& \frac12 \mu_n(\bA_k)\mu_1(\bZ_\uptok^\top \bZ_\uptok)^{-1}\left\|\left(\mu_n(\bA_k)\mu_1(\bZ_\uptok^\top \bZ_\uptok)^{-1}\bSigma^{-1}_\uptok + \bI_k\right)^{-1/2}\bSigma^{-1/2}_\uptok\bmu_\uptok\right\|^2 \\
+&\|\bmu_\ktoinf\|^2 - 9\mu_n(\bA_k)^{-1}\|\bQ_\uptok^\top\bmu_\ktoinf\|^2
\end{align*}

\subsubsection{Bound from above}
In this section we bound $\bmu^\top\muperpridge$ from above. This is easier than bounding it from below. Indeed, recall the decomposition from the previous section:
\begin{align*}
\bmu^\top\muperpridge =& \bmu_\uptok^\top (\bI_k - \bQ_\uptok^\top\bA^{-1}\bQ_\uptok)\bmu_\uptok\\
+& \bmu_\ktoinf^\top (\bI_{p-k} - \bQ_\ktoinf^\top\bA^{-1}\bQ_\ktoinf)\bmu_\ktoinf\\
-& 2 \bmu_\uptok^\top\bQ_\uptok^\top\bA^{-1}\bQ_\ktoinf\bmu_\ktoinf.
\end{align*}
For the first term we had
\[
\bmu_\uptok^\top (\bI_k - \bQ_\uptok^\top\bA^{-1}\bQ_\uptok)\bmu_\uptok = \bmu_\uptok^\top\left(\bI_k + \bQ_\uptok^\top \bA_k^{-1}\bQ_\uptok\right)^{-1}\bmu_\uptok.
\]
and for the cross-term
\begin{align*}
&\left|\bmu_\uptok^\top\bQ_\uptok^\top\bA^{-1}\bQ_\ktoinf\bmu_\ktoinf\right|\\
\leq& w \bmu_\uptok^\top\left(\bI_k + \bQ_\uptok^\top \bA_k^{-1}\bQ_\uptok\right)^{-1}\bmu_\uptok + w^{-1}\|\bA_k^{-1/2}\bQ_\ktoinf\bmu_\ktoinf\|^2,
\end{align*}
for any $w > 0$. Here we will take $w = 1$. For the second term we simply write 
\[
\bmu_\ktoinf^\top (\bI_{p-k} - \bQ_\ktoinf^\top\bA^{-1}\bQ_\ktoinf)\bmu_\ktoinf \leq \|\bmu_\ktoinf\|^2.
\]

Combining everything together, we get
\begin{align*}
\bmu^\top\muperpridge \leq& 3\bmu_\uptok^\top\left(\bI_k + \bQ_\uptok^\top \bA_k^{-1}\bQ_\uptok\right)^{-1}\bmu_\uptok + \|\bmu_\ktoinf\|^2 + 2\|\bA_k^{-1/2}\bQ_\ktoinf\bmu_\ktoinf\|^2\\
=&  3\bmu_\uptok^\top\bSigma_\uptok^{-1/2}\left(\bSigma_\uptok^{-1} + \bZ_\uptok^\top \bA_k^{-1}\bZ_\uptok\right)^{-1}\bSigma_\uptok^{-1/2}\bmu_\uptok + \|\bmu_\ktoinf\|^2 + 2\|\bA_k^{-1/2}\bQ_\ktoinf\bmu_\ktoinf\|^2\\
\leq& 3\mu_1(\bA_k)\mu_k(\bZ_\uptok^\top\bZ_\uptok)^{-1}\left\|\left(\mu_k(\bZ_\uptok^\top \bZ_\uptok)^{-1}\mu_1(\bA_k)\bSigma^{-1}_\uptok + \bI_k\right)^{-1/2}\bSigma^{-1/2}_\uptok\bmu_\uptok\right\|^2\\
+& \|\bmu_\ktoinf\|^2 + 2\|\bA_k^{-1/2}\bQ_\ktoinf\bmu_\ktoinf\|^2.
\end{align*}

%
%
%
%
\subsection{\texorpdfstring{$\|\muperpridge\|_\bSigma^2$}{Regression bias}}
The quantity $\|\muperpridge\|_\bSigma^2$ is exactly the bias term from \cite{BO_ridge}: indeed, plug in $\bmu$ instead of their $\theta^*$, $\bQ$ instead of their $X$, $\bSigma$ instead of their $\Sigma$, and  $\bA$ instead of their $A$. We could, in principle, just borrow an algebraic bound from their Lemma 28. However, for the sake of completeness, and because we would like a bound in a slightly different form, we do our own derivation here.

As before, we start with the first $k$ components and use Lemma \ref{lm::application of SMW}:
\begin{align*}
&\left\|[\muperpridge]_\uptok\right\|^2_{\bSigma_\uptok}/2 \\
\leq& \left\|(\bI_k - \bQ_\uptok^\top\bA^{-1}\bQ_\uptok)\bmu_\uptok\right\|_{\bSigma_\uptok}^2 + \left\|\bQ_\uptok^\top\bA^{-1}\bQ_\ktoinf\bmu_\ktoinf\right\|_{\bSigma_\uptok}^2 \\
=& \left\| \left(\bI_k + \bQ_\uptok^\top \bA_k^{-1}\bQ_\uptok\right)^{-1}\bmu_\uptok\right\|_{\bSigma_\uptok}^2\\ 
+&\left\| \left(\bI_k + \bQ_\uptok^\top \bA_k^{-1}\bQ_\uptok\right)^{-1}\bQ_\uptok^\top\bA_k^{-1}\bQ_\ktoinf\bmu_\ktoinf\right\|_{\bSigma_\uptok}^2\\
=& \left\| \left(\bSigma_\uptok^{-1} + \bZ_\uptok^\top \bA_k^{-1}\bZ_\uptok\right)^{-1}\bSigma_\uptok^{-1/2}\bmu_\uptok\right\|^2\\ 
+&\left\| \left(\bSigma_\uptok^{-1} + \bZ_\uptok^\top \bA_k^{-1}\bZ_\uptok\right)^{-1}\bZ_\uptok^\top\bA_k^{-1}\bQ_\ktoinf\bmu_\ktoinf\right\|^2\\
\leq& \left(\mu_k(\bZ_\uptok^\top \bZ_\uptok)\mu_1(\bA_k)^{-1}\right)^{-2}\left\| \left(\mu_1(\bZ_\uptok^\top \bZ_\uptok)^{-1}\mu_n(\bA_k)\bSigma_\uptok^{-1} + \bI_k\right)^{-1}\bSigma_\uptok^{-1/2}\bmu_\uptok\right\|^2\\ 
+& \mu_1(\bA_k)^2\mu_n(\bZ_\uptok^\top \bZ_\uptok)^{-2}\mu_1(\bZ_\uptok\bZ_\uptok^\top)\mu_n(\bA_k)^{-2}\left\|\bQ_\ktoinf\bmu_\ktoinf\right\|^2,
\end{align*}
where in the last transition we used Lemma \ref{lm::preserve sigma uptok inverse} and the following observation:
\[
\left\|\left(\bSigma_\uptok^{-1} + \bZ_\uptok^\top \bA_k^{-1}\bZ_\uptok\right)^{-1}\right\| \leq \mu_k(\bZ_\uptok^\top \bA_k^{-1}\bZ_\uptok)^{-1} \leq \mu_1(\bA_k)\mu_k(\bZ_\uptok^\top \bZ_\uptok)^{-1}.
\]

When it comes to the rest of the components, we write
\begin{align*}
&[\muperpridge]_\ktoinf \\
=& \bmu_\ktoinf - \bQ_\ktoinf^\top\bA^{-1}\bQ\bmu\\
=& \bmu_\ktoinf - \bQ_\ktoinf^\top\bA^{-1}\bQ_\ktoinf\bmu_\ktoinf - \bQ_\ktoinf^\top\bA^{-1}\bQ_\uptok\bmu_\uptok\\
=& \bmu_\ktoinf - \bQ_\ktoinf^\top\bA^{-1}\bQ_\ktoinf\bmu_\ktoinf - \bQ_\ktoinf^\top\bA_k^{-1}\bQ_\uptok\left(\bI_k + \bQ_\uptok^\top \bA_k^{-1}\bQ_\uptok\right)^{-1}\bmu_\uptok\\
=& \bmu_\ktoinf - \bQ_\ktoinf^\top\bA^{-1}\bQ_\ktoinf\bmu_\ktoinf - \bQ_\ktoinf^\top\bA_k^{-1}\bZ_\uptok\left(\bSigma_\uptok^{-1} + \bZ_\uptok^\top \bA_k^{-1}\bZ_\uptok\right)^{-1}\bSigma_\uptok^{-1/2}\bmu_\uptok,
\end{align*}
which yields
\begin{align*}
&\|[\muperpridge]_\ktoinf\|_{\bSigma_\ktoinf}\\
\leq& \|\bmu_\ktoinf\|_{\bSigma_\ktoinf} + \|\bQ_\ktoinf\bSigma_\ktoinf\bQ_\ktoinf^\top\|^{1/2}\mu_n(\bA)^{-1}\|\bQ_\ktoinf\bmu_\ktoinf\|\\
+&\|\bQ_\ktoinf\bSigma_\ktoinf\bQ_\ktoinf^\top\|^{1/2}\mu_n(\bA_k)^{-1}\mu_1(\bZ_\uptok^\top \bZ_\uptok)^{1/2}\left\|\left(\bSigma_\uptok^{-1} + \bZ_\uptok^\top \bA_k^{-1}\bZ_\uptok\right)^{-1}\bSigma_\uptok^{-1/2}\bmu_\uptok\right\|\\
\leq& \|\bmu_\ktoinf\|_{\bSigma_\ktoinf} + \|\bQ_\ktoinf\bSigma_\ktoinf\bQ_\ktoinf^\top\|^{1/2}\mu_n(\bA_k)^{-1}\|\bQ_\ktoinf\bmu_\ktoinf\|\\
+&\|\bQ_\ktoinf\bSigma_\ktoinf\bQ_\ktoinf^\top\|^{1/2}\frac{\mu_n(\bA_k)^{-1}\mu_1(\bZ_\uptok^\top \bZ_\uptok)^{1/2}}{\mu_k(\bZ_\uptok^\top \bZ_\uptok)\mu_1(\bA_k)^{-1}}\\
\times& \left\|\left(\mu_1(\bZ_\uptok^\top \bZ_\uptok)^{-1}\mu_n(\bA_k)\bSigma_\uptok^{-1} + \bI_k\right)^{-1}\bSigma_\uptok^{-1/2}\bmu_\uptok\right\|,
\end{align*}
where we used Lemma \ref{lm::preserve sigma uptok inverse} and the fact that $\mu_n(\bA) \geq \mu_n(\bA_k)$ in the last transition. Combining everything together and using the inequality $(a + b + c)^2 \leq 3(a^2 + b^2 + c^2)$ yields the final bound.
%
%
%
%
\subsection{\texorpdfstring{$\tr(\bA^{-1}\bQ\bSigma\bQ^\top\bA^{-1})$}{Regression expected variance}}
\label{sec::algebraic decomposition for V}
The quantity $\tr(\bA^{-1}\bQ\bSigma\bQ^\top\bA^{-1})$ is exactly the variance term from \cite{BO_ridge}: as for the bias term,  plug in  $\bQ$ instead of their $X$, $\bSigma$ instead of their $\Sigma$, and $\bA$ instead of their $A$.  As before, we could in principle use their algebraic decomposition from their Lemma 27, but we provide our own derivation because we want to obtain the bound in a slightly different form.

\begin{align*}
&\tr(\bA^{-1}[\bQ_\uptok, \bQ_\ktoinf]\bSigma[\bQ_\uptok, \bQ_\ktoinf]^\top\bA^{-1})\\
=& \tr(\bA^{-1}\bQ_\uptok\bSigma_\uptok\bQ_\uptok^\top\bA^{-1}) + \tr(\bA^{-1}\bQ_\ktoinf\bSigma_\ktoinf\bQ_\ktoinf^\top\bA^{-1})\\
=& \tr\left(\bA_k^{-1}\bQ_\uptok\left(\bI_k + \bQ_\uptok^\top \bA_k^{-1}\bQ_\uptok\right)^{-1}\bSigma_\uptok\left(\bI_k + \bQ_\uptok^\top \bA_k^{-1}\bQ_\uptok\right)^{-1}\bQ_\uptok^\top\bA_k^{-1}\right)\\
+& \tr(\bA^{-1}\bQ_\ktoinf\bSigma_\ktoinf\bQ_\ktoinf^\top\bA^{-1})\\
=& \tr\left(\bA_k^{-1}\bZ_\uptok\left(\bSigma_\uptok^{-1} + \bZ_\uptok^\top \bA_k^{-1}\bZ_\uptok\right)^{-2}\bZ_\uptok^\top\bA_k^{-1}\right)\\
+& \tr(\bA^{-1}\bQ_\ktoinf\bSigma_\ktoinf\bQ_\ktoinf^\top\bA^{-1})
\end{align*}

Now let's bound these two terms separately. First, recall that for any PSD matrices $\bM_1$ and $\bM_2$ the following holds: $\tr(\bM_1\bM_2) \leq \|\bM_1\|\tr(\bM_2)$. We use this to bound the first term as follows:
\begin{align*}
&\tr\left(\bA_k^{-1}\bZ_\uptok\left(\bSigma_\uptok^{-1} + \bZ_\uptok^\top \bA_k^{-1}\bZ_\uptok\right)^{-2}\bZ_\uptok^\top\bA_k^{-1}\right)\\
\leq& \mu_n(\bA_k)^{-2}\tr\left(\bZ_\uptok\left(\bSigma_\uptok^{-1} + \bZ_\uptok^\top \bA_k^{-1}\bZ_\uptok\right)^{-2}\bZ_\uptok^\top\right)\\
=&\mu_n(\bA_k)^{-2}\tr\left(\left(\bSigma_\uptok^{-1} + \bZ_\uptok^\top \bA_k^{-1}\bZ_\uptok\right)^{-2}\bZ_\uptok^\top\bZ_\uptok\right)\\
\leq&\mu_n(\bA_k)^{-2}\mu_1(\bZ_\uptok^\top\bZ_\uptok)\tr\left(\left(\bSigma_\uptok^{-1} + \bZ_\uptok^\top \bA_k^{-1}\bZ_\uptok\right)^{-2}\right)\\
\leq& \frac{\mu_n(\bA_k)^{-2}\mu_1(\bZ_\uptok^\top\bZ_\uptok)}{\mu_k(\bZ_\uptok^\top \bZ_\uptok)^2\mu_1(\bA_k)^{-2}}\tr\left(\left(\mu_1(\bZ_\uptok^\top \bZ_\uptok)^{-1}\mu_n(\bA_k)\bSigma_\uptok^{-1} + \bI_k\right)^{-2}\right),
\end{align*}
where we used Lemma \ref{lm::preserve sigma uptok inverse} in the last transition.

When it comes to the second term, we simply write
\begin{align*}
\tr(\bA^{-1}\bQ_\ktoinf\bSigma_\ktoinf\bQ_\ktoinf^\top\bA^{-1}) \leq \mu_n(\bA_k)^{-2}\tr\left(\bQ_\ktoinf\bSigma_\ktoinf\bQ_\ktoinf^\top\right).
\end{align*}
%
%
%
%
\subsection{\texorpdfstring{$\|\bA^{-1}\bQ\bSigma\bQ^\top\bA^{-1}\|$}{Regression variance deviation}}
\label{sec::variance spectral norm algebraic}
In  \cite{BO_ridge} the deviations of the variance term in noise were dealt with in the following way: Hanson-Wright inequality states that a quadratic form $\beps^\top \bM \beps$, where $\beps$ is a vector with i.i.d. centered sub-Gaussian components concentrates around $\tr(\bM)$, with deviations being composed of a sub-Gaussian tail controlled by $\|\bM\|_F$ and sub-exponential tail controlled by $\|\bM\|$. In  \cite{BO_ridge} the latter two quantities were bounded as $\|\bM\|_F^2 \leq \tr(\bM)^2$ and $\|\bM\| \leq \tr(\bM)$, so only the bound on the trace of the matrix $\bA^{-1}\bQ\bSigma\bQ^\top\bA^{-1}$ was required. Instead of making such step, we can bound the spectral norm separately, and then use $\|\bM\|_F^2 \leq \|\bM\|\tr(\bM)$. This section shows the following:
\[
\|\bA^{-1}\bQ\bSigma\bQ^\top\bA^{-1}\| \leq \frac{\mu_1(\bA_k^{-1})^2}{\mu_n(\bA_k^{-1})^2}\frac{\mu_1(\bZ_\uptok^\top\bZ_\uptok)}{\mu_k(\bZ_\uptok^\top\bZ_\uptok)^2} + \frac{\|\bQ_\ktoinf\bSigma_\ktoinf\bQ_\ktoinf^\top \|}{\mu_n(\bA_k)^2}.
\]

We bound the operator norm as follows: first, we decompose into two terms
\[
\bA^{-1}\bQ \bSigma \bQ^\top \bA^{-1} = \bA^{-1}\bQ_\uptok \bSigma_\uptok \bQ_\uptok^\top \bA^{-1} + \bA^{-1}\bQ_\ktoinf \bSigma_\ktoinf \bQ_\ktoinf^\top \bA^{-1}.
\]
The second term is straightforward:
\begin{align*}
&\|\bA^{-1}\bQ_\ktoinf \bSigma_\ktoinf \bQ_\ktoinf^\top \bA^{-1}\|\\
\leq &\|\bA^{-1}\|\|\bQ_\ktoinf\bSigma_\ktoinf\bQ_\ktoinf^\top \|\|\bA^{-1}\|\\
=&\frac{\|\bQ_\ktoinf\bSigma_\ktoinf\bQ_\ktoinf^\top \|}{\mu_n(\bA_k)^2}.
\end{align*}

For the first term we use Equation \eqref{eq::favorite identity} to write
\begin{gather*}
 \|\bA^{-1}\bQ_\uptok \bSigma_\uptok \bQ_\uptok^\top \bA^{-1}\|\\
  = \left\|\bA_k^{-1}\bQ_\uptok\left(\bI_k + \bQ_\uptok^\top \bA_k^{-1}\bQ_\uptok\right)^{-1}\bSigma_\uptok \left(\bI_k + \bQ_\uptok^\top \bA_k^{-1}\bQ_\uptok\right)^{-1}\bQ_\uptok^\top \bA_k^{-1}\right\|\\
  = \|\bA_k^{-1}\bZ_\uptok\left(\bSigma_\uptok^{-1} + \bZ_\uptok^\top \bA_k^{-1}\bZ_\uptok\right)^{-2}\bZ_\uptok^\top \bA_k^{-1}\|\\
  \leq \left\|\left(\bSigma_\uptok^{-1} + \bZ_\uptok^\top \bA_k^{-1}\bZ_\uptok\right)^{-2}\right\|\|\bZ_\uptok\bZ_\uptok^\top\|\|\bA_k^{-1}\|^2\\
  \leq \left(\lambda_1^{2} \wedge \left\|\left(\bZ_\uptok^\top \bA_k^{-1}\bZ_\uptok\right)^{-2}\right\|\right)\|\bZ_\uptok\bZ_\uptok^\top\|\|\bA_k^{-1}\|^2\\
  \leq \frac{\mu_1(\bA_k)^2}{\mu_n(\bA_k)^2}\frac{\mu_1(\bZ_\uptok^\top\bZ_\uptok)}{\mu_k(\bZ_\uptok^\top\bZ_\uptok)^2} \wedge \frac{\lambda_1^{2}\mu_1(\bZ_\uptok^\top\bZ_\uptok)}{\mu_n(\bA_k)^2}.
\end{gather*}

\section{Randomness in covariates}
\label{sec::randomness in Q}

\begin{restatable}[Randomness in covariates]{lemma}{randomnessinQ}\label{lm::randomness in Q but no eigvals of A}
Consider some $L> 1$. There exists a constant $c$ that only depends on $c_B$ and $L$ such that the following holds. Denote 
\[
\Lambda = \lambda + \sum_{i > k}\lambda_i
\]
Assume that  $k < n/c$ and 
\[
\Lambda >  cn\lambda_{k+1} \vee \sqrt{n\sum_{i > k}\lambda_i^2}.
\]
 Then  all the following hold on the event $\sA_k(L) \cap \sB_k(c_B)$:
\begin{enumerate}
\item
\begin{align*}
\|\bA^{-1}\bnu\| \leq& c\left(n^{-1/2}\left\|\left(\Lambda n^{-1}\bSigma_\uptok^{-1} + \bI_k\right)^{-1}\bSigma_\uptok^{-1/2}\bmu_\uptok\right\| + \Lambda^{-1}{\sqrt{n}\|\bmu_\ktoinf\|_{\bSigma_\ktoinf}}\right);
\end{align*}
\item 
\begin{align*}
c\bmu^\top\muperpridge \geq\frac{\Lambda}{n}\left\|\left(\Lambda n^{-1}\bSigma_\uptok^{-1} + \bI_k\right)^{-1/2}\bSigma_\uptok^{-1/2}\bmu_\uptok\right\| + \|\bmu_\ktoinf\|^2 \geq \bmu^\top\muperpridge/c;
\end{align*}

\item 
\begin{align*}
\|\muperpridge\|_\bSigma^2/c
\leq \Lambda^2n^{-2}\left\| \left(\Lambda n^{-1}\bSigma_\uptok^{-1} + \bI_k\right)^{-1}\bSigma_\uptok^{-1/2}\bmu_\uptok\right\|^2 + \left\|\bmu_\ktoinf\right\|_{\bSigma_\ktoinf}^2;
\end{align*}
\item 
\begin{align*}
\tr(\bA^{-1}\bQ\bSigma\bQ^\top\bA^{-1})
\leq& c\left(n^{-1}\tr\left(\left(\Lambda n^{-1}\bSigma_\uptok^{-1} + \bI_k\right)^{-2}\right)+ \Lambda^{-2}n\sum_{i > k}\lambda_i^2\right);
\end{align*}
\item 
\begin{align*}
\|\bA^{-1}\bQ\bSigma\bQ^\top\bA^{-1}\| \leq&   c\left(\frac{1}{n} \wedge \frac{n\lambda_1^2}{\Lambda^2} + \frac{n\lambda_{k+1}^2 + \sum_{i > k}\lambda_i^2}{\Lambda^2}\right).
\end{align*}

\end{enumerate}
\end{restatable}
\begin{proof}
Recall that on $\sA_k(L)$ we have $\Lambda/L \leq \mu_n(\bA_k) \leq \mu_1(\bA_k)\leq L\Lambda$.

 The proof is rather straightforward: we plug the bounds from the definition of the event $\sB_k(c_B)$ from Section \ref{sec::definition of sB_k},  and the bounds on eigenvalues of $\bA_k$ from the definition of the event $\sA_k(L)$ into the result of Lemma \ref{lm::algebraic decompositions}.  Recall that $c_B$ is the constant from the definition of $\sB_k(c_B)$. 
 
 On $\sA_k(L)\cap \sB_k(c_B)$ all the following hold:
\begin{enumerate}
\item 
\begin{align*}
\|\bA^{-1}\bnu\| \leq& \frac{\mu_1(\bA_k)}{\mu_n(\bA_k)}\frac{\sqrt{\mu_1(\bZ_\uptok^\top\bZ_\uptok})}{\mu_k(\bZ_\uptok^\top\bZ_\uptok)}\left\|\left(\mu_1(\bZ_\uptok^\top\bZ_\uptok)^{-1}\mu_n(\bA_k)\bSigma_\uptok^{-1} + \bI_k\right)^{-1}\bSigma_\uptok^{-1/2}\bmu_\uptok\right\|\\
+& \mu_n(\bA_k)^{-1}\|\bQ_\ktoinf \bmu_\ktoinf\|\\
\leq&\frac{ c_B^{3/2}L^2}{\sqrt{n}}\cdot c_BL\left\|\left(\Lambda n^{-1}\bSigma_\uptok^{-1} + \bI_k\right)^{-1}\bSigma_\uptok^{-1/2}\bmu_\uptok\right\| + c_B^{1/2}L\Lambda^{-1}{\sqrt{n}\|\bmu_\ktoinf\|_{\bSigma_\ktoinf}}.
\end{align*}
Note that in the last transition we used that for a positive scalar $a < 1$ we can write 
\[
\left\|\left(a\Lambda n^{-1}\bSigma_\uptok^{-1} + \bI_k\right)^{-1}\bSigma_\uptok^{-1/2}\bmu_\uptok\right\| \leq a^{-1}\left\|\left(\Lambda n^{-1}\bSigma_\uptok^{-1} + \bI_k\right)^{-1}\bSigma_\uptok^{-1/2}\bmu_\uptok\right\|.
\]

It is easy to see that this is correct since both matrices $\bI_k$ and $\bSigma_\uptok$ are diagonal. We will make such a transition several more times throughout this proof, as well as the following for $b \geq 1$:
\[
\left\|\left(b\Lambda n^{-1}\bSigma_\uptok^{-1} + \bI_k\right)^{-1}\bSigma_\uptok^{-1/2}\bmu_\uptok\right\| \geq b^{-1}\left\|\left(\Lambda n^{-1}\bSigma_\uptok^{-1} + \bI_k\right)^{-1}\bSigma_\uptok^{-1/2}\bmu_\uptok\right\|.
\]

\item 
We start with the lower bound on $\bmu^\top\muperpridge$:
\begin{align*}
\bmu^\top\muperpridge \geq& \frac12 \mu_n(\bA_k)\mu_1(\bZ_\uptok^\top \bZ_\uptok)^{-1}\left\|\left(\mu_n(\bA_k)\mu_1(\bZ_\uptok^\top \bZ_\uptok)^{-1}\bSigma^{-1}_\uptok + \bI_k\right)^{-1/2}\bSigma^{-1/2}_\uptok\bmu_\uptok\right\|^2 \\
+&\|\bmu_\ktoinf\|^2 - 9\mu_n(\bA_k)^{-1}\|\bQ_\uptok^\top\bmu_\ktoinf\|^2\\
\geq& \frac{\Lambda}{2Lc_Bn}\cdot\frac{1}{c_BL}\left\|\left(\Lambda n^{-1}\bSigma_\uptok^{-1} + \bI_k\right)^{-1/2}\bSigma_\uptok^{-1/2}\bmu_\uptok\right\|\\
+&  \|\bmu_\ktoinf\|^2 - 9Lc_B\Lambda^{-1}n\|\bmu_\ktoinf\|^2_{\bSigma_\ktoinf}\\
\geq&  \frac{\Lambda}{2L^2c_B^2n}\left\|\left(\Lambda n^{-1}\bSigma_\uptok^{-1} + \bI_k\right)^{-1/2}\bSigma_\uptok^{-1/2}\bmu_\uptok\right\| + \|\bmu_\ktoinf\|^2(1 - 9Lc_Bn\lambda_{k+1}\Lambda^{-1}),
\end{align*}
where in the last line we used that $\|\bmu_\ktoinf\|^2_{\bSigma_\ktoinf} \leq \lambda_{k+1}\|\bmu_\ktoinf\|^2$. Note that if we take $c > 18c_BL$ in the end, then $1 - 9Lc_Bn\lambda_{k+1}\Lambda^{-1} > 0.5$ since we assumed that $\Lambda > cn\lambda_{k+1}$. 

Now, we do the upper bound:
\begin{align*}
&\bmu^\top\muperpridge\\
\leq& 3\mu_1(\bA_k)\mu_k(\bZ_\uptok^\top\bZ_\uptok)^{-1}\left\|\left(\mu_k(\bZ_\uptok^\top \bZ_\uptok)^{-1}\mu_1(\bA_k)\bSigma^{-1}_\uptok + \bI_k\right)^{-1/2}\bSigma^{-1/2}_\uptok\bmu_\uptok\right\|^2\\
+& \|\bmu_\ktoinf\|^2 + 2\|\bA_k^{-1/2}\bQ_\ktoinf\bmu_\ktoinf\|^2\\
\leq& 3Lc_B\Lambda n^{-1}\cdot c_BL\left\|\left(\Lambda n^{-1}\bSigma^{-1}_\uptok + \bI_k\right)^{-1/2}\bSigma^{-1/2}_\uptok\bmu_\uptok\right\|^2\\
+& \|\bmu_\ktoinf\|^2 + 2Lc_B\Lambda^{-1}n\|\bmu_\ktoinf\|^2_{\bSigma_\ktoinf}\\
\leq& 3L^2c_B^2\Lambda n^{-1}\left\|\left(\Lambda n^{-1}\bSigma^{-1}_\uptok + \bI_k\right)^{-1/2}\bSigma^{-1/2}_\uptok\bmu_\uptok\right\|^2 + \|\bmu_\ktoinf\|^2(1 + 2Lc_B\lambda_{k+1}\Lambda^{-1}n),
\end{align*}
where, as before, we used $\|\bmu_\ktoinf\|^2_{\bSigma_\ktoinf} \leq \lambda_{k+1}\|\bmu_\ktoinf\|^2$ in the last line. Note that here $\Lambda > cn\lambda_{k+1}$ implies that $2Lc_B\lambda_{k+1}\Lambda^{-1}n \leq 2Lc_B/c < 1$ for $c$ large enough.
\item 
The upper bound on $\|\muperpridge\|_\bSigma^2$, is  very similar to the bound on the bias term in \cite{BO_ridge}, but has a slightly different form. We derive it below.
\begin{align*}
&{\|\muperpridge\|_\bSigma^2} \\
 \leq&  2\mu_k(\bZ_\uptok^\top \bZ_\uptok)^{-2}\mu_1(\bA_k)^{2}\left\| \left(\mu_1(\bZ_\uptok^\top \bZ_\uptok)^{-1}\mu_n(\bA_k)\bSigma_\uptok^{-1} + \bI_k\right)^{-1}\bSigma_\uptok^{-1/2}\bmu_\uptok\right\|^2\\ 
+& 2\frac{\mu_1(\bA_k)^2\mu_1(\bZ_\uptok\bZ_\uptok^\top)}{\mu_n(\bZ_\uptok^\top \bZ_\uptok)^{2}\mu_n(\bA_k)^{2}}\left\|\bQ_\ktoinf\bmu_\ktoinf\right\|^2\\
+&3\|\bmu_\ktoinf\|_{\bSigma_\ktoinf}^2 + 3\|\bQ_\ktoinf\bSigma_\ktoinf\bQ_\ktoinf^\top\|\mu_n(\bA_k)^{-2}\|\bQ_\ktoinf\bmu_\ktoinf\|^2\\
+&3\|\bQ_\ktoinf\bSigma_\ktoinf\bQ_\ktoinf^\top\|\frac{\mu_1(\bA_k)^{2}\mu_1(\bZ_\uptok^\top \bZ_\uptok)}{\mu_k(\bZ_\uptok^\top \bZ_\uptok)^2\mu_n(\bA_k)^{2}}\\
\times&\left\|\left(\mu_1(\bZ_\uptok^\top \bZ_\uptok)^{-1}\mu_n(\bA_k)\bSigma_\uptok^{-1} + \bI_k\right)^{-1}\bSigma_\uptok^{-1/2}\bmu_\uptok\right\|^2\\
\leq& 2c_B^2L^2\Lambda^2n^{-2}\cdot c_B^2L^2\left\| \left(\Lambda n^{-1}\bSigma_\uptok^{-1} + \bI_k\right)^{-1}\bSigma_\uptok^{-1/2}\bmu_\uptok\right\|^2\\ 
+& 2L^4c_B^3n^{-1}\cdot c_Bn\left\|\bmu_\ktoinf\right\|_{\bSigma_\ktoinf}^2\\
+&3\|\bmu_\ktoinf\|_{\bSigma_\ktoinf}^2 + 3c_B\left(n\lambda_{k+1}^2 + \sum_{i > k}\lambda_i^2\right)\cdot L^2\Lambda^{-2} c_Bn\left\|\bmu_\ktoinf\right\|_{\bSigma_\ktoinf}^2\\
+&3c_B\left(n\lambda_{k+1}^2 + \sum_{i > k}\lambda_i^2\right)\cdot L^4c_B^3 n^{-1} \cdot c_B^2L^2\left\|\left(\Lambda n^{-1}\bSigma_\uptok^{-1} + \bI_k\right)^{-1}\bSigma_\uptok^{-1/2}\bmu_\uptok\right\|^2\\
=& \left\| \left(\Lambda n^{-1}\bSigma_\uptok^{-1} + \bI_k\right)^{-1}\bSigma_\uptok^{-1/2}\bmu_\uptok\right\|^2 \cdot \left(2c_B^4L^4\Lambda^2n^{-2} + 3c_B^6L^6n^{-1}\left(n\lambda_{k+1}^2 + \sum_{i > k}\lambda_i^2\right)\right)\\
+& \left\|\bmu_\ktoinf\right\|_{\bSigma_\ktoinf}^2 \cdot \left(2L^4c_B^4 + 3 + 3c_B^2L^2\Lambda^{-2}n\left(n\lambda_{k+1}^2 + \sum_{i > k}\lambda_i^2\right)\right)
\end{align*}

Recall that we imposed the assumption that 
\[
\Lambda > n\lambda_{k+1},\quad \Lambda > \sqrt{n\sum_{i > k}\lambda_i^2}.
\]
Therefore
\[
n\left(n\lambda_{k+1}^2 + \sum_{i > k}\lambda_i^2\right) \leq 2\Lambda^2. 
\]
Plugging this inequality in and taking $c$ large enough depending on $c_B$ and $L$ gives the bound.

\item The upper bound on $\tr(\bA^{-1}\bQ\bSigma\bQ^\top\bA^{-1})$, is also very similar to the one derived in \cite{BO_ridge}, where it is exactly the the variance term, but has a slightly different form. We derive it below:
\begin{align*}
&\tr(\bA^{-1}\bQ\bSigma\bQ^\top\bA^{-1})\\
\leq& \frac{\mu_1(\bA_k)^{2}\mu_1(\bZ_\uptok^\top\bZ_\uptok)}{\mu_k(\bZ_\uptok^\top \bZ_\uptok)^2\mu_n(\bA_k)^{2}}\tr\left(\left(\mu_1(\bZ_\uptok^\top \bZ_\uptok)^{-1}\mu_n(\bA_k)\bSigma_\uptok^{-1} + \bI_k\right)^{-2}\right)\\
+&  \mu_n(\bA_k)^{-2}\tr(\bQ_\ktoinf\bSigma_\ktoinf \bQ_\ktoinf^\top)\\
\leq& L^4 c_B^3 n^{-1} \cdot c_BL \tr\left(\left(\Lambda n^{-1}\bSigma_\uptok^{-1} + \bI_k\right)^{-2}\right)+ L^2c_B\Lambda^{-2}n\sum_{i > k}\lambda_i^2.
\end{align*}
\item 
\begin{align*}
&\|\bA^{-1}\bQ\bSigma\bQ^\top\bA^{-1}\|\\
\leq&  \frac{\mu_1(\bA_k)^2}{\mu_n(\bA_k)^2}\frac{\mu_1(\bZ_\uptok^\top\bZ_\uptok)}{\mu_k(\bZ_\uptok^\top\bZ_\uptok)^2} \wedge \frac{\lambda_1^{2}\mu_1(\bZ_\uptok^\top\bZ_\uptok)}{\mu_n(\bA_k)^2} + \frac{\|\bQ_\ktoinf\bSigma_\ktoinf\bQ_\ktoinf^\top \|}{\mu_n(\bA_k)^2}\\
\leq& \frac{L^4c_B^3}{n} \wedge \frac{\lambda_1^{2}c_Bn}{L^{-2}\Lambda^2} + c_BL^2\Lambda^{-2}\left(n\lambda_{k+1}^2 + \sum_{i > k}\lambda_i^2\right).
\end{align*}

\end{enumerate}

In all the cases above we see that taking $c$ large enough depending on $c_B$ and $L$ yields the  result.
\end{proof}

\section{Proof of the main lower bound}
\label{sec::main bound put together}

First of all, we combine Lemma \ref{lm::randomness in y} with Lemma \ref{lm::randomness in Q but no eigvals of A} to obtain high probability bounds on all the terms that appear in quantities of interest. The result is given by the following
\begin{restatable}[High probability bounds on separate terms]{lemma}{highprobbounds}\label{lm::high probability bounds}
Consider some $L> 1$. There exists a constant $c$ that only depends on $c_B$ and $L$ and an absolute constant $c_y$ such that the following holds. Assume that $\eta < c^{-1}$. 
 Assume that $k < n/c$
\[
\Lambda > cn\lambda_{k+1} \vee \sqrt{n\sum_{i > k}\lambda_i^2}.
\] 

For any $t \in (0, \sqrt{n}/c_y)$, conditionally on the event $\sA_k(L)\cap \sB_k(c_B)$, with probability is at least $1 - c_ye^{-t^2/2}$ over the draw of $(\by, \hat{\by})$ all the following hold:
\begin{enumerate}
\item
\begin{align*}
|\bnu^\top \bA^{-1}\by| \vee |\bnu^\top \bA^{-1}\hat{\by}| \leq& ct\Diamond,
\end{align*}
\item
\begin{align*}
|\bnu^\top \bA^{-1}\Delta\by|\leq& ct\sigma_\eta\Diamond,
\end{align*}
\item
\begin{align*}
|\Delta\by^\top \bA^{-1}\by| \leq& c\sigma_{\eta}n\Lambda^{-1},
\end{align*}
\item
\begin{align*}
\by^\top \bA^{-1}\hat{\by} \geq& c^{-1}n\Lambda^{-1},
\end{align*}
\item
\begin{align*}
cn\Lambda^{-1}\geq \by^\top \bA^{-1}\by \geq c^{-1}n\Lambda^{-1},
\end{align*}
\item
\begin{align*}
\|\bQ^\top\bA^{-1}\Delta\by\|_\bSigma^2 \leq& c\sigma_{\eta}^2(V + t^2\Delta V),
\end{align*}
\item
\begin{align*}
\|\bQ^\top\bA^{-1}\by\|_\bSigma^2 \leq&  c(V + t^2\Delta V),
\end{align*}
\item
\begin{align*}
 cM \geq \bmu^\top\muperpridge \geq& c^{-1}M,
\end{align*}
\item
\begin{align*}
\|\muperpridge\|_\bSigma \leq& c\Lambda\Diamond/\sqrt{n},
\end{align*}

\end{enumerate}
\end{restatable}
\begin{proof}
Parts 1, 2, 6 and 7 can be obtained directly from the corresponding parts of Lemma \ref{lm::randomness in y} by plugging in the bounds from Lemma \ref{lm::randomness in Q but no eigvals of A}. Parts 8 and 9 are exactly parts 2 and 3 of Lemma \ref{lm::randomness in Q but no eigvals of A}. Thus, only parts 3, 4, and 5 require additional explanation, which we provide below.

First of all, note that $\|\bA^{-1}\| \leq \|\bA_k^{-1}\| = \mu_n(\bA_k)^{-1}$, since $\bA$ is larger than $\bA_k$ w.r.t. Loewner order. Thus, on $\sA_k(L)$ we have
\[
\|\bA^{-1}\| \leq L\Lambda^{-1}, \quad \mu_1(\bA_k)^{-1}\geq L^{-1}\Lambda^{-1},
\]

Now let's explain parts 3, 4, 5 starting with the corresponding parts of Lemma \ref{lm::randomness in y}. Denote the (absolute) constant from that lemma as $c_1$. In all the following we plug in the bounds on eigenvalues of $\bA_k$,  together with $\eta < 1/c$, $k < n/c$ and $t < \sqrt{n}/c$. In the end of each derivation we need to take $c$ large enough depending on $L$ and $c_1$. 

By Lemma \ref{lm::randomness in y} for every $t \in (0, \sqrt{n}/c_1)$ on $\sA_k(L) \cap \sB_k(c_B)$ we have with probability at least $1 - c_1 e^{-t^2/2}$ over the draw of $(\by, \hat{\by})$
\newcommand\setItemnumber[1]{\setcounter{enumi}{\numexpr#1-1\relax}}
\begin{enumerate}
\setItemnumber{3}
\item 
\begin{align*}
|\Delta\by^\top \bA^{-1}\by| \leq& c_1\|\bA^{-1}\|\left(n\eta + t\sigma_\eta(\sqrt{n} + t)\right)\\
\leq& c_1 L\Lambda(n\sigma_\eta + \sqrt{n}\sigma_\eta(\sqrt{n} + \sqrt{n})\\
\leq& cn\Lambda^{-1}\sigma_\eta.
\end{align*}

\item 
\begin{align*}
\by^\top \bA^{-1}\hat{\by} \geq& (n - n\eta - c_1t\sigma_\eta\sqrt{n} - k)\mu_1(\bA_k)^{-1}\\
-&  (n\eta + c_1t\sigma_\eta\sqrt{n} + c_1t\sqrt{n} + c_1t^2)\|\bA^{-1}\|\\
\geq& L^{-1}\Lambda^{-1}n(1 - 1/c -c_1\sigma_\eta/c - 1/c -L^2/c - c_1L^2\sigma_\eta/c - L^2c_1/c)\\
\geq& c^{-1}n\Lambda^{-1}.
\end{align*}
\item 
\[
c n \Lambda^{-1} \geq nL\Lambda^{-1} \geq n\|\bA^{-1}\| \geq \by^\top \bA^{-1}\by,
\]
and
\begin{align*}
\by^\top \bA^{-1}\by \geq& (n - k)\mu_1(\bA_k)^{-1} -  c_1(t\sqrt{n} + t^2)\|\bA^{-1}\|\\
\geq& (n - n/c)L^{-1}\Lambda^{-1} - c_1(n/c + n/c^2)L\Lambda^{-1}\\
\geq& c^{-1}n\Lambda^{-1}.
\end{align*}
\end{enumerate}
\end{proof}

\mainresult*
\begin{proof}
Note that increasing $c$ only makes the statement weaker. From the very beginning let's put $c$ to be large enough, so that $c > 1$ and $\sigma_\eta < 1$.

Recall that
\[
\Delta V := \frac{k \wedge1}{n} + \frac{n\lambda_{k+1}^2 + \sum_{i > k}\lambda_i^2}{\left(\lambda + \sum_{i > k}\lambda_i\right)^2}.
\]

The plan is to plug in the bounds from Lemma \ref{lm::high probability bounds} into quantities of interest, the formulas for which are given by Lemma \ref{lm::solution formulas}. First of all, however, we need to make sure that $S > 0$. This is required to write $\|S\solnridge\|_{\bSigma} = S\|\solnridge\|_{\bSigma}$ and then cancel $S$ in the numerator and denominator. This is indeed the case since $S = (1 + \bnu^\top \bA^{-1}\by)^2 + \by^\top\bA^{-1}\by\bmu^\top\muperpridge$, and in Lemma \ref{lm::high probability bounds} we bound $\bmu^\top\muperpridge$ and $\by^\top\bA^{-1}\by$ from below by strictly positive quantities.

Let's plug in the bounds Lemma \ref{lm::high probability bounds} in the formulas from Lemma \ref{lm::solution formulas}:  denote the constant from Lemma \ref{lm::high probability bounds} as $c_1$ and write
\begin{enumerate}
\item $S\bmu^\top\solnridge$: recall that $\by_C$ is the vector of labels of clean points: $\by_C = \by + \Delta\by/2 =  \hat{\by} - \Delta\by/2$. Now we write
\begin{align*}
&\by^\top\bA^{-1}\hat{\by} \bmu^\top\muperpridge + (1 +  \bnu^\top\bA^{-1}\by) \bnu^\top \bA^{-1}\hat{\by}\\
=& \by^\top\bA^{-1}\hat{\by} \bmu^\top\muperpridge + \bnu^\top \bA^{-1}\hat{\by} + \bnu^\top\bA^{-1}(\by_C - \Delta\by/2) \bnu^\top \bA^{-1}(\by_C + \Delta\by/2)\\
=&\by^\top\bA^{-1}\hat{\by} \bmu^\top\muperpridge + \bnu^\top \bA^{-1}\hat{\by} - (\bnu^\top\bA^{-1}\Delta\by)^2/4 + (\bnu^\top\bA^{-1}\by_C)^2\\
\geq& \by^\top\bA^{-1}\hat{\by} \bmu^\top\muperpridge + \bnu^\top \bA^{-1}\hat{\by} - (\bnu^\top\bA^{-1}\Delta\by)^2/4\\
\geq& \frac{n}{c_1^2\Lambda}M - c_1t\Diamond - c_1^2t^2\sigma_\eta^2\Diamond^2\\
=& \frac{N}{c_1^2} - c_1t\Diamond - c_1^2t^2\sigma_\eta^2\Diamond^2.
\end{align*}
Recall that by Lemma \ref{lm::relations} we have $N \geq n\Diamond^2$, which yields
\[
\frac{N}{c_1^2} - c_1^2t^2\sigma_\eta^2\Diamond^2 \geq \frac{N}{c_1^2} - \frac{c_1^2}{c^2}n\Diamond^2 \geq \frac{N}{2c_1^2},
\]
where the last transition is correct if $c$ is taken large enough depending on $c_1$. Thus, we get
\[
\by^\top\bA^{-1}\hat{\by} \bmu^\top\muperpridge + (1 +  \bnu^\top\bA^{-1}\by) \bnu^\top \bA^{-1}\hat{\by} \geq \frac{N}{2c_1^2} - c_1t\Diamond.
\]

\item $S\|\solnridge\|_\bSigma$, the first term:
\begin{align*}
&\left[(1 + |\bnu^\top \bA^{-1}\by|)^2 + \by^\top\bA^{-1}\by\bmu^\top\muperpridge\right] \|\bQ^\top\bA^{-1}\Delta\by\|_\bSigma\\
\leq&\left[(1 + c_1t\Diamond)^2 + c_1^2M n\Lambda^{-1}\right] \sqrt{c_1\sigma_\eta^2(V + t^2\Delta V)}\\
\leq& c_1^{2.5}\left[\left(1 +t\Diamond\right)^2 +  n\Lambda^{-1}M\right] \sigma_\eta\sqrt{ V + t^2\Delta V}.
\end{align*}
\item $S\|\solnridge\|_\bSigma$, the second term
\begin{align*}
&\left[(1 + |\bnu^\top \bA^{-1}\by|)(1 + |\bnu^\top\bA^{-1}\Delta\by|) + |\Delta\by^\top \bA^{-1}\by|\bmu^\top\muperpridge\right]\|\bQ^\top\bA^{-1}\by\|_\bSigma\\
\leq&\left[(1 + c_1t\Diamond)(1 + c_1\sigma_\eta t\Diamond) + c_1^2M \sigma_\eta n\Lambda^{-1}\right]\sqrt{c_1(V + t^2\Delta V)}\\
\leq& c_1^{2.5}\Biggl[1 + (1 + \sigma_\eta)t\Diamond + \sigma_\eta t^2\Diamond^2+  n\Lambda^{-1}M\sigma_\eta \Biggr] \sqrt{ V + t^2\Delta V}\\
\leq& 2c_1^3\Biggl[1 + t\Diamond + n\Lambda^{-1}M\sigma_\eta \Biggr] \sqrt{ V + t^2\Delta V},
\end{align*}
where we used that $\sigma_\eta < 1$ and $t^2\Diamond^2 < n\Diamond^2 < n\Lambda^{-1}M$ in the last transition (by Lemma \ref{lm::relations}).

\item $S\|\solnridge\|_\bSigma$, the third term
\begin{align*}
& \left[\by^\top \bA^{-1}\by + (1 + |\bnu^\top \bA^{-1}\by|)|\Delta\by^\top \bA^{-1}\by| +\by^\top\bA^{-1}\by|\bnu^\top\bA^{-1}\Delta\by|\right]\|\muperpridge\|_\bSigma\\
\leq&\left[\frac{c_1n}{\Lambda} + (1 + c_1t\Diamond)\frac{c_1n\sigma_\eta}{\Lambda} + \frac{c_1^2nt\sigma_\eta\Diamond}{\Lambda}\right]\frac{c_1\Lambda}{\sqrt{n}}\Diamond\\
\leq& 2c_1^3\sqrt{n}(1 + t\sigma_\eta \Diamond)\Diamond,
\end{align*}
where we used $c_1n\Lambda^{-1}(1 + \sigma_\eta) \leq 2 c_1n\Lambda^{-1}$ in the last line to reduce the number of terms.

\end{enumerate}

Combining all the terms for $S\|\solnridge\|_\bSigma$, we get that for some new constant $c_2$ that only depends on $L$ and $c_B$ under the condition that $t \leq \sqrt{n}/c_2$ and $\eta < 1/c_2$
\begin{align*}
S\|\solnridge\|_\bSigma/c_2 \leq& \left[\left(1 +t\Diamond\right)^2 +  n\Lambda^{-1}M\right] \sigma_\eta\sqrt{ V + t^2\Delta V}\\
+&\left[1 + t\Diamond + n\Lambda^{-1}M\sigma_\eta \right] \sqrt{ V + t^2\Delta V}\\
+&\sqrt{n}(1 + t\sigma_\eta \Diamond)\Diamond\\
=& \Diamond^2\cdot t\sigma_\eta (\sqrt{n} + t\sqrt{ V + t^2\Delta V}) \\
+& \Diamond \cdot \bigl(\sqrt{n} + t(1 + 2\sigma_\eta)\sqrt{ V + t^2\Delta V}\bigr)\\
+& \left[1 + \sigma_\eta + 2n\Lambda^{-1}M\sigma_\eta \right]\sqrt{ V + t^2\Delta V}
\end{align*}

By Lemma \ref{lm::relations}
\[
V \leq 2, \quad \Delta V\leq 3/n, \quad t^2\Delta V \leq 3.
\]
This allows us to obtain the final bound on $S\|\solnridge\|_\bSigma$: plug in the following inequalities:
\begin{align*}
\sqrt{n} + t\sqrt{ V + t^2\Delta V} \leq& (1 + \sqrt{5})\sqrt{n}, \\
\sqrt{n} + t(1 + 2\sigma_\eta)\sqrt{ V + t^2\Delta V} \leq& (1 + 3\sqrt{5})\sqrt{n},\\
1 + \sigma_\eta \leq& 2.
\end{align*}
We get for some $c_3$ that only depends on $L, c_B$:
\begin{equation}
\label{eq::pre-final form for denominator}
S\|\solnridge\|_\bSigma \leq c_3\left(\left[1 + n\Lambda^{-1}M\sigma_\eta \right]\sqrt{ V + t^2\Delta V} + \Diamond\sqrt{n} + t\sigma_\eta\Diamond^2\sqrt{n}\right).
\end{equation}

Finally, note that by Lemma \ref{lm::relations} we have $\Diamond^2 \leq \Lambda^{-1}M\sqrt{n\Delta V}$, so 
\[
t\sigma_\eta\Diamond^2\sqrt{n} \leq t\sigma_\eta M\Lambda^{-1}n\sqrt{\Delta V} = \sigma_\eta M\Lambda^{-1}n\sqrt{t^2\Delta V}  \leq n\Lambda^{-1}M\sigma_\eta\sqrt{ V + t^2\Delta V}.
\]
We see that the term $t\sigma_\eta\Diamond^2\sqrt{n}$ is dominated by another term up to a constant factor, so it can be removed. This gives the final form of the bound.

\end{proof}

\section{Proof of tightness}
\label{sec::tightness appendix}
The goal of this section is to prove a constant probability upper bound on $\bmu^\top\solnridge/\|\solnridge\|_\bSigma$ for the case without label flipping noise (that is, $\by = \hat{\by}$). We are going to do it by separately bounding $S\bmu^\top\solnridge$ from above and $\|S\solnridge\|_\bSigma$ from below, where $S$ is the scalar from Lemma \ref{lm::solution formulas}. With our techniques, the bounds from below are usually more complicated then the bounds from above. This happens because of the cross-terms: one can use Cauchy-Schwarz to bound them from above, but not from below. To overcome this issue, we introduce two additional random signs to the data. More precisely, introduce two independent Rademacher random variables $\eps_y$ and $\eps_q$, which are independent from $\by$ and $\bQ$, and denote
\begin{align}
\bar{\bQ} :=& [\bQ_\uptok, \eps_q\bQ_\ktoinf],\\
\bar{\by} :=& \eps_y\by,\\
\solnridgebar :=& (\bar{\bQ} + \bar{\by}\bmu^\top)^\top (\underbrace{\bar{\bQ}\bar{\bQ}^\top + \lambda \bI_n}_{= \bA})^{-1}\bar{\by}.
\end{align}

Note that since the distribution of $\by$ is symmetric (i.e. $\by$ and $-\by$ have the same distribution), the distribution of $\bar{\by}$ is the same as the distribution of $\by$. We are also going to assume that $\bQ_\ktoinf$ is independent from $\bQ_\uptok$ and that the distribution of $\bQ_\ktoinf$ is symmetric, which implies that the $\bQ$ is has the same distribution as $\bar{\bQ}$.  Moreover, note that the expressions in the definitions of the events $\sB_k(c_B)$ and $\sA_k(L)$ don't change if we substitute $\bQ$ by $\bar{\bQ}$ in those definitions. Both random signs $\eps_y$ and $\eps_y$ cancel. For example, this implies Lemma \ref{lm::randomness in Q but no eigvals of A} applies if we substitute $\bQ$ by $\bar{\bQ}$, and its result holds almost surely over $\eps_q$. 

Introduction of those random signs allows us to say that the cross terms are non-negative with probability $0.5$ independently of $\bQ$ and $\by$, and thus we don't need to lower bound them to obtain results with constant probability.

By Lemma \ref{lm::solution formulas} 
\begin{align*}
\bar{S} :=& (1 + \bar{\bnu}^\top \bA^{-1}\bar{\by})^2 + \bmu^\top\muperpridgebar\bar{\by}^\top\bA^{-1}\bar{\by},\\
\bar{S}\solnridgebar =& (1 + \bar{\bnu}^\top \bA^{-1}\bar{\by})\bar{\bQ}^\top\bA^{-1}\bar{\by} + \bar{\by}^\top \bA^{-1}\bar{\by}\muperpridgebar,\\
\bar{S}\bmu^\top\solnridgebar =&  \bar{\by}^\top\bA^{-1}\bar{\by} \bmu^\top\muperpridgebar + (1 +  \bar{\bnu}^\top\bA^{-1}\bar{\by}) \bar{\bnu}^\top \bA^{-1}\bar{\by},
\end{align*}
where we introduced $\bar{\bnu} := \bar{\bQ}\bmu$ and $\muperpridgebar = (\bI_p - \bar{\bQ}^\top\bA^{-1}\bar{\bQ})\bmu$.

The remainder of this section is organized as follows: in Section \ref{sec::numerator upper bound appendix} we bound $\bar{S}\bmu^\top\solnridgebar$ from above, in Section \ref{sec::denominator lower bound appendix} we bound $\|\bar{S}\bmu^\top\solnridgebar\|_{\bSigma}$ from below. In Section \ref{sec::upper bound on the ratio appendix} we combine those bouds into the upper bound on $\bar{S}\bmu^\top\solnridgebar/|\bar{S}\bmu^\top\solnridgebar\|_{\bSigma}$, and thus ${\bmu^\top\solnridge}/{\|\solnridge\|_\bSigma}$ too because it has the same distribution.
\subsection{Numerator}
\label{sec::numerator upper bound appendix}

We start with the following auxiliary lemma, which gives separate bounds on two quantities of interest that arise in the proof of the upper bound on $\bar{S}\bmu^\top\solnridgebar$.
\begin{lemma}
\label{lm::high probability lower bounds numerator}
 Suppose that the distribution of the rows of $\bZ$ is $\sigma_x$-sub-Gaussian.
For any $L \geq 1$ there exists a constant $c$ that only depends on $L,\sigma_x$ and $c_B$ such that the following holds. Suppose that $k < n/c$ and $\bQ_\uptok$ is independent from $\bQ_\ktoinf$. There exists an event $\sC$ whose probability is at least $1 - ce^{-n/c}$ such that all the following hold on the event $ \sA_{k}(L)\cap \sB_k(c_B) \cap \sC$:
\begin{align*}
\|\bA^{-1}\bQ_\uptok\bmu_\uptok\|^2 \geq& c^{-1}n^{-1} \left\|\left(\Lambda n^{-1}\bSigma_\uptok^{-1} + \bI_k\right)^{-1} \bSigma_\uptok^{-1/2}\bmu_\uptok\right\|^2,\\
\|\bA^{-1}\bQ_\ktoinf\bmu_\ktoinf\|^2 \geq& c^{-1}\Lambda^{-2}n\|\bmu_\ktoinf\|_{\bSigma_\ktoinf}^2.
\end{align*}
\end{lemma}
\begin{proof}
We prove the inequalities separately. Recall that $c_B$ is the constant from the definition of $\sB_k(c_B)$ in Section \ref{sec::definition of sB_k}. 
\begin{enumerate}
\item Using the  expressions that we derived in Section \ref{sec:: A inv nu decomposition} we have on the event $\sA_k(L)\cap \sB_k(c_B)$
\begin{align*}
&\|\bA^{-1}\bQ_\uptok\bmu_\uptok\|^2\\
=& \left\|\bA_k^{-1}\bZ_\uptok\left(\bSigma_\uptok^{-1} + \bZ_\uptok^\top \bA_k^{-1}\bZ_\uptok\right)^{-1}\bSigma_\uptok^{-1/2}\bmu_\uptok\right\|^2\\
\geq& \mu_1(\bA_k)^{-2}\mu_n(\bZ_\uptok^\top\bZ_\uptok) \left\|\left(\bSigma_\uptok^{-1} + \bZ_\uptok^\top \bA_k^{-1}\bZ_\uptok\right)^{-1} \bSigma_\uptok^{-1/2}\bmu_\uptok\right\|^2\\
\geq& \frac{\mu_n(\bA_k)^{2}\mu_n(\bZ_\uptok^\top\bZ_\uptok)}{\mu_1(\bA_k)^2\mu_1(\bZ_\uptok^\top \bZ_\uptok)^2}\left\|\left(\mu_1(\bA_k)\mu_k(\bZ_\uptok^\top \bZ_\uptok)^{-1}\bSigma_\uptok^{-1} + \bI_k\right)^{-1} \bSigma_\uptok^{-1/2}\bmu_\uptok\right\|^2\\
\geq& L^4c_B^3n^{-1}\cdot Lc_B \left\|\left(\Lambda n^{-1}\bSigma_\uptok^{-1} + \bI_k\right)^{-1} \bSigma_\uptok^{-1/2}\bmu_\uptok\right\|^2.
\end{align*}
where we used Lemma \ref{lm::preserve sigma uptok inverse} in the penultimate line.
\item Use Sherman-Morrison-Woodbury for the matrix $\bA^{-1}$ (Equation \eqref{eq::SMW for A}):
\[
\bA^{-1} = (\bA_k + \bQ_\uptok\bQ_\uptok^\top)^{-1} = \bA_k^{-1} - \bA_k^{-1}\bQ_\uptok(\bI_k + \bQ_\uptok^\top\bA_k^{-1}\bQ_\uptok)^{-1}\bQ_\uptok^\top\bA_k^{-1}.
\]
Let's consider the matrix $\bA_k\bA^{-1}$ and see what happens when we multiply it by $\bQ_\ktoinf\bmu_\ktoinf$. The idea is to say that the column span of $\bQ_\uptok$ is independent of $\bQ_\ktoinf\bmu_\ktoinf$, and thus the part that lies that span doesn't influence the norm of the vector much. Formally, denote the projector on the orthogonal complement to the span of the columns of $\bQ_\uptok$ as $\bP_\uptok^\perp\in \R^{n\times n}$, and note (from Equation \eqref{eq::SMW for A}) that $\bP_\uptok^\perp\bA_k\bA^{-1} = \bP_\uptok^\perp$. We write
\begin{align*}
&\|\bA^{-1}\bQ_\ktoinf\bmu_\ktoinf\|^2\\
\geq& \mu_1(\bA_k)^{-2}\|\bA_k\bA^{-1}\bQ_\ktoinf\bmu_\ktoinf\|^2\\
\geq&\mu_1(\bA_k)^{-2}\|\bP_\uptok^\perp\bA_k\bA^{-1}\bQ_\ktoinf\bmu_\ktoinf\|^2\\
\geq& \mu_1(\bA_k)^{-2}\|\bP_\uptok^\perp\bQ_\ktoinf\bmu_\ktoinf\|^2.
\end{align*}

Now note that the vector $\bQ_\ktoinf\bmu_\ktoinf$ has i.i.d. components, whose variances are equal to $\|\bmu_\ktoinf\|_\bSigma^2$ and whose sub-Gaussian constants don't exceed $\sigma_x\|\bmu_\ktoinf\|_\bSigma$. Moreover, $\bQ_\ktoinf\bmu_\ktoinf$ is independent from $\bP_\uptok^\perp$, and  since $\bP_\uptok^\perp$ is a projector, we have
\[
\|\bP_\uptok^\perp\bQ_\ktoinf\bmu_\ktoinf\|^2 = (\bQ_\ktoinf\bmu_\ktoinf)^\top\bP_\uptok^\perp\bQ_\ktoinf\bmu_\ktoinf.
\]
Thus, by Hanson-Wright inequality (Lemma \ref{lm::quadratic form concentration}) for some absolute constant $c_1$ and any $s > 0$, the probability of the following event is at most $2\exp(-s/c_1)$
\[
|\|\bP_\uptok^\perp\bQ_\ktoinf\bmu_\ktoinf\|^2 - \|\bmu_\ktoinf\|_\bSigma^2\tr(\bP_\uptok^\perp) | > \sigma_x^2\|\bmu_\ktoinf\|_\bSigma^2\max(\sqrt{s}\|\bP_\uptok^\perp\|_F,s\|\bP_\uptok^\perp\|).
\]

Once again, $\bP_\uptok^\perp$ is a projector of rank $n-k$, so
\[
\tr(\bP_\uptok^\perp) = n-k > n/2, \quad \|\bP_\uptok^\perp\|_F = \sqrt{n-k} \leq \sqrt{n}, \quad \|\bP_\uptok^\perp\|_F = 1.
\]

Taking $s = n/c_2$ for a large enough constant $c_2$ that only depends on $\sigma_x$ we see that the probability of the following event is at least $1 - 2\exp\left\{-n/(c_1c_2)\right\}$:
\[
\sC := \left\{\|\bP_\uptok^\perp\bQ_\ktoinf\bmu_\ktoinf\|^2 >  n \|\bmu_\ktoinf\|_\bSigma^2\left(\frac12 - \sigma_x^2/\sqrt{c_2}\right)\right\}.
\]

For $c_2 > 16\sigma_x^4$ on$\sC $ we have $\|\bP_\uptok^\perp\bQ_\ktoinf\bmu_\ktoinf\|^2 \geq  n \|\bmu_\ktoinf\|_\bSigma^2/4$.

Combining everything together, on $\sA_k(L)\cap\sC$ we get 
\begin{align*}
\|\bA^{-1}\bQ_\ktoinf\bmu_\ktoinf\|^2 \geq& 0.25L^{-2}\Lambda^{-2}n\|\bmu_\ktoinf\|_\bSigma^2.
\end{align*}
\end{enumerate}
\end{proof}

Our upper bound on $\bar{S}\bmu^\top\solnridgebar$ is given by the following lemma.
\begin{lemma}
\label{lm::numerator upper} Suppose that the distribution of the rows of $\bZ$ is $\sigma_x$-sub-Gaussian.
Consider some $L > 1$. There exist  large constants $a, c$ that only depend on $\sigma_x, c_B$ and $L$ and an absolute constant $c_y$ such that the following holds.  Assume that  $k < n/c$ and
\[
\Lambda > cn\lambda_{k+1} \vee \sqrt{n\sum_{i > k}\lambda_i^2}.
\] 
\begin{enumerate}
\item If $n\Lambda^{-1}M \geq a^{-1}  \Diamond$, then on $\sA_k(L)\cap \sB_k(c_B)$ for any $t \in (0, \sqrt{n})$ with probability at least $1 - c_ye^{-t^2/c_y}$ over the draw of $\by$ almost surely over the draw of $(\eps_q, \eps_y)$ 
\[
\bar{S}\bmu^\top\solnridgebar < c(1 + t)n\Lambda^{-1}M
\]

\item If $n\Lambda^{-1}M < a^{-1}  \Diamond$, there exists an event $\sC$ that only depends on $\bQ$, whose probability is at least $1 - ce^{-n/c}$ such that then on $\sA_k(L)\cap \sB_k(c_B) \cap \sC$ with probability at least $c_y^{-1}$ over the draw of $\by$ and $(\eps_q, \eps_y)$ 
\[
\bar{S}\bmu^\top\solnridgebar < 0.
\]
\end{enumerate}
\end{lemma}
\begin{proof}
Recall that 
\[
\bar{S}\bmu^\top\solnridgebar =  \bar{\by}^\top\bA^{-1}\bar{\by} \bmu^\top\muperpridgebar + (1 +  \bar{\bnu}^\top\bA^{-1}\bar{\by}) \bar{\bnu}^\top \bA^{-1}\bar{\by}.
\]

First of all, denote the constant from Lemma \ref{lm::randomness in Q but no eigvals of A} as $c_1$. By that lemma on $\sA_k(L)\cap \sB_k(c_B)$ we have
\begin{align*}
\bar{\by}^\top\bA^{-1}\bar{\by} \bmu^\top\muperpridgebar
\leq&\mu_n(\bA)^{-1}\|\bar{\by}\|^2 \bmu^\top\muperpridgebar\\
\leq& L\Lambda^{-1}n\cdot c_1M;\\
\|\bA^{-1}\bar{\bnu}\| \leq& c_1\Diamond.
\end{align*}

Recall that we indeed can apply Lemma \ref{lm::randomness in Q but no eigvals of A} to $\|\bA^{-1}\bar{\bnu}\|$ instead of $\|\bA^{-1}{\bnu}\|$ because introducing $\eps_q$ into the matrix $\bQ$ does not change the definitions of the events $\sA_k(L)$ and $\sB_k(c_B)$. 

In the same way as in the proof of Lemma \ref{lm::randomness in y}  since  $\by$ is a sub-Gaussian vector with sub-Gaussian norm bounded by an absolute constant, we have for some absolute constant $c_{y, 1}$ that for any $t > 0$ on $\sA_k(L)$ and $\sB_k(c_B)$ with probability at least $1 - c_{y,1}e^{-t^2/c_{y,1}}$ over the draw of $\by$ almost surely over the draw of $\eps_q$
\[
|\bar{\bnu}^\top \bA^{-1}\bar{\by}|= |\bar{\bnu}^\top \bA^{-1}\by| \leq c_1t\Diamond.
\]
We can make that statement almost surely over $\eps_q$, because it can only take two values, so we can just do multiplicity correction by adjusting the constant $c_{y, 1}$.

These upper bounds directly imply the first part of the lemma. Indeed, 
\begin{align*}
\bar{S}\bmu^\top\solnridgebar =&  \bar{\by}^\top\bA^{-1}\bar{\by} \bmu^\top(\bI_p - \bar{\bQ}^\top\bA^{-1}\bar{\bQ}) + (1 +  \bar{\bnu}^\top\bA^{-1}\bar{\by}) \bar{\bnu}^\top \bA^{-1}\bar{\by}\\
\leq& c_1^2n\Lambda^{-1}M + c_1t\Diamond + (c_1t\Diamond)^2\\
\leq& c_1^2n\Lambda^{-1}M  + c_1t\Diamond + c_1n\Diamond^2\\
\leq&n\Lambda^{-1}M(c_1^2 + c_1) + c_1tan\Lambda^{-1}M,
\end{align*}
where we used that $t < \sqrt{n}$,  $\Diamond^2 \leq \Lambda^{-1}M$ (Lemma \ref{lm::relations}) and $\Diamond < an\Lambda^{-1}M$ in the last line. In the end, we just need $c$ to be large enough depending on $c_1$ and $a$. 

When it comes to the second part, we leave the same bounds for  the terms $\bar{\by}^\top\bA^{-1}\bar{\by} \bmu^\top(\bI_p - \bar{\bQ}^\top\bA^{-1}\bar{\bQ})$ and $(\bar{\bnu}^\top\bA^{-1}\bar{\by})^2$, but show that the term $\bar{\bnu}^\top\bA^{-1}\bar{\by}$ can be negative with large enough magnitude to pull the whole bound in the negative direction.

We take the event $\sC$ to be the same as in  Lemma \ref{lm::high probability lower bounds numerator}, by which there exists a constant $c_2$ that only depends on $L, \sigma_x, c_B$ such that on $\sA_k(L)\cap\sB_k(c_B)\cap \sC$
\begin{align*}
\|\bA^{-1}\bQ_\uptok\bmu_\uptok\|^2 \geq& c_2^{-1}n^{-1} \left\|\left(\Lambda n^{-1}\bSigma_\uptok^{-1} + \bI_k\right)^{-1} \bSigma_\uptok^{-1/2}\bmu_\uptok\right\|^2,\\
\|\bA^{-1}\bQ_\ktoinf\bmu_\ktoinf\|^2 \geq& c_2^{-1}\Lambda^{-2}n\|\bmu_\ktoinf\|_{\bSigma_\ktoinf}^2.
\end{align*}

Now we use the same expressions as we derived in Section \ref{sec:: A inv nu decomposition} to write
\begin{align*}
\|\bA^{-1}\bar{\bnu}\|^2 =& \|\bA^{-1}(\bQ_\uptok\bmu_\uptok + \eps_q \bQ_\ktoinf\bmu_\ktoinf)\|^2\\
=& \left\|\bA_k^{-1}\bZ_\uptok\left(\bSigma_\uptok^{-1} + \bZ_\uptok^\top \bA_k^{-1}\bZ_\uptok\right)^{-1}\bSigma_\uptok^{-1/2}\bmu_\uptok\right\|^2\\
+& \|\bA^{-1}\bQ_\ktoinf\bmu_\ktoinf\|^2\\
+& 2\eps_q (\bQ_\ktoinf\bmu_\ktoinf)^\top\bA^{-2}(\bQ_\uptok\bmu_\uptok)\\
\geq& c_2^{-1}\Diamond^2/2 + 2\eps_q (\bQ_\ktoinf\bmu_\ktoinf)^\top\bA^{-2}(\bQ_\uptok\bmu_\uptok).
\end{align*}

Note that conditionally on $\sA_k(L)\cap\sB_k(c_B)\cap \sC$ with probability $0.5$ over the draw of $\eps_q$ the term involving $\eps_q$ is non-negative, that is $\|\bA^{-1}\bar{\bnu}\| \geq (2c_2)^{-1/2}\Diamond$. That statement doesn't involve $\by$, so $\by$ is still independent of this event. Thus, conditionally on it, by Lemma \ref{lm::scalar prod const prob lower} for an absolute constant $c_{y, 2}$ with probability at least $c_{y, 2}^{-1}$ over the choice of $\by$ we have $|\by^\top\bA^{-1}\bar{\bnu}| \geq c_{y, 2}^{-1}\|\bA^{-1}\bar{\bnu}\|$. Moreover, we've seen in the first part of the proof that with probability at least $1 - c_{y,1}e^{-n/c_{y,1}}$ over the draw of $\by$ $|\by^\top\bA^{-1}\bar{\bnu}| \leq c_1\sqrt{n}\Diamond$. Finally, with probability $0.5$ over $\eps_y$ we have $\eps_y\by^\top\bA^{-1}\bar{\bnu} = -|\by^\top\bA^{-1}\bar{\bnu}|$. Combining everything together (recall that $\eps_q, \eps_y, \bQ$ and $\by$ are independent) we get that on $\sA_k(L)\cap\sB_k(c_B)\cap \sC$ with probability at least $0.25\left(c_{y, 2}^{-1} - c_{y,1}e^{-n/c_{y,1}}\right)$ over the draw of $\by$ and $(\eps_q, \eps_y)$
\begin{align*}
\bar{\by}^\top\bA^{-1}\bar{\bnu} \leq& -c_{y,2}^{-1}(2c_2)^{-1/2}\Diamond,\\
\bar{S}\bmu^\top\solnridgebar \leq& n\Lambda^{-1}M(c_1^2 + 2c_1)  -c_{y,2}^{-1}(2c_2)^{-1/2}\Diamond\\
<& 0,
\end{align*}
where the last transition holds for $a$ large enough depending on $c_1, c_2, c_{y,1}, c_{y, 2}$ since  $n\Lambda^{-1}M < a^{-1}  \Diamond$.
\end{proof}

\subsection{Denominator}
\label{sec::denominator lower bound appendix} 
The next step is to lower-bound the denominator $\|\bar{S}\bmu^\top\solnridgebar\|_{\bSigma}$. Recall that $\eps_y, \eps_q, \by, \bQ$ are all independent from each other. We factor out the randomness in each of those variables one-by-one, starting with the following
\begin{lemma}
\label{lm:randomness in eps denominator lower}
With probability at least $0.25$ over the choice of $(\eps_y, \eps_q)$ (that is, conditionally on $\by$ and $\bQ$)
\begin{align*}
\|\bar{S}\bmu^\top\solnridgebar\|_{\bSigma}^2\geq& \frac12 \left\|{\bQ}^\top\bA^{-1}{\by} \right\|_\bSigma^2\\
+& \frac12({\by}^\top \bA^{-1}{\by})^2\left(\left\|(\bI_k - \bQ_\uptok^\top\bA^{-1}\bQ_\uptok)\bmu_\uptok\right\|_{\bSigma_{\uptok}}^2 + 0.5\|\bmu_\ktoinf\|_{\bSigma_\ktoinf}^2\right)\\
 -&\frac72({\by}^\top \bA^{-1}{\by})^2\left\| \bQ_\ktoinf^\top\bA^{-1}\bQ_\ktoinf\bmu_\ktoinf\right\|_{\bSigma_\ktoinf}^2\\
 -& 7(\bar{\bnu}^\top \bA^{-1}{\by})^2\|{\bQ}^\top\bA^{-1}{\by}\|_\bSigma^2
\end{align*}
\end{lemma}
\begin{proof}
Note that for any vectors $\bu, \bv$ of the same dimension the following holds:
\begin{align*}
\|\bu + \bv\|^2 =& \|\bu\|^2 + \|\bv\|^2 + 2\bu^\top\bv\\
\geq& \|\bu\|^2 + \|\bv\|^2 - 2\left(0.25 \|\bu\|^2 + 4\|\bv\|^2\right)\\
=& 0.5 \|\bu\|^2 - 7\|\bv\|^2.
\end{align*}
 Thus, we write
\begin{align*}
\|\bar{S}\solnridgebar\|_\bSigma^2 \geq&0.5\|\bar{\bQ}^\top\bA^{-1}\bar{\by} + \bar{\by}^\top \bA^{-1}\bar{\by}\muperpridgebar\|_\bSigma^2 - 7(\bar{\bnu}^\top \bA^{-1}\bar{\by})^2\|\bar{\bQ}^\top\bA^{-1}\bar{\by}\|_\bSigma^2.
\end{align*}
For the last term note that 
\[
\left\|\bar{\bQ}^\top\bA^{-1}\bar{\by} \right\|_\bSigma^2 = \left\|\bar{\bQ}^\top\bA^{-1}{\by} \right\|_\bSigma^2 = \left\|\bQ_\uptok^\top\bA^{-1}{\by} \right\|_{\bSigma_\uptok}^2 + \left\|\eps_y\bQ_\ktoinf^\top\bA^{-1}{\by} \right\|_{\bSigma_\ktoinf}^2 = \left\|{\bQ}^\top\bA^{-1}{\by} \right\|_\bSigma^2.
\]

Next, we decompose the first term as follows:
\begin{align*}
&\left\|\bar{\bQ}^\top\bA^{-1}\bar{\by} + \bar{\by}^\top \bA^{-1}\bar{\by}\muperpridgebar\right\|_\bSigma^2\\
=&\left\|\bar{\bQ}^\top\bA^{-1}{\by} \right\|_\bSigma^2 + ({\by}^\top \bA^{-1}{\by})^2\left\|\muperpridgebar\right\|_\bSigma^2
+ 2\eps_yf_1(\bSigma, \bQ, \bmu,\eps_q\by),
\end{align*}
where $f_1(\bSigma, \bQ, \bmu,\eps_q\by)$ is a cross-term, which doesn't involve $\eps_y$. Recall that for the first term we have $\left\|\bar{\bQ}^\top\bA^{-1}{\by} \right\|_\bSigma^2  = \left\|{\bQ}^\top\bA^{-1}{\by} \right\|_\bSigma^2$.

For the second term in that decomposition we go a step further and write
\begin{align*}
&\left\|\muperpridgebar\right\|_\bSigma^2\\
=& \left\|\bmu_\uptok - \bQ_\uptok^\top\bA^{-1}\bar{\bQ}\bmu\right\|_{\bSigma_{\uptok}}^2 + \left\|\bmu_\ktoinf - \eps_q\bQ_\ktoinf^\top\bA^{-1}\bar{\bQ}\bmu\right\|_{\bSigma_\ktoinf}^2\\
=& \left\|(\bI_k - \bQ_\uptok^\top\bA^{-1}\bQ_\uptok)\bmu_\uptok\right\|_{\bSigma_{\uptok}}^2 + \left\|\bmu_\ktoinf - \bQ_\ktoinf^\top\bA^{-1}\bQ_\ktoinf\bmu_\ktoinf\right\|_{\bSigma_\ktoinf}^2\\
+& \left\| \eps_q\bQ_\uptok^\top\bA^{-1}\bQ_\ktoinf\bmu_\ktoinf\right\|_{\bSigma_{\uptok}}^2 + \left\| \eps_q\bQ_\ktoinf^\top\bA^{-1}\bQ_\uptok\bmu_\uptok\right\|_{\bSigma_\ktoinf}^2\\
+& \eps_q f(\bSigma, \bQ, \bmu)\\
\geq&  \left\|(\bI_k - \bQ_\uptok^\top\bA^{-1}\bQ_\uptok)\bmu_\uptok\right\|_{\bSigma_{\uptok}}^2 + \left\|\bmu_\ktoinf - \bQ_\ktoinf^\top\bA^{-1}\bQ_\ktoinf\bmu_\ktoinf\right\|_{\bSigma_\ktoinf}^2 +  \eps_q f_2(\bSigma, \bQ, \bmu)\\
\geq&  \left\|(\bI_k - \bQ_\uptok^\top\bA^{-1}\bQ_\uptok)\bmu_\uptok\right\|_{\bSigma_{\uptok}}^2 + 0.5\|\bmu_\ktoinf\|_{\bSigma_\ktoinf}^2 -7\left\| \bQ_\ktoinf^\top\bA^{-1}\bQ_\ktoinf\bmu_\ktoinf\right\|_{\bSigma_\ktoinf}^2 \\
+&  \eps_q f_2(\bSigma, \bQ, \bmu).
\end{align*}
where $f_2(\bSigma, \bQ, \bmu)$ is the cross term, which is independent from $\eps_q$, $\eps_y$ and $\by$. 

The statement of the lemma holds on the following event
\[
\{\eps_q f_2(\bSigma, \bQ, \bmu) \geq 0,\eps_yf_1(\bSigma, \bQ, \bmu,\eps_q\by) \geq 0,\}
\]
whose probability is at least $0.25$ conditionally on $\by, \bQ$ since $\eps_q$ and $\eps_y$ are independent random signs.
\end{proof}

Note that the last term in the lemma above still depends on $\eps_q$ through $\bar{\bnu}$ and $\bar{\bQ}$. This is not a problem since we will bound that term almost surely over the draw of $\eps_q$ conditionally on $\sA_k(L)\cap\sB_k(c_B)$ and $\by$.

The next step is to obtain lower bounds w.r.t. randomness that comes from $\by$. Once again, we don't touch the terms that we subtract yet, and only lower-bound the positive terms.
\begin{lemma}
\label{lm:randomness in y denominator lower}
There exists an absolute constant $c_y$ such that for any fixed value of $\bQ$ for any $t \in (0, \sqrt{n}/c_y)$ with probability at least $c_y^{-1} - c_ye^{-t^2/c_y}$ over the draw of $\by$ all the following hold almost surely over the draw of $\eps_q$:
\begin{align*}
\left\|{\bQ}^\top\bA^{-1}{\by} \right\|_\bSigma^2 \geq& c_y^{-1}\tr(\bA^{-1}\bQ\bSigma\bQ^\top\bA^{-1}),\\
\by^\top \bA^{-1}\by\geq& (n - k)\mu_1(\bA_k)^{-1} -  c_y(t\sqrt{n} + t^2)\|\bA^{-1}\|,\\
\by^\top \bA^{-1}\by \leq& n\|\bA^{-1}\|,\\
\|\bQ^\top\bA^{-1}\by\|_\bSigma^2 \leq& c_y(\tr(\bA^{-1}\bQ\bSigma\bQ^\top\bA^{-1}) + t^2\|\bA^{-1}\bQ\bSigma\bQ^\top\bA^{-1}\|),\\
|\bar{\bnu}^\top\bA^{-1}\by| \leq& c_yt\|\bA^{-1}\bar{\bnu}\|.
\end{align*}
\end{lemma}
\begin{proof}
The first inequality is a direct application of Lemma \ref{lm::quad form const prob lower},  and the remaining were shown as a part of Lemma \ref{lm::randomness in y}. Note that only the last inequality depends on $\eps_q$, and formally Lemma \ref{lm::randomness in y} only shows it for a fixed value of $\eps_q$. However, there are only two possible values of $\eps_q$, so the uniform result can be obtained by straightforward multiplicity correction (the constant from Lemma \ref{lm::randomness in y} should be doubled).
\end{proof}

At this point we just need the lower bounds on the quantities that only involve $\bQ$ and $\bmu$. These are done by the following
\begin{lemma}
\label{lm::high probability lower bounds denominator} Suppose that the distribution of the rows of $\bZ$ is $\sigma_x$-sub-Gaussian.
For any $L \geq 1$ there exists a constant $c$ that only depends on $L, c_B$ and $\sigma_x$ such that the following holds.  Suppose that $k < n/c$ and $\bQ_\uptok$ is independent from $\bQ_\ktoinf$. There exists an event $\sC$ whose probability is at least $1 - ce^{-n/c}$ such that all the following hold on the event $ \sA_{k}(L)\cap \sB_k(c_B)\cap\sC$:
\begin{align*}
\tr(\bA^{-1}\bQ\bSigma\bQ^\top\bA^{-1}) \geq& c^{-1}V,\\
\|(\bI_k - \bQ_\uptok^\top\bA^{-1}\bQ_\uptok)\bmu_\uptok\|_{\bSigma_\uptok}^2 \geq& c^{-1}\Lambda^2n^{-2}\left\|\left(\Lambda n^{-1}\bSigma_\uptok^{-1} + \bI_k\right)^{-1}\bSigma_\uptok^{-1/2}\bmu_\uptok\right\|^2.\\
\end{align*}
\end{lemma}
\begin{proof}
We prove the inequalities separately:
\begin{enumerate}
\item  
\begin{align*}
\tr(\bA^{-1}\bQ\bSigma\bQ^\top\bA^{-1})
=& \tr(\bA^{-1}\bQ_\uptok\bSigma_\uptok\bQ_\uptok^\top\bA^{-1}) + \tr(\bA^{-1}\bQ_\ktoinf\bSigma_\ktoinf\bQ_\ktoinf^\top\bA^{-1})
\end{align*}

For the first term we use the same formula as in Section \ref{sec::algebraic decomposition for V}:
\begin{align*}
&\tr(\bA^{-1}\bQ_\uptok\bSigma_\uptok\bQ_\uptok^\top\bA^{-1})\\
=&\tr\left(\bA_k^{-1}\bZ_\uptok\left(\bSigma_\uptok^{-1} + \bZ_\uptok^\top \bA_k^{-1}\bZ_\uptok\right)^{-2}\bZ_\uptok^\top\bA_k^{-1}\right)\\
\geq& \mu_1(\bA_k)^{-2}\mu_k(\bZ_\uptok^\top\bZ_\uptok)\tr\left(\left(\bSigma_\uptok^{-1} + \bZ_\uptok^\top \bA_k^{-1}\bZ_\uptok\right)^{-2}\right)\\
\geq& \frac{\mu_n(\bA_k)^2\mu_k(\bZ_\uptok^\top\bZ_\uptok)}{\mu_1(\bA_k)^{2}\mu_1(\bZ_\uptok^\top\bZ_\uptok)^2}\tr\left(\left(\mu_k(\bZ_\uptok^\top\bZ_\uptok)^{-1}\mu_1(\bA_k)^{-1}\bSigma_\uptok^{-1} + \bI_k\right)^{-2}\right)\\
\geq& L^4c_B^3n^{-1}\cdot c_B^2L^2 \tr\left(\left(\Lambda n^{-1}\bSigma_\uptok^{-1} + \bI_k\right)^{-2}\right),
\end{align*}
where we used Lemma \ref{lm::preserve sigma uptok inverse} in the penultimate line  and the definition of the event $\sB_k(c_B)$ from Section \ref{sec::definition of sB_k} in the last transition. 

When it comes to the second term, we once again (as in the proof of Lemma  \ref{lm::high probability lower bounds denominator}) are going to use the fact that 
\[
\bP_\uptok^\perp\bA_k\bA^{-1} = \bP_\uptok^\perp.
\]
Thus, we write
\begin{align*}
&\tr(\bA^{-1}\bQ_\ktoinf\bSigma_\ktoinf\bQ_\ktoinf^\top\bA^{-1})\\
\geq& \mu_1(\bA_k)^{-2}\tr(\bA_k\bA^{-1}\bQ_\ktoinf\bSigma_\ktoinf\bQ_\ktoinf^\top\bA^{-1}\bA_k)\\
\geq& \mu_1(\bA_k)^{-2}\tr(\bP_\uptok^\perp\bA_k\bA^{-1}\bQ_\ktoinf\bSigma_\ktoinf\bQ_\ktoinf^\top\bA^{-1}\bA_k\bP_\uptok^\perp)\\
=&\mu_1(\bA_k)^{-2}\tr(\bP_\uptok^\perp\bQ_\ktoinf\bSigma_\ktoinf\bQ_\ktoinf^\top\bP_\uptok^\perp)\\
=&\mu_1(\bA_k)^{-2} \sum_{i > k} \lambda_i^2\bz_i^\top\bP_\uptok^\perp\bz_i,
\end{align*}
where $\bz_i$ are columns of $\bZ$. Note that for every fixed $i > k$ the vector $\bz_i$ has i.i.d. components with variance $1$ and sub-Gaussian constant at most $\sigma_x$ which are independent of $\bP_\uptok$. As in the proof of Lemma  \ref{lm::high probability lower bounds denominator}, by Hanson-Wright inequality (Lemma \ref{lm::quadratic form concentration}) for some absolute constant $c_1$ and any $s > 0$  with probability at least $1 - 2e^{-s/c_1}$
\begin{align*}
|\bz_i^\top\bP_\uptok^\perp\bz_i - \tr(\bP_\uptok^\perp) | =& |\bz_i^\top\bP_\uptok^\perp\bz_i - (n-k) |\\
<& \sigma_x^2\max(\sqrt{s}\|\bP_\uptok^\perp\|_F,s\|\bP_\uptok^\perp\|)\\
=& \sigma_x^2\max(\sqrt{s}\sqrt{n-k},s).
\end{align*}
So, for a large enough constant $c_2$ that only depends on $\sigma_x$, given that $c > c_2$ (i.e., $k < n/c_2$) for any separate $i$ with probability at least $1 - c_2e^{-n/c_2}$
\[
\bz_i^\top\bP_\uptok^\perp\bz_i \geq n/c_2.
\]
By Lemma 9 from \cite{benign_overfitting}, we can combine separate high-probability lower bounds on non-negative terms into a high-probability lower bound on the sum, that is,  with probability at least $1 - 2c_2e^{-n/c_2}$
\[
\sum_{i > k} \lambda_i^2\bz_i^\top\bP_\uptok^\perp\bz_i \geq \frac{n}{2c_2}\sum_{i > k}\lambda_i^2. 
\]
Take this event as $\sC$.

Overall, we get that on $\sA_k(L)\cap \sC$
\[
\tr(\bA^{-1}\bQ_\ktoinf\bSigma_\ktoinf\bQ_\ktoinf^\top\bA^{-1})\geq \frac{1}{2L^2c_2}\Lambda^{-2}n\sum_{i > k}\lambda_i^2.
\]

\item Using Lemma \ref{lm::application of SMW} we have
\begin{align*}
&\|(\bI_k - \bQ_\uptok^\top\bA^{-1}\bQ_\uptok)\bmu_\uptok\|_{\bSigma_\uptok}^2\\
=&\left\|\bSigma_\uptok^{1/2}\left(\bI_k + \bQ_\uptok^\top \bA_k^{-1}\bQ_\uptok\right)^{-1}\bmu_\uptok\right\|^2\\
=& \left\|\left(\bSigma_\uptok^{-1} + \bZ_\uptok^\top \bA_k^{-1}\bZ_\uptok\right)^{-1}\bSigma_\uptok^{-1/2}\bmu_\uptok\right\|^2\\
\geq& \mu_1(\bZ_\uptok^\top \bZ_\uptok)^{-2}\mu_n(\bA_k)^2\left\|\left(\mu_k(\bZ_\uptok^\top \bZ_\uptok)^{-1}\mu_1(\bA_k)\bSigma_\uptok^{-1} + \bI_k\right)^{-1}\bSigma_\uptok^{-1/2}\bmu_\uptok\right\|^2\\
\geq& L^{-2}c_B^{-2}\Lambda^2 n^{-2}\cdot c_B^{-2}L^{-2}\left\|\left(\Lambda n^{-1}\bSigma_\uptok^{-1} + \bI_k\right)^{-1}\bSigma_\uptok^{-1/2}\bmu_\uptok\right\|^2,
\end{align*}
where we used Lemma \ref{lm::preserve sigma uptok inverse} in the penultimate line  and the definition of the event $\sB_k(c_B)$ from Section \ref{sec::definition of sB_k} in the last transition. 
\end{enumerate}
\end{proof}

Finally, we can put everything together in the following
\begin{lemma}
\label{lm::denominator lower}
 Suppose that the distribution of the rows of $\bZ$ is $\sigma_x$-sub-Gaussian.
Take some $L > 1$. There is an absolute constant $c_y$ and a constant $c$ that only depends on $L$ and $\sigma_x$, such that if $k < n/c$ and 
\begin{equation}
\label{eq::assumption for b denominator lower}
\Lambda > c\left(n\lambda_{k+1} + \sqrt{n\sum_{i > k}\lambda_i^2}\right),
\end{equation}
then there exists an event $\sC$ which only depends on $\bQ$, whose probability is at least $1 - ce^{-n/c}$ such that conditionally on $\sA_k(L)\cap \sB_k(c_B)\cap\sC$ with probability at least $c_y^{-1} - c_y e^{-n/c}$ over the draw of $\by$ with probability at least $0.25$ over the draw of $(\eps_y, \eps_q)$
\[
\|\bar{S}\bmu^\top\solnridgebar\|_{\bSigma}^2 \geq c^{-1}(V + n\Diamond^2).
\]
\end{lemma}
\begin{proof}
Take the event $\sC$ to be the same as in Lemma \ref{lm::high probability lower bounds denominator}, and denote the constant from it as $c_1$. By that lemma on $\sA_k(L)\cap\sB_k(c_B)\cap \sC$
\begin{align*}
\tr(\bA^{-1}\bQ\bSigma\bQ^\top\bA^{-1}) \geq& c_1^{-1}V,\\
\|(\bI_k - \bQ_\uptok^\top\bA^{-1}\bQ_\uptok)\bmu_\uptok\|_{\bSigma_\uptok}^2 \geq& c_1^{-1}\Lambda^2n^{-2}\left\|\left(\Lambda n^{-1}\bSigma_\uptok^{-1} + \bI_k\right)^{-1}\bSigma_\uptok^{-1/2}\bmu_\uptok\right\|^2.\\
\end{align*}

Moreover, denote the constant from Lemma \ref{lm::randomness in Q but no eigvals of A} as $c_2$. Note that $k < n/c_2$ and $\Lambda > c_2n\lambda_{k+1} \vee \sqrt{n\sum_{i > k}\lambda_i^2}$ on  if $c$ is large enough. Thus, we by Lemma \ref{lm::randomness in Q but no eigvals of A} on $\sA_k(L)\cap \sB_k(c_B)$
\begin{align*}
\|\bA^{-1}\bar{\bnu}\| \leq& c_2\Diamond,\\
\tr(\bA^{-1}\bQ\bSigma\bQ^\top\bA^{-1})
\leq& c_2V\\
\|\bA^{-1}\bQ\bSigma\bQ^\top\bA^{-1}\| \leq&   c_2\Delta V.
\end{align*}

Now denote the constant from Lemma \ref{lm:randomness in y denominator lower} as $c_y$. Combining that lemma with the results we stated above we get that conditionally on the event $\sA_k(L)\cap\sB_k(c_B)\cap\sC$ for any $t \in (0, \sqrt{n}/c_y)$ with probability at least $1 - c_ye^{-t^2/c_y}$ all the following hold almost surely over the draw of $\eps_q$:

\begin{align*}
\left\|{\bQ}^\top\bA^{-1}{\by} \right\|_\bSigma^2\geq& c_y^{-1}c_1^{-1}V,\\
\by^\top \bA^{-1}\by\geq& (n - k)\mu_1(\bA_k)^{-1} -  c_y(t\sqrt{n} + t^2)\|\bA^{-1}\|\\
\geq& n\Lambda^{-1}(L^{-1}(1 - k/n) - c_yL\Lambda^{-1}(t\sqrt{n} + t^2),\\
\by^\top \bA^{-1}\by \leq&nL\Lambda^{-1},\\
\|\bQ^\top\bA^{-1}\by\|_\bSigma^2 \leq& c_yc_2(V + t^2\Delta V),\\
|\bar{\bnu}^\top\bA^{-1}\by| \leq& c_yc_2t\Diamond.
\end{align*}

For the second inequality let's restrict $t$ to the range $(0, \sqrt{n}/c_3)$, where $c_3$ is a large enough constant depending on $\sigma_x, L$, so that inequality implies $\by^\top\bA^{-1}\by \geq c_3^{-1}n\Lambda^{-1}$.

Moreover, on $\sA_k(L)\cap\sB_k(c_B)$ we can write
\begin{align*}
&\left\| \bQ_\ktoinf^\top\bA^{-1}\bQ_\ktoinf\bmu_\ktoinf\right\|_{\bSigma_\ktoinf}^2\\
\leq& \left\|\bQ_\ktoinf\bSigma_\ktoinf \bQ_\ktoinf^\top\right\|\mu_n(\bA)^{-2}\left\|\bQ_\ktoinf\bmu_\ktoinf\right\|^2\\
\leq& c_B\left(\sum_{i > k}\lambda_i^2 + n\lambda_{k+1}^2\right)\cdot L^2\Lambda^{-2}\cdot c_Bn\|\bmu_\ktoinf\|_{\bSigma_\ktoinf}^2\\
\leq& \frac{c_B^2L^2}{c^2}\|\bmu_\ktoinf\|_{\bSigma_\ktoinf}^2,
\end{align*}
where in the last line we used the assumption from Equation \eqref{eq::assumption for b denominator lower}.

Combining all that with Lemma \ref{lm:randomness in eps denominator lower} gives that conditionally on $\sA_k(L)\cap\sB_k(c_B)\cap\sC$ for any $t \in (0, \sqrt{n}/c_3)$ with probability at least $c_y^{-1} - c_ye^{-t^2/c_y}$ over the draw of $\by$ with probability at least $0.25$ over the draw of $(\eps_y, \eps_q)$
\begin{align*}
\|\bar{S}\bmu^\top\solnridgebar\|_{\bSigma}^2\geq& \frac12 \left\|{\bQ}^\top\bA^{-1}{\by} \right\|_\bSigma^2\\
+& \frac12({\by}^\top \bA^{-1}{\by})^2\left(\left\|(\bI_k - \bQ_\uptok^\top\bA^{-1}\bQ_\uptok)\bmu_\uptok\right\|_{\bSigma_{\uptok}}^2 + 0.5\|\bmu_\ktoinf\|_{\bSigma_\ktoinf}^2\right)\\
 -&\frac72({\by}^\top \bA^{-1}{\by})^2\left\| \bQ_\ktoinf^\top\bA^{-1}\bQ_\ktoinf\bmu_\ktoinf\right\|_{\bSigma_\ktoinf}^2\\
 -& 7(\bar{\bnu}^\top \bA^{-1}{\by})^2\|{\bQ}^\top\bA^{-1}{\by}\|_\bSigma^2\\
 \geq& \frac{1}{2c_1c_y}V\\
 +& \frac{1}{2c_3^2}n^2\Lambda^{-2}\cdot\left(c_1^{-1}\Lambda^2n^{-2}\left\|\left(\Lambda n^{-1}\bSigma_\uptok^{-1} + \bI_k\right)^{-1}\bSigma_\uptok^{-1/2}\bmu_\uptok\right\|^2 + 0.5\|\bmu_\ktoinf\|_{\bSigma_\ktoinf}^2\right)\\
 -&\frac72 n^2\Lambda^{-2}\cdot\frac{c_B^2L^2}{c^2}\|\bmu_\ktoinf\|_{\bSigma_\ktoinf}^2 - 7c_y^2c_2^2t^2\Diamond^2\cdot c_yc_2(V + t^2\Delta V).
\end{align*}

If $c$ is large enough, namely if $7c_B^2L^2c_3^2 < 0.25c^2$, then we have
\begin{align*}
&\frac{1}{2c_3^2}n^2\Lambda^{-2}\cdot\left(c_1^{-1}\Lambda^2n^{-2}\left\|\left(\Lambda n^{-1}\bSigma_\uptok^{-1} + \bI_k\right)^{-1}\bSigma_\uptok^{-1/2}\bmu_\uptok\right\|^2 + 0.5\|\bmu_\ktoinf\|_{\bSigma_\ktoinf}^2\right)\\
 -&\frac72 n^2\Lambda^{-2}\cdot\frac{c_B^2L^2}{c^2}\|\bmu_\ktoinf\|_{\bSigma_\ktoinf}^2\\
 \geq& \frac{1}{2c_3^2}\cdot\left(c_1^{-1}\left\|\left(\Lambda n^{-1}\bSigma_\uptok^{-1} + \bI_k\right)^{-1}\bSigma_\uptok^{-1/2}\bmu_\uptok\right\|^2 + 0.25n^2\Lambda^{-2}\|\bmu_\ktoinf\|_{\bSigma_\ktoinf}^2\right)\\
 \geq& \frac{\min(c_1^{-1}, 0.25)}{2c_3^2}n \Diamond^2/2.
\end{align*}
That is, for a large enough constant $c_4$ that only depends on $\sigma_x, c_B, L$, we have
\begin{align*}
\|\bar{S}\bmu^\top\solnridgebar\|_{\bSigma}^2\geq& c_4^{-1}(V + n\Diamond^2) - c_4t^2\Diamond^2(V + t^2\Delta V).
\end{align*}

By Lemma \ref{lm::relations} $V \leq 2$ and $t^2\Delta V \leq 3t^2/n\leq 3$. Thus, 
\[
\|\bar{S}\bmu^\top\solnridgebar\|_{\bSigma}^2 \geq  c_4^{-1}(V + n\Diamond^2(1 - 5c_4^2t^2/n) ) \geq c_5^{-1}(V + n\Diamond^2),
\]
provided that $c_5$ is a large enough constant depending on $c_4$, and $t = \sqrt{n/c_5}$.

\end{proof}

\subsection{The ratio}
\label{sec::upper bound on the ratio appendix}

Finally we can put the bound on the numerator together with the bound on the denominator and obtain the following
\mainupper*
\begin{proof}
First of all, it is enough to show the statement for $\solnridgebar$ (as defined in the beginning of Section \ref{sec::tightness appendix}) instead of  $\solnridge$ as it has the same distribution. 

The straightforward combination of Lemmas \ref{lm::denominator lower} and \ref{lm::numerator upper} almost does the job, but we also need to show that $\bar{S} > 0$ with high probability.  

Recall that
\[
\bar{S} = (1 + \bar{\bnu}^\top \bA^{-1}\bar{\by})^2 + \bmu^\top\muperpridgebar\bar{\by}^\top\bA^{-1}\bar{\by}.
\]

By Lemma \ref{lm::randomness in Q but no eigvals of A}, if $c$ is large enough depending on $\sigma_x, L$, then on the event $\sA_k(L)\cap\sB_k(c_B)$ we have 
\[
\bmu^\top\muperpridgebar \geq M/c > 0.
\]
Moreover, $(1 + \bar{\bnu}^\top \bA^{-1}\bar{\by})^2 \geq 0$ almost surely,  and $\bar{\by}^\top\bA^{-1}\bar{\by} > 0$ on  $\sA_k(L)$. Thus, if $c$ is large enough, then $\bar{S} > 0$ on $\sA_k(L)\cap\sB_k(c_B)$.

Recall that since the data is sub-Gaussian, we can take $c_B$ in the definition of the event $\sB_k(c_B)$ large enough depending only on $\sigma_x$ such that $\P(\sB_k(c_B)) \geq 1 - c_Be^{-n/c_B}$ (as shown in Section \ref{sec::definition of sB_k}).
Now the first part of the theorem is a direct consequence of part 2 of Lemma \ref{lm::numerator upper}, while the second part of the theorem is a direct combination of Lemma \ref{lm::denominator lower} with the first part of \ref{lm::numerator upper}.
\end{proof}

\quantiletightness*
\begin{proof}
First of all, denote the constants from Theorem \ref{th::main upper bound} as $a_u, c_u, c_{y, u}$ (here index $u$ stands for ``upper bound"). Note that by that theorem, regardless whether $n\Lambda^{-1}M < a^{-1}\Diamond$ or $n\Lambda^{-1}M \geq a^{-1}\Diamond$ it still holds for any $t_u \in (0, \sqrt{n}/c_{y, u})$ that the probability of the event
\[
\left\{\frac{\bmu^\top\solnridge}{\|\solnridge\|_\bSigma} \leq c_u(1 + t_u)\frac{n\Lambda^{-1}M}{\sqrt{V + n\Diamond^2}}\right\}
\]
is a least 
\[
(c_{y,u}^{-1} - c_{y_u}e^{-t_u^2/c_{y, u}} - c_{y, u}e^{-n/c_u})_+(\P(\sA_k(L)) - c_ue^{-n/c_u})_+.
\]

Thus, if $t_u, n, \delta$ and $\eps$ are such that
\[
(c_{y,u}^{-1} - c_{y,u}e^{-t_u^2/c_{y, u}} - c_ye^{-n/c_u})_+(1 - \delta - c_ue^{-n/c_u})_+ > \eps,
\]
then 
\[
\alpha_\eps < c_u(1 + t_u)\frac{n\Lambda^{-1}M}{\sqrt{V + n\Diamond^2}}.
\]

When it comes to the lower bound, recall that by Lemma \ref{lm::sB under sub-Gaussianity}, the event $\sB_k(c_B)$ holds with probability at least $1 - c_Be^{-n/c_B}$ for a constant $c_B$ that only depends on $\sigma_x$. Thus, Theorem \ref{th::main} is applicable. Denote the constant from Theorem \ref{th::main} as $c_\ell$ (here index $\ell$ stands for ``lower bound"). Then we have that for the case $\eta = 0$ (i.e. $\by = \hat{\by}$) for  any $t_\ell \in (0, \sqrt{n}/c_{\ell})$, conditionally on the event $\sA_k(L)\cap \sB_k(c_B)$, with probability at least $1 - c_{\ell}e^{-t_\ell^2/2}$ over the draw of $\by$ 
\[
\frac{\bmu^\top\solnridge}{\|\solnridge\|_\bSigma} \geq c_\ell^{-2}\frac{N - c_\ell^2t_\ell\Diamond}{\sqrt{ V + t_\ell^2\Delta V} + \Diamond\sqrt{n} } \geq c_\ell^{-2}\frac{N/2}{\sqrt{ V(1 + 4t_\ell^2)} + \Diamond\sqrt{n} },
\]
where the last transition is made under the assumption that $c_\ell^2t_\ell\Diamond \leq N$, and also uses the fact that $\Delta V \leq 4 V$ (Lemma \ref{lm::relations}).

Thus, if we take $a = 2c_\ell^2t_\ell$ and $t_\ell, n, \eps, \delta$ are such that
\[
(1 - \delta - c_Be^{-n/c_B})_+(1 - c_{ \ell}e^{-t_\ell^2/2})_+\geq 1 - \eps,
\]
then under the condition $n\Lambda^{-1}M \geq a\Diamond$ we get 
\[
\alpha_\eps \geq c_\ell^{-2}\frac{n\Lambda^{-1}M/2}{\sqrt{ V(1 + 4t_\ell^2)} + \Diamond\sqrt{n} }.
\]

Finally, to finish the proof we just need to choose $c_1$, $t_l$ and $t_u$ that can only depend on $L, \sigma_x$ and absolute constants $\delta, \eps$ such that for any $n > c_1$
\begin{align*}
(1 - \delta - c_Be^{-n/c_B})_+(1 - c_{y, \ell}e^{-t_\ell^2/2})_+ >& 1 - \eps,\\
(c_{y,u}^{-1} - c_{y,u}e^{-t_u^2/c_{y, u}} - c_ye^{-n/c_u})_+(1 - \delta - c_ue^{-n/c_u})_+ >& \eps. 
\end{align*}
This is  easy to do: first choose $t_u$ large enough so that $ c_{y_u}e^{-t_u^2/c_{y, u}} < c_{y_u}^{-1}/2$. Note that $t_u$ is an absolute constant. Second, take $\eps =  0.5 \wedge (c_{y_u}^{-1}/16)$ --- an absolute constant. Third, take $c_1$ large enough depending on $c_B, c_{y, u}, c_u, \eps$ so that 
\[
c_ye^{-c_1/c_u} \leq c_{y,u}^{-1}/4, \quad c_ye^{-c_1/c_u} \leq \frac14, \quad c_Be^{-c_1/c_B}\leq \eps/4.
\]
Fourth, take $\delta = \eps/4$ -- an absolute constant. Finally, take $t_\ell$ such that $c_{\ell}e^{-t_\ell^2/2} \leq \eps/2$ --- a constant that only depends on $c_{\ell}$ (which, in its turn, only depends on $L$ and $\sigma_x$). Combining all gives
\begin{align*}
&(1 - \delta - c_Be^{-n/c_B})_+(1 - c_{y, \ell}e^{-t_\ell^2/2})_+ \\
\geq& (1 - \eps/4 - \eps/4)_+ (1 -\eps/2)_+\\
>& 1 - \eps,\\
&(c_{y,u}^{-1} - c_{y,u}e^{-t_u^2/c_{y, u}} - c_ye^{-n/c_u})_+(1 - \delta - c_ue^{-n/c_u})_+\\
>& (c_{y,u}^{-1} - c_{y,u}^{-1}/2 - c_{y,u}^{-1}/4)(1 - 1/4 - 1/4)\\
=& c_{y_u}^{-1}/16 \geq \eps, 
\end{align*}
which finishes the proof.
\end{proof}

\section{Analysis of ridge regularization}
\label{sec::ridge analysis appendix}

\rationovardecreasing*
\begin{proof}
The idea is to  introduce the vector-valued function $\bw(t) := (\bI_p + t\bM)^{-1}\bv$ and to compute the derivative of $f\left(\bw(t)\right)$ in $t$ using the chain rule as 
\[
\frac{d}{dt}f\left(\bw(t)\right) = \left(\frac{d}{dt}\bw(t)\right)^\top\left(\nabla f(\bw)\big|_{\bw = \bw(t)}\right).
\]

We write
\begin{align*}
f(\bw):=& \bv^\top\bw/\|\bw\|,\\
\nabla_{\bw} f(\bw) =& \bv/\|\bw\| - \bw\cdot   \bv^\top\bw/\|\bw\|^3,\\
\dot{\bw} := \frac{d}{dt}{\bw} =  & -\bM(\bI_p + t\bM)^{-2}\bv\\
=&-t^{-1}(\bI_p + t\bM - \bI_p)(\bI_p + t\bM)^{-2}\bv\\
=&-t^{-1}(\bI_p + t\bM)^{-1}\bv + t^{-1}(\bI_p + t\bM)^{-2}\bv\\
=& -t^{-1}\left(\bw + (\bI_p + t\bM)^{-1}\bw\right).\\
t\bv^\top\dot{\bw} =& -\bv^\top\bw + \|\bw\|^2,\\
t\bw^\top\dot{\bw} =& -\|\bw\|^2 + \bw^\top(\bI_p + t\bM)^{-1}\bw,\\
t\|\bw\|^3\langle\nabla f(\bw), \dot{\bw}\rangle =& \|\bw\|^4 - \bv^\top\bw\cdot\bw^\top(\bI_p + t\bM)^{-1}\bw \\
=& \left\langle(\bI_p + t\bM)^{-1/2}\bv, (\bI_p + t\bM)^{-3/2}\bv\right\rangle^2\\
&- \|(\bI_p + t\bM)^{-1/2}\bv\|^2 \cdot \|(\bI_p + t\bM)^{-3/2}\bv\|^2 \\
\leq& 0,
\end{align*}
where the last transition is by Cauchy-Schwartz and we used the brackets $\langle \bu_1, \bu_2\rangle$ to denote the scalar product $\bu_1^\top\bu_2$. We see that the derivative of $f(\bw(t))$ is non-positive when $t > 0$, thus the function $f(\bw(t))$ is non-increasing in $t$ on $[0, +\infty)$. 
\end{proof}

\largelambdasmallbound*
\begin{proof}
First of all, note that $\Lambda(\lambda') > \Lambda(\lambda)> n\lambda_k$. Thus, by Lemma \ref{lm::alternative form of bounds} we can show that for some absolute constant $c_1 > 0$
\[
\frac{\Numalt(\Lambda_1)}{\sqrt{\Varalt(\Lambda_1)} \vee \sqrt{n}\Diamondalt(\Lambda_1)} \leq c_1\left(1 + \frac{\Numalt(\Lambda_0)}{\sqrt{\Varalt(\Lambda_0)} \vee \sqrt{n}\Diamondalt(\Lambda_0)} \right),
\]
where we denoted $\Lambda_0 = \Lambda(\lambda)$, $\Lambda_1 = \Lambda(\lambda')$ and used the notation from Lemma \ref{lm::alternative form of bounds}.

From now on we forget about the notion of $k$ and only study the following quantity as a function of $\Lambda$:
\[
\frac{\Numalt(\Lambda)}{\sqrt{\Varalt(\Lambda)}}\wedge\frac{\Numalt(\Lambda)}{\sqrt{n}\Diamondalt(\Lambda)}.
\]

Note that if we denote
\begin{align*}
t := \Lambda/n,\quad
\bv := \bSigma^{-1/2}\bmu,\quad
\bw := (\bSigma + t\bI_p)^{-1}\bSigma^{1/2}\bmu
= (\bI_p + t\bSigma^{-1})^{-1}\bv,
\end{align*}
then it becomes 
\[
\frac{\Numalt(\Lambda)}{\sqrt{n}\Diamondalt(\Lambda)} = \frac{\bv^\top\bw}{\|\bw\|}.
\]
Thus, by Lemma \ref{lm::ratio with no variance is non-increasing}, $\frac{\Numalt(\Lambda)}{\sqrt{n}\Diamondalt(\Lambda)}$ is a non-increasing function of $\Lambda$, i.e.,  the benefit of regularization could only potentially come from the term $\frac{\Numalt(\Lambda)}{\sqrt{V(\Lambda)}}$.  More precisely, suppose that
\[
\frac{\Numalt(\Lambda_0)}{\sqrt{\Varalt(\Lambda_0)}}\wedge\frac{\Numalt(\Lambda_0)}{\sqrt{n}\Diamondalt(\Lambda_0)} < \frac{\Numalt(\Lambda_1)}{\sqrt{\Varalt(\Lambda_1)}}\wedge\frac{\Numalt(\Lambda_1)}{\sqrt{n}\Diamondalt(\Lambda_1)}.
\]

Then $\sqrt{\Varalt(\Lambda_0)} \geq \sqrt{n}\Diamondalt(\Lambda_0)$, otherwise we would have
\[
\frac{\Numalt(\Lambda_0)}{\sqrt{\Varalt(\Lambda_0)}}\wedge\frac{\Numalt(\Lambda_0)}{\sqrt{n}\Diamondalt(\Lambda_0)} = \frac{\Numalt(\Lambda_0)}{\sqrt{n}\Diamondalt(\Lambda_0)} \geq \frac{\Numalt(\Lambda_1)}{\sqrt{n}\Diamondalt(\Lambda_1)} \geq \frac{\Numalt(\Lambda_1)}{\sqrt{\Varalt(\Lambda_1)}}\wedge\frac{\Numalt(\Lambda_1)}{\sqrt{n}\Diamondalt(\Lambda_1)}.
\]

Moreover, if $\sqrt{\Varalt(\Lambda_1)} < \sqrt{n}\Diamondalt(\Lambda_1)$ then by Intermediate Value Theorem we can  take such $\Lambda_{0.5}$ that $\sqrt{\Varalt(\Lambda_{0.5})} = \sqrt{n}\Diamondalt(\Lambda_{0.5})$, and we'll once again have
\[
\frac{\Numalt(\Lambda_{0.5})}{\sqrt{\Varalt(\Lambda_{0.5})}}\wedge\frac{\Numalt(\Lambda_{0.5})}{\sqrt{n}\Diamondalt(\Lambda_{0.5})} = \frac{\Numalt(\Lambda_{0.5})}{\sqrt{n}\Diamondalt(\Lambda_{0.5})} \geq \frac{\Numalt(\Lambda_1)}{\sqrt{n}\Diamondalt(\Lambda_1)} \geq \frac{\Numalt(\Lambda_1)}{\sqrt{\Varalt(\Lambda_1)}}\wedge\frac{\Numalt(\Lambda_1)}{\sqrt{n}\Diamondalt(\Lambda_1)}.
\]

In case $\Lambda_{0.5}$ as above exists set $\Lambda = \Lambda_{0.5}$, otherwise set $\Lambda = \Lambda_1$. Now we have
\begin{gather*}
\Lambda_1 \geq \Lambda > \Lambda_0,\\
\sqrt{\Varalt(\Lambda)} \geq \sqrt{n}\Diamondalt(\Lambda), \quad\sqrt{\Varalt(\Lambda_0)} \geq \sqrt{n}\Diamondalt(\Lambda_0),\\
\frac{\Numalt(\Lambda_0)}{\sqrt{\Varalt(\Lambda_0)}}\wedge\frac{\Numalt(\Lambda_0)}{\sqrt{n}\Diamondalt(\Lambda_0)} <\frac{\Numalt(\Lambda)}{\sqrt{\Varalt(\Lambda)}}\wedge\frac{\Numalt(\Lambda)}{\sqrt{n}\Diamondalt(\Lambda)} \geq \frac{\Numalt(\Lambda_1)}{\sqrt{\Varalt(\Lambda_1)}}\wedge\frac{\Numalt(\Lambda_1)}{\sqrt{n}\Diamondalt(\Lambda_1)}.
\end{gather*}

Let's study $\frac{\Numalt(\Lambda)}{\sqrt{\Varalt(\Lambda)}}$. The idea is to re-introduce $k$, but the ``right one" (basically choose $k = k^*$). Then split into the $\uptok$ and $\ktoinf$ part and say that increasing regularization does nothing to the tail, but also cannot make the $\uptok$ part more than a constant. Formally,  take $k_0 = \min\{\kappa: \lambda_{\kappa + 1} < \Lambda_0/n\}$. Such $k_0 < n$ exists since $\Lambda_0 > n\lambda_{k+1}$.

Now we write
\begin{align*}
\frac{\Numalt(\Lambda)}{\sqrt{\Varalt(\Lambda)}} =&\frac{ \sum_i \frac{\mu_i^2}{\lambda_i + \Lambda/n}}{\sqrt{\sum_i \frac{\lambda_i^2/n }{(\lambda_i + \Lambda/n)^2}}}\\
=& \frac{ \sum_{i=1}^{k_0} \frac{\mu_i^2}{\lambda_i + \Lambda/n}}{\sqrt{\sum_i \frac{\lambda_i^2/n }{(\lambda_i + \Lambda/n)^2}}} + \frac{ \sum_{i> k_0} \frac{\mu_i^2}{\lambda_i + \Lambda/n}}{\sqrt{\sum_i \frac{\lambda_i^2/n }{(\lambda_i + \Lambda/n)^2}}}\\
\leq& \frac{ \sum_{i=1}^{k_0} \frac{\mu_i^2}{\lambda_i + \Lambda/n}}{\sqrt{\sum_i \frac{\lambda_i^2/n }{(\lambda_i + \Lambda/n)^2}}} + \frac{ \sum_{i> k_0} \frac{\mu_i^2}{\Lambda/n}}{\sqrt{\sum_i \frac{\lambda_i^2/n }{(\lambda_i + \Lambda/n)^2}}}\\
=& \frac{ \sum_{i=1}^{k_0} \frac{\mu_i^2}{\lambda_i + \Lambda/n}}{\sqrt{\sum_i \frac{\lambda_i^2/n }{(\lambda_i + \Lambda/n)^2}}} + \frac{ \sum_{i> k_0} \mu_i^2}{\sqrt{\sum_i \frac{\lambda_i^2/n }{(n\lambda_i/\Lambda + 1)^2}}} 
\end{align*}

The second term is a decreasing function of $\Lambda$, which implies 
\[
\frac{ \sum_{i> k_0} \mu_i^2}{\sqrt{\sum_i \frac{\lambda_i^2/n }{(n\lambda_i/\Lambda + 1)^2}}} \leq \frac{ \sum_{i> k_0} \mu_i^2}{\sqrt{\sum_i \frac{\lambda_i^2/n }{(n\lambda_i/\Lambda_0 + 1)^2}}} = \frac{ \sum_{i> k_0} \frac{\mu_i^2}{\Lambda_0/n}}{\sqrt{\sum_i \frac{\lambda_i^2/n }{(\lambda_i + \Lambda_0/n)^2}}}  \leq 2\frac{ \sum_i \frac{\mu_i^2}{\lambda_i + \Lambda_0/n}}{\sqrt{\sum_i \frac{\lambda_i^2/n }{(\lambda_i + \Lambda_0/n)^2}}}= 2\frac{\Numalt(\Lambda_0)}{\sqrt{\Varalt(\Lambda_0)}},
\]
where we used that $(\Lambda_0/n)^{-1} \leq 2(\lambda_i + \Lambda_0/n)^{-1}$ for $i > k_0$ in the last inequality.

Now let's study the part that comes from the first $k_0$ components. We can write
\[
\frac{ \sum_{i=1}^{k_0} \frac{\mu_i^2}{\lambda_i + \Lambda/n}}{\sqrt{\sum_i \frac{\lambda_i^2/n }{(\lambda_i + \Lambda/n)^2}}}  =\frac{ \sum_{i=1}^{k_0} \frac{\mu_i^2}{\lambda_i + \Lambda/n}}{\sqrt{\sum_{i=1}^{k_0} \frac{\lambda_i\mu_i^2}{(\lambda_i + \Lambda/n)^2}}} \cdot \frac{\sqrt{\sum_{i=1}^{k_0} \frac{\lambda_i\mu_i^2}{(\lambda_i + \Lambda/n)^2}}}{\sqrt{\sum_i \frac{\lambda_i^2/n }{(\lambda_i + \Lambda/n)^2}}}
\]
By Lemma \ref{lm::ratio with no variance is non-increasing}, the first multiplier is a non-increasing function of $\Lambda$. Moreover, if we plug $\Lambda_0$ instead of $\Lambda$ we get
\begin{multline*}
\frac{ \sum_{i=1}^{k_0} \frac{\mu_i^2}{\lambda_i + \Lambda_0/n}}{\sqrt{\sum_{i=1}^{k_0} \frac{\lambda_i\mu_i^2}{(\lambda_i + \Lambda_0/n)^2}}} \leq 2\frac{ \sum_{i=1}^{k_0} \frac{\lambda_i \mu_i^2}{(\lambda_i + \Lambda_0/n)^2}}{\sqrt{\sum_{i=1}^{k_0} \frac{\lambda_i\mu_i^2}{(\lambda_i + \Lambda_0/n)^2}}} = 2\sqrt{\sum_{i=1}^{k_0} \frac{\lambda_i\mu_i^2}{(\lambda_i + \Lambda_0/n)^2}}\\
\leq 2\sqrt{n}\Diamondalt(\Lambda_0) \leq 2\sqrt{\Varalt(\Lambda_0)} \leq 2\sqrt{2},
\end{multline*}
where we used that $\lambda_i \geq \Lambda_0/n$ for $i \leq k_0$. In the last transition we also used that $\Varalt(\Lambda_0) < V(\Lambda_0)$ (Lemma \ref{lm::alternative form of bounds}) and $V(\Lambda_0) < 2$ (Lemma \ref{lm::relations}).

Thus, the first multiplier starts less than a constant and stays less than a constant. For the second multiplier we have
\[
\frac{\sqrt{\sum_{i=1}^{k_0} \frac{\lambda_i\mu_i^2}{(\lambda_i + \Lambda/n)^2}}}{\sqrt{\sum_i \frac{\lambda_i^2/n }{(\lambda_i + \Lambda/n)^2}}} \leq \frac{\sqrt{n}\Diamondalt(\Lambda)}{\sqrt{\Varalt(\Lambda)}} \leq 1.
\]

Overall, we've got that either
\[
\frac{\Numalt(\Lambda_0)}{\sqrt{\Varalt(\Lambda_0)}}\wedge\frac{\Numalt(\Lambda_0)}{\sqrt{n}\Diamondalt(\Lambda_0)} \geq \frac{\Numalt(\Lambda_1)}{\sqrt{\Varalt(\Lambda_1)}}\wedge\frac{\Numalt(\Lambda_1)}{\sqrt{n}\Diamondalt(\Lambda_1)},
\]
or 
\[
\frac{\Numalt(\Lambda_1)}{\sqrt{\Varalt(\Lambda_1)}}\wedge\frac{\Numalt(\Lambda_1)}{\sqrt{n}\Diamondalt(\Lambda_1)} \leq \frac{\Numalt(\Lambda)}{\sqrt{\Varalt(\Lambda)}} \leq 2\sqrt{2} + 2\frac{\Numalt(\Lambda_0)}{\sqrt{\Varalt(\Lambda_0)}} = 2\sqrt{2} + 2\frac{\Numalt(\Lambda_0)}{\sqrt{\Varalt(\Lambda_0)}}\wedge\frac{\Numalt(\Lambda_0)}{\sqrt{n}\Diamondalt(\Lambda_0)},
\]
which implies the desired result.

\end{proof}

\lambdadoesntmatter*
\begin{proof}
Denote the constants from Theorem \ref{th::constant quantile tight bounds} as $a_0, c_0, \delta_0$ and $\eps_0$.

First of all, let's show that  if $a$ is chosen to be equal to $a_0$, then for any $\lambda' \geq \lambda$ it holds $n\Lambda(\lambda')^{-1}M(\lambda')  \geq a_0\Diamond(\lambda')$. Indeed, if $\|\bmu_\uptok\| = 0$, then 
\begin{align*}
&\frac{n\Lambda(\lambda')^{-1}M(\lambda') }{\Diamond(\lambda')}\\
=& \frac{ \left\|\left(\Lambda(\lambda') n^{-1}\bSigma_\uptok^{-1} + \bI_k\right)^{-1/2}\bSigma_\uptok^{-1/2}\bmu_\uptok\right\|^2 + n\Lambda^{-1}\|\bmu_\ktoinf\|^2}{\sqrt{n^{-1}\left\|\left(\Lambda(\lambda') n^{-1}\bSigma_\uptok^{-1} + \bI_k\right)^{-1}\bSigma_\uptok^{-1/2}\bmu_\uptok\right\|^2 + n\Lambda^{-2}\|\bmu_\ktoinf\|_{\bSigma_\ktoinf}^2}}\\
=& \frac{\sqrt{n}\|\bmu_\ktoinf\|^2}{\|\bmu_\ktoinf\|_{\bSigma_\ktoinf}}
\end{align*}
--- doesn't depend on $\lambda'$ at all.

In the case that $\bmu$ is an eigenvector of $\bSigma$, that is, $\bmu = \mu_i\be_i$ for some $i \in [k]$,  we have
\begin{align*}
&\frac{n\Lambda(\lambda')^{-1} M(\lambda') }{\Diamond(\lambda')}\\
=&\frac{\left(1 + \lambda_i^{-1}n^{-1}\Lambda(\lambda')\right)^{-1}\lambda_i^{-1}\mu_i^2}{\sqrt{n^{-1}\left(1 + \lambda_i^{-1}n^{-1}\Lambda(\lambda')\right)^{-1}\lambda_i^{-1}\mu_i^2}}\\
=& \sqrt{n/\lambda_i}
\end{align*}
--- doesn't depend on $\lambda'$ once again.

That is 
\[
\frac{n\Lambda(\lambda')^{-1}M(\lambda') }{\Diamond(\lambda')} = \frac{n\Lambda(\lambda)^{-1}M(\lambda) }{\Diamond(\lambda)} \geq a = a_0.
\]

Note also that 
\[
\Lambda(\lambda') \geq \Lambda(\lambda) >  c\left(n\lambda_{k+1}\vee \sqrt{n\sum_{i > k}\lambda_i^2}\right),
\]
and 
\[
\P(\sA_k(L, \lambda')) \geq \P(\sA_k(L, \lambda)) \geq 1 - \delta_0,
\]
so if $c > c_0$, $\eps = \eps_0$ and $\delta = \delta_0$ then the assumptions of Theorem \ref{th::constant quantile tight bounds} are satisfied for both $\lambda$ and $\lambda'$, which yields
\[
\alpha_\eps(\lambda)/c_0 \leq \frac{n\Lambda(\lambda)^{-1}M(\lambda) }{\sqrt{V(\lambda) } + \sqrt{n}\Diamond(\lambda) } \leq c_0\alpha_\eps(\lambda) , \quad \alpha_\eps(\lambda')/c_0 \leq \frac{n\Lambda(\lambda')^{-1}M(\lambda')}{\sqrt{V(\lambda')} + \sqrt{n}\Diamond(\lambda')} \leq c_0\alpha_\eps(\lambda').
\]

We've already seen that $\frac{n\Lambda(\lambda')^{-1}M(\lambda') }{\sqrt{n}\Diamond(\lambda')} = \frac{n\Lambda(\lambda)^{-1}M(\lambda) }{\sqrt{n}\Diamond(\lambda)}$, so the only thing we need to study is $V(\lambda)$. Namely, we are going to show that under the assumptions we made $V(\lambda) < n\Diamond^2(\lambda)$ and $V(\lambda') < n\Diamond^2(\lambda')$. That will finish the proof since in that case
\begin{multline*}
\alpha_\eps(\lambda')/c_0 \leq \frac{n\Lambda(\lambda')^{-1}M(\lambda')}{\sqrt{V(\lambda')} + \sqrt{n}\Diamond(\lambda')} \leq \frac{n\Lambda(\lambda')^{-1}M(\lambda')}{\sqrt{n}\Diamond(\lambda')}\\ = \frac{n\Lambda(\lambda')^{-1}M(\lambda)}{\sqrt{n}\Diamond(\lambda)} \leq 2\frac{n\Lambda(\lambda')^{-1}M(\lambda)}{\sqrt{V(\lambda')} + \sqrt{n}\Diamond(\lambda)} \leq 2c_0\alpha_\eps(\lambda),
\end{multline*}
and analogously $\alpha_\eps(\lambda)/c_0 \leq 2c_0\alpha_\eps(\lambda')$.

Thus, in the rest of the proof we show that $V(\lambda) < n\Diamond^2(\lambda)$ and $V(\lambda') < n\Diamond^2(\lambda')$.
 Let's write out two cases:
\begin{enumerate}
\item $\bmu$ is supported on the tail. Then
\begin{align*}
V(\lambda) :=& n^{-1}\tr\left(\left(\Lambda(\lambda)  n^{-1}\bSigma_\uptok^{-1} + \bI_k\right)^{-2}\right)+ \Lambda(\lambda) ^{-2}n\sum_{i > k}\lambda_i^2,\\
\Diamond^2 :=& n^{-1}\left\|\left(\Lambda(\lambda)  n^{-1}\bSigma_\uptok^{-1} + \bI_k\right)^{-1}\bSigma_\uptok^{-1/2}\bmu_\uptok\right\|^2 + n\Lambda(\lambda) ^{-2}\|\bmu_\ktoinf\|_{\bSigma_\ktoinf}^2\\
=& n\Lambda(\lambda) ^{-2}\|\bmu_\ktoinf\|_{\bSigma_\ktoinf}^2
\end{align*}
We want to show that
\[
n^{-1}\tr\left(\left(\bSigma_\uptok^{-1} + n\Lambda(\lambda) ^{-1}\bI_k\right)^{-2}\right)+ n^{-1}\sum_{i > k}\lambda_i^2 \leq \|\bmu_\ktoinf\|_{\bSigma_\ktoinf}^2,
\]
which holds since
\begin{align*}
&n^{-1}\tr\left(\left(\bSigma_\uptok^{-1} + n\Lambda(\lambda) ^{-1}\bI_k\right)^{-2}\right)+ n^{-1}\sum_{i > k}\lambda_i^2\\
<&n^{-1}\tr\left(\left(\bSigma_\uptok^{-1}\right)^{-2}\right)+ n^{-1}\sum_{i > k}\lambda_i^2\\
=&\sum_i \lambda_i^2\\
\leq& n\|\bmu_\ktoinf\|_{\bSigma_\ktoinf}^2.
\end{align*}
We showed that $V(\lambda) \leq n\Diamond^2(\lambda)$. Note that $V(\lambda') \leq n\Diamond^2(\lambda')$ by exactly the same argument.

\item $\bmu = \be_i$ for $i\leq k$. Write out $V$ and $\Diamond$ again:
\begin{align*}
V(\lambda) :=& n^{-1}\tr\left(\left(\Lambda(\lambda) n^{-1}\bSigma_\uptok^{-1} + \bI_k\right)^{-2}\right)+ \Lambda(\lambda)^{-2}n\sum_{i > k}\lambda_i^2,\\
\Diamond(\lambda)^2 :=& n^{-1}\left\|\left(\Lambda(\lambda) n^{-1}\bSigma_\uptok^{-1} + \bI_k\right)^{-1}\bSigma_\uptok^{-1/2}\bmu_\uptok\right\|^2 + n\Lambda(\lambda)^{-2}\|\bmu_\ktoinf\|_{\bSigma_\ktoinf}^2\\
=& n^{-1}\frac{\lambda_i^{-1}\mu_i^2}{(n^{-1}\Lambda(\lambda)\lambda_i^{-1} + 1)^2}\\
= & \frac{n\lambda_i\mu_i^2}{(\Lambda(\lambda) + n\lambda_i)^2}
\end{align*}

By the same argument as before, since
\[
\frac{n\lambda_i\mu_i^2}{(1 + n\lambda_i/\Lambda(\lambda))^2} \geq \sum_i \lambda_i^2, 
\]
we have $V(\lambda) < n\Diamond(\lambda)^2$. When it comes to $\lambda'$, we can write
\[
\frac{n\lambda_i\mu_i^2}{(1 + n\lambda_i/\Lambda(\lambda'))^2} \geq \frac{n\lambda_i\mu_i^2}{(1 + n\lambda_i/\Lambda(\lambda))^2} \geq \sum_i \lambda_i^2, 
\]
which yields $V(\lambda') < n\Diamond(\lambda')^2.$
\end{enumerate}
\end{proof}

\smallregfixedkrecipe*
\begin{proof}
Take $a, \eps, \delta$ to be the same as in Theorem \ref{th::constant quantile tight bounds}. Denote the constant $c$ from that theorem as $c_1$. In the end we will take  $c$ large enough depending on $c_1$. If $c > c_1$ then Theorem \ref{th::constant quantile tight bounds} implies that
\[
\alpha_\eps(\lambda) \geq c_1^{-1}\frac{n\Lambda(\lambda)^{-1}M(\lambda)}{\sqrt{V(\lambda)} + \sqrt{n}\Diamond(\lambda)} \geq c_1^{-1}\frac{n\Lambda(\lambda)^{-1}M(\lambda)}{2\sqrt{n}\Diamond(\lambda)}.
\]

At the same time, by Theorem \ref{th::constant quantile tight bounds} for any $\lambda' > \lambda$
\[
\alpha_\eps(\lambda') \leq c_1\frac{n\Lambda(\lambda')^{-1}M(\lambda')}{\sqrt{V(\lambda')} + \sqrt{n}\Diamond(\lambda')} \leq c_1\frac{n\Lambda(\lambda')^{-1}M(\lambda')}{\sqrt{n}\Diamond(\lambda')}.
\]
Take such $\hat{\lambda}$ that $\Lambda(\hat{\lambda}) = C\Lambda(\lambda)$. Note that $\hat{\lambda} > \lambda$. By our construction, due to the fact that $\lambda_k > C\Lambda(\lambda) = \Lambda(\hat{\lambda})$  and Equation \eqref{eq::mu uptok balances mu ktoinf} we have
\begin{align*}
n\Lambda(\lambda)^{-1}M(\lambda) \geq& n\Lambda(\lambda)^{-1}\|\bmu_\ktoinf\|^2,\\
n\Lambda(\hat{\lambda})^{-1}M(\hat{\lambda}) \leq& \|\bmu_\uptok\|_{\bSigma_\uptok^{-1}}^2 + n\Lambda(\hat{\lambda})^{-1}\|\bmu_\ktoinf\|^2\\
\leq& 2n\Lambda(\hat{\lambda})^{-1}\|\bmu_\ktoinf\|^2,\\
n\Diamond(\lambda)^2 \leq& \|\bmu_\uptok\|_{\bSigma_\uptok^{-1}}^2 + n^2\Lambda({\lambda})^{-2}\|\bmu_\ktoinf\|_{\bSigma_\ktoinf}^2\\
\leq& 2\|\bmu_\uptok\|_{\bSigma_\uptok^{-1}}^2,\\
n\Diamond(\hat{\lambda})^2 \geq& \left\|\left(\Lambda(\hat{\lambda})n^{-1}\bSigma_\uptok^{-1} + \bI_k\right)^{-1}\bSigma_\uptok^{-1/2}\bmu_\uptok\right\|^2 \\
\geq& \frac14\|\bmu_\uptok\|_{\bSigma_\uptok^{-1}}^2,
\end{align*}
where we used that $\Lambda(\hat{\lambda})n^{-1}\bSigma_\uptok^{-1} + \bI_k$ is a diagonal matrix whose diagonal  elements are at most $2$ in the last transition.

Combining everything together we get that
\begin{multline*}
\alpha_\eps(\lambda)  \geq c_1^{-1}\frac{n\Lambda(\lambda)^{-1}M(\lambda)}{\sqrt{n}\Diamond(\lambda)} \geq \frac{1}{2c_1}\frac{n\Lambda(\lambda)^{-1}\|\bmu_\ktoinf\|^2}{\sqrt{2}\|\bmu_\uptok\|_{\bSigma_\uptok^{-1}}} =\\
=\frac{C}{8\sqrt{2}c_1}\frac{2n\Lambda(\hat{\lambda})^{-1}\|\bmu_\ktoinf\|^2}{\frac12\|\bmu_\uptok\|_{\bSigma_\uptok^{-1}}} \geq \frac{C}{8\sqrt{2}c_1^2}\cdot c_1\frac{n\Lambda(\hat{\lambda})^{-1}M(\hat{\lambda})}{\sqrt{n}\Diamond(\hat{\lambda})} \geq \frac{C}{8\sqrt{2}c_1^2} \alpha_\eps(\hat{\lambda}).
\end{multline*}

We obtained the result for $\lambda' = \hat{\lambda}$. To extend the result for all $\lambda' > \hat{\lambda}$ note that by Lemma \ref{lm::alternative form of bounds} and the fact that $\Numalt(\lambda)/\Diamondalt(\lambda)$ is a non-increasing function of $\lambda$ for some absolute constant $c_2 > 1$ we can write
\begin{multline*}
\alpha_\eps(\lambda)  \geq \frac{C}{8\sqrt{2}c_1^2}\cdot c_1\frac{n\Lambda(\hat{\lambda})^{-1}M(\hat{\lambda})}{\sqrt{n}\Diamond(\hat{\lambda})} \geq \frac{C}{8\sqrt{2}c_1^2c_2}\cdot c_1\frac{\Numalt(\hat{\lambda})}{\sqrt{n}\Diamondalt(\hat{\lambda})} \geq\\
\geq \frac{C}{8\sqrt{2}c_1^2c_2}\cdot c_1\frac{\Numalt({\lambda'})}{\sqrt{n}\Diamondalt({\lambda'})} \geq \frac{C}{8\sqrt{2}c_1^2c_2^2}\cdot c_1\frac{n\Lambda({\lambda'})^{-1}M({\lambda'})}{\sqrt{n}\Diamond({\lambda'})} \geq \frac{C}{8\sqrt{2}c_1^2c_2^2}\alpha_\eps({\lambda'}).
\end{multline*}
Taking $c = 8\sqrt{2}c_1^2c_2^2$ finishes the proof.
\end{proof}

\negativeregfixedkexample*
\begin{proof}
Since we consider Gaussian data, $\sigma_x$ is an absolute constant.
 Let's take $L = 2$ and denote the corresponding constant $c$ from Lemma \ref{lm::small regularization is better no change in k star} to be $c_1$. We are going to use that lemma to construct such distribution of data that the quantile $\alpha_{\eps}(\lambda)$ is minimized for a negative $\lambda$. Note that to do that it is enough to take $C= c_1$ and $\lambda = -\frac{c_1-1}{c_1}\sum_{i > k}\lambda_i$ as this condition is equivalent to $\Lambda(0) = c_1\Lambda(\lambda)$.

Let's take infinite-dimensional Gaussian data with slow exponential decay in the tail, that is
\[
\lambda_i = \begin{cases}
\ell, & i \leq k,\\
e^{-\alpha(i - k)}, & i > k.
\end{cases},\quad
\mu_i = 
 \begin{cases}
m_\uptok, & i \leq k,\\
m_\ktoinf e^{-\frac{\beta}{2}(i-k)}, & i > k.
\end{cases}
\]
Thus, the whole classification problem is described by scalars $\ell, \alpha, \beta, m_\uptok, m_\ktoinf, n, k$.

Let's see how we need to choose those scalars in order to follow the recipe from Lemma \ref{lm::small regularization is better no change in k star}.  This is an absolute constant since $L$ is an absolute constant and the data is Gaussian. As discussed before, we fix $C = c_1$ and put $\lambda = -\frac{c_1-1}{c_1}\sum_{i > k}\lambda_i$

Due to Lemma \ref{lm:: eigenvalues of A_k indep coord}, for Gaussian data the statement $\P(\sA_k(L, \lambda)) \geq 1 - \delta$ follows from the statement
\[
\Lambda(\lambda) \geq b\left(n\lambda_{k+1}\vee \sqrt{n\sum_{i > k}\lambda_i^2}\right),
\]
where $b$ is a large constant that depends on $\delta$. Let's also take it to be larger than $c_1$ in order to fully satisfy step 1 of Lemma \ref{lm::small regularization is better no change in k star}. For our covariance and regularization this translates into
\[
\frac{1}{c_1}\frac{e^{-\alpha}}{1 - e^{-\alpha}} \geq b\left(ne^{-\alpha} \vee \sqrt{n\frac{e^{-2\alpha}}{1 - e^{-2\alpha}}}\right),
\]
which can be equivalently rewritten as
\[
1 - e^{-\alpha} \leq \frac{1}{bc_1\sqrt{n}} \left(\frac{1}{\sqrt{n}}\wedge \sqrt{1 - e^{-2\alpha}}\right). 
\]
For $x \in (0, 1)$ it holds that $1-x < e^{-x} < 1-x(1 - e^{-1})< 1 - x/2.$ Thus, assuming that $\alpha < 0.5$ the condition above follows from the following:
\[
\alpha \leq \frac{1}{bc_1\sqrt{n}} \left(\frac{1}{\sqrt{n}}\wedge \sqrt{\alpha}\right), 
\]
that is 
\[
\alpha \leq \frac{1}{bc_1n}\wedge \frac{1}{b^2c_1^2n} = \frac{1}{b^2c_1^2n}.
\]

Let's take $c_2 = b^2c_1^2$ and put $\alpha = c_2^{-1}n^{-1}$. Then conditions from step 1 of Lemma \ref{lm::small regularization is better no change in k star} are satisfied.

The second part of Lemma \ref{lm::small regularization is better no change in k star} states that we require $n\lambda_k \geq C\Lambda(\lambda)$, that is, $n\ell \geq C\left(\frac{e^{-\alpha}}{1 - e^{-\alpha}} + \lambda\right) = \frac{e^{-\alpha}}{1 - e^{-\alpha}}$. Note that we also need $\ell \geq  1$ in order for the sequence $\{\lambda_i\}$ to be non-increasing. Given our previous choice of $\alpha$ we have
\[
\frac{e^{-\alpha}}{1 - e^{-\alpha}}\leq \frac{2}{\alpha} = 2c_2n.
\]
Thus, we can take $\ell = 2c_2$.

The third part of Lemma \ref{lm::small regularization is better no change in k star} requires the following
\[
\|\bmu_\ktoinf\|_{\bSigma_\ktoinf}^2 \leq C^{-1}n^{-1}\Lambda(\lambda)\|\bmu_\ktoinf\|^2,
\]
which we equivalently transform below:
\begin{gather*}
\sum_{i > k}e^{-(\alpha + \beta)(i-k)} \leq C^{-2}n^{-1}\frac{e^{-\alpha}}{1 - e^{-\alpha}}\sum_{i > k}e^{-\beta(i-k)},\\
\frac{e^{-(\alpha + \beta)}}{1 - e^{-(\alpha + \beta)}} \leq C^{-2}n^{-1}\frac{e^{-\alpha}}{1 - e^{-\alpha}}\frac{e^{-\beta}}{1 - e^{-\beta}},\\
\frac{1 - e^{-\beta}}{1 - e^{-(\alpha + \beta)}} \leq C^{-2}n^{-1}\frac{1}{1 - e^{-\alpha}}.
\end{gather*}
Let's restrict the range of $\beta$ so that $\alpha + \beta < 1$. Then it is sufficient to choose $\beta$ that satisfies the following stronger condition:
\[
\frac{2\beta}{\alpha + \beta} \leq \frac{1}{C^2n\alpha}.
\]
Plugging the expression for $\alpha$ yields
\[
2\beta \leq \frac{c_2}{C^2}\left(\beta + \frac{1}{c_2n}\right).
\]
Actually, since $c_2 = b^2C^2 > 2C^2$, the inequality above always holds, so we can take any $\beta < 1 + \alpha$, for example, $\beta = \ln(2)$ (to make further computations simpler).

Next, part 4 of Lemma \ref{lm::small regularization is better no change in k star} requires
\[
nC^{-1}\Lambda(\lambda)^{-1}\|\bmu_\ktoinf\|^{2} \geq \|\bmu_\uptok\|_{\bSigma_\uptok^{-1}}^2 \geq n^2\Lambda(\lambda)^{-2}\|\bmu_\ktoinf\|_{\bSigma_\ktoinf}^2,
\]
that is,
\[
n\frac{1 - e^{-\alpha}}{e^{-\alpha}} \frac{e^{-\beta}}{1 - e^{-\beta}} \geq \frac{km_\uptok^2}{\ell m_\ktoinf^2} \geq n^2C^2\frac{(1 - e^{-\alpha})^2}{e^{-2\alpha}} \frac{e^{-(\alpha + \beta)}}{1 - e^{-(\alpha + \beta)}}.
\]
Plugging in $\beta = \ln(2)$ and simplifying yields
\[
n\frac{1 - e^{-\alpha}}{e^{-\alpha}}  \geq \frac{km_\uptok^2}{\ell m_\ktoinf^2} \geq n^2C^2\frac{(1 - e^{-\alpha})^2}{e^{-\alpha}} \frac{1}{2 - e^{-\alpha}}.
\]
As before, let's replace it by a stronger condition:
\[
\frac{n\alpha}{2} \geq \frac{km_\uptok^2}{\ell m_\ktoinf^2} \geq 2n^2C^2 \alpha^2.
\]
Plugging in $\alpha = (c_2n)^{-1}$ and $\ell = 2c_2$ yields
\begin{gather*}
\frac{1}{2c_2} \geq \frac{km_\uptok^2}{2c_2 m_\ktoinf^2} \geq \frac{2C^2}{c_2^2},\\
1 \geq \frac{km_\uptok^2}{ m_\ktoinf^2} \geq \frac{4C^2}{c_2} = \frac{4}{b^2}.
\end{gather*}
We see that since $b > 4$, we can simply put $km_\uptok^2/m_\ktoinf^2 = 1$, and part 4 of Lemma \ref{lm::small regularization is better no change in k star} is satisfied.

At last, we need to check the last part of the lemma. Let's start with writing out and transforming the expressions for $n\Diamond^2(\lambda), V(\lambda),$ and $n\Lambda^{-1}(\lambda)M(\lambda)$:
\begin{align*}
\Lambda(\lambda) =& C^{-1}\frac{e^{-\alpha}}{1 - e^{-\alpha}} \in \left(\frac{1}{2C\alpha},  \frac{2}{C\alpha}\right),\\
V(\lambda) :=& n^{-1}\tr\left(\left(\Lambda(\lambda) n^{-1}\bSigma_\uptok^{-1} + \bI_k\right)^{-2}\right)+ \Lambda^{-2}n\sum_{i > k}\lambda_i^2\\
=& n^{-1} \frac{k}{\left(1 + n^{-1}\ell^{-1}C^{-1}\frac{e^{-\alpha}}{1 - e^{-\alpha}}\right)^2} + C^2ne^{2\alpha}(1 - e^{-\alpha})^2\frac{e^{-2\alpha}}{1 - e^{-2\alpha}}\\
\leq&   \frac{k/n}{\left(1 + n^{-1}\ell^{-1}C^{-1}/(2\alpha)\right)^2} + C^2n\frac{\alpha^2}{\alpha}\\
=&  \frac{k/n}{\left(1 + n^{-1}(2c_2)^{-1}C^{-1}c_2n/2\right)^2} + C^2n\alpha\\
=& \frac{k/n}{\left(1 + 0.25C^{-1}\right)^2} + C^2c_2^{-1}\\
\leq& 2.
\end{align*}
\begin{align*}
n\Lambda^{-1}(\lambda)M(\lambda) :=& \left\|\left(\Lambda(\lambda) n^{-1}\bSigma_\uptok^{-1} + \bI_k\right)^{-1/2}\bSigma_\uptok^{-1/2}\bmu_\uptok\right\|^2 + n\Lambda(\lambda) ^{-1}(\lambda)\|\bmu_\ktoinf\|^2\\
=& k\frac{m_\uptok^2}{\ell(1 + n^{-1}\ell^{-1}\Lambda(\lambda))} + m_\ktoinf^2n\Lambda(\lambda)^{-1}\frac{e^{-\beta}}{1 - e^{-\beta}}\\
\geq& k\frac{m_\uptok^2}{\ell(1 + 2n^{-1}\ell^{-1}/(C\alpha))} + \frac12m_\ktoinf^2nC\alpha(\lambda)\frac{e^{-\beta}}{1 - e^{-\beta}}\\
=&\frac{km_\uptok^2}{2c_2 + 2c_2/C} + \frac12m_\ktoinf^2nC/(c_2n)\\
\geq& \frac{1}{4c_2}(km_\uptok^2 + m_\ktoinf^2) = \frac{m_\ktoinf^2}{2c_2} .
\end{align*}

\begin{align*}
n\Diamond^2(\lambda) :=& \left\|\left(\Lambda(\lambda) n^{-1}\bSigma_\uptok^{-1} + \bI_k\right)^{-1}\bSigma_\uptok^{-1/2}\bmu_\uptok\right\|^2 + n^2\Lambda(\lambda)^{-2}\|\bmu_\ktoinf\|_{\bSigma_\ktoinf}^2\\
=& k\frac{m_\uptok^2}{\ell(1 + n^{-1}\ell^{-1}\Lambda(\lambda))^2} + m_\ktoinf^2n^2\Lambda(\lambda)^{-2}\frac{e^{-(\alpha + \beta)}}{1 - e^{-(\alpha + \beta)}}\\
=& \frac{2c_2m_\ktoinf^2}{(2c_2 + \Lambda(\lambda)/n)^2} + m_\ktoinf^2n^2\Lambda(\lambda)^{-2}\frac{e^{-\alpha}}{2 - e^{-\alpha}}\\
\end{align*}

For $n\Diamond^2$ we need bounds from both sides. In what follows we write them separately.
\begin{align*}
n\Diamond^2(\lambda) \leq& \frac{2c_2m_\ktoinf^2}{(2c_2 + 0.5n^{-1}C^{-1}\alpha^{-1})^2} + m_\ktoinf^2n^2(2C\alpha)^2\\
=&\frac{2c_2m_\ktoinf^2}{(2c_2 + 0.5c_2C^{-1})^2} + m_\ktoinf^2(2C/c_2)^2
\leq& 5m_\ktoinf^2/c_2,
\end{align*}
where in the last transition we used that $c_2/C = b > \sqrt{c_2}$. 

When it comes to the bound from below, we write
\begin{align*}
n\Diamond^2(\lambda) \geq& \frac{2c_2m_\ktoinf^2}{(2c_2 + 2n^{-1}C^{-1}\alpha^{-1})^2} + m_\ktoinf^2n^2(2C\alpha)^2/3\\
=&\frac{2c_2m_\ktoinf^2}{(2c_2 + 2c_2C^{-1})^2} + m_\ktoinf^2(2C/c_2)^2/3\\
\geq& m_\ktoinf^2/(8c_2),
\end{align*}
where we used $2c_2 + 2c_2C^{-1} \leq 4c_2$ in the last transition.

Finally, we can write out the conditions from part 5 of Lemma  \ref{lm::small regularization is better no change in k star}. According to the bounds above, the following conditions on $m_\ktoinf$ are sufficient:
\[
m_\ktoinf^2/(8c_2) \geq 2,\quad \frac{m_\ktoinf^2}{2c_2} \geq a\sqrt{\frac{5}{c_2n}}m_\ktoinf,
\]
that is $m_\ktoinf^2 \geq (16c_2)\vee (5c_2a^2/n) = 16c_2$ given that $n$ is large enough.
\end{proof}

\smallregchangingkrecipe*
\begin{proof}
Take $a, \eps, \delta$ to be the same as in Theorem \ref{th::constant quantile tight bounds}. Denote the constant $c$ from that theorem as $c_1$. In the end we will take  $c$ large enough depending on $c_1$. If $c > c_1$ then Theorem \ref{th::constant quantile tight bounds} implies that
\[
\alpha_\eps(\lambda) \geq c_1^{-1}\frac{n\Lambda(\lambda)^{-1}M(\lambda)}{\sqrt{V(\lambda)} + \sqrt{n}\Diamond(\lambda)} \geq c_1^{-1}\frac{n\Lambda(\lambda)^{-1}M(\lambda)}{2\sqrt{n}\Diamond(\lambda)}.
\]

At the same time, by Theorem \ref{th::constant quantile tight bounds} for any $\lambda' > \lambda$
\[
\alpha_\eps(\lambda') \leq c_1\frac{n\Lambda(\lambda')^{-1}M(\lambda')}{\sqrt{V(\lambda')} + \sqrt{n}\Diamond(\lambda')} \leq c_1\frac{n\Lambda(\lambda')^{-1}M(\lambda')}{\sqrt{n}\Diamond(\lambda')}.
\]

Since $\|\bmu_\ktoinf\| = 0$, $\Lambda(\lambda) \leq n\lambda_k$ and $\Lambda(\lambda') > n\lambda_1$, we can write
\begin{align*}
n\Lambda(\lambda)^{-1}M(\lambda) =& \left\|\left(\Lambda(\lambda)n^{-1}\bSigma_\uptok^{-1} + \bI_k\right)^{-1/2}\bSigma_\uptok^{-1/2}\bmu_\uptok\right\|^2\\
\geq& \frac12\|\bmu_\uptok\|_{\bSigma_\uptok^{-1}}^2,\\
n\Lambda(\lambda')^{-1}M(\lambda') =& \left\|\left(\Lambda(\lambda')n^{-1}\bI_k + \bSigma_\uptok\right)^{-1/2}\bmu_\uptok\right\|^2\\
\leq& 2n\Lambda(\lambda')^{-1}\|\bmu_\uptok\|^2,\\
n\Diamond(\lambda)^2 \leq& \|\bmu_\uptok\|_{\bSigma_\uptok^{-1}}^2,\\
n\Diamond(\lambda')^2 =& \left\|\left(\Lambda(\lambda')n^{-1}\bSigma_\uptok^{-1} + \bI_k\right)^{-1}\bSigma_\uptok^{-1/2}\bmu_\uptok\right\|^2 \\
\geq& \frac14n^2\Lambda(\lambda')^{-2}\|\bmu_\uptok\|_{\bSigma_\uptok}^2,
\end{align*}
where we used the fact that $\Lambda(\lambda)n^{-1}\bSigma_\uptok^{-1} + \bI_k$ and  $\Lambda(\lambda')n^{-1}\bI_k + \bSigma_\uptok$ are both diagonal matrices whose diagonal elements are at most $2$.

Combining everything together we get that
\begin{multline*}
\alpha_\eps(\lambda)  \geq c_1^{-1}\frac{n\Lambda(\lambda)^{-1}M(\lambda)}{\sqrt{n}\Diamond(\lambda)} \geq \frac{1}{2c_1}\frac{\frac12\|\bmu_\uptok\|_{\bSigma_\uptok^{-1}}^2}{\|\bmu_\uptok\|_{\bSigma_\uptok^{-1}}} = \frac{1}{4c_1}\|\bmu_\uptok\|_{\bSigma_\uptok^{-1}} \geq
\frac{C}{4c_1}\frac{\|\bmu_\uptok\|^2}{\|\bmu_\uptok\|_{\bSigma_\uptok}} =\\
=\frac{C}{16c_1}\frac{2n\Lambda(\lambda')^{-1}\|\bmu_\uptok\|^2}{\frac12n\Lambda(\lambda')^{-1}\|\bmu_\uptok\|_{\bSigma_\uptok}} \geq \frac{C}{16c_1^2}\cdot c_1\frac{n\Lambda(\lambda')^{-1}M(\lambda')}{\sqrt{n}\Diamond(\lambda')} \geq\frac{C}{16c_1^2} \alpha_\eps(\lambda').
\end{multline*}
Taking $c = 16c_1^2$ finishes the proof.
\end{proof}

\negativeregchangingkexample*
\begin{proof}
Since we consider Gaussian data, $\sigma_x$ is an absolute constant.
 Let's take $L = 2$ and denote the corresponding constants $a, c$ from Lemma \ref{lm::smallregchangingkrecipe} to be $a_1, c_1$. We are going to use that lemma to construct such distribution of data that the quantile $\alpha_{\eps}(\lambda)$ is minimized for a negative $\lambda$. Note that to do that it is enough to take $C= c_1$ and $\lambda = -\frac{c_1-1}{c_1}\sum_{i > k}\lambda_i$ as this condition is equivalent to $\Lambda(0) = c_1\Lambda(\lambda)$.

Let's take finite-dimensional Gaussian data with exponential decay in the first $k$ components and isotropic tails:
\[
\lambda_i = \begin{cases}
e^{-\alpha i}, & i \leq k,\\
\ell, & i > k.
\end{cases}
\]
Note that $\Lambda(\lambda) = \Lambda(0)/c_1 = \ell(p-k)/c_1$.

 We take $\bmu_\ktoinf$ to be zero, in accordance with Lemma \ref{lm::smallregchangingkrecipe}. When it comes to $\bmu_\uptok$, we just put all its components to be equal, that is
\[
\mu_i = 
 \begin{cases}
m, & i \leq k,\\
0, & i > k.
\end{cases}
\]
Thus, the whole classification problem is described by scalars $\ell, m, \alpha, n, k, p$.

Our goal is to find the values of these parameters such that the conditions from Lemma \ref{lm::smallregchangingkrecipe} are satisfied for some $L$. We take $L = 2$.

The first part of that lemma says that $\sA_k(2)$ should be satisfied with probability at least $1-\delta$ and that 
\[
 \ell(p-k)/c_1 = \Lambda(\lambda) \geq c_1 \left(n\lambda_{k+1}\vee \sqrt{n\sum_{i > k}\lambda_i^2}\right) = c_1\ell(n \vee \sqrt{n(p-k)}).
\]
Due to Lemma \ref{lm:: eigenvalues of A_k indep coord}, for Gaussian data both these conditions follow from $p > bn$ and $n > b$, where $b$ is a large enough absolute constant.

The second part of Lemma \ref{lm::smallregchangingkrecipe} requires $n\lambda_i > \Lambda(\lambda)$, that is, $ne^{-\alpha k} \geq \ell (p-k)/c_1$, so we can take $\ell = c_1ne^{-\alpha k}/p$. Note that since $p > c_1n$ we have $\lambda_{k+1} = \ell < e^{-\alpha k} =  \lambda_k$, so the eigenvalues remain in the right order.

The third part of Lemma \ref{lm::smallregchangingkrecipe} demands $\|\bmu_\uptok\|_{\bSigma_\uptok}\|\bmu_\uptok\|_{\bSigma_\uptok^{-1}} \geq C\|\bmu_\uptok\|^{2}$, that is,
\begin{gather*}
\sqrt{\left(m^2\sum_{i=1}^ke^{-\alpha i}\right)\left(m^2\sum_{i=1}^ke^{\alpha i}\right)} \geq c_1km^2,\\
\sqrt{\frac{1 - e^{-k\alpha}}{1 - e^{-\alpha}}\frac{e^{k\alpha} - 1}{e^{\alpha} - 1}}\geq c_1k.
\end{gather*}
Note that for $\alpha > 0$ we have $1-e^{-k\alpha} > 1 - e^{-\alpha}$. Moreover, $e^{k\alpha} - 1 > (e^\alpha - 1)e^{(k-1)\alpha}$. Thus, it is enough to satisfy the following weaker condition:
\[
e^{(k-1)\alpha/2}\geq c_1k,
\]
so we need to take $\alpha \geq 2\ln(c_1k)/(k-1)$. Since $k$ is lower bounded by a large constant $b$, we can take $\alpha = 4\ln(k)/k$.
Plugging it into equation for $\ell$ yields $\ell = c_1ne^{-\alpha k}/(ep) = c_1n/(epk^4)$. We take $c = c_1/e$, so $\ell = cnp^{-1}k^{-4}$.

Finally, we need to take $m$ large enough so that part 4 of Lemma \ref{lm::smallregchangingkrecipe} is satisfied. To check that part, we start with writing  the expressions for $n\Diamond^2(\lambda), V(\lambda),$ and $n\Lambda^{-1}(\lambda)M(\lambda)$ and bounding them.
\begin{align*}
\Lambda(\lambda) =& \ell(p-k)/c_1 = \frac{n(p-k)}{epk^4},\\
V(\lambda) :=& n^{-1}\tr\left(\left(\Lambda(\lambda) n^{-1}\bSigma_\uptok^{-1} + \bI_k\right)^{-2}\right)+ \Lambda^{-2}n\sum_{i > k}\lambda_i^2\\
\leq& \frac{k}{n} + \Lambda^{-2}n\sum_{i > k}\lambda_i^2\\
=& \frac{k}{n} + \frac{n}{p-k}.
\end{align*}
\begin{align*}
n\Lambda^{-1}(\lambda)M(\lambda) :=& \left\|\left(\Lambda(\lambda) n^{-1}\bSigma_\uptok^{-1} + \bI_k\right)^{-1/2}\bSigma_\uptok^{-1/2}\bmu_\uptok\right\|^2 + n\Lambda(\lambda) ^{-1}(\lambda)\|\bmu_\ktoinf\|^2\\
=& m^2\sum_{i=1}^k\frac{e^{\alpha i}}{1 + e^{\alpha i}(p-k)/(epk^4)}\\
\geq& \frac{m^2}{1 + e^{\alpha k}/(ek^4)}\sum_{i=1}^ke^{\alpha i}\\
=& \frac{m^2}{1 + e^{-1}}\sum_{i=1}^ke^{\alpha i}.
\end{align*}
Let's bound the sum of exponents separately. We write
\[
\sum_{i=1}^ke^{\alpha i} = 
e^{\alpha}\frac{e^{\alpha k} - 1}{e^{\alpha} - 1}= \frac{k^4 - 1}{1 - e^{-4\ln(k)/k}} \begin{cases}
\geq \frac{k^5}{5\ln(k)}\\
\leq \frac{k^5}{2\ln(k)}
\end{cases}
\]
where we used that $\alpha < 1$ (since $k$ is large enough) and for $x \in (0, 1)$ it holds $x > 1 - e^{-x} > x/2$.

Thus,
\[
n\Lambda^{-1}(\lambda)M(\lambda) \geq \frac{m^2k^5}{10\ln(k)}.
\]

When it comes to $\Diamond$, the derivation is very similar as above:
\begin{align*}
n\Diamond^2(\lambda) :=& \left\|\left(\Lambda(\lambda) n^{-1}\bSigma_\uptok^{-1} + \bI_k\right)^{-1}\bSigma_\uptok^{-1/2}\bmu_\uptok\right\|^2 + n^2\Lambda(\lambda)^{-2}\|\bmu_\ktoinf\|_{\bSigma_\ktoinf}^2\\
=& m^2\sum_{i=1}^k\frac{e^{\alpha i}}{\left(1 + e^{\alpha i}(p-k)/(epk^4)\right)^2} 
\end{align*}

For $n\Diamond^2$ we need bounds from both sides. In what follows we write them separately.
\begin{align*}
n\Diamond^2(\lambda) \leq& m^2\sum_{i=1}^k\frac{e^{\alpha i}}{1}\\
\leq& \frac{m^2k^5}{2\ln(k)},\\
n\Diamond^2(\lambda) \geq& \frac{m^2}{\left(1 + e^{\alpha k}/(ek^4)\right)^2} \sum_{i=1}^ke^{\alpha i}\\
=& \frac{m^2}{(1 + e^{-1})^2}\sum_{i=1}^ke^{\alpha i}\\
\geq& \frac{m^2k^5}{20\ln(k)}.
\end{align*}

Finally, to satisfy part 4 of Lemma \ref{lm::smallregchangingkrecipe} we need
\[
n\Diamond^2(\lambda) \geq V(\lambda), \quad n\Lambda^{-1}(\lambda)M(\lambda) \geq a_1\Diamond(\lambda),
\]
that is, it is enough to have
\begin{align*}
\frac{m^2k^5}{20\ln(k)} \geq \frac{k}{n} + \frac{n}{p-k}, \quad
\frac{m^2k^5}{10\ln(k)} \geq a_1\sqrt{\frac{m^2k^5}{2n\ln(k)}},
\end{align*}
which is equivalent to
\[
\frac{m^2k^5}{20\ln(k)} \geq \left(\frac{k}{n} + \frac{n}{p-k}\right)\vee \frac{50a_1}{n}.
\]
For example, we can put 
\[
m = \frac{b\ln(k)}{k^5}\left(\frac{k}{n} + \frac{n}{p}\right)
\]
given that $b$ is a large enough constant. 
\end{proof}

\section{Comparisons with earlier results}
\label{sec::comparissons proofs appendix}
\comparisonchatterjilinearnoise*
\begin{proof}
First of all, due to assumptions on $\lambda_i$, $\lambda$ and $k$ we can write
\begin{align*}
\sum_i\lambda_i^2 \leq& \Lambda,\\
\Lambda =& \sum_{i}\lambda_i \in [\kappa p, p],\\
V =& \Lambda^{-2}n\sum_{i > k}\lambda_i^2 \leq n/\Lambda \leq \frac{n}{\kappa p},\\
\Delta V \leq& \frac{n\lambda_1^2}{\Lambda^2} + \frac{n\lambda_{1}^2 + \sum_{i }\lambda_i^2}{\Lambda^2} \leq \frac{2n}{\kappa^2p^2} + \frac{1}{\kappa p} \leq \frac{3}{\kappa^2p},\\
\Diamond^2 =& n\Lambda^{-2}\|\bmu\|_{\bSigma}^2 \leq \frac{n\|\bmu\|^2}{\kappa^2p^2},\\
M =& \|\bmu\|^2.
\end{align*}

Plugging in those bounds together with $\sigma_\eta < 1$ yields
\begin{align*}
n\Lambda^{-1}M - ct\Diamond 
\geq& \frac{n}{p}\|\bmu\|^2 - ct\frac{\sqrt{n}\|\bmu\|}{\kappa{p}} ,\\
\left[1 + n\Lambda^{-1}M\sigma_\eta \right]\sqrt{ V + t^2\Delta V} + \Diamond\sqrt{n}
\leq& \left[1 + \frac{n}{\kappa p}\|\bmu\|^2 \right]\sqrt{\frac{n}{\kappa p} + t^2\frac{3}{\kappa^2p}} + \frac{n\|\bmu\|}{\kappa{p}}.
\end{align*}

Next, since $n\|\bmu\|^2/p \leq \kappa$ for $t^2 \leq n\kappa$ we can write
\begin{align*}
\left[1 + n\Lambda^{-1}M\sigma_\eta \right]\sqrt{ V + t^2\Delta V} + \Diamond\sqrt{n}
\leq& [1 + 1]\sqrt{\frac{4n}{\kappa p}} + \frac{n\|\bmu\|}{\kappa{p}} \leq 5\sqrt{\frac{n}{p\kappa}}.
\end{align*}

Finally, if $\|\bmu\| \geq 2ct/(\kappa \sqrt{n})$, then $\frac{n}{p}\|\bmu\|^2 - ct\frac{\sqrt{n}\|\bmu\|}{\kappa{p}} \geq \frac{n}{2p}\|\bmu\|^2$. Plugging in that bound in yields the result.
\end{proof}

\comparisonGaoGuBelkin*
\begin{proof}
Let's write out the definitions of $\Lambda, M, \Diamond, V, \Delta V$ for the case $k = 0$ with $n\lambda_1 < \lambda + \sum_i \lambda_i$:
\begin{align*}
\Lambda =& \lambda + \sum_{i}\lambda_i = \lambda + \tr(\bSigma),\\
V =& \Lambda^{-2}n\sum_{i > k}\lambda_i^2 = \frac{n\|\bSigma\|_F^2}{\left(\lambda + \tr(\bSigma)\right)^2},\\
\Delta V =& \frac{n\lambda_1^2}{\Lambda^2} + \frac{n\lambda_{1}^2 + \sum_{i }\lambda_i^2}{\Lambda^2} = \frac{2n\|\bSigma\|^2 + \|\bSigma\|_F^2}{\left(\lambda + \tr(\bSigma)\right)^2},\\
\Diamond^2 =& n\Lambda^{-2}\|\bmu\|_{\bSigma}^2 = \frac{n\|\bmu\|_{\bSigma}^2}{\left(\lambda + \tr(\bSigma)\right)^2},\\
M =& \|\bmu\|^2.
\end{align*}

Now we can rewrite our bound as
\begin{align*}
&\frac{n\Lambda^{-1}M - ct\Diamond}{\sqrt{V + t^2\Delta V} + \sqrt{n}\Diamond}\\
=& \frac{nM - ct\Lambda\Diamond}{\sqrt{\Lambda^2V + t^2\Lambda^2\Delta V} + \sqrt{n}\Lambda\Diamond}\\
=& \frac{n\|\bmu\|^2 - ct\sqrt{n}\|\bmu\|_{\bSigma}}{\sqrt{n\|\bSigma\|_F^2 + t^2\left(2n\|\bSigma\|^2 + \|\bSigma\|_F^2\right)} + n\|\bmu\|_{\bSigma}}.
\end{align*}

We see that for $t \leq \sqrt{n}$ the condition $\|\bmu\|^2 \geq 2c\|\bmu\|_\bSigma$ ensures that the numerator greater or equal to $n\|\bmu_\ktoinf\|^2/2$. At the same time, the denominator doesn't exceed $n\|\bmu\|_\bSigma + \sqrt{2n}\|\bSigma\|_F + 2n\|\bSigma\|$ Thus, we obtain the desired result.
\end{proof}

\comparisonwangbinary*
\begin{proof}
Note that  under  assumption $\|\bmu_\uptok\| = 0$  we have
\begin{align*}
M = \|\bmu_\ktoinf\|^2 = \|\bmu\|^2,\quad
\Diamond^2 =n\Lambda^{-2}\|\bmu_\ktoinf\|^2_{\bSigma_\ktoinf} =n\Lambda^{-2}\|\bmu\|^2_{\bSigma} ,
\end{align*}
and thus, the condition $\|\bmu\|^2 \geq 2c\|\bmu\|_{\bSigma}$ can be rewritten as $M \geq 2c\sqrt{n}\Lambda \Diamond$. Therefore, it implies that $n\Lambda^{-1}M - ct\Diamond \geq n\Lambda^{-1}M/2$ for $t \leq \sqrt{n}$. Thus, 
\begin{equation}
\label{eq::our bound k=1 mu in tail}
\frac{n\Lambda^{-1}M - ct\Diamond}{\sqrt{V + t^2\Delta V} +  \sqrt{n}\Diamond} \geq \frac12  \frac{\|\bmu\|^2}{n^{-1}\Lambda\sqrt{V + t^2\Delta V} + \|\bmu\|_\bSigma}.
\end{equation}
We see that in order to compare Equation \eqref{eq::our bound k=1 mu in tail} and \eqref{eq::wang bi-level}, we need to compare $A + B + \lambda_j$ to $n^{-1}\Lambda\sqrt{V + t^2\Delta V}$
Note that
\begin{align*}
A \geq& \frac{\lambda_1\Lambda}{n\lambda_1 + \Lambda}  = n^{-1}\left((n\lambda_1)^{-1} + \Lambda^{-1}\right)^{-1} \geq \frac{1}{2n}\left(\Lambda \wedge n\lambda_1\right),\\
B \geq& \sqrt{\sum_{i \neq 1, j}\lambda_i^2},\\
V =& n^{-1}(1 + n^{-1}\Lambda \lambda_1^{-1})^{-2} + \Lambda^{-2}n\sum_{i > 1}\lambda_i^2\\
=& n\Lambda^{-2}\left(n^{-2}(\Lambda^{-1} + n^{-1} \lambda_1^{-1})^{-2} + \sum_{i > 1}\lambda_i^2\right)\\
\leq&  n\Lambda^{-2}(A^2 + B^2 + \lambda_j^2),\\
\Delta V =& \frac{1}{n}\wedge \frac{n\lambda_1^2}{\Lambda^2} + \frac{n\lambda_2^2 + \sum_{i > 1}\lambda_i^2}{\Lambda^2}\\
=&\frac{1}{n\Lambda^2}\left(\Lambda \wedge n\lambda_1\right)^2  + \frac{n^{-1}(n\lambda_2)^2 + \sum_{i > 1}\lambda_i^2}{\Lambda^2}\\
\leq& \frac{1}{n\Lambda^2}\left(\Lambda \wedge n\lambda_1\right)^2  + \frac{n^{-1}(n\lambda_1 \wedge \Lambda)^2 + \sum_{i > 1}\lambda_i^2}{\Lambda^2}\\
\leq&\frac{1}{n\Lambda^2}\left(2nA\right)^2  + \frac{n^{-1}(2nA)^2 + B^2 + \lambda_j^2}{\Lambda^2}\\
=& n\Lambda^{-2}(8A^2 + B^2/n + \lambda_j^2/n),
\end{align*}
where we used that $n\lambda_2 \leq n\lambda_1$ and $n\lambda_2 \leq \sum_{i > 1}\lambda_i \leq \Lambda$ for $\lambda > 0$ to write $n\lambda_2 \leq n\lambda_1 \wedge \Lambda$ when we bounded $\Delta V$. Overall, we get
\[
n^{-2}\Lambda^2(V + t^2\Delta V) \leq \frac{1}{n} (A^2 + B^2 + \lambda_j^2) + \frac{t^2}{n}(8A^2 + B^2/n + \lambda_j^2/n),
\]
that is, for $t< \sqrt{n}$ 
\begin{align*}
n^{-2}\Lambda^2(V + t^2\Delta V) \leq& 9A^2 + 2B^2/n + 2\lambda_j^2/n,\\
n^{-1}\Lambda\sqrt{V + t^2\Delta V} \leq& 3(A + B + \lambda_j),
\end{align*}
which yields the result.
\end{proof}

\comparisonmuthukumarclassification*
\begin{proof}
Throughout the proof we treat $q, r, s$ as constants and use small-oh notation $o(1)$ to denote a function of $n, q, r, s$ that converges to zero as $n$ goes to infinity. Each time we use this notation it denotes a different function.

First of all, let's write out the quantities of interest and plug in the expressions for $\lambda_i$.
\begin{align*}
\bmu:=& \sqrt{\frac{2\lambda_1}{\pi}}\be_1 = \sqrt{2/\pi}n^{(s - q - r)/2}\be_1 ,\\ 
\Lambda =& \sum_{i > k}\lambda_i = (n^s - n^r)\lambda_{k+1} =  (n^s - n^r)\cdot(1-n^{-q})/(1 - n^{r-s})\\
=& n^s - n^{s-q} = n^s(1 - o(1)),\\
\Lambda n^{-1}\lambda_1^{-1} =&n^s(1 - o(1)) \cdot n^{-1}\cdot n^{-s + q + r} = n^{q + r - 1}(1 - o(1)),\\
\Diamond^2 =& n^{-1}\left\|\left(\Lambda n^{-1}\bSigma_\uptok^{-1} + \bI_k\right)^{-1}\bSigma_\uptok^{-1/2}\bmu_\uptok\right\|^2 + n\Lambda^{-2}\|\bmu_\ktoinf\|_{\bSigma_\ktoinf}^2\\
=& n^{-1}\cdot (1 + \Lambda n^{-1}\lambda_1^{-1})^{-2}\lambda_1^{-1}\cdot \frac{2\lambda_1}{\pi}\\
=& \frac{2 + o(1)}{\pi n\left(1 +  n^{q + r - 1}\right)^2},\\
n\Lambda^{-1}M =& \left\|\left(\Lambda n^{-1}\bSigma_\uptok^{-1} + \bI_k\right)^{-1/2}\bSigma_\uptok^{-1/2}\bmu_\uptok\right\|^2 + n\Lambda^{-1}\|\bmu_\ktoinf\|^2\\
=& (1 + \Lambda n^{-1}\lambda_1^{-1})^{-1}\lambda_1^{-1}\cdot \frac{2\lambda_1}{\pi}\\
=& \frac{2 + o(1)}{\pi(1 + n^{q + r - 1})},\\
V =& n^{-1}\tr\left(\left(\Lambda n^{-1}\bSigma_\uptok^{-1} + \bI_k\right)^{-2}\right)+ \Lambda^{-2}n\sum_{i > k}\lambda_i^2\\
=&n^{-1}n^r(\Lambda n^{-1}\lambda_1^{-1} + 1)^{-2} + \Lambda^{-2}n(n^s - n^r)\lambda_{k+1}^2\\
=& \frac{1 + o(1)}{n^{1-r}(1 + n^{q + r - 1})^2} + n^{1 - s}(1 + o(1)), \\
\Delta V =& \frac{1}{n}\wedge\frac{n\lambda_1^2}{\Lambda^2} + \frac{n\lambda_{k+1}^2 + \sum_{i > k}\lambda_i^2}{\Lambda^2}\\
=& n^{-1}\wedge n^{1-2r-2q}(1 + o(1)) + n^{-s}(1 + o(1)).
\end{align*}

Now let's plug this into the quantity of interest.
Note that as long as $t = o(\sqrt{n})$, we have $t\Diamond = o(n\Lambda^{-1}M)$. Moreover, $\Delta V/V = O(n^{-r})$, indeed
\[
n^{r}\Delta V = \left(\frac{1}{n^{1-r} \vee n^{2(q + r - 1) - r}} + n^{r-s}\right)(1 + o(1)) \leq V(1 + o(1)),
\]
since $r < 1$. Thus, if $t^2 = o(n^r)$, then $V + t^2\Delta V = V(1 + o(1))$. Now note that $n\Lambda^{-1}M = \Diamond\sqrt{n}(\sqrt{2/\pi} + o(1))$. That is,
\[
\frac{n\Lambda^{-1}M}{\sqrt{V } +  \sqrt{n}\Diamond} = \frac{\sqrt{2/\pi} + o(1)}{1 + \sqrt{V/(n\Diamond^2)}}.
\]
So, the only thing left is to compare $V$ and $n\Diamond^2$:
\[
\frac{V}{n\Diamond^2} = \frac{\pi + o(1)}{2}\left( n^{r-1} + n^{1-s}(1 + n^{q + r - 1})^2\right).
\]
Since $r < 1 < s$, this ratio goes to infinity if and only if $n^{1-s}(n^{q + r - 1})^2$ goes to infinity, that is $2q + 2r - 1 - s > 0$, which yields the result.
\end{proof}

\end{document}